\documentclass{article} % For LaTeX2e
\usepackage{iclr2022_conference,times}

\iclrfinalcopy

\newcommand{\linyi}[1]{{\color{olive}(Linyi: #1)}}

\usepackage{soul}
\usepackage[normalem]{ulem}
\renewcommand\sout{\bgroup\markoverwith
{\color{orange!90!black}{\rule[.5ex]{2pt}{1pt}}}\ULon}

\usepackage{amsmath}
\usepackage{amsfonts}
\usepackage{amssymb}

\usepackage{graphicx}
\usepackage{xcolor}

\usepackage{multicol}

\usepackage{algorithmic}  
\usepackage[ruled,vlined,linesnumbered]{algorithm2e}
\usepackage{adjustbox}
\SetKwComment{Comment}{$\triangleright$\ }{}

\usepackage{nicefrac}

\usepackage[shortlabels]{enumitem}

\usepackage{hyperref}
\usepackage[nameinlink]{cleveref}
\crefname{ineq}{inequality}{inequalities}
\creflabelformat{ineq}{#2{\upshape(#1)}#3} 

\usepackage{longtable}

\usepackage{amsthm}

\usepackage{bbm}
\usepackage{dsfont}

\usepackage[super]{nth}

\theoremstyle{definition}
\newtheorem{definition}{Definition}
\theoremstyle{plain}
\newtheorem{theoremappend}{Theorem}
\theoremstyle{plain}
\theoremstyle{plain}

\usepackage{thmtools}
\usepackage{thm-restate}

\theoremstyle{plain}

\theoremstyle{plain}
\declaretheorem[name=Lemma]{lemma}

\newcommand{\eg}[0]{\textit{e.g.}}

\newcommand{\ie}[0]{\textit{i.e.}}

\newcommand{\eps}[0]{\varepsilon}

\newcommand{\wtilde}[0]{\widetilde}

\newcommand{\namef}[0]{perturbed return function\xspace}
\newcommand{\nameF}[0]{Perturbed return function\xspace}

\newcommand{\namepj}[0]{perturbed cumulative reward\xspace}

\newcommand{\ebound}[0]{expectation bound\xspace}
\newcommand{\pbound}[0]{percentile bound\xspace}

\newcommand{\sigrtj}[0]{$\sigma$-randomized trajectory\xspace}
\newcommand{\sigrtjs}[0]{$\sigma$-randomized trajectories\xspace}
\newcommand{\sigrpi}[0]{$\sigma$-randomized policy\xspace}

\usepackage{accents}

\newcommand{\uit}[1]{\underline{\textit{#1}}}
\newcommand{\texttc}[1]{\texttt{\textsc{#1}}}

\newcommand{\argmax}{\mathop{\mathrm{argmax}}}

\newcommand{\lip}[0]{Lipschitz\xspace}
\newcommand{\env}[0]{environment\xspace}   
\newcommand{\qpi}[0]{Q^\pi}
\newcommand{\tpi}[0]{\tilde\pi}
\newcommand{\wqpi}[0]{\wtilde Q^\pi}

\newcommand{\vmin}[0]{V_{\min}}
\newcommand{\vmax}[0]{V_{\max}}
\newcommand{\jmin}[0]{J_{\min}}
\newcommand{\jmax}[0]{J_{\max}}
\newcommand{\jcur}[0]{J_{\rm cur}}
\newcommand{\jglobal}[0]{J_{\rm global}}

\newcommand{\epsball}[0]{\cal B^\eps}
\newcommand{\epsbball}[0]{\cal B^{\bar\eps}}

\renewcommand{\vert}[0]{\oplus}
\newcommand{\vertt}[2]{\vert_{t=#1}^{#2}}

\newcommand{\uj}[0]{\underline{J}}
\newcommand{\uje}[0]{\underline{J_E}}
\newcommand{\ujp}[0]{\underline{J_p}}

\newcommand{\sysname}[0]{\textsc{CROP}\xspace}

\newcommand{\staters}[0]{\textsc{\sysname-LoAct}\xspace}
\newcommand{\glbrs}[0]{\textsc{\sysname-GRe}\xspace}
\newcommand{\adasearch}[0]{\textsc{\sysname-LoRe}\xspace}

\newcommand{\norm}[1]{\lVert{#1}\rVert}

\newcommand{\bb}[1]{\mathbb{#1}}

\newcommand{\pmat}[2]{
                    \begin{pmatrix}
                    #1\\#2
                    \end{pmatrix}
                    }

\renewcommand{\cal}[1]{\mathcal{#1}}
\usepackage{float}

\usepackage{titlesec}
\titlespacing{\section}{0pt}{6pt}{6pt}
\titlespacing{\subsection}{0pt}{6pt}{6pt}

\usepackage{subcaption}

\usepackage{booktabs}
\usepackage{makecell}

\usepackage{wrapfig}

\makeatletter
\newcommand{\pushright}[1]{\ifmeasuring@#1\else\omit\hfill$\displaystyle#1$\ignorespaces\fi}
\makeatother

\makeatletter
\let\save@mathaccent\mathaccent
\newcommand*\if@single[3]{%
  \setbox0\hbox{${\mathaccent"0362{#1}}^H$}%
  \setbox2\hbox{${\mathaccent"0362{\kern0pt#1}}^H$}%
  \ifdim\ht0=\ht2 #3\else #2\fi
  }
\newcommand*\rel@kern[1]{\kern#1\dimexpr\macc@kerna}
\newcommand*\widebar[1]{\@ifnextchar^{{\wide@bar{#1}{0}}}{\wide@bar{#1}{1}}}
\newcommand*\wide@bar[2]{\if@single{#1}{\wide@bar@{#1}{#2}{1}}{\wide@bar@{#1}{#2}{2}}}
\newcommand*\wide@bar@[3]{%
  \begingroup
  \def\mathaccent##1##2{%
    \let\mathaccent\save@mathaccent
    \if#32 \let\macc@nucleus\first@char \fi
    \setbox\z@\hbox{$\macc@style{\macc@nucleus}_{}$}%
    \setbox\tw@\hbox{$\macc@style{\macc@nucleus}{}_{}$}%
    \dimen@\wd\tw@
    \advance\dimen@-\wd\z@
    \divide\dimen@ 3
    \@tempdima\wd\tw@
    \advance\@tempdima-\scriptspace
    \divide\@tempdima 10
    \advance\dimen@-\@tempdima
    \ifdim\dimen@>\z@ \dimen@0pt\fi
    \rel@kern{0.6}\kern-\dimen@
    \if#31
      \overline{\rel@kern{-0.6}\kern\dimen@\macc@nucleus\rel@kern{0.4}\kern\dimen@}%
      \advance\dimen@0.4\dimexpr\macc@kerna
      \let\final@kern#2%
      \ifdim\dimen@<\z@ \let\final@kern1\fi
      \if\final@kern1 \kern-\dimen@\fi
    \else
      \overline{\rel@kern{-0.6}\kern\dimen@#1}%
    \fi
  }%
  \macc@depth\@ne
  \let\math@bgroup\@empty \let\math@egroup\macc@set@skewchar
  \mathsurround\z@ \frozen@everymath{\mathgroup\macc@group\relax}%
  \macc@set@skewchar\relax
  \let\mathaccentV\macc@nested@a
  \if#31
    \macc@nested@a\relax111{#1}%
  \else
    \def\gobble@till@marker##1\endmarker{}%
    \futurelet\first@char\gobble@till@marker#1\endmarker
    \ifcat\noexpand\first@char A\else
      \def\first@char{}%
    \fi
    \macc@nested@a\relax111{\first@char}%
  \fi
  \endgroup
}
\makeatother
\usepackage{math_commands}

\usepackage{hyperref}
\usepackage[hyphenbreaks]{breakurl}

\usepackage{url}

\usepackage[utf8]{inputenc} % allow utf-8 input
\usepackage[T1]{fontenc}    % use 8-bit T1 fonts
\usepackage{booktabs}       % professional-quality tables
\usepackage{amsfonts}       % blackboard math symbols
\usepackage{nicefrac}       % compact symbols for 1/2, etc.
\usepackage{microtype}      % microtypography

\title{\sysname: Certifying Robust Policies for Reinforcement Learning through Functional Smoothing}

\author{Fan Wu$^{1}$
\quad
Linyi Li$^{1}$ %\thanks{Equal contribution.}
\quad
Zijian Huang$^{1}$ %\footnotemark[1]
\quad
Yevgeniy Vorobeychik$^{2}$ %\footnotemark[1]
\quad
Ding Zhao$^{3}$ 
\quad
Bo Li$^{1}$
\\
$^{1}$University of Illinois at Urbana-Champaign
\qquad
$^{2}$Washington University in St. Louis
\\
$^{3}$Carnegie Mellon University
\\
{\tt\footnotesize \{fanw6,linyi2,zijianh4,lbo\}@illinois.edu}
\qquad
{\tt\footnotesize yvorobeychik@wustl.edu}
\qquad
\\
{\tt\footnotesize dingzhao@andrew.cmu.edu}
}

\begin{document}

\maketitle

\begin{abstract}
As reinforcement learning (RL) has achieved great success and been even adopted in safety-critical domains such as autonomous vehicles, a range of empirical studies have been conducted to improve its robustness against adversarial attacks. However, how to certify its robustness with theoretical guarantees still remains challenging. In this paper, we present the first unified framework \sysname (\underline{C}ertifying  \underline{Ro}bust \underline{P}olicies for RL) to provide robustness certification on both action and reward levels.
% \mo{which provides state level robustness certification and the \textit{first} certification for cumulative rewards.}
In particular, we propose two robustness certification criteria: \textit{robustness of per-state actions} and \textit{lower bound of cumulative rewards}.
    %We focus on \ding{define.. for; because there could be other kinds of certification} the robustness of both \textit{per-state action}  and \textit{perturbed cumulative reward}, by providing the certification for  the action stability and lower bound of the cumulative reward under bounded state perturbations.
    %We achieve these two types of robustness certifications via the local and global smoothing strategies over the input state trajectory.
 We then develop a {local smoothing} algorithm for policies derived from Q-functions 
%  smoothed with Gaussian noise over each encountered state 
 to guarantee the robustness of actions taken along the trajectory;
    we also develop a \textit{global smoothing} algorithm for certifying the lower bound of a finite-horizon cumulative reward, as well as
  a novel \textit{local smoothing} algorithm to perform adaptive search in order to obtain tighter reward certification.
    %In particular, we propose to certify the robustness of per-state action via different local smoothing strategies to achieve the expected and percentile bounds, and certify the lower bound of perturbed cumulative reward via both global smoothing and local smoothing based adaptive search algorithm \adasearch.
    Empirically, we apply \sysname to evaluate several existing empirically robust RL algorithms, including adversarial training and different robust regularization, in four environments (two representative Atari games, Highway, and CartPole).
    %Based on our proposed RL robustness certification framework, we evaluate 6 empirically robust RL algorithms such as Adversarial-training and different regularizations on 2 Atari games with different variants of DQNs. 
    % We show that SA-MDP (PGD), SA-MDP (CVX), and RadialRL achieve high certified robustness among these.
    % \fan{true for Highway}
    Furthermore, by evaluating these algorithms against adversarial attacks, we demonstrate that our certifications are often tight. All experiment results are available at website \url{https://crop-leaderboard.github.io}.
    % \href{https://crop-leaderboard.github.io/}{CROP-LEADERBOARD}\footnote{The website link is \href{https://crop-leaderboard.github.io/}{https://crop-leaderboard.github.io/}.}.
    %We also evaluate the algorithms under different adversarial attacks to demonstrate the tightness of our certification. 
    %As the first robustness certification framework for RL, we believe this work will shed light on future directions of developing more certifiably robust $Q$-learning and other RL algorithms with tighter certifications.
    % the first certified robustness analysis on reinforcement learning against adversarial perturbations on state observations. 
    % We leverage the principle of randomized smoothing to construct a smoothed action-value function for any given trained value network.
    % This way, the Lipschitz continuity of the derived value function is inferred, which further offers a robustness certification for a given policy against diverse test time perturbations on state observations.
    % We prove that given a smoothed policy based on our smoothing procedures, the value of the model prediction can be bounded under a perturbed state observation.
    % In addition, we explore Gaussian data augmentation during training in reinforcement learning and find it can empirically improve both the test reward and
    % the certified robustness under adversarial attacks.
    % Extensive experiments on four atari games, variants of DQN architectures, and different types of attacks demonstrate the strength of
    % smoothed policies, and provide theoretical and empirical justifications on existing work promoting to train smoothed policies.
\end{abstract}

\section{Introduction}

\vspace{-2mm}
Reinforcement learning (RL) has been widely applied to different applications, such as robotics~\citep{kober2013reinforcement,deisenroth2013survey,polydoros2017survey}, autonomous driving vehicles~\citep{shalev2016safe,sallab2017deep}, and trading~\citep{deng2016deep,almahdi2017adaptive,ye2020reinforcement}.
However, recent studies have shown that learning algorithms are  vulnerable to adversarial attacks~\citep{goodfellow2014explaining,kurakin2016adversarial,moosavi2016deepfool,jia2017adversarial,eykholt2018robust}, and a range of attacks have also been proposed against the input states and trained policies of RL~\citep{huang2017adversarial,kos2017delving,lin2017tactics,behzadan2017vulnerability}.
As more and more safety-critical applications are being deployed in  real-world~\citep{christiano2016transfer,fisac2018general,cheng2019end,eop2020promoting}, how to test and improve their robustness before massive production is of great importance.

To defend against adversarial attacks in RL, different \textit{empirical defenses} have been proposed~\citep{mandlekar2017adversarially,behzadan2017whatever,pattanaik2018robust,fischer2019online,zhang2021robust,oikarinen2020robust,donti2020enforcing,shen2020deep,eysenbach2021maximum}. In particular, adversarial training~\citep{kos2017delving,behzadan2017whatever,pattanaik2018robust} and regularized RL algorithms by enforcing the smoothness of the trained models~\citep{shen2020deep,zhang2021robust} have been studied to improve the robustness of trained policies. 
However, several strong adaptive attacks have been proposed against these empirical defenses~\citep{gleave2019adversarial,hussenot2019copycat,russo2019optimal} and it is important to provide  \textit{robustness certification} for a given learning algorithm to end such repeated game between attackers and defenders.

To provide \textit{robustness certification}, several studies have been conducted on classification. For instance, both deterministic~\citep{ehlers2017formal,katz2017reluplex,cheng2017maximum,tjeng2017evaluating,weng2018towards,zhang2018efficient,singh2019abstract,gehr2018ai2,wong2018provable,raghunathan2018certified} and probabilistic approaches \citep{lecuyer2019certified,cohen2019certified,lee2019tight,salman2019provably,carmon2019unlabeled,jeong2020consistency} have been explored to provide a lower bound of classification accuracy given bounded adversarial perturbation.
Considering the sequential decision making property of RL, which makes it  more challenging to be directly certified compared to classification, in this paper we ask: \textit{How to provide efficient and effective robustness certification for RL algorithms?} \textit{What criteria should be used to certify the robustness of RL algorithms?}

Different from classification which involves one-step prediction only,  RL algorithms provide both action prediction and reward feedback, making \textit{what to certify} and \textit{how to certify} robustness of RL challenging.
% Considering different strategies and outputs of RL algorithms, 
In this paper we focus on  Q-learning  and propose two certification criteria: \textit{per-state action stability} and \textit{lower bound of perturbed cumulative reward}.
In particular, to certify the \textit{per-state action stability}, we propose the \textit{local smoothing} on each input state and therefore derive the certified radius for  perturbation  at each state, within which the  action prediction will not be altered.
To certify the \textit{lower bound of cumulative reward}, we propose both  \textit{global smoothing} over the finite trajectory to obtain the expectation or percentile bounds given trajectories smoothed with sampled noise sequences; and 
\textit{local smoothing} to calculate an absolute lower bound based on our adaptive search algorithm. 

We leverage our framework to test nine empirically robust RL algorithms on multiple RL environments. We show that the certified robustness depends on both the algorithm and the environment properties. For instance, RadialRL~\citep{oikarinen2020robust} is the most certifiably robust method on Freeway.
% , and freeway is the most robust and stable games among all we considered. 
In addition,
based on the per-state certification, we observe that for some environments such as Pong, some states are more certifiably robust and such pattern is periodic. Given the information of which states are more vulnerable, it is possible to design robust algorithms to specifically focus on these vulnerable states. 
Based on the lower bound of perturbed cumulative reward, we show that our certification is tight  by comparing our bounds with empirical results under adversarial attacks.
% \bo{more observations to add xxxx}

\vspace{-1mm}
{\bf \underline{Technical Contributions.}\;}
In this paper, we take an important step
towards providing robustness certification for Q-learning. 
We make contributions on both theoretical
and empirical fronts.
\vspace{-2mm}
\begin{itemize}[noitemsep,leftmargin=*]
\item We propose a framework for certifying the robustness of Q-learning algorithms, which is notably the \textit{first} that provides the robustness certification w.r.t. the cumulative reward.
\item We propose two robustness certification criteria for Q-learning algorithms, together with corresponding certification algorithms based on global and local smoothing strategies.
\item We theoretically prove the certification radius for input state and lower bound of perturbed cumulative reward under bounded adversarial state perturbations.
\item \looseness=-1 We conduct extensive experiments to provide certification for nine empirically robust RL algorithms on multiple RL environments. We provide several interesting observations which would further inspire the development of robust RL algorithms.
\end{itemize}
\vspace{-4mm}

\section{Preliminaries}
\label{sec:prelim}
\vspace{-2mm}
% In this section we will provide
% background on Q-learning and randomized smoothing.
% , and then establish the connections between them to show how to provide probabilistic robustness certification for deep Q-Networks via smoothing.

\looseness=-1 \textbf{Q-learning and Deep Q-Networks (DQNs).}\quad
%In this paper, we focus on 
Markov decision processes (MDPs) are at the core of RL.
Our focus is on discounted discrete-time MDPs, which are defined by tuple $(\cal S,\cal A, R,P,\gamma, d_0)$, 
where $\cal S$ is a set of states (each with dimensionality $N$), 
$\cal A$ represents a set of discrete actions, 
$R:\cal S\times \cal A\rightarrow\bb R$ is the reward function, 
and $P: \cal S\times \cal A \rightarrow \cal P(\cal S)$ is the transition function with $\cal P(\cdot)$ defining the set of probability measures,
% Concretely, the transition probability
% $p(s'|s,a) = \Pr\left[s_{t+1}=s'\mid s_t = s,a_t=a\right]$.
$\gamma \in [0,1]$ is the discount factor, 
and $d_0 \in \cal P(\cal S)$ is the distribution over the initial state. 
At time step $t$, the agent is in the state $s_t \in \cal S$. After choosing action $a_t \in \cal A$, the agent transitions to the next state $s_{t+1}\sim P(s_t,a_t)$ and receives reward $R(s_t,a_t)$. The goal is to learn a policy $\pi:\cal S\rightarrow \cal P(\cal A)$ that maximizes the expected cumulative reward $\mathbb{E}[\sum\nolimits_t \gamma^t r_t]$.

Q-learning~\citep{watkins1992q} learns an action-value function (Q-function), $Q^\star(s,a)$, which is 
the maximum expected cumulative reward  the agent can achieve after taking action $a$ in state $s$:
    $Q^\star(s,a) = R(s,a) + \gamma \underset{s'\sim P(s,a)}{\bb E}\left[\max_{a'} Q^\star (s',a') \right]$.
% The (optimal) policy is defined as taking the best action in each state determined by the Q-function:
% $
% \pi(s) = \arg\max_{a \in \cal A} Q^\star(s,a)
% $.
In deep Q-Networks (DQNs)~\citep{mnih2013playing},  $Q^\star$ is approximated using a neural network parametrized by $\theta$, \ie, $Q^\pi(s,a;\theta) \approx Q^\star(s,a)$.
Let $\rho\in\cal P(\cal S \times \cal A)$ be the observed distribution defined over states $s$ and actions $a$, the network can be trained via minimizing loss function 
    $\cal L(\theta) = \underset{(s,a)\sim \rho,s'\sim P(s,a)}{\bb E}\left[ \left(R(s,a)+\gamma \max _{a^{\prime}} Q^\pi\left(s^{\prime}, a^{\prime} ; \theta\right)-Q^\pi(s, a ; \theta)\right)^{2}\right].$
The greedy policy $\pi$ is defined as taking the action with highest $Q^\pi$ value in each state $s$: $\pi(s) = \arg\max_{a\in \cal A} Q^\pi(s,a)$.

\textbf{Certified Robustness for Classifiers via Randomized Smoothing.}\quad
    Randomized smoothing~\citep{cohen2019certified} has been proposed to provide probabilistic certified robustness for  classification.
    It achieves state-of-the-art certified robustness on large-scale datasets such as ImageNet under $\ell_2$-bounded constraints~\citep{salman2019provably,yang2020randomized}.
    % achieved the state-of-the-art certified robustness on large ImageNet and CIFAR-10 dataset under $L_2$ norm.
    % For a single ML model, to obtain robust certification for large radius classifier it needs to predict correctly with high probability even when the input is corrupted by random noise.
    % Several approaches have further improved it by:
    % (1)~choosing different smoothing distributions for different $L_p$ norms~\citep{dvijotham2019framework,zhang2020black,yang2020randomized},
    % and (2)~training more robust smoothed classifiers, using data augmentation~\citep{cohen2019certified}, unlabeled data~\citep{carmon2019unlabeled}, adversarial training~\citep{salman2019provably}, regularization~\citep{li2019certified,zhai2019macer}, and denoising~\citep{salman2020black}.
    %     However, within our knowledge, there is no work studying how to customize randomized smoothing for RL.
    In particular, given a base model and a test instance, a smoothed model is constructed by outputting the most probable prediction over different Gaussian perturbed inputs.
    % Formally, 
    % let $\delta \sim \gN(0, \sigma^2 I)$ be a random noise drawn from the Gaussian distribution,
    % the smoothed classifier would predict 
    % \begin{small}
    % $$\wtilde F(x) = \argmax_{c\in\cal C} \Pr_{\delta \sim \gN(0, \sigma^2 I)} (F(x + \delta) = c)$$
    % \end{small}where $F: \sR^d \to \{1,\dots,|\cal C|\}$ denotes the base classification model with the class set $\cal C$.
    % % Intuitively, the confidence score for each class is given by the probability of predicting that class under noised inputs.
    % In this paper, we will focus on how to leverage different smoothing strategies at local and global levels to smooth a given policy and provide the certified robustness for Q-learning based on different criteria.

\vspace{-2mm}
\section{Robustness Certification in Q-learning}
\label{sec:criteria}
\vspace{-2mm}

In this section, we first introduce the threat model, 
followed by two robustness certification criteria for the Q-learning algorithm: \emph{per-state action} and \emph{cumulative reward}.
% \noindent\textbf{Threat Model.}
We consider the standard adversarial setting in Q-learning \citep{huang2017adversarial,kos2017delving,zhang2021robust}, where the adversary can apply $\ell_2$-bounded perturbation $\epsball = \{\delta \in \sR^n \mid \norm{\delta}_2 \leq \eps\}$ to  input state observations of the agent during decision (test) time to cause the policy to select suboptimal actions.
% in a $\ell_2$-ball with radius $\eps$, denoted by
% $\norm{\delta}_2$ is bounded by $\eps$ for all $s\in \cal S$.
The agent observes the perturbed state and takes action $a' = \pi(s+\delta)$, following policy $\pi$.
%and may transit from the actual state $s$ to an ``adversarial" state (\eg, car crash) following the trained optimal policy $\pi$.
%In this work, we focus on such test-time 
Following the Kerckhoff's principle~\citep{shannon1949communication}, we consider a \textit{worst-case} adversary  who applies adversarial perturbations 
%from $\epsball$ on 
to \emph{every} state at decision time.
Our analysis and methods are generalizable to other $\ell_p$ norms following~\citep{yang2020randomized,lecuyer2019certified}.
%observations at \emph{all} time steps, 
%to evaluate the robustness of given Q-learning algorithms, and in this paper, more concretely, different DQN variants.
% Our robustness certification will therefore offer guarantees for a broad space of adversarial DQN settings without making specific assumptions about the timing of attacks (\ie, when perturbations to state observations begin or end during policy execution).

% We aim to provide the certified performance guarantees for
% Q-learning, and more concretely, different DQN variants. 
% Such certified robustness under different scenarios and criteria provides rigorous analysis for the robustness of  given Q-learning algorithms.

\vspace{-2mm}
\subsection{Robustness Certification for Q-learning with Different Criteria}
\vspace{-2mm}

% Compared with a single classification model which provides only a final prediction, the Q-learning algorithms provide different outputs such as the \textit{action} given current state input as well as the corresponding \textit{cumulative reward}. \ding{very similar to a sentence mentioned above. Maybe change the sentence structure.}
To provide the robustness certification for  Q-learning, we propose two certification criteria:   \textit{per-state action robustness} and  \textit{lower bound of the cumulative reward}.
% given adversarially perturbed state inputs. 

\textbf{Robustness Certification for Per-State Action.}\quad
We first aim to explore the robustness (stability/consistency) of the per-state action given adversarially perturbed input states.
% We first consider that given input state perturbation, how stable the action selection would be. 
    % We first present the definition of the robustness certification for per-state action.
    \begin{definition}[Certification for per-state action]
    \label{def:per-state}
    Given a trained network $\qpi$ with policy $\pi$, we define the robustness certification for per-state action as the \emph{maximum perturbation magnitude} $\bar\eps$, 
    % supremum of the magnitude of perturbation $\eps$ at the given state $s$, called $\bar\eps$, 
    such that for any perturbation $\delta\in \epsbball$, the predicted action under the perturbed state will be the same as the action taken in the clean environment, \ie,
    $
        \pi(s+\delta)=\pi(s)
        , \, \forall 
        \delta\in\epsbball \label{eq:bar-eps}
    $.
    % \vspace{-\topsep}
    % \begin{small}
    % \begin{align}
    %     % \bar\eps(s) = \sup_{\eps} 
    %     % \left\{ \eps \mid \pi(s+\delta)=\pi(s)\text{ for all }\delta\in\epsball\right\}.
    %     \pi(s+\delta)=\pi(s)
    %     , \, \forall 
    %     \delta\in\epsball, \eps \leq \bar\eps. \label{eq:bar-eps}
    % \end{align}
    % \end{small}
    % \vspace{-\topsep}
    \end{definition}
    % In~\Cref{sec:cert-rad},
    % we will provide such certification for a smoothed policy derived based on a given  $\qpi$, where we will introduce policy smoothing procedures and certification strategies.
    \vspace{-2mm}
\textbf{Robustness Certification for Cumulative Reward.}\quad 
Given that the cumulative reward is important for RL, here in addition to the per-state action, we also define the robustness certification regarding the cumulative reward under input state perturbation.
    % Given a {deterministic} environment,  we define $\Gamma(s,a)$ as the \emph{transition function} and $R(s,a)$  the \emph{reward function}.
    % We consider the greedy policy $\pi$ \ding{maybe also define it with a short equ to make things clear, e.g. pi = arg min balabala. Also, can our algorithm work on non-greedy policy. If so, why should we emphasize it is greedy or not here?} 
    % \fan{Two considerations here: 1) the adaptive search algorithm relies on a deterministic policy and the greedy policy we use is by nature a deterministic policy; 2) to build connection with the Q-learning, otherwise the whole section seems unrelated with Q-learning.
    % How about changing the sentence to \textbf{``
    % We consider policy $\pi$, more specifically, a greedy policy $\pi=\argmax_{a}\qpi(s,a)$ given  network $\qpi$ in  Q-learning. We  define the robustness certification for the smoothed $\pi$ in terms of the \textit{lower bound of its cumulative reward}. 
    % We next define the finite horizon cumulative reward, followed by its certification.
    \begin{definition}[Cumulative reward]
    \label{def:cum-reward}
    Let $P: \cal S \times \cal A \rightarrow \cal P(\cal S)$ be the transition function of the environment with $\cal P(\cdot)$ defining the set of probability measures.  
    % Let $\Gamma:\cal S\times \cal A\rightarrow \cal S$ be the transition function, 
    Let $R,d_0,\gamma,\qpi, \pi$ be the reward function, initial state distribution, discount factor, a given trained Q-network, and the corresponding greedy policy as introduced in~\Cref{sec:prelim}.
    % Suppose each state trajectory contains at most $H$ time steps, and 
    $J(\pi)$  represents the \emph{cumulative reward} and $J_\eps(\pi)$  represents the \emph{\namepj} under perturbations $\delta_t\in \epsball$ at each time step $t$: 
        % For infinite horizon setting, we can introduce a discount factor $\gamma$ in (\ref{eq:cum-reward}) and (\ref{eq:cum-pert-reward}), and
    % when $t$ is large enough, the term $\gamma ^t R(s_t,\pi(s_t))$ will be infinitesimal and have no contribution to the cumulative reward. 
    % This way, the infinite horizon problem is reduced to a finite horizon problem with $H=\bar t$ for a $\bar t$ satisfying the aforementioned condition.\bo{update here}
    % \vspace{-\topsep}
    \vspace{8mm}
    \begin{small}
    \begin{equation}
    \begin{aligned}[c]
    J(\pi)  &:= \sum_{t=0}^{\infty} \gamma^t R(s_t, \pi(s_t)),\\
    \text{ where }s_{t+1}&\sim P(s_t,a_t),s_0\sim d_0,
    \end{aligned}
    \qquad\text{and}\qquad
    \begin{aligned}[c]
    J_\eps(\pi)  &:= \sum_{t=0}^{\infty} \gamma^t R(s_t,\pi(s_t+\delta_t)),\\
    \text{ where }s_{t+1}&\sim P(s_t,\pi(s_t+\delta_t)),s_0\sim d_0.
    \end{aligned}
    \end{equation}
    % \begin{align}
    % \label{eq:cum-reward}
    %     J(\pi)  &:= \sum_{t=0}^{\infty} \gamma^t R(s_t, \pi(s_t)),\text{ where }s_{t+1}=\Gamma(s_t,a_t),s_0\sim d_0;\\
    % \label{eq:cum-pert-reward}
    %     J_\eps(\pi)  &:= \sum_{t=0}^{\infty} \gamma^t R(s_t,\pi(s_t+\delta_t)),\text{ where }s_{t+1}=\Gamma(s_t,\pi(s_t+\delta_t)),s_0\sim d_0.
    % \end{align}
    \end{small}
    \end{definition}
    \vspace{-3mm}
    \looseness=-1 The randomness of $J(\pi)$ arises from the environment dynamics, while that of $J_\eps(\pi)$ includes additional randomness from the perturbations $\{\delta_t\}$.
    %_{t=0}^{H-1}$.
    We focus on a finite horizon $H$ in this paper, where a sufficiently large $H$ can approximate $J(\pi)$ and $J_\eps(\pi)$ to arbitrary precision when $\gamma < 1$.
    % Since for a sufficiently large finite horizon $H$ we can approximate $J(\pi)$ and $J_\eps(\pi)$ to arbitrary precision, we henceforth focus on a finite horizon $H$.
    
%    Considering a finite-horizon state trajectory of length $H$,
%    when $H$ grows sufficiently large, the discount factor $\gamma$ will reduce the infinite-horizon cumulative reward to the finite one in this case.
    % \begin{definition}[Certification of the cumulative reward]
    % We define the \emph{certification} of the cumulative reward as three statistics of $J_\eps(\pi)$ defined in~\Cref{def:cum-reward}: 
    % the lower bound of its expectation $\E\left[J_\eps(\pi)\right]$, the lower bound of its $p$-th percentile, and its absolute lower bound. We use the briefed notations $\underline{J_E},\underline{J_p}, \uj$ for the three statistics where $\eps,H,\pi$ are clear from context.
    % \end{definition}

    \begin{definition}[Robustness certification for cumulative reward]
    \label{def:cert-cum-reward}
    The robustness certification for cumulative reward is the \emph{lower bound} of \textit{\namepj} $\uj$ such that $\uj \leq J_\eps({\pi})$ under perturbation in $\epsball = \{\delta \in \sR^n \mid \norm{\delta}_2 \leq \eps\}$ applied to all time steps.
    % three statistics of $J_\eps(\pi)$ defined in~\Cref{def:cum-reward}: 
    % the lower bound of its expectation $\E\left[J_\eps(\pi)\right]$, the lower bound of its $p$-th percentile, and its absolute lower bound. We use the briefed notations $\underline{J_E},\underline{J_p}, \uj$ for the three statistics where $\eps,H,\pi$ are clear from context.
    \end{definition}
    \vspace{-2mm}    
    We will provide details on the certification of per-state action in~\Cref{sec:cert-rad} and the certification of cumulative reward in~\Cref{sec:cert-cum} 
    % In~\Cref{sec:cert-cum}, we will provide details on the certification of cumulative reward
    % , concretely, $\uje, \ujp, \uj$, 
    based on different smoothing strategies and certification methods.
    % We provide 
    % in section \Cref{sec:global-rand-smooth}, we also provide certification methods for two slightly different definitions (\Cref{def:xx,def:xx})
    
    % We consider all these statistics (mean, $p$-th percentile, and the worst case) in order to provide a comprehensive view of the property of the random variable $J_\eps(\pi)$.

\vspace{-2mm}
\section{Robustness Certification Strategies for Per-State Action}
\label{sec:cert-rad}
\vspace{-2mm}

In this section, we discuss the  robustness certification for \textit{per-state action}, aiming to calculate a lower bound of \textit{maximum perturbation magnitude} $\bar\eps$  in~\Cref{def:per-state}.

\vspace{-2mm}
\subsection{Certification for Per-State Action via Action-Value Functional Smoothing}
\label{sec:cert-act-thm}
\vspace{-2mm}

    \looseness=-1
    Let $Q^\pi$ be the action-value function given by the trained network $Q$ with policy $\pi$.
    We derive a smoothed function $\wtilde Q^\pi$ through per-state \textit{local smoothing}. Specifically, at each time step $t$, for each action $a\in\cal A$, we draw random noise from a Gaussian distribution $\gN(0,\sigma^2I_N)$ to smooth $\qpi(\cdot,a)$. 
    % where $I_N$ denotes an identity matrix of size $N\times N$. 
    % We then obtain the following locally smoothed value function $\wtilde Q$ and the associated locally smoothed policy $\tilde\pi$:
    \vspace{8mm}
    \begin{small}
    \begin{align}
        \label{eq:smooth-q}
        \wtilde Q^\pi(s_t,a)
        % := (Q^\pi\circ \cal N(0,\sigma^2 I))(s,a) 
        := \underset{\Delta_t\sim \gN(0,\sigma^2I_N)}{\E}\qpi(s_t+\Delta_t,a)\quad \forall s_t\in\cal S, a\in\cal A,
        \ \ \textrm{and}\ \
        % \label{eq:smooth-pi}
        \tilde \pi(s_t) := \argmax_a \wtilde Q^\pi(s_t,a)\quad \forall s_t\in\cal S.
    \end{align}
    \end{small}
    \vspace{5mm}
% \begin{lemma}[\lip continuity of the smoothed value function]
\begin{restatable}[\lip continuity of the smoothed value function]{lemma}{lipcontq}
    \label{lem:lipschitz}
    Given the action-value function $Q^\pi: \cal S\times \cal A\rightarrow [V_{\min}, V_{\max}]$, the smoothed function $\wtilde Q^\pi$ with smoothing parameter $\sigma$  is L-\lip continuous with  
    $L=\frac{ V_{\max} - V_{\min} }{\sigma}\sqrt{\nicefrac{2}{\pi}}$ w.r.t. the state input.
\end{restatable}
% \end{lemma}
\vspace{-3mm}
% The procedure is sketched in ~\Cref{alg:smooth-q}.

% \subsection{Provable Performance Guarantee under Lipschitz Condition}
% \subsection{Provable Performance Guarantee for Smoothed Action Value Function}

% \paragraph{Certified robustness for per-state action.}
The proof is given in~\Cref{append:proof-lip-local}. 
Leveraging the \lip continuity in~\Cref{lem:lipschitz}, we derive the following theorem for certifying the robustness of per-state action.

% \begin{theorem}
\begin{restatable}{theorem}{thmrad}
\label{thm:radius}
\looseness=-1 Let $Q^\pi: \cal S \times \cal A\rightarrow [\vmin,\vmax]$ be a trained value network, 
$\wtilde Q^\pi$ be the smoothed function with (\ref{eq:smooth-q}).
% \linyi{\textbf{[old]}: At time step $t$, we can certify a radius $r_t$ for state $s_t$, such that $r_t = \bar\eps(s_t)$, \ie, 
% the induced action $\tpi(s_t+\delta_t)$ under any state perturbation $\delta_t$ will be the same as the action given clean state observation $\tpi(s_t)$ as long as $\eps < r_t$.
% The radius $r_t$ can be calculated as below:}
At time step $t$ with state $s_t$, we can compute the lower bound $r_t$ of maximum perturbation magnitude $\bar\eps(s_t)$~(\ie, $r_t \le \bar\eps(s_t)$, $\bar\eps$ defined in~\Cref{def:per-state}) for locally smoothed policy $\tilde \pi$:
% \begin{align*}
%     R(s)
%     &:= \frac{1}{2}\left( \eta_{a_1}(s) - \eta_{a_2}(s) \right),
% \end{align*}
\vspace{8mm}
\begin{small}
\begin{align}
        r_t = \frac{\sigma}{2}\left(
    \Phi^{-1}\left(\frac{\wtilde Q^\pi(s_t,a_1)-V_{\min}}{V_{\max}-V_{\min}}\right)-
    \Phi^{-1}\left(\frac{\wtilde Q^\pi(s_t,a_2)-V_{\min}}{V_{\max}-V_{\min}}\right)
    \right), \label{eq:R-t}
\end{align}
\end{small}
% \vspace{10mm}
\vspace{4mm}

where $\Phi^{-1}$ is the inverse CDF function, $a_1$ is the action with the highest $\wqpi$ value at state $s_t$, and $a_2$ is the runner-up  action. We name the lower bound $r_t$ as certified radius for the state $s_t$.
\end{restatable}

The proof is omitted to~\Cref{append:proof-radius}. The theorem provides a certified radius $r_t$ for per-state action given smoothed policy: As long as the perturbation is bounded by $r_t$, \ie, $\|\delta_t\|_2 \le r_t$, the action does not change: $\tpi(s_t + \delta_t) = \tpi(s_t)$.
To achieve high certified robustness for per-state action,~\Cref{thm:radius} implies a tradeoff between value function smoothness and the margin between the values of top two actions: If a larger smoothing parameter $\sigma$ is applied, the action-value function would be smoother and therefore more stable; however, it would shrink the margin between the top two action values leading to smaller certified radius. Thus, there exists a proper smoothing parameter to balance the tradeoff, which depends on   the actual environments and algorithms.
% and learning algorithms. 

\vspace{-2mm}
\subsection{\staters: Local Randomized Smoothing for Certifying Per-State Action}
\label{sec:cert-rad-algo}
\vspace{-2mm}

\looseness=-1 Next we introduce the algorithm to achieve the certification for per-state action.
Given a $Q^\pi$ network, we apply (\ref{eq:smooth-q}) to derive a smoothed network $\wqpi$. 
At each state $s_t$, we obtain the greedy action $\tilde a_t$ w.r.t. $\wqpi$, and then compute the certified radius $r_t$.
We present the complete algorithm in~\Cref{append:staters}.

There are some additional challenges in smoothing the value function and computing the certified radius in Q-learning compared with the standard classification task~\citep{cohen2019certified}.
\textbf{Challenge 1:}
% First of all, 
In classification, the output range of the confidence $[0,1]$ is known a priori; however, in Q-learning, for a given $\qpi$, its range $[V_{\min},V_{\max}]$ is unknown. 
% We therefore need a pre-processing step to compute the range.
% The range can be exactly computed via enumerating all input states.
% To reduce time complexity, we only constrain ourselves to 
% We estimate the range based on a finite set of valid states $\cal S_{\rm{sub}} \subseteq \cal S$.
% In order to obtain an accurate estimation, we sample a set of sufficiently large size. 
% Second, 
\textbf{Challenge 2:}
In the classification task, the lower  and upper bounds of the top two classes' prediction probabilities can be directly computed via the confidence interval base on multinomial proportions~\citep{goodman1965simultaneous}.
For  Q-networks, the outputs are not probabilities and calculating the multinomial proportions becomes challenging.
% Thus, we need different approaches to compute the confidence interval. 
% , we instead resort to the Hoeffding's inequality~\cite{hoeffding1994probability} to compute the interval.
% We next provide detailed approaches to address these challenges to certify the robustness of per-state action. 
% a walkthrough of the algorithm addressing the challenges in detail.

\textbf{Pre-processing.}\quad
To address \textbf{Challenge 1},
 we estimate the output range $[V_{\min},V_{\max}]$ of a given network $\qpi$ based on a finite set of valid states $\cal S_{\rm{sub}} \subseteq \cal S$.
% We take the minimum and maximum of the output of $\qpi$ on $\cal S_{\rm{sub}}$ to be $\vmin$ and $\vmax$.
In particular, we craft a sufficiently large set $\cal S_{\rm{sub}}$ to estimate $\vmin$ and $\vmax$ for $\qpi$ on $\cal S$, which can be used later in per-state smoothing.
The details for estimating $\vmin$ and $\vmax$ are deferred to~\Cref{append:staters}.
% We next present the detailed steps to certify the robustness of per-state action. 

\looseness=-1
\textbf{Certification.}\quad 
To smooth a given state $s_t$, we use Monte Carlo sampling~\citep{cohen2019certified} to sample noise applied to $s_t$, and then 
estimate the corresponding smoothed value function $\wqpi$ at $s_t$ with (\ref{eq:smooth-q}).
In particular, we sample $m$ Gaussian noise $\Delta_i \sim \cal N(0,\sigma^2 I_N)$, clip the Q-network output to ensure that it falls within the range $[V_{\min},V_{\max}]$, and then take the average of the output to obtain the smoothed action prediction based on $\wqpi$.
% \textbf{Step 2: Per-state certification.}\quad
We then employ~\Cref{thm:radius} to compute the certified radius $r_t$.
We omit the detailed inference procedure to~\Cref{append:staters}.
% Since the smoothed value $\wqpi(s,a)$ is obtained over a finite number of smoothed input states, 
% we need to compute its lower and upper bounds under a given confidence level.
To address \textbf{Challenge 2}, we leverage Hoeffding's inequality~\citep{hoeffding1994probability}  to compute a lower bound of $\wqpi(s_t,a_1)$ and an upper bound of $\wqpi(s_t,a_2)$ with one-sided confidence level parameter $\alpha$ given the top two actions $a_1$ and $a_2$. 
When the former is higher than the latter, we can certify a positive radius for the given state $s_t$.

\vspace{-2mm}
\section{Robustness Certification Strategies for the Cumulative Reward}
\label{sec:cert-cum}
\vspace{-2mm}

    % \emph{infinite horizon cumulative reward} and by
    % \begin{align}
    %     % \bar Q^\pi(s,a) & := \E\left[ \sum_{t=0}^\infty \gamma^t r_t \Big| s_0 = s, a_0=a, \pi\right], \\
    %     % \bar V^\pi(s) & := \E\left[ \sum_{t=0}^\infty \gamma^t r_t \Big| s_0 = s, \pi\right], \\
    %     J^\infty(\pi)  := \E\left[\sum_{t=0}^\infty \gamma^t r_t \Big| s_0 \sim d_0, \pi\right]
    %     \quad\text{and}\quad
    %     J(\pi)  := \E\left[\sum_{t=0}^{H}  r_t \Big| s_0 \sim d_0, \pi\right].
    % \end{align}
   
    In this section, we present robustness certification strategies for the cumulative reward. The goal is to provide the lower bounds for the \emph{\namepj} in~\Cref{def:cum-reward}.
    In particular, we propose both \textit{global smoothing} and \textit{local smoothing} strategies to certify the perturbed cumulative reward. In the global smoothing, we view the whole state trajectory as a function to smooth, which would lead to relatively loose certification bound.
    We then propose the local smoothing by smoothing each state individually to obtain the absolute lower bound.
    % for any given state trajectory based
    % by proposing an adaptive search algorithm \adasearch.
    
    % We start from considering a \emph{\namef} function defined over the state trajectory. 
    % We first provide an \emph{\ebound} for it.\bo{it?}
    % Noticing that the \lip constant in the \ebound is algorithm agnostic, we then improve the method to derive a more practical and tighter \emph{\pbound}.
    % However, since \emph{\namef} views the entire trajectory as the input and applies global smoothing over it, the derived lower bounds are inevitably loose.
    % We thus take a step further to propose an adaptive search algorithm\bo{name} built on top of the state-wise smoothing strategy, which achieves \emph{an absolute lower bound} for any trajectory with a given initial state.
    % \bo{global and local, then xxxx, top-down!}
    
\vspace{-2mm}
\subsection{Certification of Cumulative Reward based on Global Smoothing}
\label{sec:cert-glb}
\vspace{-2mm}

    \looseness=-1
    In contrast to~\Cref{sec:cert-rad} where we perform per-state smoothing to achieve the certification for per-state action, here, we aim to perform \textit{global smoothing} on the state trajectory by viewing the entire trajectory as a function.
    In particular, we first derive the \textit{expectation bound} of the cumulative reward based on global smoothing by estimating  the \lip constant for the cumulative reward w.r.t. the trajectories.
    Since the \lip estimation in the \ebound is algorithm agnostic and could lead to loose estimation bound, we subsequently propose a more practical and tighter \emph{\pbound}.
    % which does not require to estimate the Lipschitz constant.
    % , and focus on the percentile certification of the perturbed cumulative reward.
    % To achieve trajectory-wise smoothing, we define the following notions on :
    % We next introduce two definitions on \textit{trajectory} level that enables global smoothing:

    % To smooth the entire state trajectory via global smoothing, we  first define the randomized trajectory and policy below by adding a sampled noise sequence to the state trajectory of interest.

    \begin{definition}[$\sigma$-randomized trajectory and $\sigma$-randomized policy]
    \label{def:sig-rand} 
    \looseness=-1 Given a state {trajectory} $(s_0,s_1,\ldots,s_{H-1})$ of length $H$ where $s_{t+1}\sim P(s_t,\pi(s_t))$, $s_0\sim d_0$, with $\pi$ the greedy policy of the action-value function $Q^\pi$, we derive a \textit{$\sigma$-randomized trajectory} as $(s_0',s_1',\ldots,s_{H-1}')$, where $s_{t+1}'\sim P(s_t',\pi(s_t'+\Delta_t))$, $\Delta_t\sim\gN(0,\sigma^2I_N)$, and $s_0'=s_0\sim d_0$.
    We correspondingly define a \textit{$\sigma$-randomized policy} $\pi'$ based on $\pi$ in the following form: 
    $\pi'(s_t):=\pi(s_t+\Delta_t)$ where $\Delta_t\sim\gN(0,\sigma^2I_N)$.
    \end{definition}
    \vspace{-\topsep}
    % This way, the state transition in $\sigma$-randomized trajectory can be re-written as $s_{t+1}' = \Gamma(s_t',\pi'(s_t'))$, which has a consistent form compared with that of the original trajectory    
    % We note that, with \sigrpi $\pi'$, the state transition in \sigrtj can be re-formulated in a consistent form as the original trajectory: $s_{t+1}' = \Gamma(s_t',\pi'(s_t'))$, implying that the \sigrtj can be viewed as generated by the \sigrpi.
    % We will next present certification results obtained based on multiple sampled $\sigma$-randomized trajectories, which aims to smooth the whole trajectory.
    
    % maybe change to ``We will next present how to achieve the certification for \namepj by smoothing over sampled $\sigma$-randomized trajectories.''
  
    % To perform the global smoothing, we will first define a 
    Let the operator $\vert$  concatenates given input states or noise that are added to each state. The sampled noise sequence is denoted by $\Delta = \vertt{0}{H-1}\Delta_t$, where $\Delta_t \sim \gN(0,\sigma^2I_N)$.
    % Without loss of generality, we define the notion   \emph{macro-input}\bo{state?} as the concatenation of $H$ input states $\delta_i \in \sR^n$, \ie, $\vert_{i=1}^{H}\delta_i$, where $\oplus$ is a concatenation operator. \bo{it's noise???}
    % Each macro-input thus has dimensionality $H\times n$.
    %We then define the \textit{\namef} that represents the cumulative return of a perturbed state trajectory, which serves as a proxy of the \textit{\namepj} $J_\eps$.
    % We will later see how to build connections between \namef $F$ and perturbed cumulative reward $J_\eps$.

    \begin{definition}[\nameF]
    \label{def:namef}
    \looseness=-1 Let $R,P,\gamma,d_0$ be the reward function, transition function, discount factor, and initial state distribution in~\Cref{def:cum-reward}. 
    % Let $\Pi$ be a policy set containing any policy including the $\sigma$-randomized policies derived when applying smoothing on the state trajectory.
    We define a bounded \textit{\namef}  $F_\pi: \sR^{H\times N}  \to [\jmin,\jmax]$
    % that maps the combination of a macro-state of $H$ noise $\{\delta_t\}_{t=0}^{H-1}$ and a policy $\pi$ to 
     representing 
    % finite horizon 
    cumulative reward with potential perturbation $\delta$:
    \vspace{7mm}
    \begin{small}
    \begin{align}
        \label{eq:global-f}
        F_\pi\left(\vert_{t=0}^{H-1}\delta_t\right) := 
            % \sum_{t=0}^H \gamma_t r_t \Big| s_0 \sim d_0, s_i = \pi(s_{i-1}+\delta_i)
            \sum_{t=0}^H \gamma^t R(s_t,\pi(s_{t}+\delta_t)), \quad\text{where }s_{t+1}\sim P(s_t,\pi(s_{t}+\delta_t)), s_0\sim d_0.
    \end{align}
    \end{small}
    \end{definition}
    \vspace{-3mm}
    \looseness=-1
    We can see when there is no perturbation ($\delta_t=\bf 0$),
     $F_\pi(\vert_{t=0}^{H-1}\mathbf{0})=J(\pi)$;
     when there are adversarial perturbations $\delta_t\in\epsball$ at each time step,
     $F_\pi(\vert_{t=0}^{H-1}\delta_t)=J_\eps(\pi)$, \ie, \textit{\namepj}.

    % Next, we show how to calculate the general lower bounds of the \textit{\namepj} $J_\eps(\pi)$ via global smoothing.
    % We first smooth the state trajectory by sampling noise sequences $\Delta$ for the state trajectory of interest, and obtain the expected \namef over these sampled noise sequences.
    % By considering multiple $\sigma$-randomized trajectories, we achieve a lower bound for the expectation of \namef. 
    % We then derive the lower bound of the expected \namepj.
    % via the \textit{mean smoothing}.
    %based on~\Cref{thm:exp-bound}. 

    % We discuss these two bounds below in detail.

    % We leverage the mean smoothing method to obtain the expectation bound, \ie, lower bound of $\bb E\left[ J_\eps(\pi')\right]$, and percentile smoothing to obtain the percentile bound, \ie, lower bound of the $p$-th percentile of $J_\eps(\pi')$.
    
    % We next discuss \emph{two} methods to compute the lower bound of smoothed $F$
    % when the perturbation is constrained by $\norm{\delta_i^m}_2\leq \eps$ for all $i$.
    \looseness=-1
    \textbf{Mean Smoothing: Expectation bound.}\quad
        %Here we leverage the mean smoothing strategy to smooth the \textit{perturbed cumulative reward} for a whole state trajectory by 
        Here we propose to sample  noise sequences $\Delta$ to perform \textit{global smoothing} for the entire state trajectory, and  calculate the lower  bound of the expected \namepj  $\bb E_{\Delta}\left[ J_\eps(\pi')\right]$ under all possible $\ell_2$-bounded perturbations within magnitude $\eps$.
        The expectation is over the noise sequence $\Delta$ involved in the \sigrpi $\pi'$ in~\Cref{def:sig-rand}.
        % over the noise sequences, where $\pi'$ is a $\sigma$-randomized policy  in~\Cref{def:sig-rand}.
        % \fanm{and the randomness of $\pi'$ arises exactly from $\Delta\sim\gN(0,\sigma^2I_{H\times N})$.}
        % where $\delta_i \in \sR^S$ is the perturbation added to each step.
        
        % The smoothed \namef is defined as the expectation of \namef $F$ on \fanm{sampled $\sigma$-randomized trajectories}:
        % \begin{align}
        %     \label{eq:smooth-global-f}
        %      \wtilde F(\delta_0\vert\cdots\vert\delta_{H-1})
        %     := \underset{\substack{\Delta_i \sim \gN(0,\sigma^2I_S),\\ 0\le i< H}}{\E} F(\delta_0+\Delta_0\vert\cdots\vert\delta_{H-1}+\Delta_{H-1})
        % \end{align}
        % \begin{align}
        %     \label{eq:smooth-global-f}
        %      \wtilde F_\pi\left(\vert_{t=0}^{H-1}\delta_t\right)
        %       := \underset{\Delta \sim \cal N (0,H\sigma^2I)}{\E} F_\pi\left(\vert_{t=0}^{H-1}(\delta_t+\Delta_t)\right),\quad
        %     % := \underset{\vert_{t=0}^{H-1}\Delta_t}{\E} F_\pi\left(\vert_{t=0}^{H-1}(\delta_t+\Delta_t)\right),\quad
        %     % \text{where }\Delta_t\sim \gN(0,\sigma^2I).
        % \end{align}
        % The following lemma and theorem provide the lower bound of $\bb E\left[J_\eps(\pi')\right]$.

        % \begin{lemma}[\lip continuity of smoothed \namef]
        \begin{restatable}[\lip continuity of smoothed \namef]{lemma}{lemlipf}
        \label{lem:global-lip}
            Let $F$ be the \namef function defined in (\ref{eq:global-f}), the smoothed \namef   $\wtilde F_\pi$ is $\frac{(J_{\max}-J_{\min})}{\sigma}\sqrt{\nicefrac{2}{\pi}}$-\lip continuous, where 
             $\wtilde F_\pi\left(\vert_{t=0}^{H-1}\delta_t\right)
              := \underset{\Delta \sim \cal N (0,\sigma^2I_{H\times N})}{\E} F_\pi\left(\vert_{t=0}^{H-1}(\delta_t+\Delta_t)\right)$.
            % := \underset{\vert_{t=0}^{H-1}\Delta_t}{\E} F_\pi\left(\vert_{t=0}^{H-1}(\delta_t+\Delta_t)\right),\quad
            % \text{where }\Delta_t\sim \gN(0,\sigma^2I).
        \end{restatable}            
        % \end{lemma}
        \vspace{-3mm}
        \begin{restatable}[Expectation bound]{theorem}{expbound}
            \label{thm:exp-bound}
            % The finite horizon cumulative reward 
%            For the \namepj $J_\eps(\pi')$ defined in (\ref{eq:cum-pert-reward}),
%            we define the \ebound $\uje$ as the lower bound of the expected perturbed reward, 
%            such that $\uje \leq \bb E\left[J_\eps(\pi')\right]$.
            %The \ebound can be computed as 
            Let $\uje=\wtilde F_\pi\left(\vertt{0}{H-1}\mathbf{0}\right) - L\eps\sqrt{H}$,
            % its expectation $\bb E\left[J_\eps(\pi')\right]$ is lower bounded by
            % $\wtilde F_\pi\left(\vertt{0}{H-1}0_{N}\right) - L\eps\sqrt{H}$,
            where $L = \frac{(J_{\max}-J_{\min})}{\sigma}\sqrt{\nicefrac{2}{\pi}}$.
            %is the \lip constant of the \namef given by~\Cref{lem:global-lip}.
            Then $\uje \leq \bb E\left[J_\eps(\pi')\right]$.
        \end{restatable}
        % \vspace{-3\topsep}
        % \begin{proof}\let\qed\relax[]
        \textit{Proof Sketch.}\quad
        We first derive the equality between  expected \namepj $\bb E\left[J_\eps(\pi')\right]$ and the smoothed \namef $\wtilde F_\pi(\vert_{t=0}^{H-1}\delta_t)$. Thus, to lower bound the former, it suffices to lower bound the latter,
        which can be calculated leveraging the \lip continuity of $\wtilde F$ in~\Cref{lem:global-lip} (proved in~\Cref{append:proof-global-lip}), noticing that the distance between $\vertt{0}{H-1}\mathbf{0}$ and the adversarial perturbations $\vertt{0}{H-1}\delta_t$ is bounded by $\eps\sqrt{H}$. 
        The complete proof is omitted to~\Cref{append:proof-exp-bound}.
        % \end{proof}
        % \vspace{-3mm}
        
        We obtain $J_{\min}$ and $J_{\max}$ in \Cref{lem:global-lip} from environment specifications which can be loose in practice. Thus the \lip constant $L$ estimation is coarse and mean smoothing is usually loose.
        % One downside of mean smoothing is that the \lip constant $L$ estimation is algorithm agnostic, which is loose due to the coarse bounds of $\jmin$ and $\jmax$  in~\Cref{lem:global-lip}.
        % As shown in~\Cref{lem:global-lip}, $L$ depends on $\jmin$ and $\jmax$, which are the smallest and largest possible cumulative reward of the game for \textit{any} algorithm.
        % Thus, the \lip constant can be quite large, and the obtained lower bound will be too loose to be considered effective.
        We next present a method that circumvents estimating the \lip constant and provides a tight percentile bound.
        % eliminating the dependence on the range of the cumulative reward.

    \textbf{Percentile Smoothing: Percentile bound.}\quad
        We now propose to apply \textit{percentile smoothing} to smooth the \textit{\namepj} 
        and obtain the lower bound of the $p$-th percentile of $J_\eps(\pi')$, where $\pi'$ is a $\sigma$-randomized policy defined in~\Cref{def:sig-rand}.
        % Given $F_\pi: \sR^{H \times S} \to \sR$ and $\Delta\sim\gN(0,\sigma^2I_{H\times N})$, we 
        % We define the $p$-percentile smoothing of $F_\pi$ as
        \vspace{8mm}
        \begin{small}
        \begin{align}
        \label{eq:wtilde-f}
            \wtilde F_{\pi}^p\left(\vertt{0}{H-1}\delta_t\right) = \mathrm{sup}_y\left\{ y\in \sR \mid \mathbb{P}\left[F_\pi\left(\vertt{0}{H-1}(\delta_t+\Delta_t)\right) \leq y\right] \leq p \right\}.
        \end{align}
        \vspace{5mm}
        \end{small}
        % \end{definition}
        \begin{restatable}[Percentile bound]{theorem}{percbound}
        % \begin{theorem}[Percentile bound]
            \label{thm:perc-bound}
%            For the \namepj $J_\eps(\pi')$ defined in (\ref{eq:cum-pert-reward}),
%            we define the \pbound as the lower bound of its $p$-th percentile $\ujp$, such that
%            $\ujp \leq $ the $p$-th percentile of $J_\eps(\pi')$.
%            The percentile bound can be computed by 
            Let $\ujp = \wtilde F_\pi^{p'}\left(\vertt{0}{H-1}\mathbf{0}\right)$,
            % its $p$-th percentile is lower-bounded by $\wbar F_{p'}\left(\vertt{0}{H-1}\mathbf{0}\right)$,
            where $p':=\Phi\left(\Phi^{-1}(p)-\nicefrac{\eps\sqrt{H}}{\sigma}\right)$.
            Then $\ujp \leq $ the $p$-th percentile of $J_\eps(\pi')$.
        % \end{theorem}
        \end{restatable}
        \vspace{-\topsep}
        \looseness=-1
        The proof is provided in~\Cref{append:proof-perc-bound} based on \citet{chiang2020detection}.
        % \textbf{Proof Sketch:} 
        % The $p$-th percentile of $J_\eps(\pi')$ is equivalent to $\wbar F_\pi^{p'}\left(\vertt{0}{H-1}\delta_t\right)$. \more
        There are several other advantages of percentile smoothing over mean smoothing.
        \underline{First}, 
        the certification given by percentile smoothing is among the cumulative rewards of the sampled \sigrtjs and is therefore achievable by a real-world policy, 
        while the \ebound is less likely to be achieved in practice given the loose \lip bound.
        % the percentile smoothing method always produces a practical bound within a reasonable range, since the value $\wbar F_{p'}\left(\vertt{0}{H-1}\mathbf{0}\right)$ is defined to be within the range of $F$.
        \underline{Second}, for a discrete function such as  \namef, the output of mean smoothing is  continuous w.r.t. $\sigma$, while the results given by percentile smoothing remain discrete. 
        Thus, the percentile smoothed function preserves properties of the base function before smoothing, and shares  similar interpretation, \eg, the number of rounds that the agent wins.
        % and therefore shares similar properties with the original function. 
        \underline{Third}, 
        taking $p=50\%$ in percentile smoothing leads to the \textit{median smoothing} which achieves additional properties such as  robustness to outliers~\citep{manikandan2011measures}.
        % We refer to this special case as \textit{median smoothing}.
        % Based on~\Cref{thm:exp-bound} and~\Cref{thm:perc-bound}, we next introduce the concrete algorithms for certifying the expectation  and the percentile bounds of \namepj.
    % Concretely, we first perform \textit{global smoothing} using Monte Carlo sampling, and then calculate the {certification for perturbed cumulative reward} based on  \Cref{thm:exp-bound} and~\Cref{thm:perc-bound} to obtain  \textit{\ebound} $\uje$ and \textit{\pbound} $\ujp$ respectively. 
    The detailed algorithm \glbrs, including the inference procedure, the estimation of $\wtilde F_\pi$, the calculation of the empirical order statistics, and the configuration of algorithm parameters $\jmin,\jmax$, are deferred to~\Cref{append:glbrs}. 
    % \bo{xxxx. check this part and can move 5.2 to appendix} 

\vspace{-2mm}    
\subsection{Certification of Cumulative Reward based on Local Smoothing}
\label{sec:cert-ada}
\vspace{-2mm}
    
    % \fanm{
    % Since the \textit{\namef} function $F$ in~\Cref{def:namef} is defined on \sigrtjs, the derived lower bounds based on \emph{global smoothing} over $F$ would be relatively loose by viewing the entire trajectory as the function input.
    % In this subsection, 
    % we aim to provide a tighter lower bound for the \textit{\namepj} $J_\eps(\tpi)$ with a given initial state, since the randomness of a trajectory only arises from the initial state in a deterministic environment with a deterministic policy, apart from the influences of the perturbations.
    % We thus introduce a method of computing the lower bound for $J_\eps$ w.r.t. the \textit{locally smoothed} policy $\tpi$ defined in (\ref{eq:smooth-q}), starting from the given $s_0$.
    % We first briefly discuss the intuition behind the algorithm, and omit the complete algorithm as well as the technical details to~\Cref{sec:tree-search-impl}.
    % }
    
    Though global smoothing  provides efficient and practical bounds for the perturbed cumulative reward, such bounds are still loose as they involve smoothing the entire trajectory at once. In this section, we aim to  provide a tighter lower bound for $J_\eps(\tpi)$ by performing  \textit{local smoothing}. 
    % Since the randomness of a trajectory arises from the initial state in a deterministic environment with a deterministic policy apart from the influences of perturbations, we introduce a method of computing the lower bound for $J_\eps$ w.r.t. the \textit{locally smoothed} policy $\tpi$ defined in (\ref{eq:smooth-q}), starting from the given $s_0$.
    % We first briefly discuss the intuition behind the algorithm, and defer the complete algorithm and technical details to~\Cref{sec:tree-search-impl}.

    Given a trajectory of $H$ time steps which is guided by the locally smoothed policy $\tpi$,
    we can compute the certified radius at each time step according to~\Cref{thm:radius}, which can be denoted as $r_0,r_1,\ldots,r_{H-1}$. 
    Recall that when the perturbation magnitude $\eps<r_t$, the optimal action $a_t$ at time step $t$ will remain unchanged.
    This implies that when $\eps < \min_{t=0}^{H-1} r_t$, none of the actions in the entire trajectory will be changed, and therefore the lower bound of the cumulative reward when $\eps <  \min_{t=0}^{H-1} r_t$ is the return of the current trajectory in a deterministic environment.
    % On the other hand, if $\eps \geq \min_{i=0}^{H-1} r_i$, this means the actions at the steps with a small radius will be subject to change. 
    %As $\eps$ grows larger, we observer two influences.
    Increasing $\eps$ has two effects.
    First, the \emph{total} number of time steps where the action is susceptible to change will increase; 
    second, at \emph{each} time step, the action can change from the best to the runner-up or the rest.
    % We next describe how to address the two challenges.
    We next introduce an extension of certified radius $r_t$ to characterize the two effects.
    % To  characterize the change of action \emph{at each time step}, we propose an extension of certified  radius $r_t$ considering the top-$k$ actions.
    \vspace{-1mm}
    % \begin{theorem} 
    \begin{restatable}{theorem}{thmextendrad}
    \label{thm:extend-radius}
        Let $(r_t^1,\ldots, r_t^{|\cal A|-1})$ be a sequence of certified radii for state $s_t$ at time step $t$, where $r_t^k$ denotes the radius such that if $\eps<r_t^k$, the possible action at time step $t$ will belong to the actions corresponding to top $k$ action values of $\wtilde Q$ at state $s_t$. 
        The  definition of $r_t$ in~\Cref{thm:radius} is equivalent to $r_t^1$ here. 
        The radii can be computed similarly as follows:
        \vspace{8mm}
        \begin{small}
        \begin{align*}
                r_t^k = \frac{\sigma}{2}\left(
            \Phi^{-1}\left(\frac{\wtilde Q^\pi(s_t,a_1)-V_{\min}}{V_{\max}-V_{\min}}\right)-
            \Phi^{-1}\left(\frac{\wtilde Q^\pi(s_t,a_{k+1})-V_{\min}}{V_{\max}-V_{\min}}\right)
            \right),\quad 1\leq k< |\cal A|,
        \end{align*}
        \end{small}
        \vspace{5mm}
        
        where $a_1$ is the action of the highest $\wtilde Q$ value at state $s_t$ and $a_{k+1}$ is the $(k+1)$-th best action. We additionally define $r_t^0(s_t)=0$, which is also compatible with the definition above.
    \end{restatable}
    % \end{theorem}
    \vspace{-1mm}    
    We defer the proof in~\Cref{append:proof-extend-rad}.
    With~\Cref{thm:extend-radius}, for any given $\eps$, we can compute all possible actions under perturbations in $\epsball$.
    % that are $\ell_2$ bounded by $\eps$.
    This allows an exhaustive search to traverse all trajectories satisfying  that all certified radii along the trajectory are smaller than $\eps$.
    Then, we can conclude that 
    $J_\eps(\tpi)$ is lower bounded by the minimum return over all these possible trajectories.
    % the abosulte lower bound of the perturbed cumulative reward under  perturbation of magnitude $\eps$ 
    % is the minimum return over all these possible trajectories.
    % We next describe an efficient adaptive search algorithm \adasearch to for this task.

\vspace{-2mm}
\subsubsection{\adasearch: Local smoothing for Certified Reward }
\label{sec:tree-search-impl}
\vspace{-2mm}

    Given a  policy $\pi$ in a {deterministic} environment, let the initial state be  $s_0$, we propose  \adasearch to certify the lower bound of $J_\eps(\tpi)$.
    At a high-level, \adasearch \textit{exhaustively} explores new trajectories leveraging \Cref{thm:extend-radius} with priority queue and \textit{effectively} updates the lower bound of cumulative reward $\uj$ by expanding a trajectory tree  dynamically. 
    The algorithm returns a collection of pairs $\{(\eps_{i}, \uj_{\eps_i})\}_{i=1}^{|C|}$ sorted in ascending order of $\eps_i$, where $|C|$ is the length of the collection. 
    % \sout{For any $(\eps_i, \uj_{\eps_i})$ in the collection, our algorithm ensures that as long as the perturbation magnitude $\eps \le \eps_i$, the cumulative reward $J_\eps(\tpi) \ge \uj_{\eps_i}$.}
    For all $\eps’$, let $i$ be the largest integer such that $ \eps_i \leq \eps’ < \eps_{i+1} $, then as long as the perturbation magnitude $\eps \le \eps'$, the cumulative reward $J_\eps(\tpi) \ge \uj_{\eps_i}$.
    % certified cumulative reward lower bound for $\eps’$ is $J_{\eps’} = J_{\eps_i}$.
    The algorithm is shown in \Cref{alg:adaptive-search-full} in~\Cref{append:adasearch}.
    
    \textbf{Algorithm Description.}\quad
    The method starts from the base case: when perturbation magnitude $\eps = 0$, the lower bound of cumulative reward $\uj$ is exactly the benign reward.
    The method then gradually increases the perturbation magnitude $\eps$~(later we will explain how a new $\eps$ is determined).
    Along the increase of $\eps$, the perturbation may cause the policy $\pi$ to take different actions at some time steps, thus resulting in new trajectories.
    Thanks to the local smoothing, the method leverages \Cref{thm:extend-radius} to figure out the exhaustive list of possible actions under current perturbation magnitude $\eps$, and effectively explore these new trajectories by formulating them as expanded branches of a trajectory tree.
    Once all new trajectories are explored, the method examines all leaf nodes of the tree and figures out the minimum reward among them, which is the new lower bound of cumulative reward $\uj$ under this new $\eps$.
    To mitigate the explosion of branches, \adasearch proposes several \textit{optimization} tricks.
    In the following, we will briefly introduce the \textit{trajectory exploration and expansion} and  \textit{growth of perturbation magnitude} steps, as well as the \textit{optimization} tricks. More algorithm details are deferred to~\Cref{append:adasearch}, where we also provide an analysis on the time complexity of the algorithm.

    \looseness=-1
    In \textbf{trajectory exploration and expansion}, 
    \adasearch organizes all possible trajectories in the form of search tree and progressively grows it.
    % Each node of the tree represents a state, and the depth of the node is equal to the time step of the corresponding state in the trajectory.
    % The root node (at depth $0$) represents the initial state $s_0$.
    For each node (representing a state), leveraging \Cref{thm:extend-radius}, we compute a non-decreasing sequence $\{r^k(s)\}_{k=0}^{|\gA|-1}$ representing the required perturbation radii for $\pi$ to choose each alternative action. 
    Suppose the current $\eps$ satisfies $r^i(s) \le \eps < r^{i+1}(s)$, we can grow $(i+1)$ branches from current state $s$ corresponding to the original action and $i$ alternative actions since $\eps \ge r^j(s)$ for $1\le j\le i$.
    We expand the tree branches using depth-first search~\citep{tarjan1972depth}.
    In \textbf{perturbation magnitude growth}, when all trajectories for perturbation magnitude $\eps$ are explored, we increase $\eps$ to seek certification under larger perturbations. 
    This is achieved by preserving a priority queue~\citep{van1977preserving} of the critical $\eps$'s that we will expand on.
    % Luckily, since the action space is discrete, we do not need to examine every $\eps \in \sR^+ $~(which is infeasible) but only need to examine the next $\eps$ where the chosen action in some step may change.
    % We leverage \textit{priority queue}~\cite{van1977preserving} to effectively find out such next ``critical'' $\eps$.
    Concretely, along the trajectory, at each tree node, we search for the possible actions and store actions corresponding to $\{r^k(s)\}_{k=i+1}^{|\gA|-1}$ into the priority queue, since these actions are exactly those need to be explored when $\eps$ grows.
    % After all trajectories for $\eps$ are fully explored, we pop out the head element from the priority queue as the next node to expand and next perturbation magnitude $\eps$ to grow.
    We repeat the procedure of \textit{perturbation magnitude growth} (\ie, popping out the head element from the queue) and \textit{trajectory exploration and expansion} (\ie, exhaustively expand all trajectories given the perturbation magnitude)
    until the priority queue  becomes empty or the perturbation magnitude $\eps$ reaches the predefined threshold.
    Additionally, we adopt a few \textbf{optimization} tricks commonly used in search algorithms to reduce the complexity of the algorithm, such as pruning and the memorization technique~\citep{michie1968memo}.
    More potential improvements are discussed in~\Cref{append:adasearch}.
    
    So far, we have presented all our certification methods. In \Cref{append:discuss-cert},
    we further discuss the \textit{advantages}, \textit{limitations}, and \textit{extensions} of these methods, and provide more \textit{detailed analysis} to help with understanding.

\vspace{-2mm}
\section{Experiments}
\label{sec:exp}
\vspace{-2mm}
\looseness=-1
In this section, we present evaluation for the proposed robustness certification framework \sysname.
% considering both the \textit{per-state action}  and \textit{lower bound of cumulative reward} robustness criteria.
Concretely, we apply our three certification algorithms (\staters, \glbrs, and \adasearch) to certify nine RL methods (\textbf{StdTrain}~\citep{mnih2013playing}, \textbf{GaussAug}~\citep{kos2017delving}, \textbf{AdvTrain}~\citep{behzadan2017whatever}, \textbf{SA-MDP (PGD,CVX)}~\citep{zhang2021robust}, \textbf{RadialRL}~\citep{oikarinen2020robust}, \textbf{CARRL}~\citep{everett2021certifiable}, \textbf{NoisyNet}~\citep{fortunato2017noisy}, and \textbf{GradDQN}~\citep{pattanaik2018robust}) on two high-dimensional Atari games (Pong and Freeway), one low dimensional control environment (CartPole), and an autonomous driving environment~(Highway).
As a summary, we find that 
(1) SA-MDP (CVX), SA-MDP (PGD), and RadialRL achieve high certified robustness in different environments;
(2) Large smoothing variance can help to improve certified robustness significantly on Freeway, while a more careful selection of the smoothing parameter is needed in Pong;
(3) For  methods that demonstrate high certified robustness, our certification of the cumulative reward is  tight. 
We defer the detailed descriptions of the environments to~\Cref{append:env} and the introduction to the RL methods and implementation details to~\Cref{append:impl-rl}.
More interesting results and discussions are omitted to~\Cref{append:results} and our leaderboard, including results on CartPole (\Cref{append:cartpole}) and Highway (\Cref{append:highway}).
As a \textit{foundational} work providing  robustness certification for RL, we expect more RL algorithms and RL environments will be certified under our framework in future work.

% We defer more interesting discussion to~\Cref{append:results}, including the 
% The results demonstrate the effectiveness and tightness of our algorithms. We further present several interesting findings. For instance, we show that the evaluated RL methods have consistent robustness ranking under two criteria.

% \subsection{Experimental Setup}

\vspace{-2mm}
\subsection{Evaluation of Robustness Certification for Per-state Action}
\label{sec:eval-cert-rad}
\vspace{-2mm}

\renewcommand{\thesubfigure}{\alph{subfigure}}
\newcommand{\mycaption}[1]% #1 = caption
{\refstepcounter{subfigure}\textbf{(\thesubfigure) }{\ignorespaces #1}}

\begin{figure}

\newlength{\utilheighta}
\settoheight{\utilheighta}{\includegraphics[width=.160\linewidth]{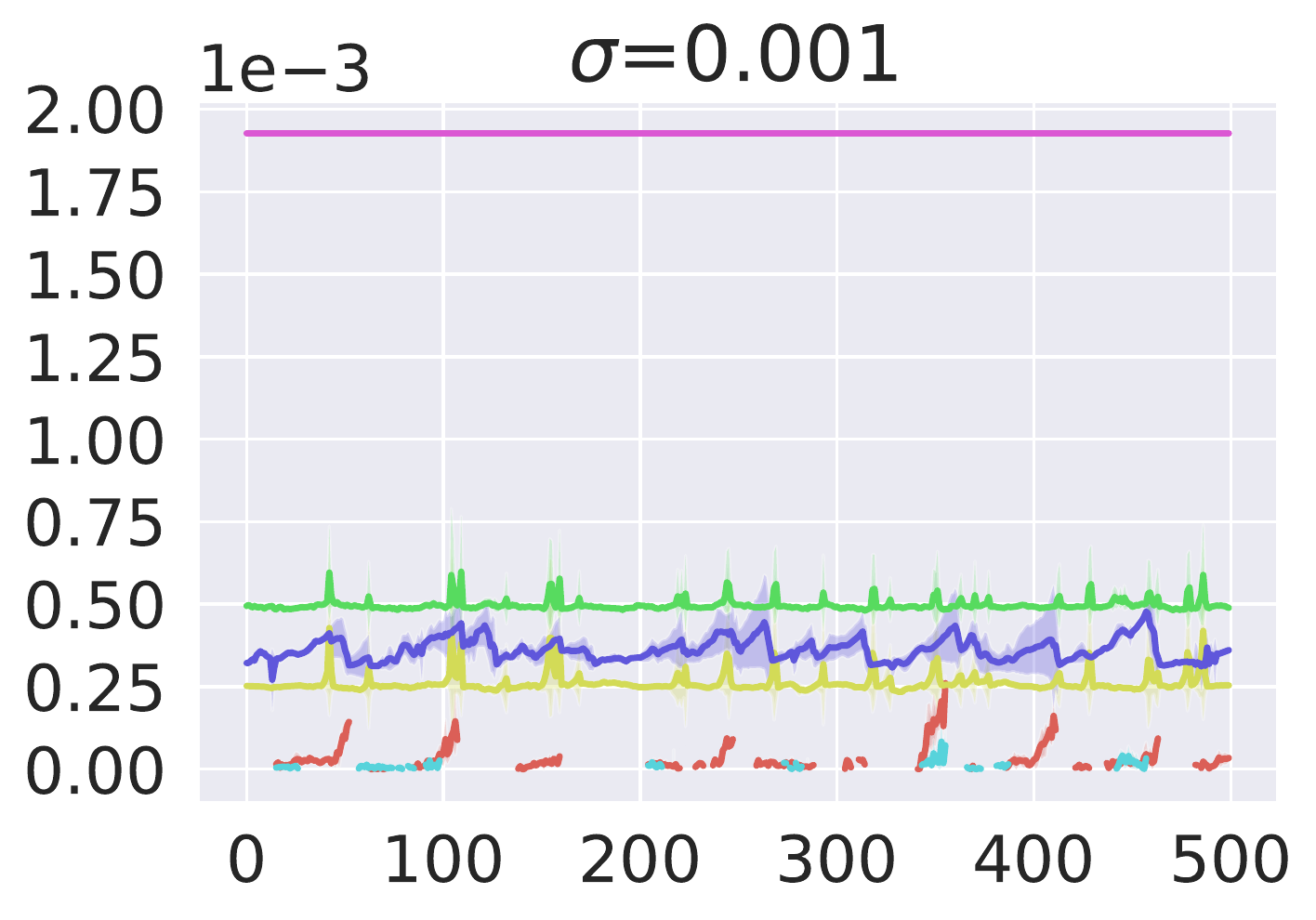}}%

\newlength{\legendheight}
\setlength{\legendheight}{0.3\utilheighta}%

\newcommand{\rowname}[1]% #1 = text
{\rotatebox{90}{\makebox[\utilheighta][c]{\tiny #1}}}

\centering

{
\renewcommand{\tabcolsep}{10pt}

\begin{subtable}[]{\linewidth}
\begin{tabular}{l}
\includegraphics[height=\legendheight]{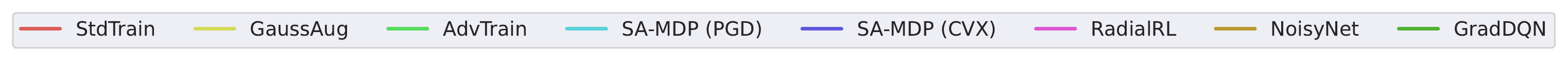}
\end{tabular}
\end{subtable}

\begin{subtable}[]{\linewidth}
\centering
% \resizebox{\linewidth}{!}{%
\begin{tabular}{@{}p{5mm}@{}c@{}c@{}c@{}c@{}c@{}c@{}}
\rowname{\makecell{Freeway\\Radius $r$}}&
\includegraphics[height=\utilheighta]{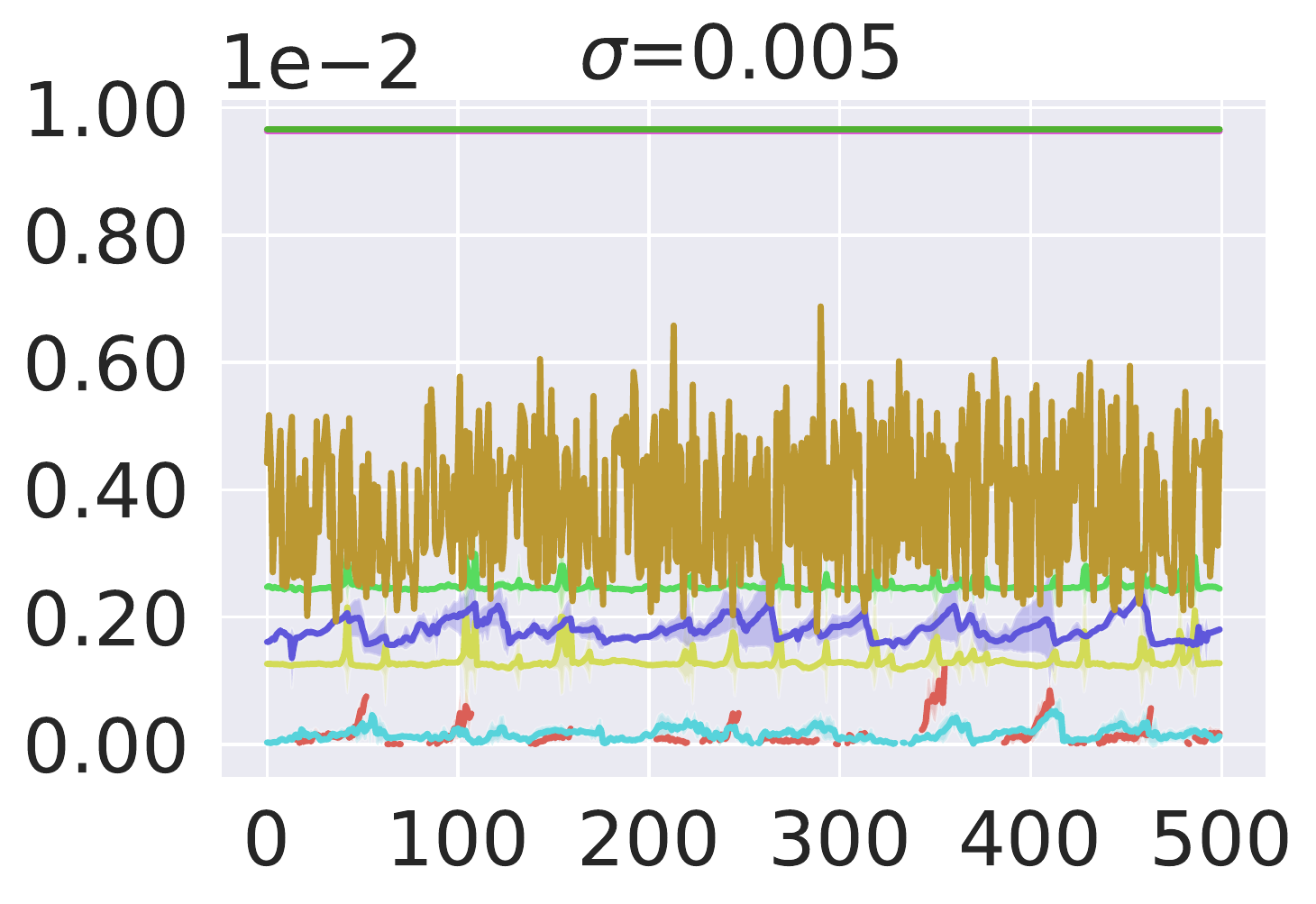}&
\includegraphics[height=\utilheighta]{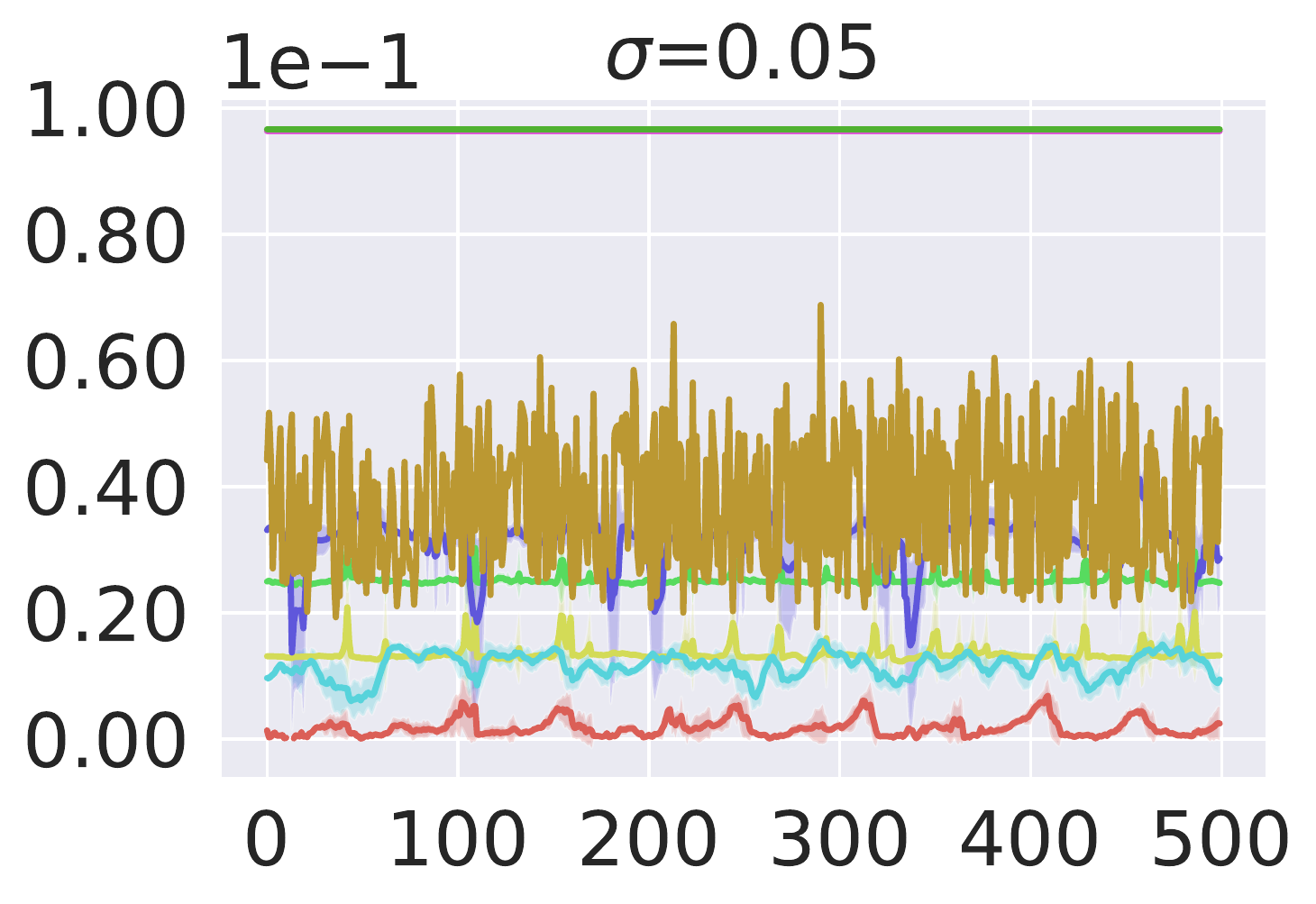}&
\includegraphics[height=\utilheighta]{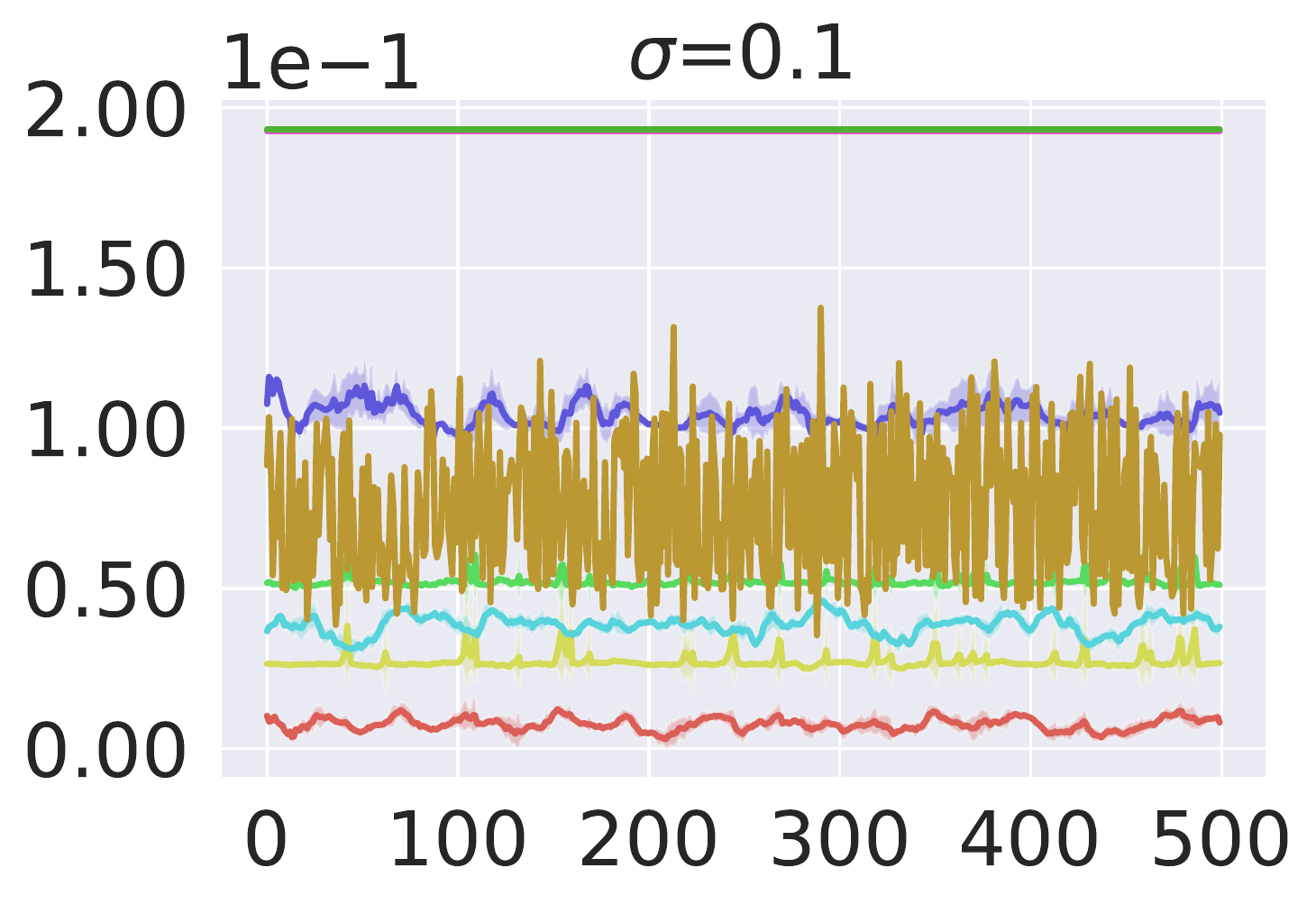}&
\includegraphics[height=\utilheighta]{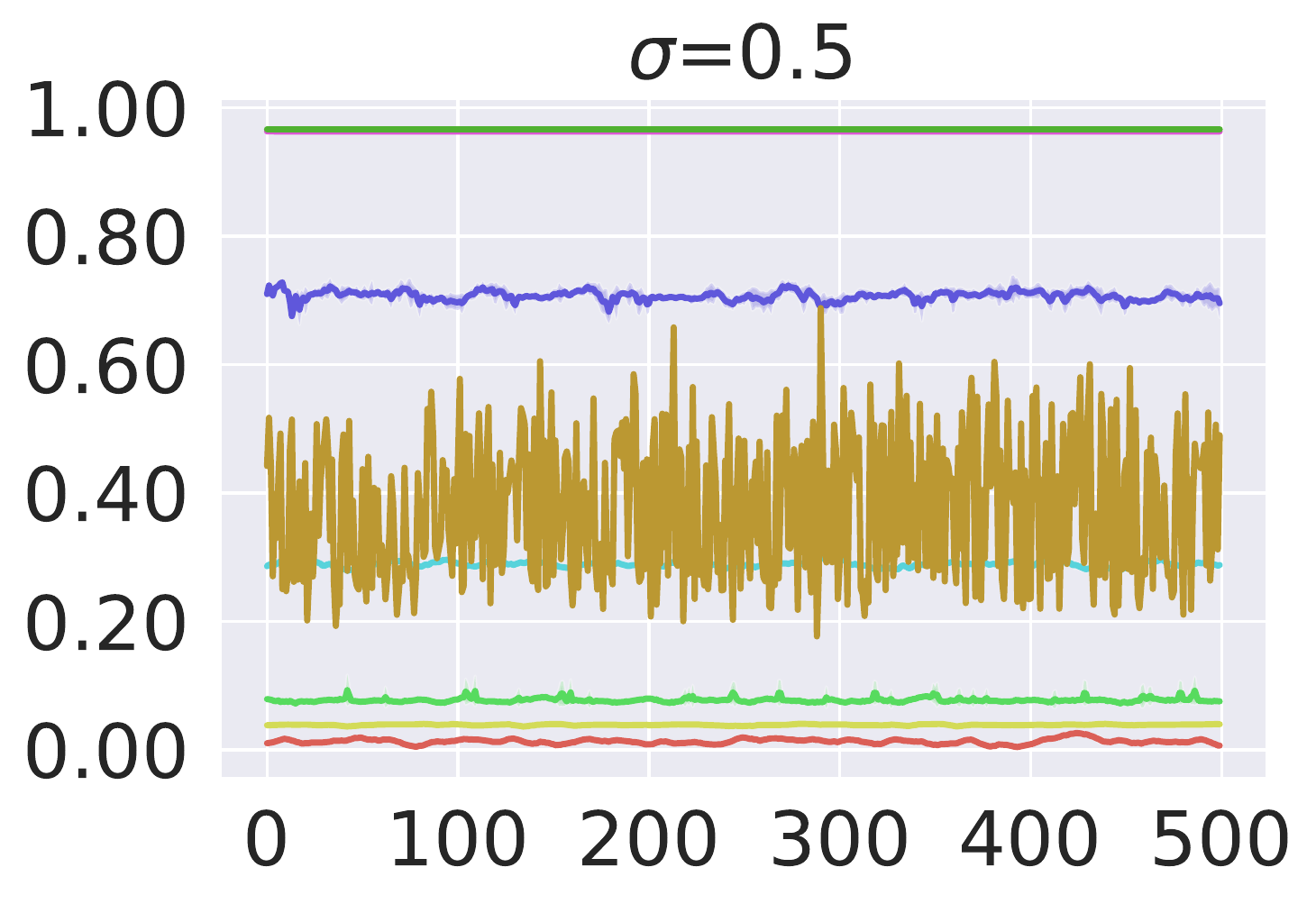}&
\includegraphics[height=\utilheighta]{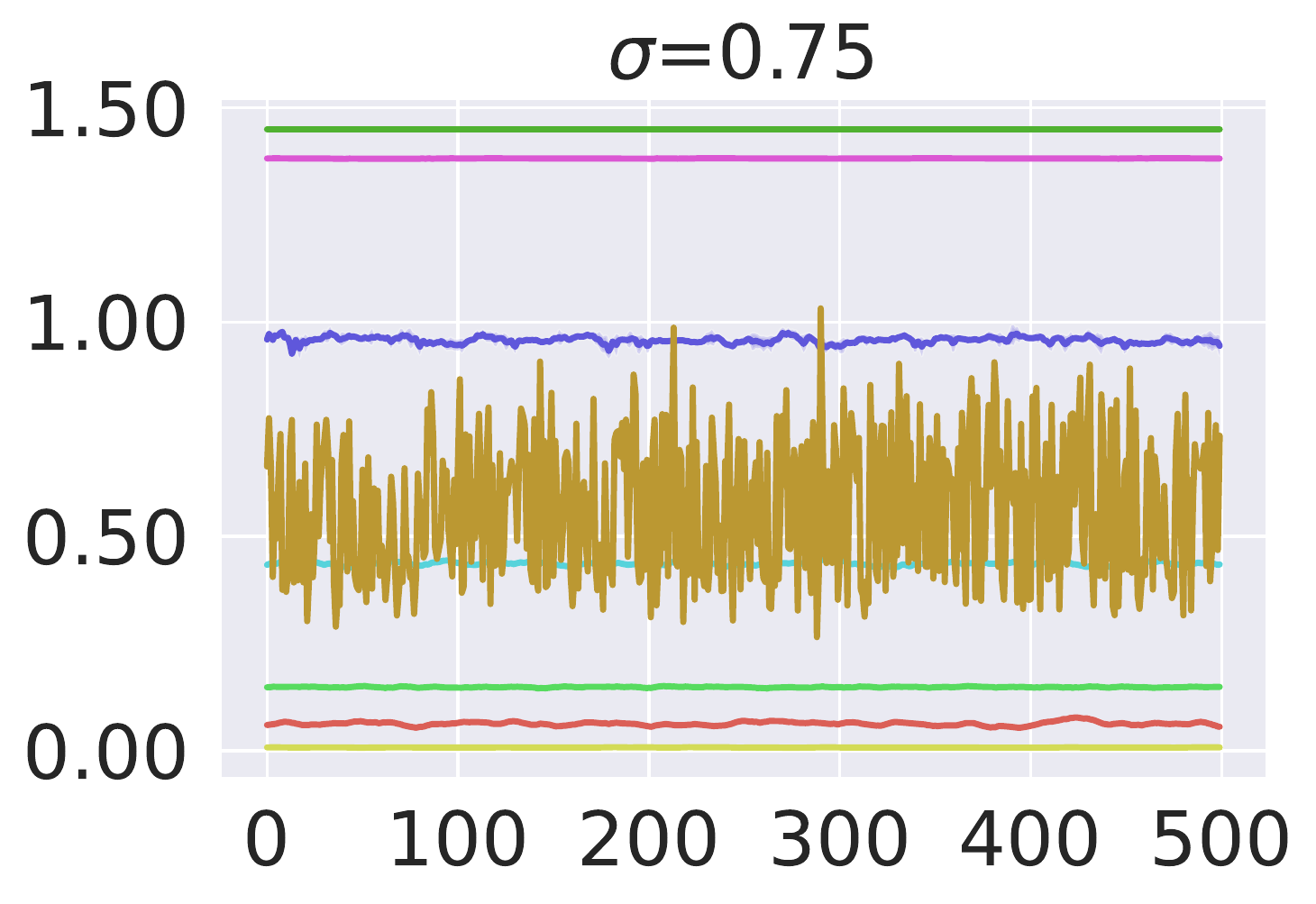}&
\includegraphics[height=\utilheighta]{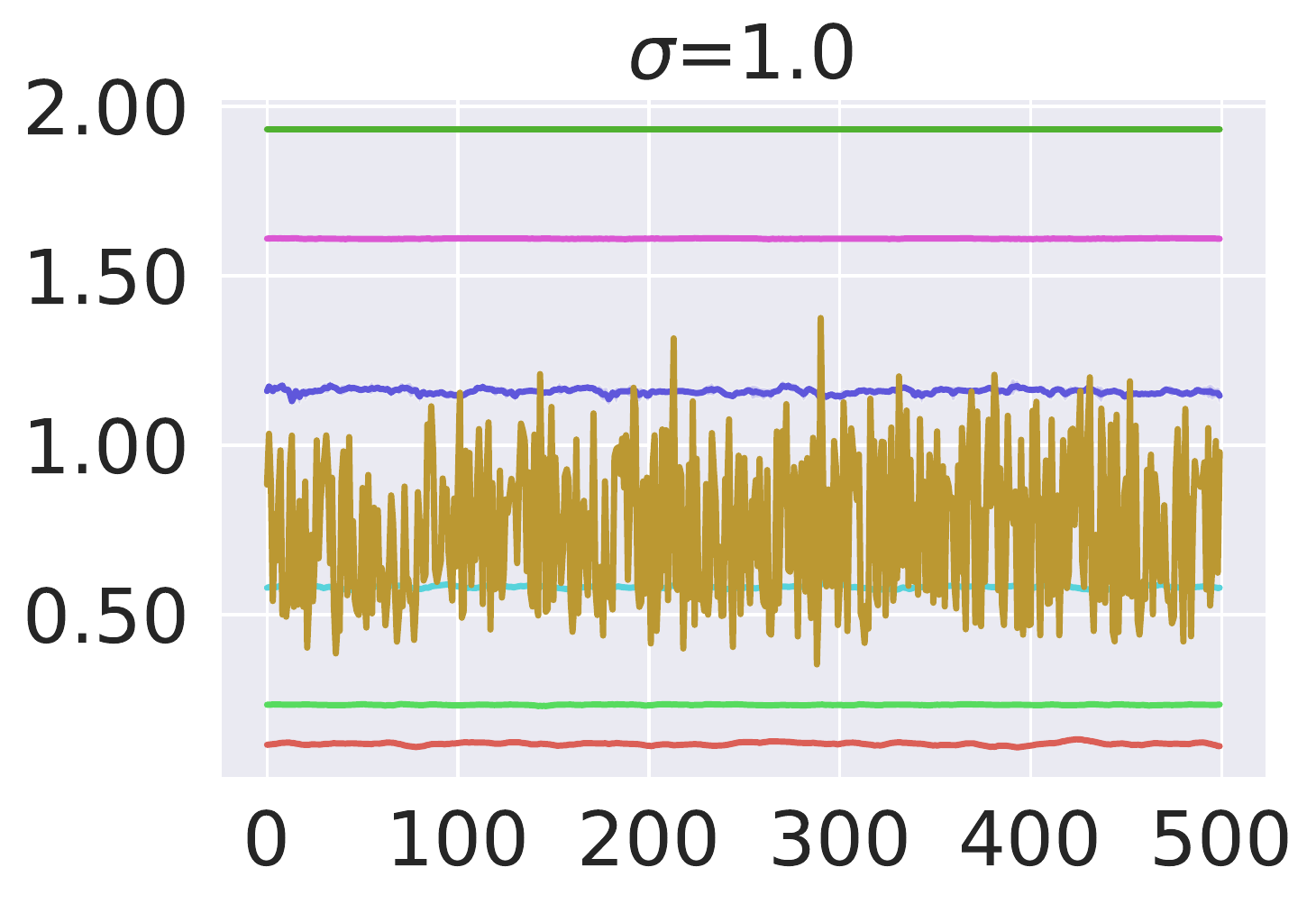}\\[-1.2ex]
\rowname{\makecell{Pong\\Radius $r$}}&
\includegraphics[height=\utilheighta]{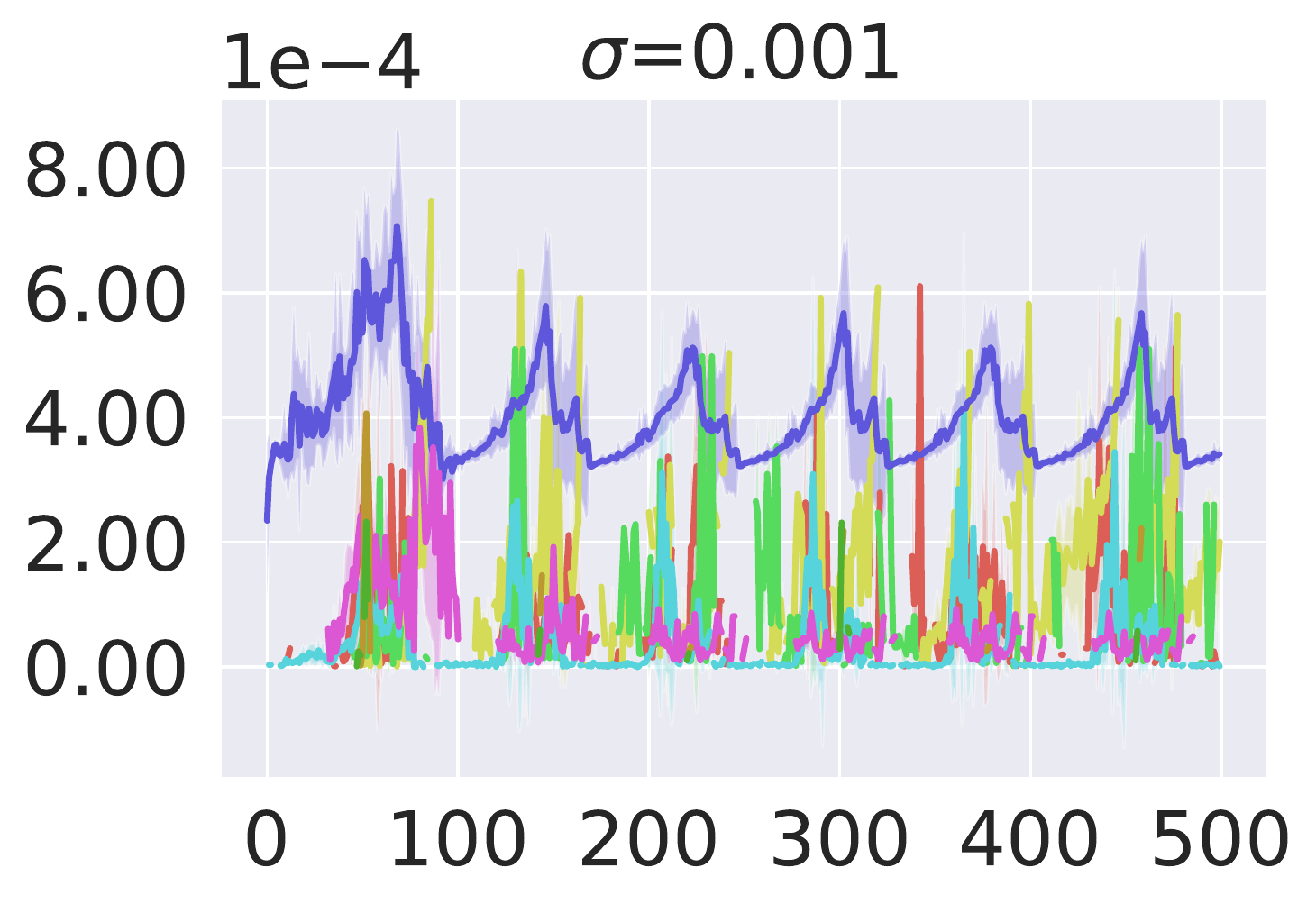}&
\includegraphics[height=\utilheighta]{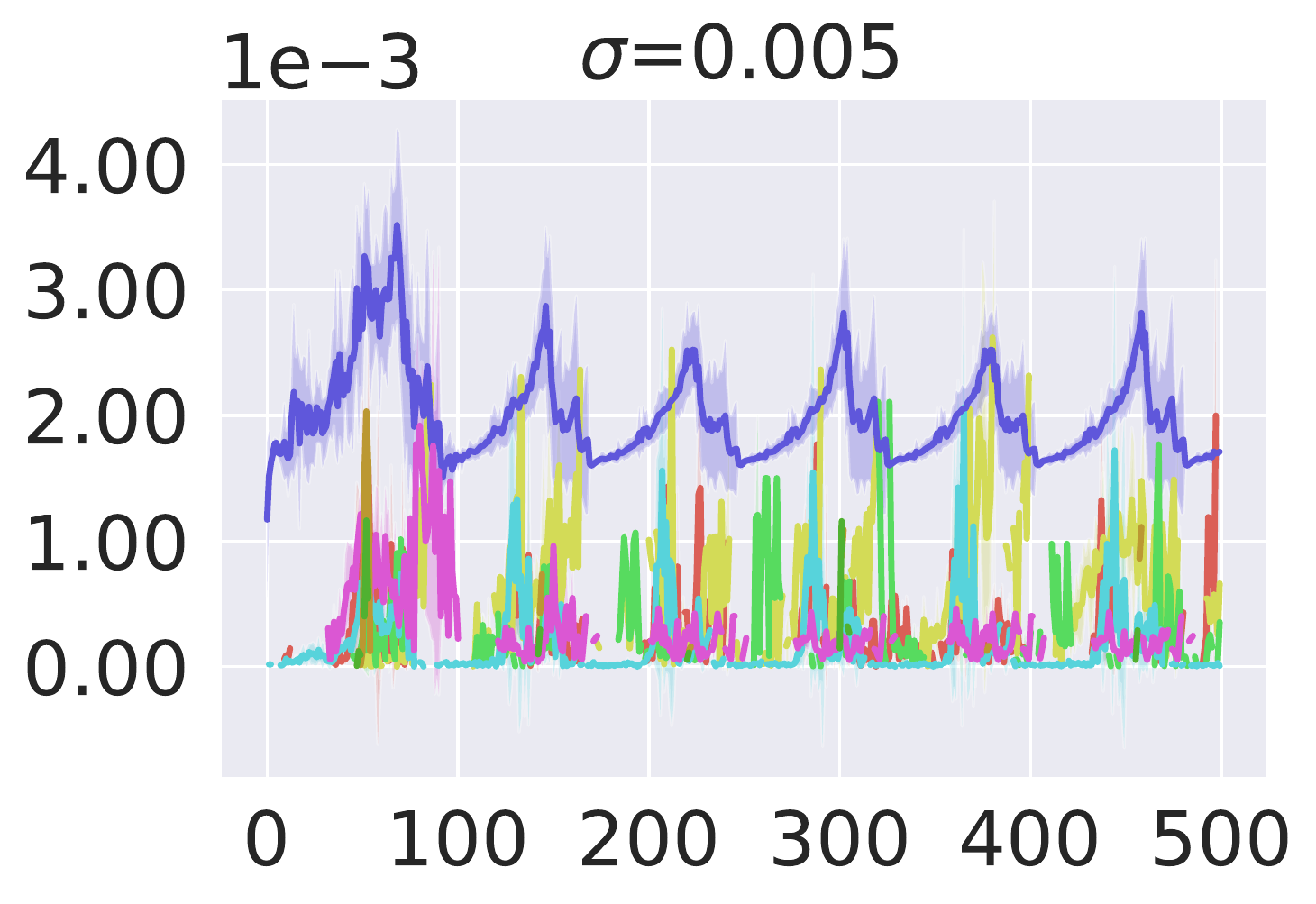}&
\includegraphics[height=\utilheighta]{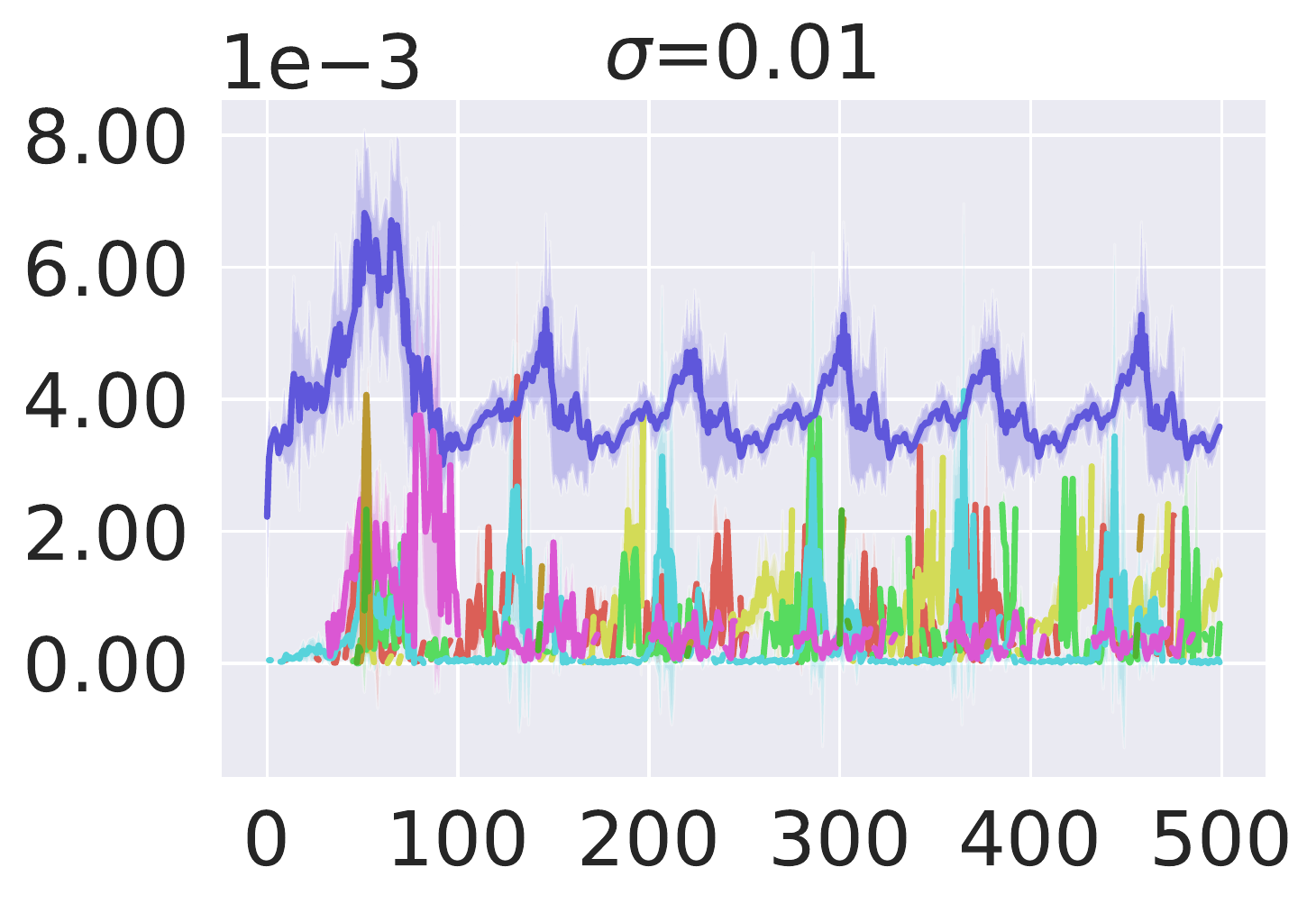}&
\includegraphics[height=\utilheighta]{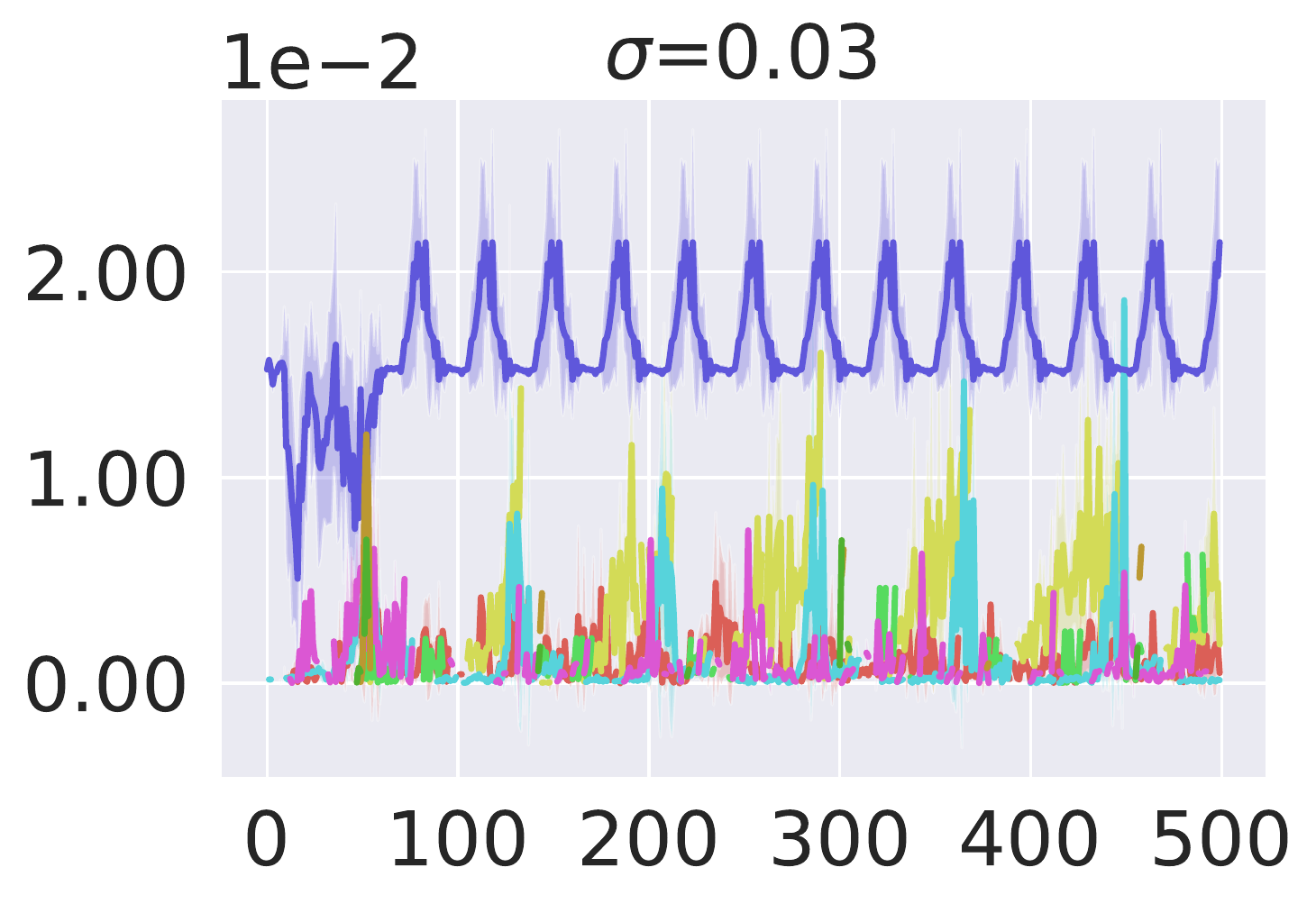}&
\includegraphics[height=\utilheighta]{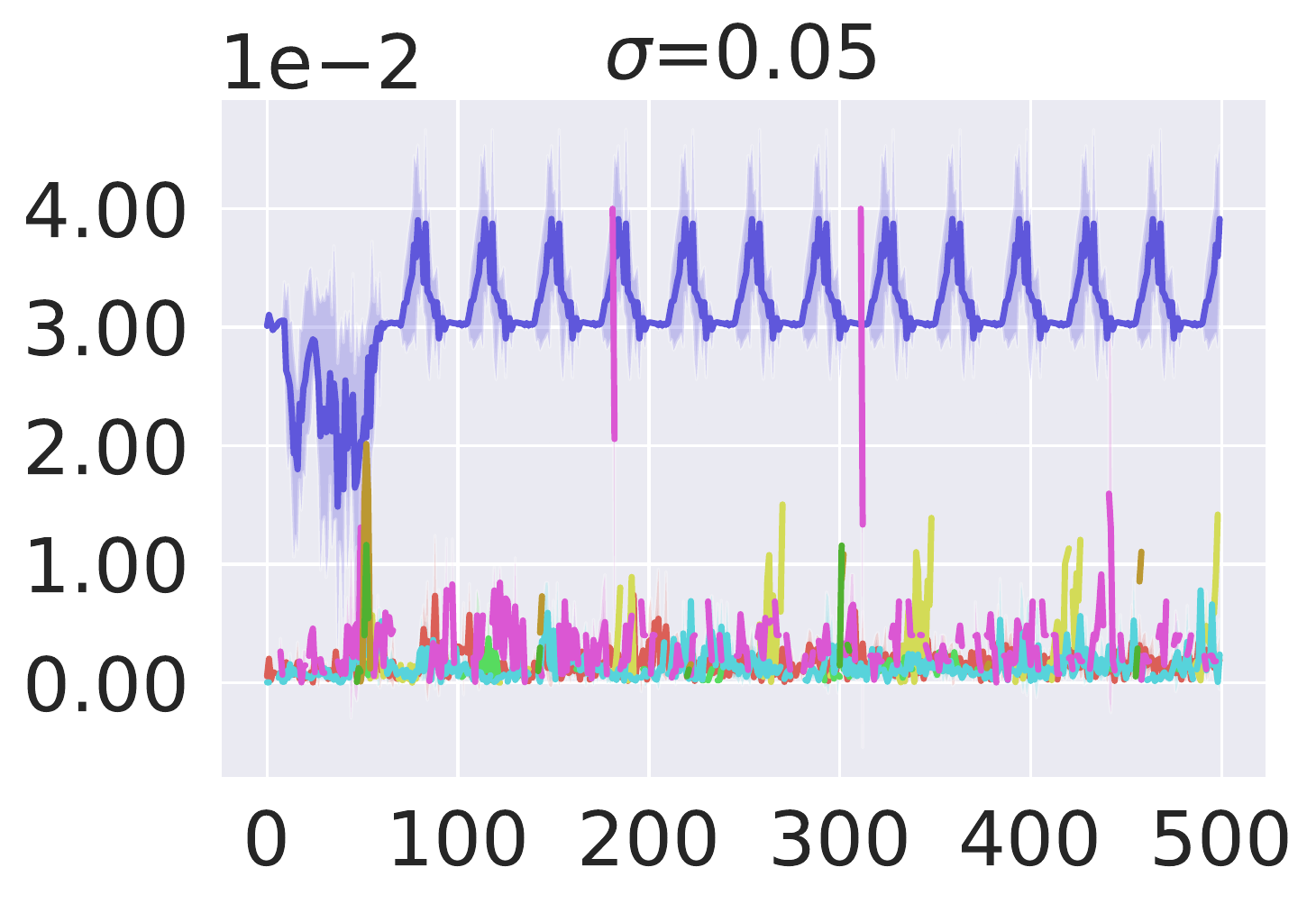}&
\includegraphics[height=\utilheighta]{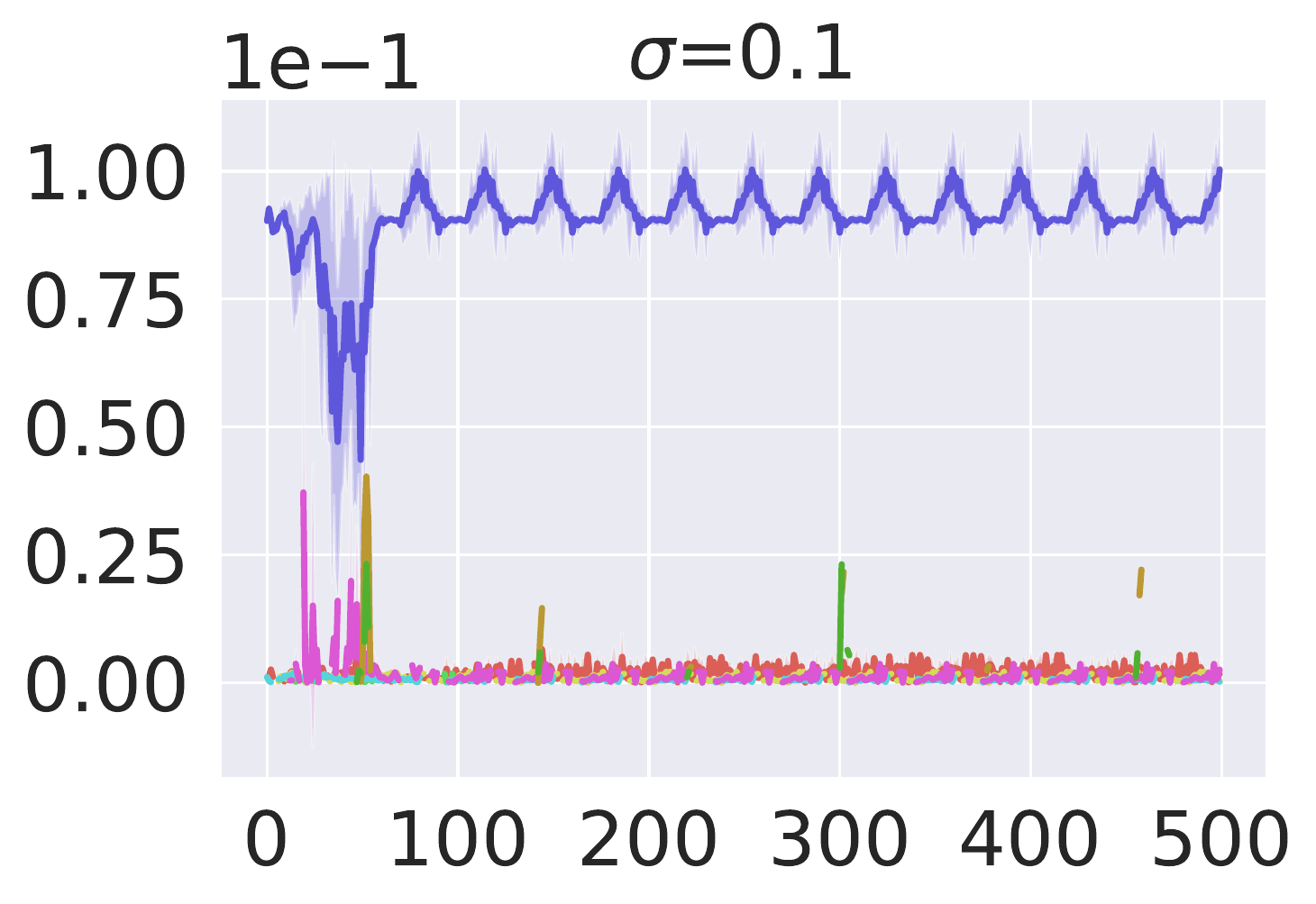}\\[-1.2ex]
% \rowname{\makecell{Highway\\Radius $r$}}&
% \includegraphics[height=\utilheighta]{figures/highway-fast_stepwise_sigma-0.005.pdf}&
% \includegraphics[height=\utilheighta]{figures/highway-fast_stepwise_sigma-0.05.pdf}&
% \includegraphics[height=\utilheighta]{figures/highway-fast_stepwise_sigma-0.1.pdf}&
% \includegraphics[height=\utilheighta]{figures/highway-fast_stepwise_sigma-0.5.pdf}&
% \includegraphics[height=\utilheighta]{figures/highway-fast_stepwise_sigma-0.75.pdf}&
% \includegraphics[height=\utilheighta]{figures/highway-fast_stepwise_sigma-1.0.pdf}\\[-1.2ex]
        & \makecell{\tiny{time step $t$}}
        & \makecell{\tiny{time step $t$}}
        & \makecell{\tiny{time step $t$}}
        & \makecell{\tiny{time step $t$}}
        & \makecell{\tiny{time step $t$}}
        & \makecell{\tiny{time step $t$}}
\end{tabular}
% }
% \caption{\small Certified radius $R_t$ along time steps}\label{tab:cert-rad}
\end{subtable}

}
\vspace{-3mm}
\caption{\small Robustness certification for \textit{per-state action} in terms of certified radius $r$ at all time steps.
Each column corresponds to a smoothing variance $\sigma$. 
% We consider six methods: StdTrain, GaussAug, AdvTrain, RegPGD, RegCVX, and RadialRL, with the last three being SOTA. 
The shaded area represents the standard deviation which is small.
RadialRL is the most certifiably robust method on Freeway, while SA-MDP (CVX) is the most robust on Pong.
}%
\label{fig:statewise}
\vspace{-2em}
\end{figure}

\newlength{\wrapheighta}
\settoheight{\wrapheighta}{\includegraphics[width=.245\textwidth]{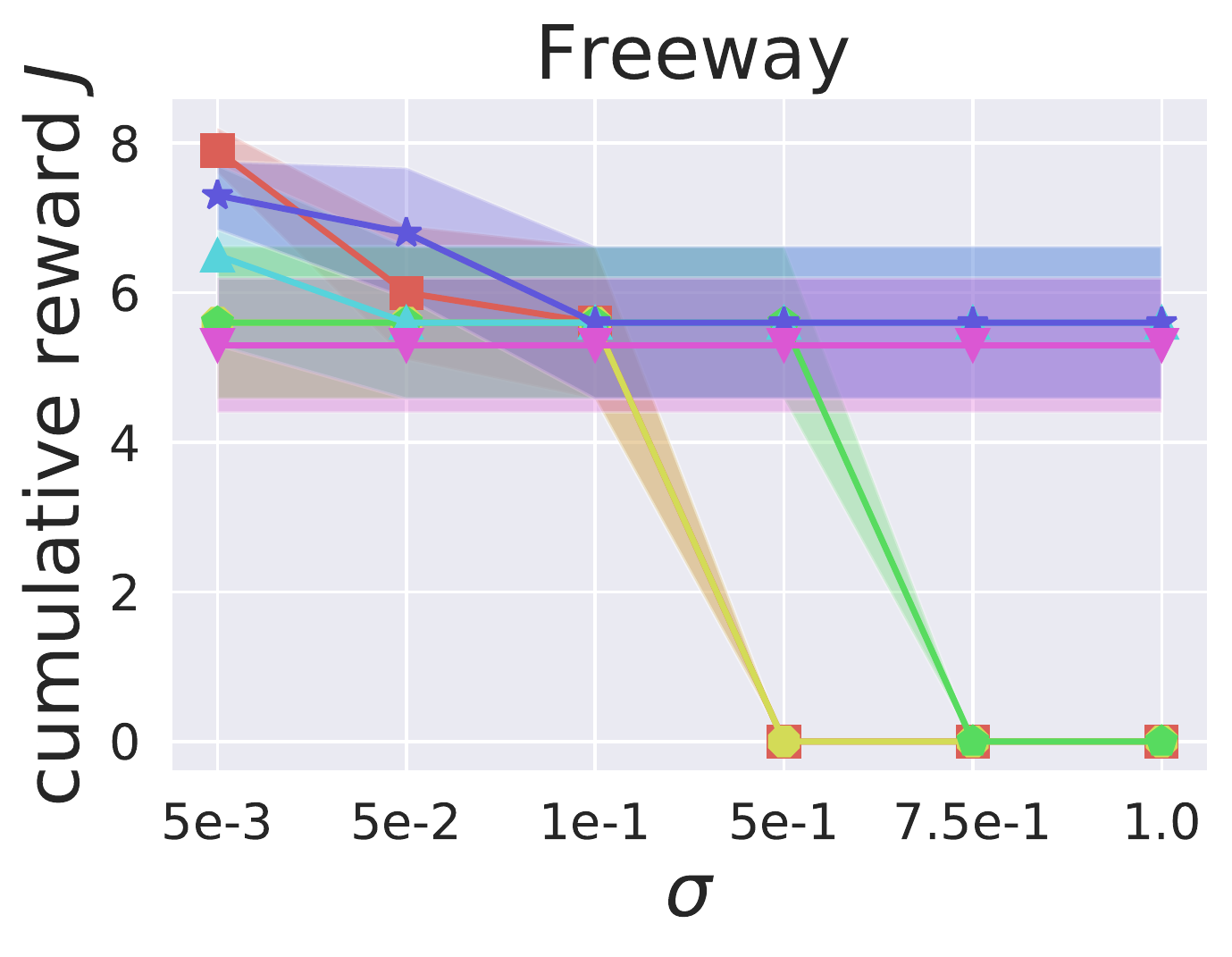}}%

\begin{wrapfigure}{R}{0.63\linewidth}
\vspace{-7mm}
\begin{center}
\includegraphics[height=\wrapheighta]{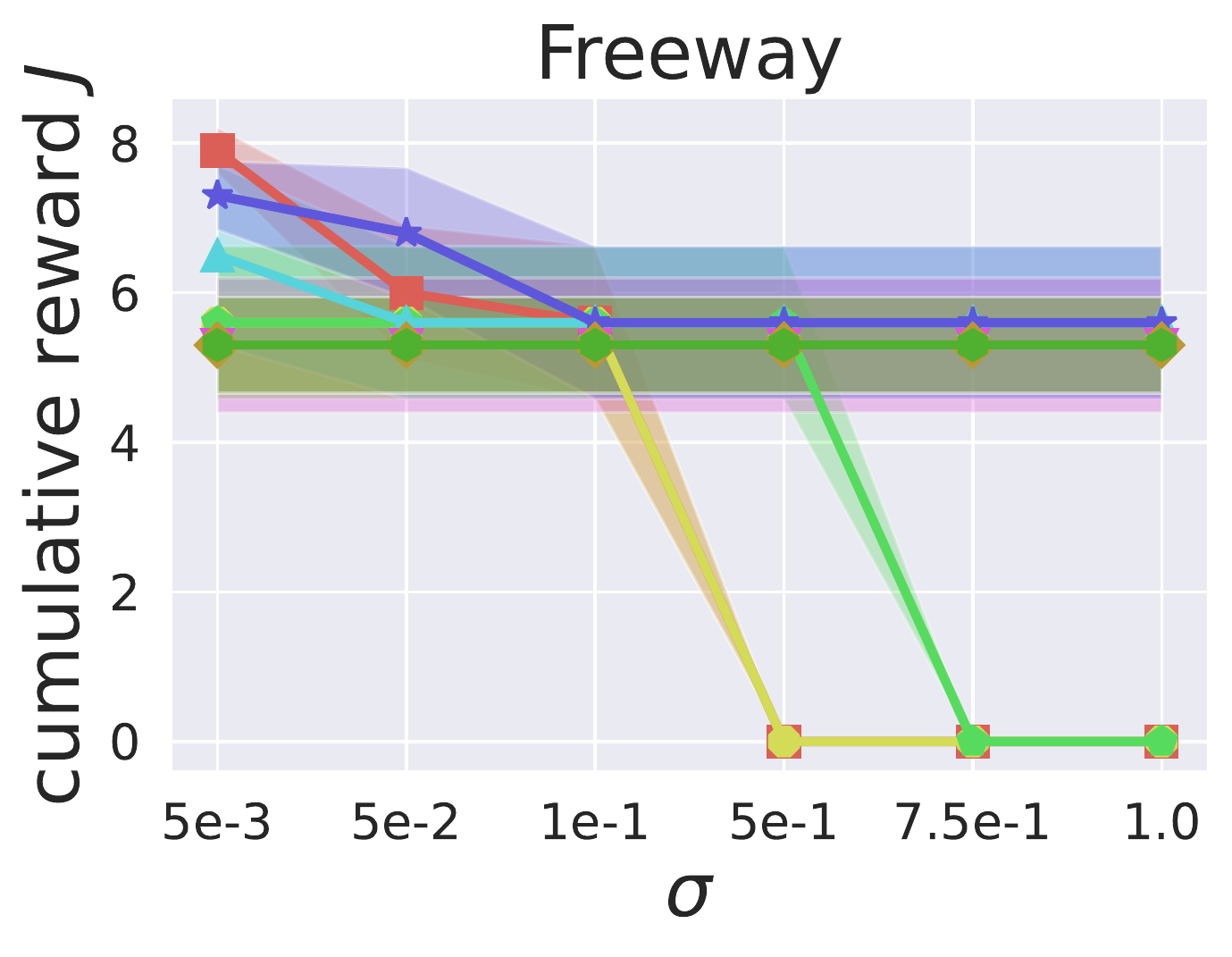}
\includegraphics[height=\wrapheighta]{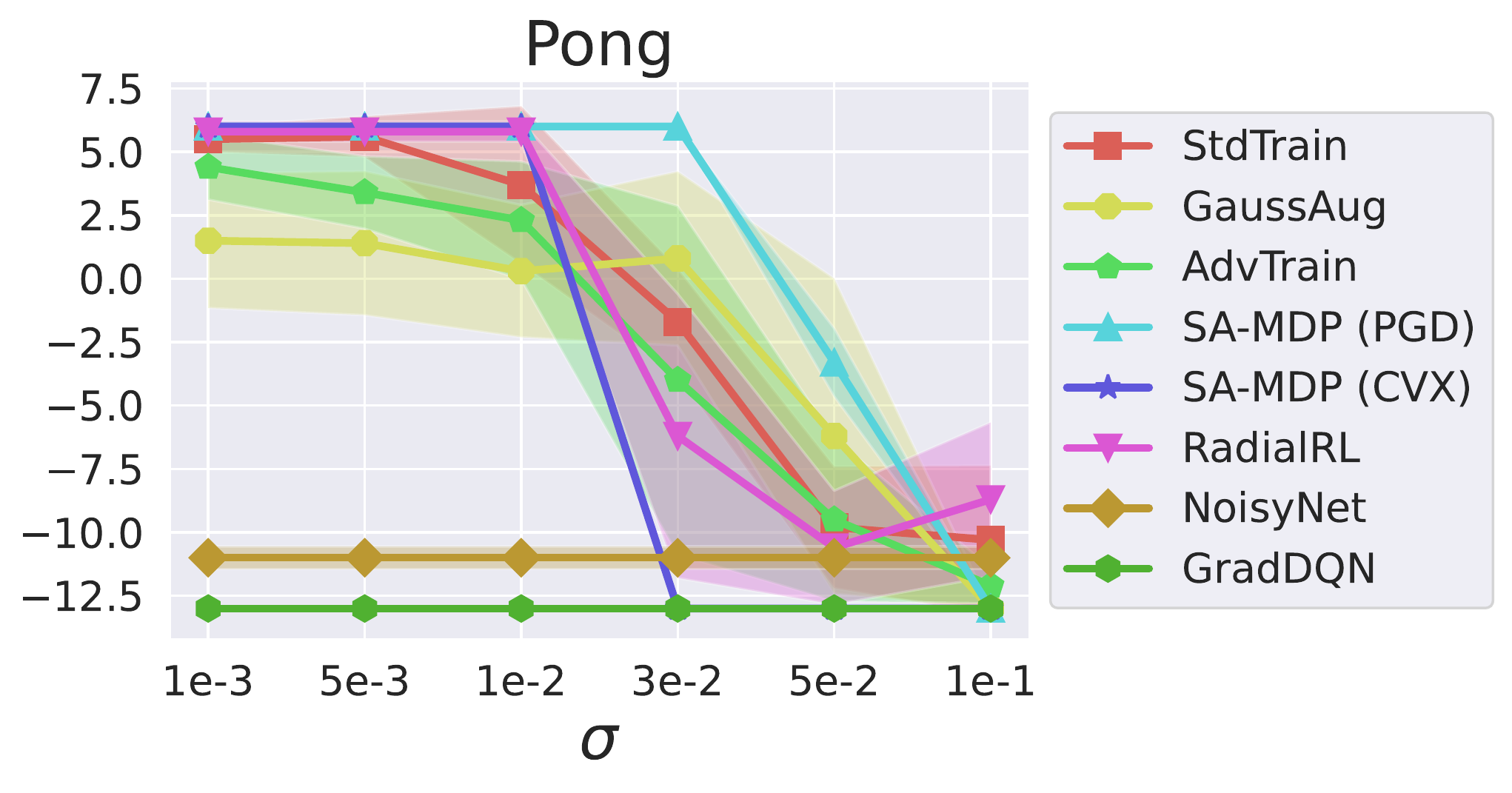}
\end{center}
\vspace{-5mm}
\caption{\small Benign performance of  locally smoothed policy $\tpi$ under different  smoothing variance $\sigma$ with clean state observations.}\label{fig:statewise-clean}
\vspace{-1em}
\end{wrapfigure}

In this subsection, we provide the robustness certification evaluation for per-state action.
% \staters introduced in~\Cref{sec:cert-rad}, and the comparison of the certified robustness for these methods.

%\vspace{-0.2mm}

\looseness=-1
\textbf{Experimental Setup and Metrics.}\quad
We evaluate the \textit{locally smoothed policy} $\tpi$ derived in (\ref{eq:smooth-q}).
We follow~\Cref{thm:radius} and report the \textit{certified radius} $r_t$ at each step $t$.
We additionally compute  \textit{certified ratio $\eta$ at radius $r$}, representing the percentage of time steps that can be certified, formally denoted as $\eta_r = \nicefrac{1}{H}\sum_{t=1}^{H}\mathds{1}_{\{r_t\geq r\}}$.
More details are omitted to~\Cref{append:exp-setup}.

\textbf{Evaluation Results of \staters.}\quad
We present the certification comparison in~\Cref{fig:statewise} and the benign  performance under different   smoothing variances in~\Cref{fig:statewise-clean}; we omit details of certified ratio to~\Cref{append:cert-ratio}.

% \textit{Comparison of  methods.}\quad
% We first note that there are some dissimilarities between the certification results displayed on Freeway and Pong. 
% Thus, we will separately consider the two games.
%\vspace{1mm}
\looseness=-1 
On \uit{Freeway}, we evaluate with a large range of smoothing parameter $\sigma$ up to $1.0$, since Freeway can tolerate large noise as shown in~\Cref{fig:statewise-clean}. 
Results under even larger $\sigma$ are deferred to~\Cref{append:freeway-large-sigma},
showing that the certified robustness can be further improved for the empirically robust methods.
% where we will show that the robust methods (SA-MDP (PGD), SA-MDP (CVX), and RadialRL) can simultaneously retain good empirical performance and achieve high certified radius under quite large $\sigma$.
From~\Cref{fig:statewise}, we see that \textit{RadialRL consistently achieves the highest certified radius across all $\sigma$'s.} This is because RadialRL explicitly optimizes over the worst-case perturbations.
% SA-MDP (CVX) and SA-MDP (PGD) are robust among the rest methods.
% another two robust methods whose certified radius increase rapidly as $\sigma$ grows.
% By nature, the two methods are trained to encourage consistent action selection under perturbations.
% In contrast, 
StdTrain, GaussAug, and AdvTrain are not as robust, and increasing $\sigma$ will not make a difference.
In \underline{\textit{Pong}}, SA-MDP (CVX) is the most certifiably robust, which may be due to the hardness of optimizing the worst-case perturbation in Pong.
More interesting, all  methods present similar periodic patterns for the certified radius on different states, which would inspire further robust training methods to take the ``confident state'' (\eg, when the ball is flying towards the paddle) into account. 
We illustrate the frames with varying certified radius in~\Cref{append:pong-periodic}.
% tend to periodically achieve a relatively high radius, indicating the existence of some more crucial time steps where the networks will give more confident output, \eg, when the ball is flying towards the agent's paddle.
Overall, the certified robustness  of these methods largely matches  empirical observations~\citep{behzadan2017whatever,oikarinen2020robust,zhang2021robust}.

% \vspace{-3mm}
% \input{fig_desc/fig_reward}

\vspace{-2mm}
\subsection{Evaluation of Robustness Certification for Cumulative Reward}
\label{sec:eval-cert-cum}
\vspace{-2mm}

Here we will discuss the evaluation  for the robustness certification regarding cumulative reward in~\Cref{sec:cert-cum}. 
We show the evaluation results for both \glbrs and \adasearch.
% , each followed by interpretations regarding the comparison of methods.
% from \textit{three} perspectives:
% 1) the robustness comparison of methods;
% 2) the impact of smoothing variance;
% 3) the tightness of our certification.
% We conclude the subsection with a comparison of different certification criteria.

\vspace{-1mm}
\looseness=-1 \textbf{Evaluation Setup and Metrics.}\quad
We evaluate the \textit{\sigrpi} $\pi'$ derived in~\Cref{def:sig-rand} for \glbrs and the \textit{locally smoothed policy} $\tpi$ derived in (\ref{eq:smooth-q}) for \adasearch.
We compute the \ebound $\uje$,  \pbound $\ujp$, and  absolute lower bound $\uj$ following~\Cref{thm:exp-bound},~\Cref{thm:perc-bound}, and~\Cref{sec:tree-search-impl}.
% We sample $m=10,000$ \sigrtjs each of length $H=500$
% and experiment with the same set of smoothing parameters as in~\Cref{sec:eval-cert-rad}.
% We take $\alpha=0.05$ as the confidence level when applying Hoeffding's inequality.
To validate the tightness of the bounds, we additionally perform empirical attacks. 
More rationales see~\Cref{append:rationale}.
The detailed parameters are omitted to~\Cref{append:exp-setup}.
%The detailed parameters for the certification and PGD attacks are omitted to~\Cref{append:exp-setup}.
% Concretely, we carry out a $10$-step PGD attack with $\ell_2$ radius within the perturbation bound at all time steps during testing to evaluate the tightness of our certification.
% the certifiable range given by the percentile smoothing method at all time steps during testing.
% and compare the empirical mean and empirical median with our certified lower bounds.

% \input{fig_desc/fig_reward}

\vspace{-1mm}
\textbf{Evaluation of \glbrs.}\quad
We present the certification of reward in~\Cref{fig:all-bounds} and mainly focus on analyzing \pbound.
% since the expectation bounds $\uje$ are too loose to be useful due to the reasons  in~\Cref{sec:cert-glb}.
% Further, all methods are almost indistinguishable under this measure.
We present the bound w.r.t. different attack magnitude $\eps$ under different smoothing parameter $\sigma$.
For each $\sigma$, there exists an upper bound of  $\eps$ that can be certified, which is positively correlated with $\sigma$.
Details see~\Cref{append:glbrs}.
% For clear comparison, we adopt the same range to plot the \ebound.
% \textit{Comparison of methods.}\quad
% We separately examine the results of Freeway and Pong.
% For \underline{\textit{Freeway}}, as we have  concluded from~\Cref{sec:eval-cert-rad}, larger $\sigma$'s are more favorable to robust methods such as SA-MDP (CVX), SA-MDP (PGD), and RadialRL.
% From the row of $\ujp$, we see that as $\sigma$ grows, the advantage of these robust methods over other methods become more apparent.
% For \textit{Pong}, the conclusions are not entirely the same.
% The more favorable range of $\sigma$ for the robust methods (\eg, SA-MDP (CVX) and SA-MDP (PGD)) on Pong is smaller compared with Freeway. 
% Within this proper range, they outperform the non-robust methods; while on larger $\sigma$, the not-so-robust methods (\eg, StdTrain, GaussAug, \fanm{and AdvTrain? [still running]}) instead present slightly better results, though almost all the bounds are quite low in this range.
We observe similar conclusion as in the per-state action certification that RadialRL is most certifiably robust for Freeway, while SA-MDP (CVX,PGD) are most robust for Pong although the highest robustness is achieved at different $\sigma$. 
% SA-MDP (PGD) achieves the highest robustness when $\sigma=0.05$ and SA-MDP (CVX)  when $\sigma=0.03$. 

\looseness=-1
\vspace{-1mm}
\textbf{Evaluation of \adasearch.}\quad
% \Cref{fig:all-bounds} show the certification results, with all the lower bounds $\uj$ converging regarding attack magnitude $\eps$.
In~\Cref{fig:all-bounds}, 
% Our results contain the bounds for all attack magnitude $\eps \in \sR^+$. 
% Specifically, we plot the range of attack magnitude $\eps$ we search for in~\Cref{alg:adaptive-search}. 
% For larger $\eps$, the lower bound $\uj$ converges.
% will remain the same as the rightmost \bo{change} end of the curve, since the rightmost \bo{change} end already represents the largest possible radius for all actions at all time steps of all possible trajectories. \bo{this last sentence should be updated using the word convergence}
% \textit{Comparison of methods.}\quad
% We similarly note that the robustness of SA-MDP (CVX), SA-MDP (PGD), and RadialRL surpasses StdTrain, GaussAug, and AdvTrain.
we note that in \underline{\textit{Freeway}}, RadialRL is the most robust, followed by SA-MDP (CVX), then SA-MDP (PGD), AdvTrain, and GaussAug.
As for \uit{Pong}, SA-MDP (CVX) outperforms RadialRL and SA-MDP (PGD). 
Remaining methods are not clearly ranked.
% The remaining methods are not clearly ranked under different perturbation magnitudes.

% For \underline{\textit{Freeway}} under the favorable large $\sigma$'s, RadialRL, SA-MDP (CVX), and SA-MDP (PGD) outperform all other methods by a large margin, capable of certifying practical lower bounds for a quite large range of $\eps$.
% On Pong, similar as the conclusion drawn from~\Cref{sec:eval-cert-rad}, only SA-MDP (CVX) can achieve non-trivial results at a small range of $\sigma$. For $\eps>4e-2$, we can certify no non-zero lower bounds for all methods.

\vspace{-2mm}
\subsection{Discussion on Evaluation Results}
\label{sec:discuss-eval}
\vspace{-2mm}

\textbf{Impact of Smoothing Parameter $\sigma$.}\quad
We draw  similar conclusions regarding the impact of smoothing variance from the results given by \staters and \glbrs.
% We separately consider different methods on the two games.
In \uit{Freeway}, 
as $\sigma$ increases, the robustness of StdTrain, GaussAug, and AdvTrain barely increases, while that of SA-MDP (PGD), SA-MDP (CVX), and RadialRL steadily increases. 
In \uit{Pong}, a $\sigma$ in the range $0.01$-$0.03$ shall be suitable for almost all methods. More explanations from the perspective of benign performance and certified results will be provided in~\Cref{append:discuss-sigma}.
For \adasearch, there are a few different conclusions.
On \uit{Freeway}, it is clear that under larger attack magnitude $\eps$, larger smoothing parameter $\sigma$ always secures higher lower bound $\uj$.
% On the other hand, \textit{within certain small $\eps$}, small smoothing variances may be able to lead to a higher lower bound than large smoothing variances, such as $\eps= 1e-3$ on Freeway. We leave the concrete discussion to~\Cref{append:discuss-sigma}.
In \uit{Pong}, \adasearch can only achieve non-zero certification with small $\sigma$ under a  small range of attack magnitude,
implying the difficulty of establishing non-trivial certification for these methods on Pong.
We defer details on the selection of $\sigma$ to~\Cref{append:discuss-sigma}.
\vspace{-1mm}

{
\renewcommand{\thesubfigure}{\alph{subfigure}}
% {\refstepcounter{subfigure}\textbf{(\thesubfigure) }{\ignorespaces #1}}

\begin{figure}
\vspace{-2em}

\newlength{\utilheightc}
\settoheight{\utilheightc}{\includegraphics[width=.160\linewidth]{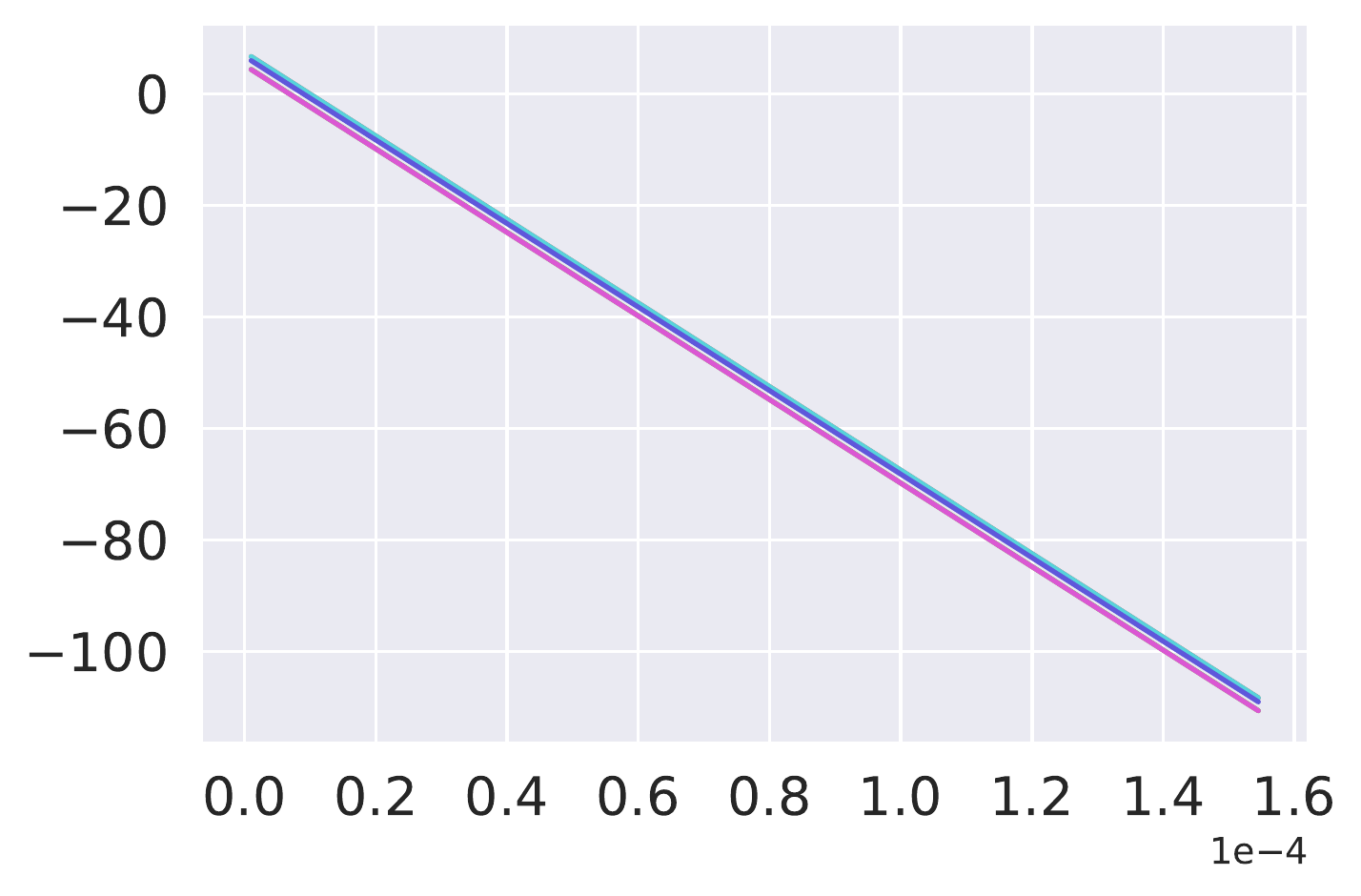}}%

\newlength{\utilheightd}
\settoheight{\utilheightd}{\includegraphics[width=.165\linewidth]{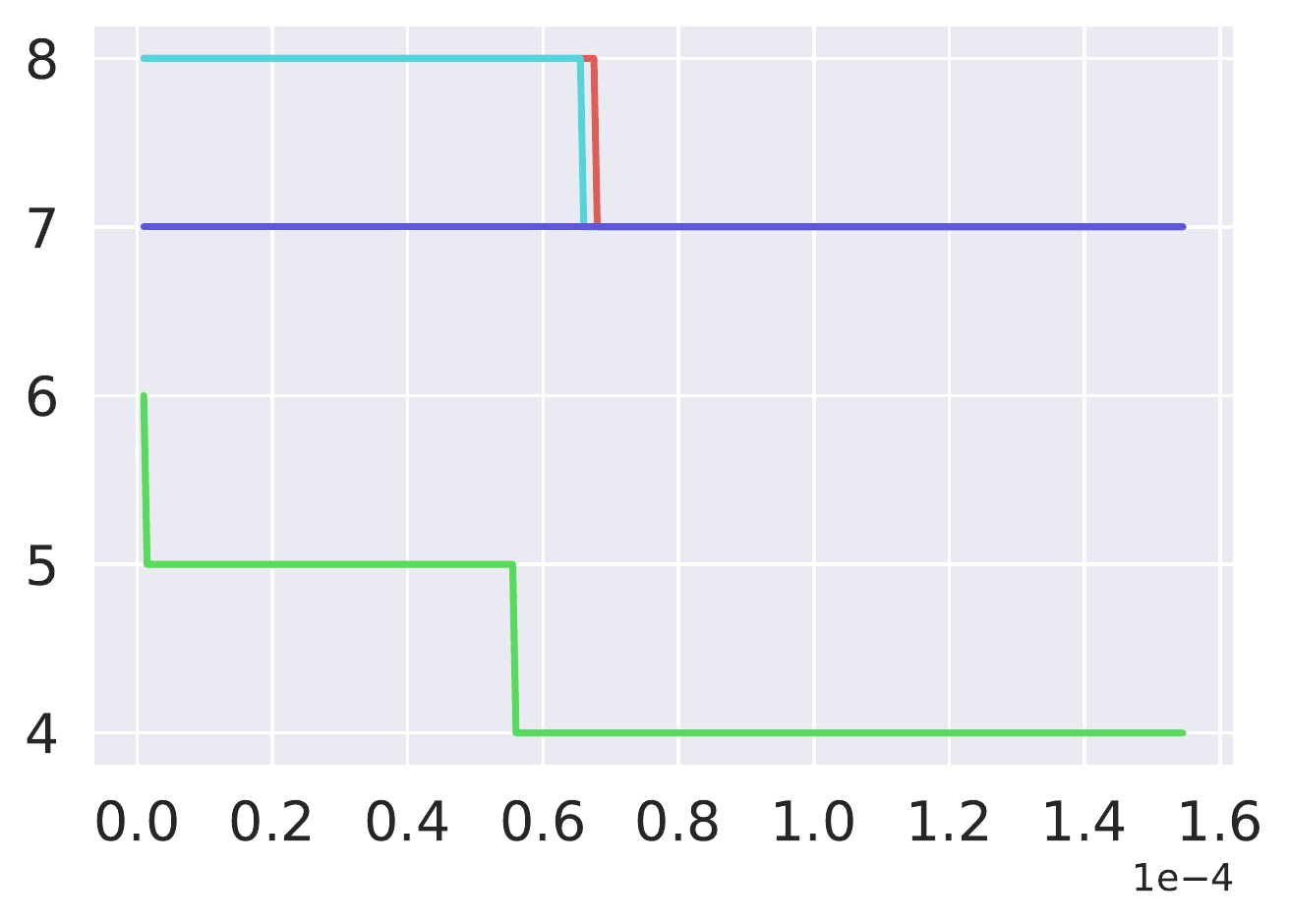}}%

\newlength{\utilheightaa}
\settoheight{\utilheightaa}{\includegraphics[width=.162\linewidth]{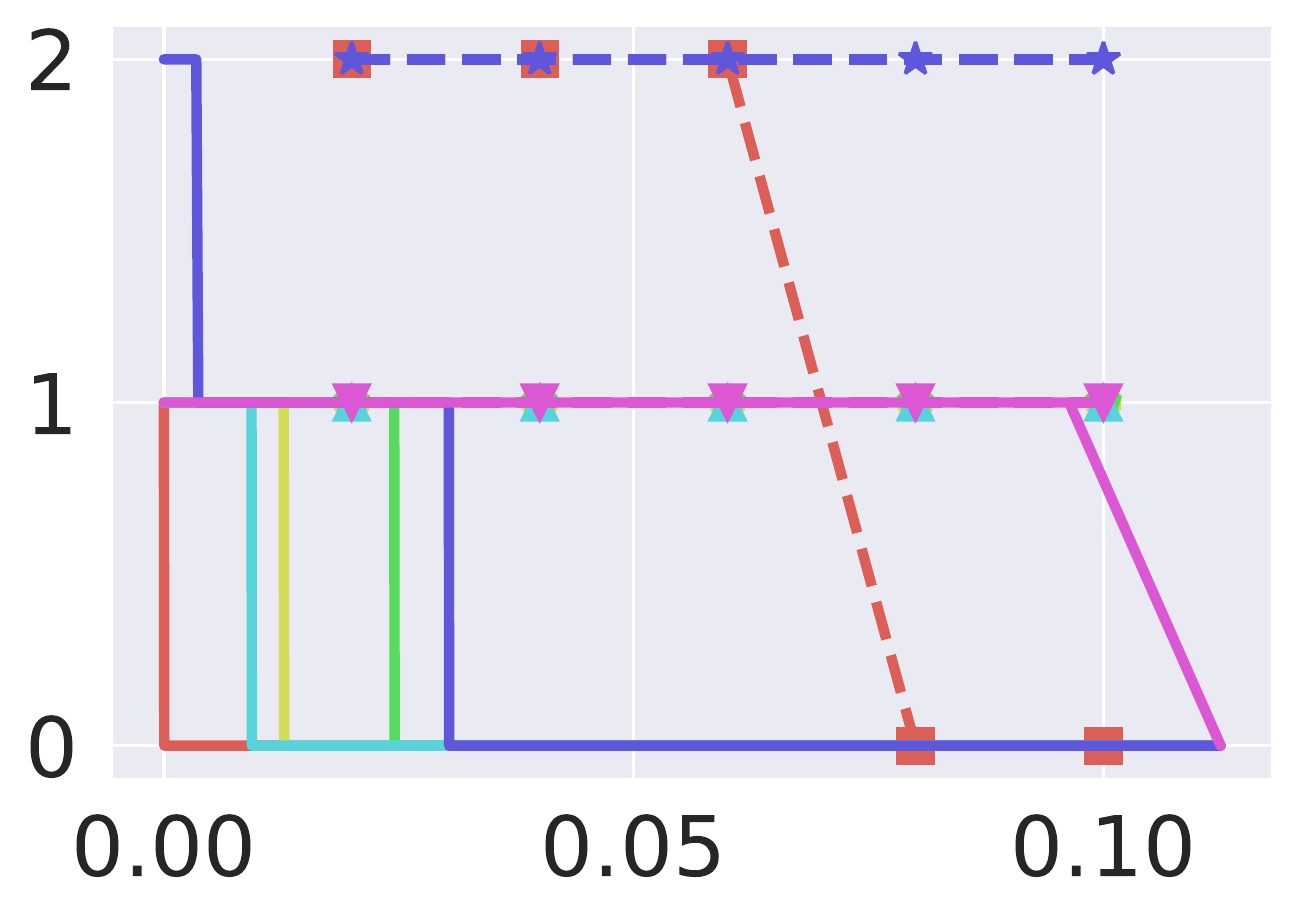}}%

\newlength{\utilheightcp}
\settoheight{\utilheightcp}{\includegraphics[width=.160\linewidth]{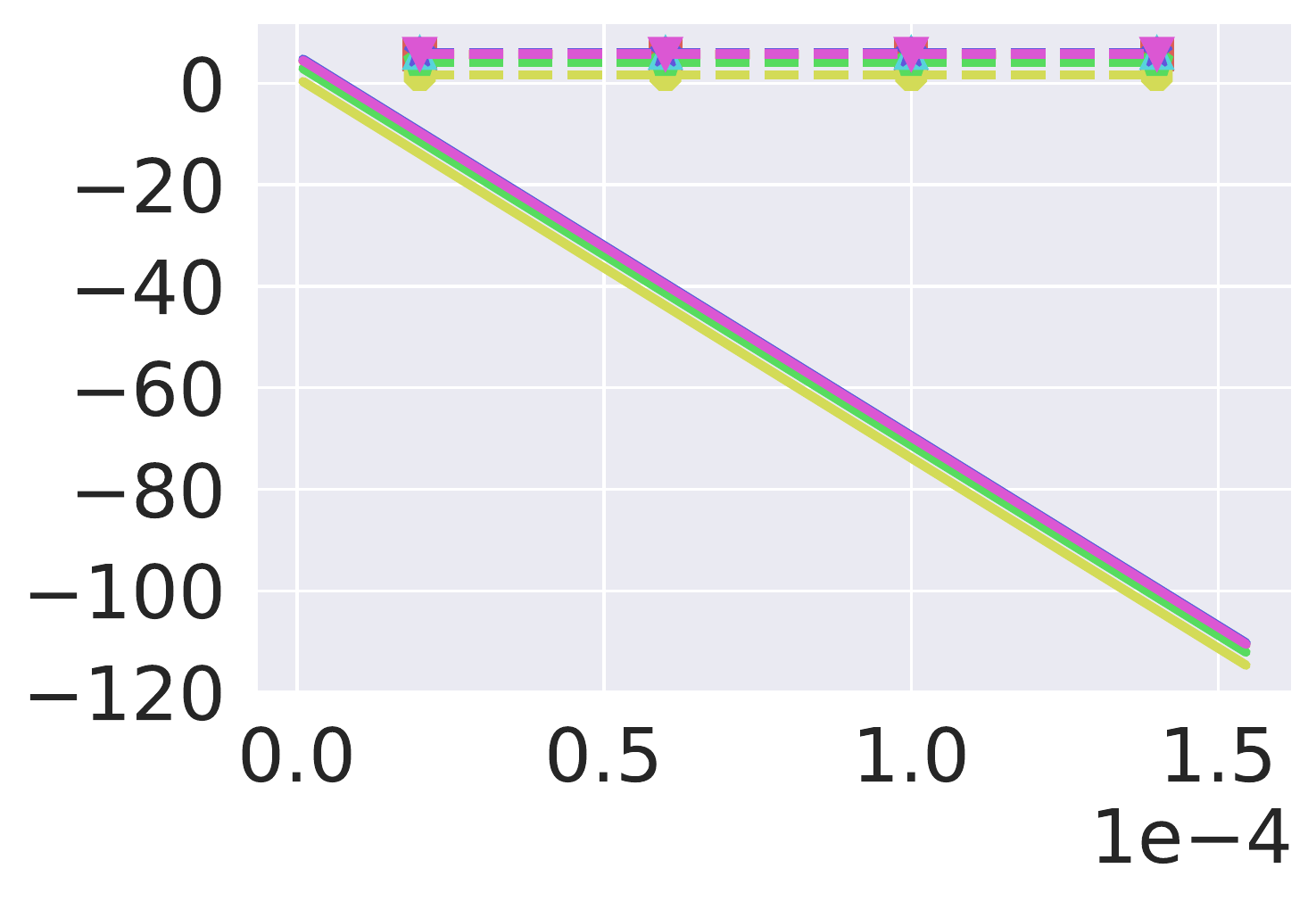}}%

\newlength{\utilheightdp}
\settoheight{\utilheightdp}{\includegraphics[width=.165\linewidth]{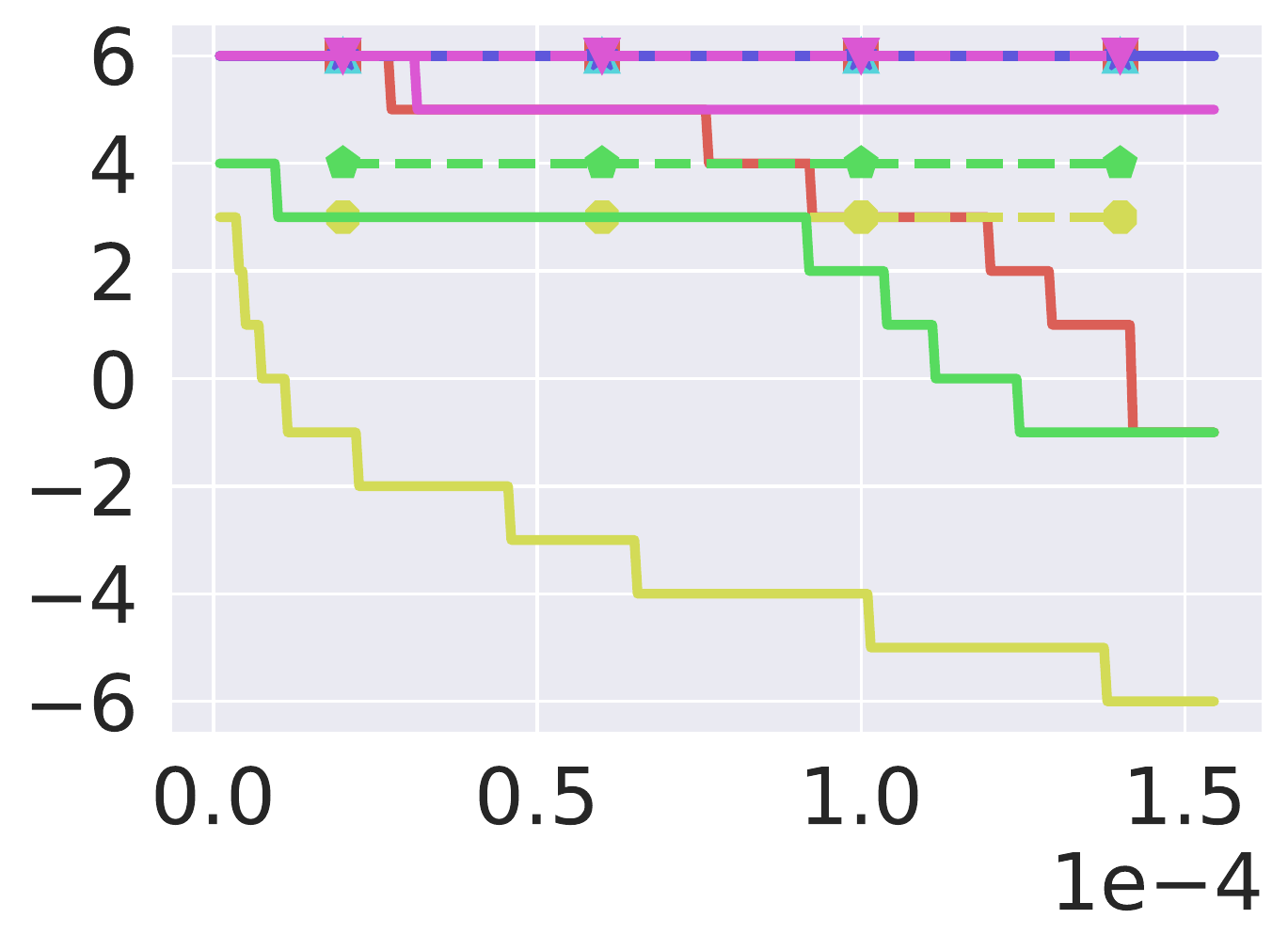}}%

\newlength{\utilheightaap}
\settoheight{\utilheightaap}{\includegraphics[width=.162\linewidth]{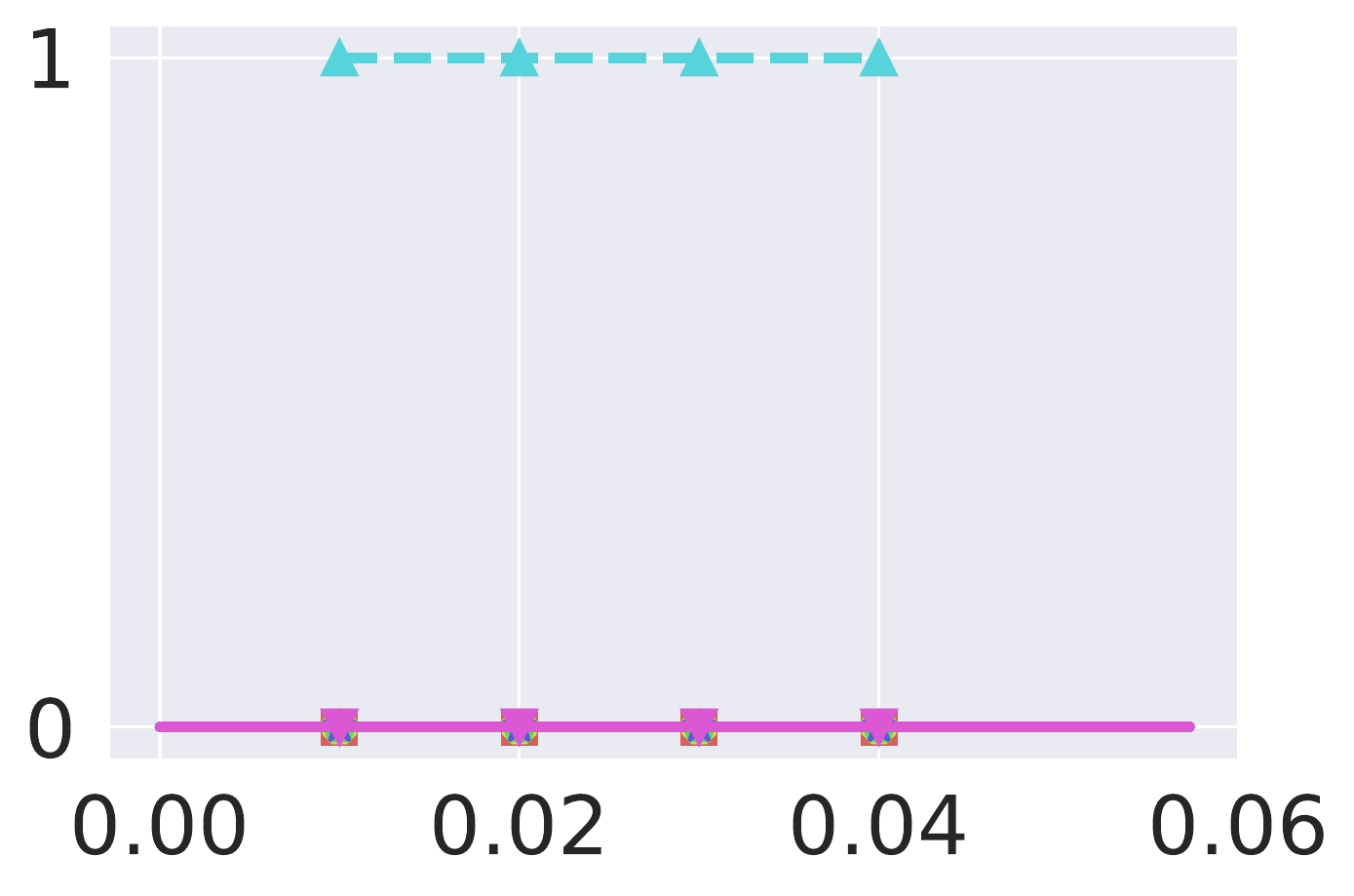}}%

% \newlength{\utilheight}
% \settoheight{\utilheight}{\includegraphics[width=.138\linewidth]{figures/twitch-DE: low degree.pdf}}%

% \newlength{\attackheightb}
% \settoheight{\attackheightb}{\includegraphics[width=.138\linewidth]{figures/twitch-DE: high degree.pdf}}%

\newlength{\legendheightb}
\setlength{\legendheightb}{0.48\utilheightc}%

\newcommand{\rownamec}[1]% #1 = text
{\rotatebox{90}{\makebox[\utilheightc][c]{\tiny #1}}}

\newcommand{\rownamed}[1]% #1 = text
{\rotatebox{90}{\makebox[\utilheightd][c]{\tiny #1}}}

\centering

{
\renewcommand{\tabcolsep}{10pt}

\begin{subtable}[]{\linewidth}
\begin{tabular}{l}
\includegraphics[height=0.7\legendheightb]{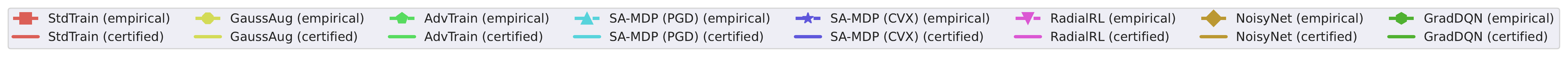}
\end{tabular}
\end{subtable}
\vspace{-2pt}
\begin{subtable}[]{\linewidth}
\centering
\resizebox{\linewidth}{!}{%
\begin{tabular}{@{}p{3mm}@{}c@{}c@{}c@{}c@{}c@{}c@{}}
        & \makecell{\tiny{$\sigma=0.005$}}
        & \makecell{\tiny{$\sigma=0.05$}}
        & \makecell{\tiny{$\sigma=0.1$}}
        & \makecell{\tiny{$\sigma=0.5$}}
        & \makecell{\tiny{$\sigma=0.75$}}
        & \makecell{\tiny{$\sigma=1.0$}}
        \vspace{-2pt}\\
\rownamec{\makecell{(Global) $\uje$}}&
\includegraphics[height=\utilheightc]{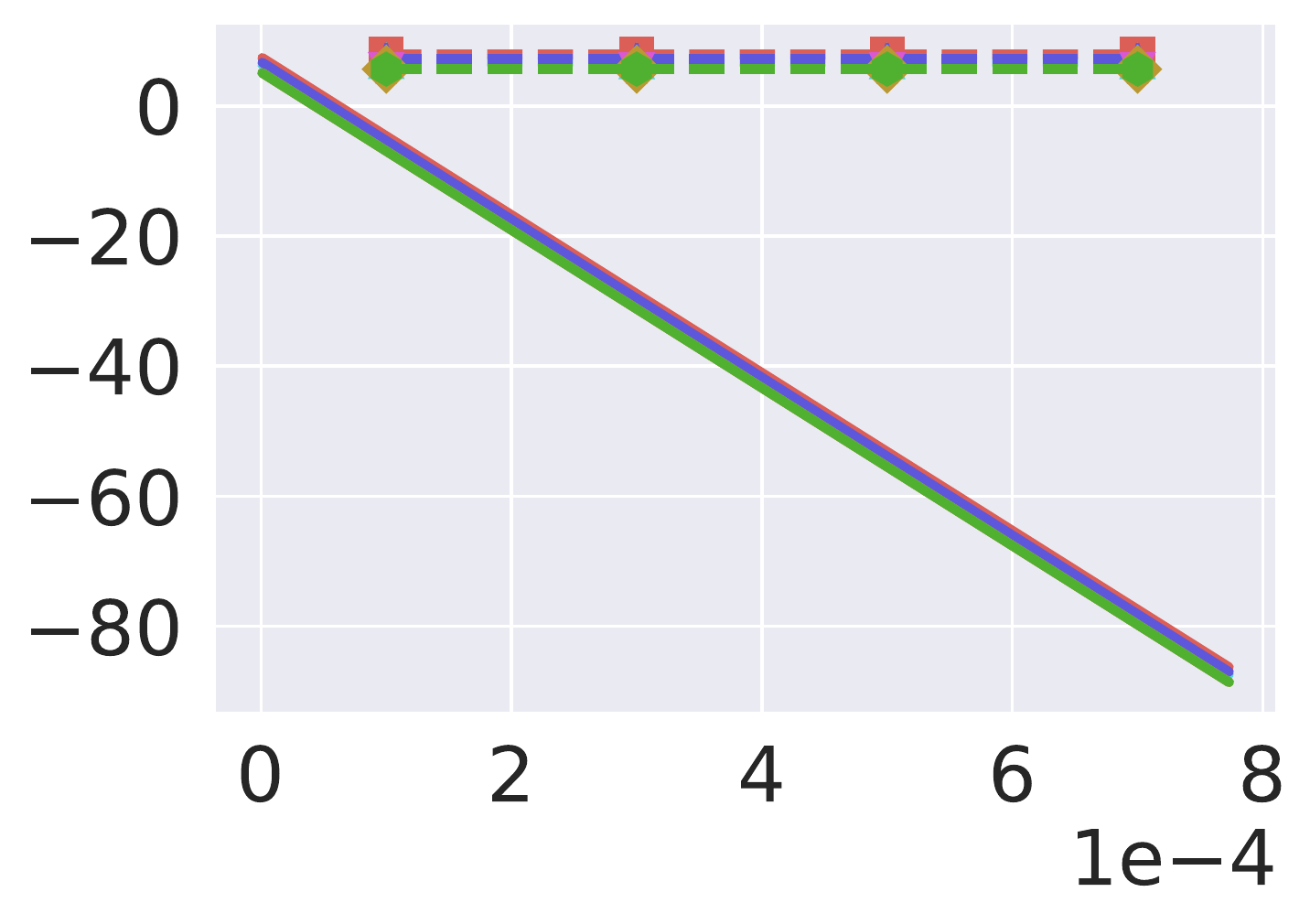}&
\includegraphics[height=\utilheightc]{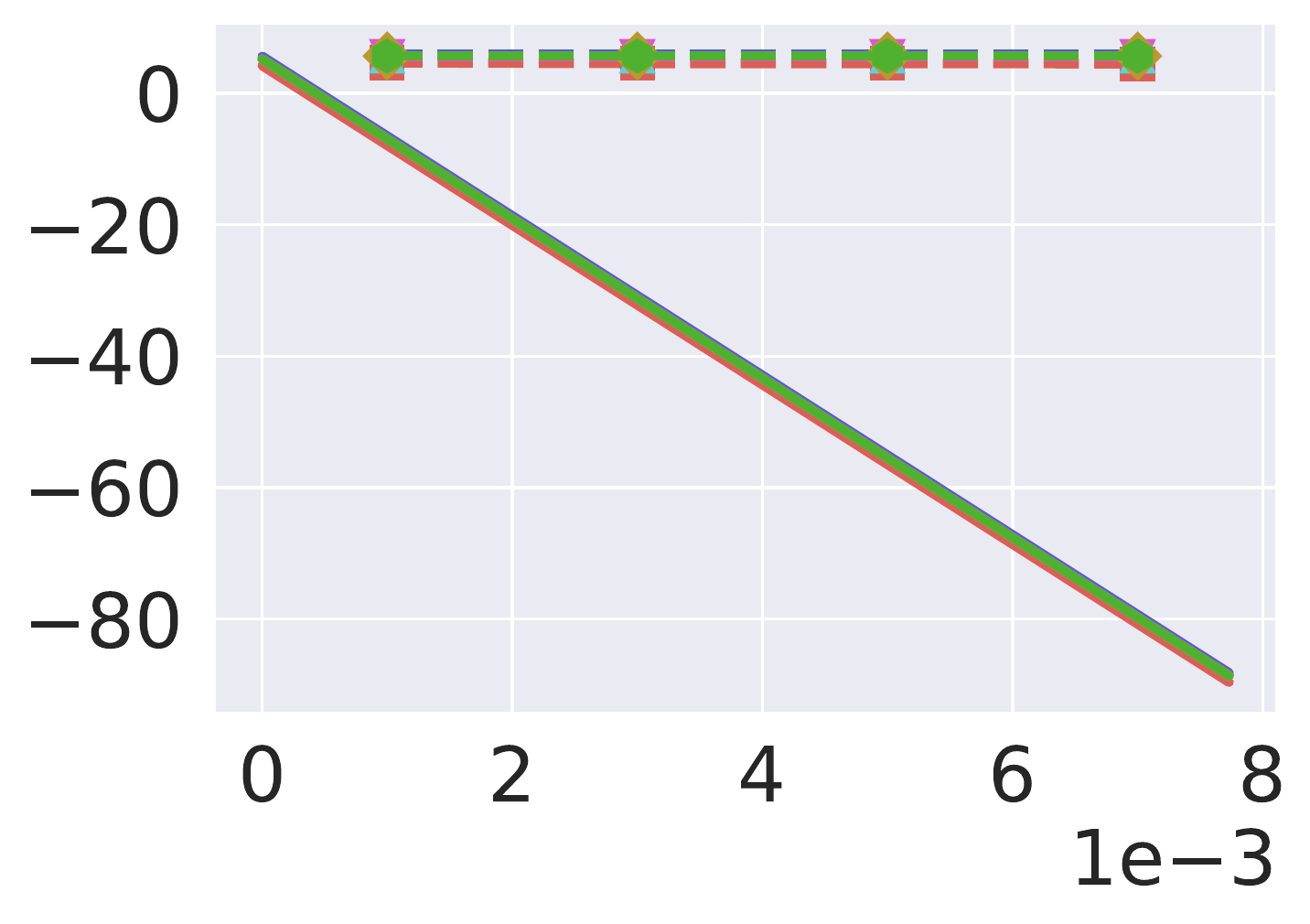}&
\includegraphics[height=\utilheightc]{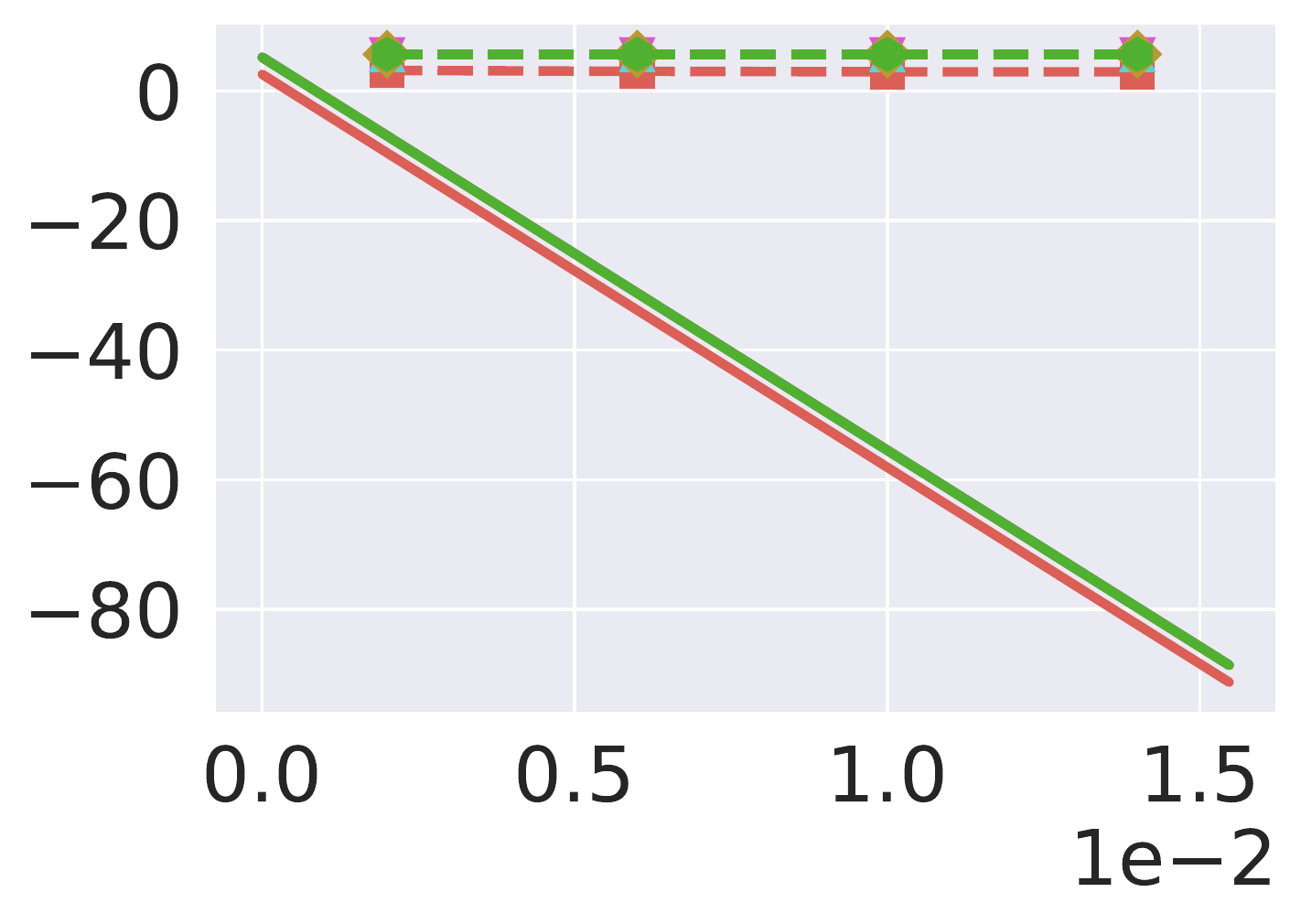}&
\includegraphics[height=\utilheightc]{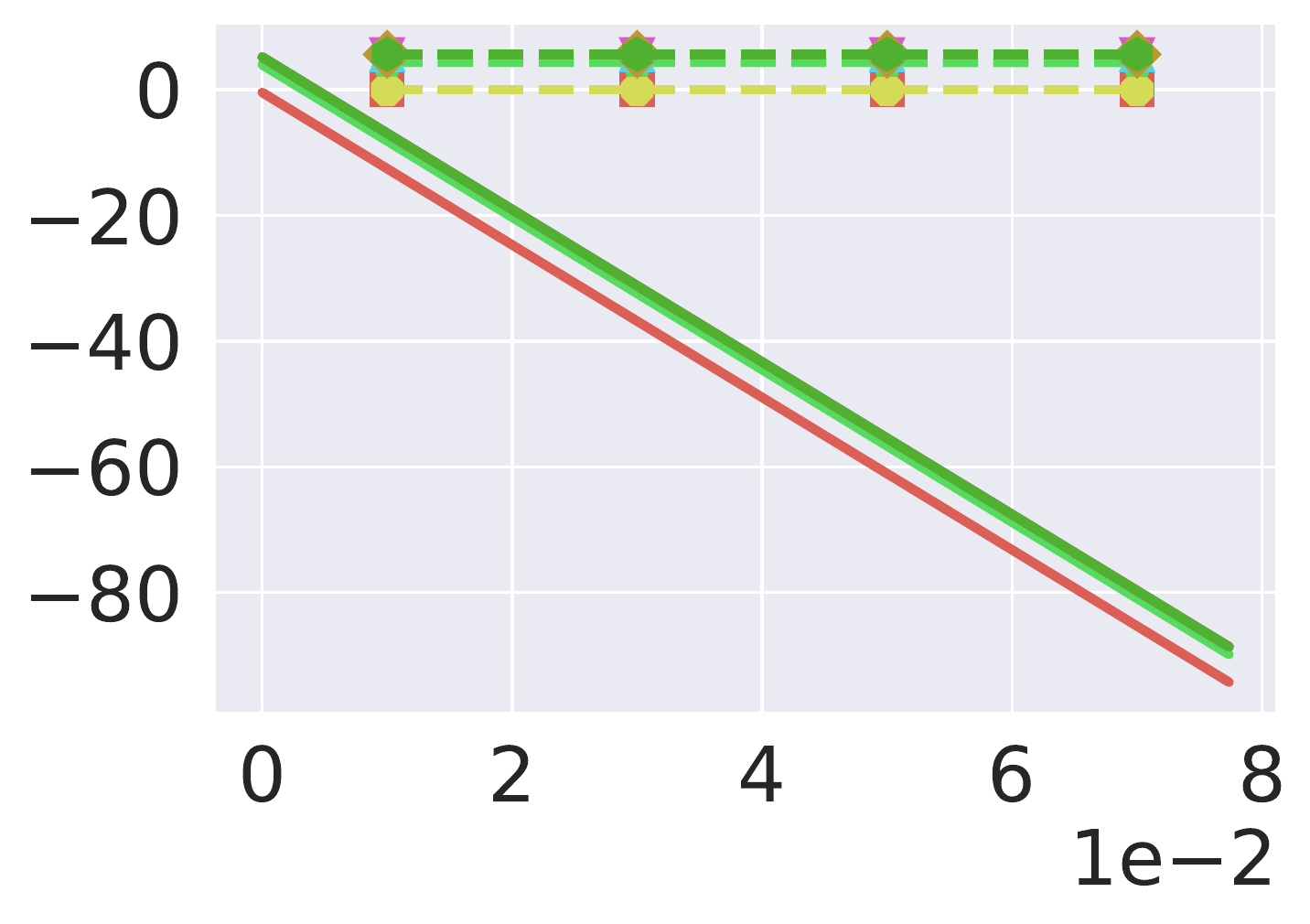}&
\includegraphics[height=\utilheightc]{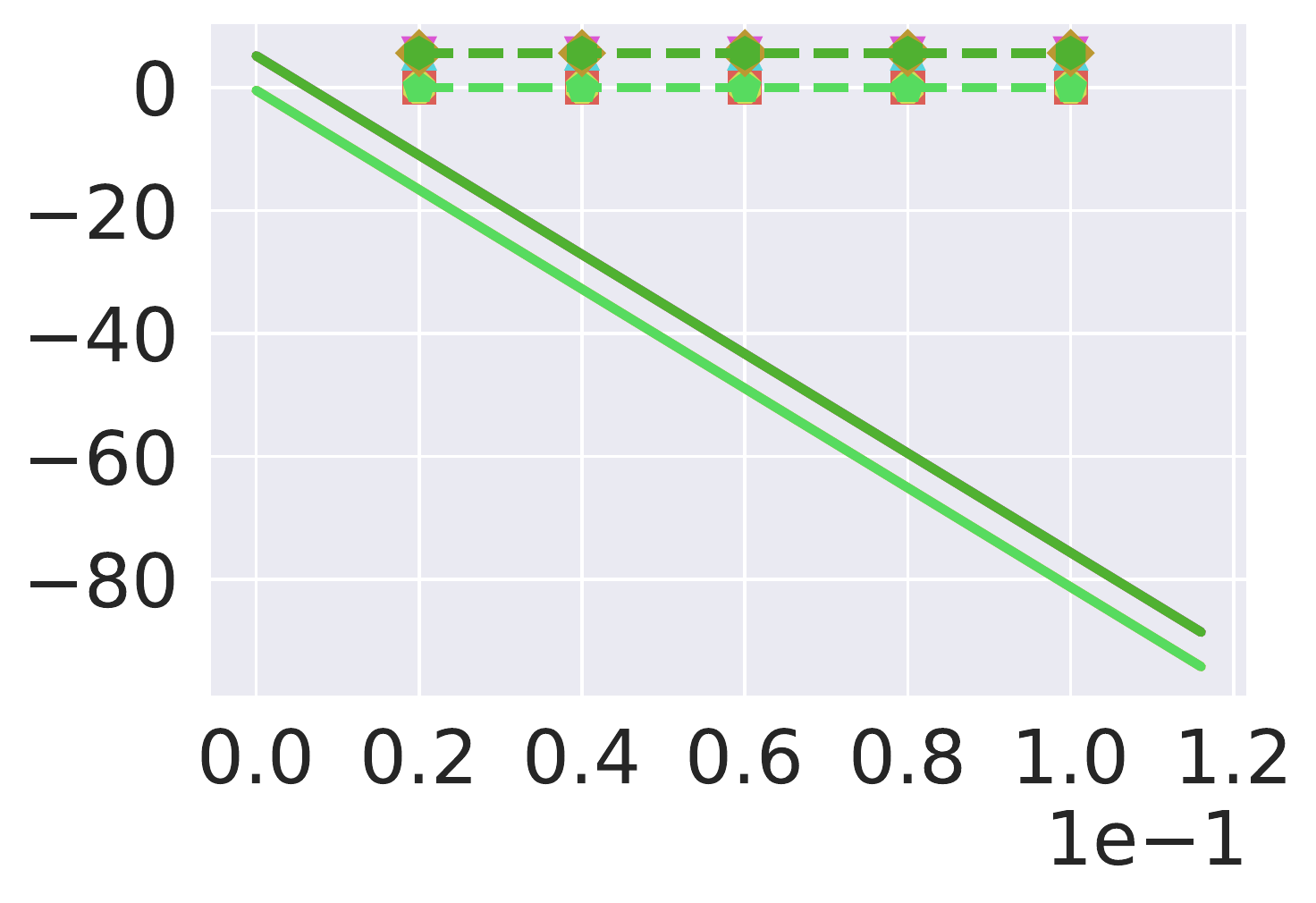}&
\includegraphics[height=\utilheightc]{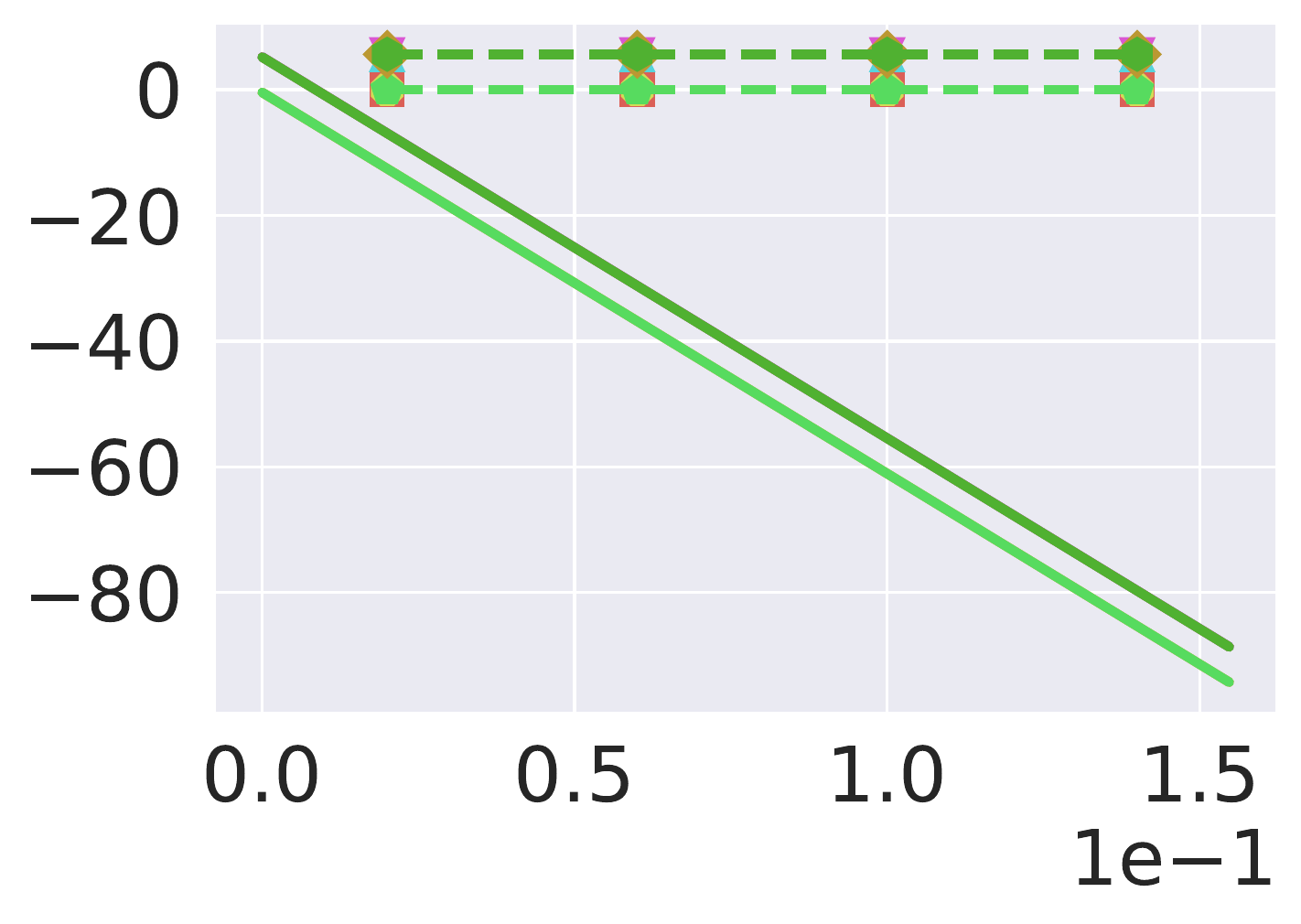}\\[-1.2ex]
\rownamed{\makecell{(Global) $\ujp$}}&
\includegraphics[height=\utilheightd]{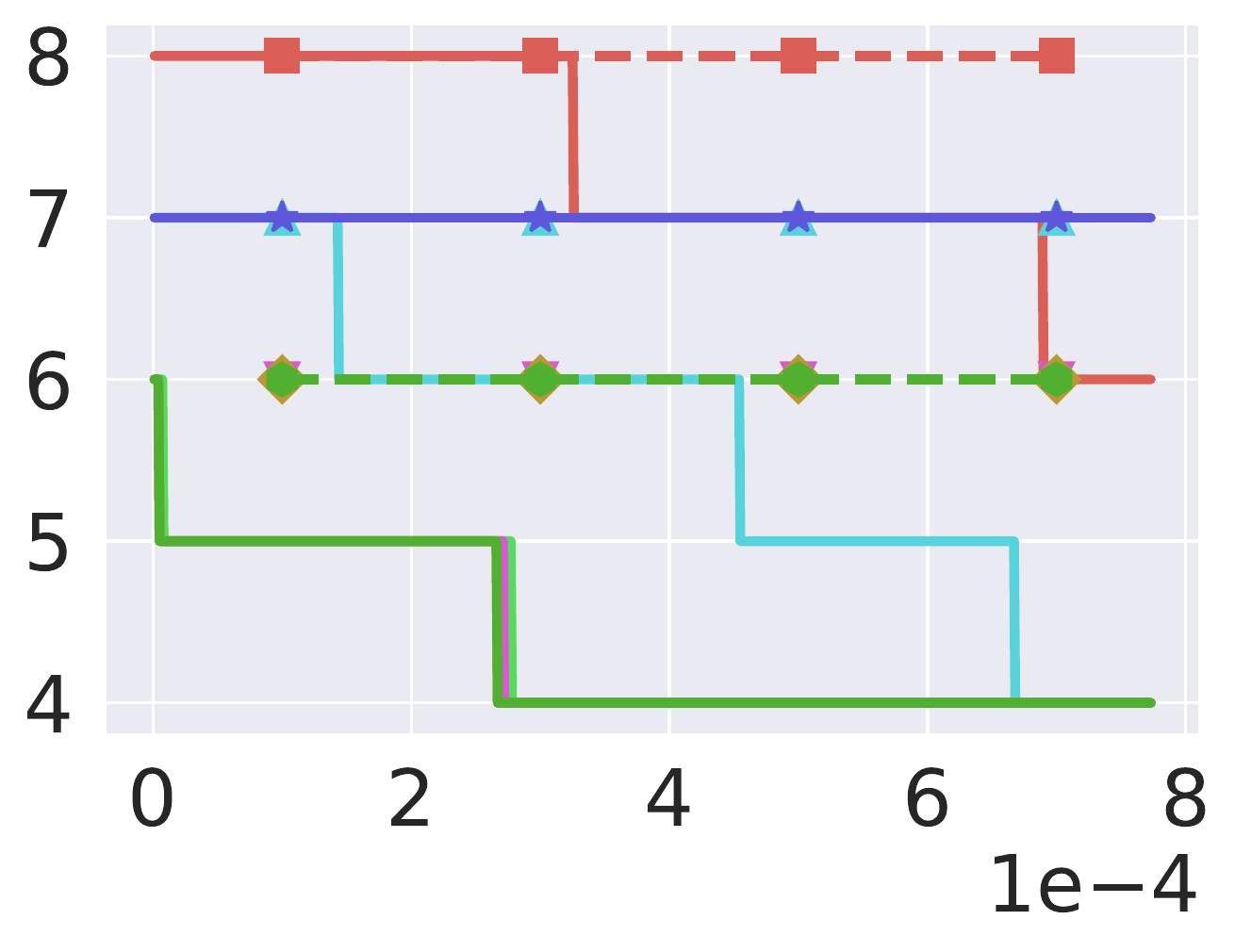}&
\includegraphics[height=\utilheightd]{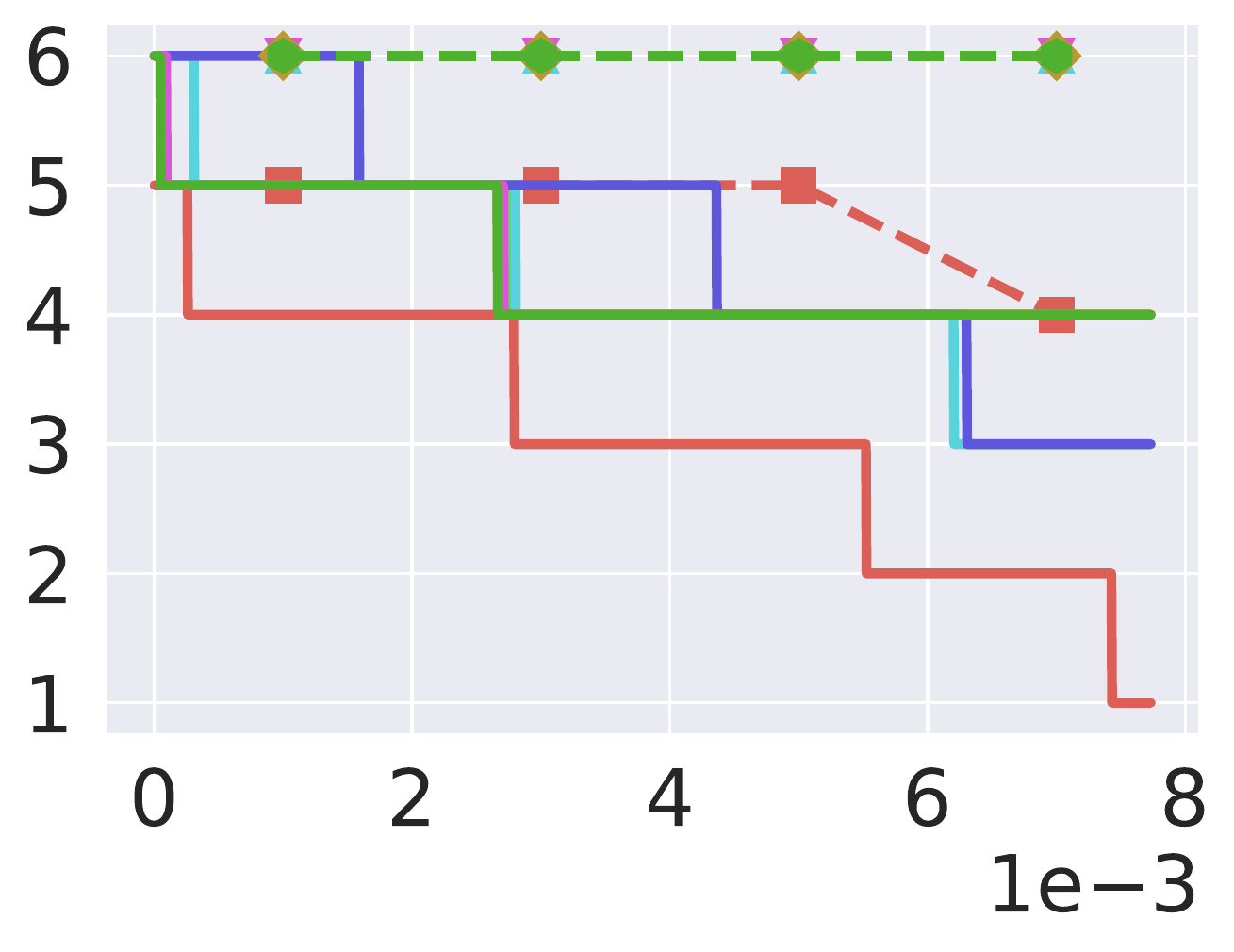}&
\includegraphics[height=\utilheightd]{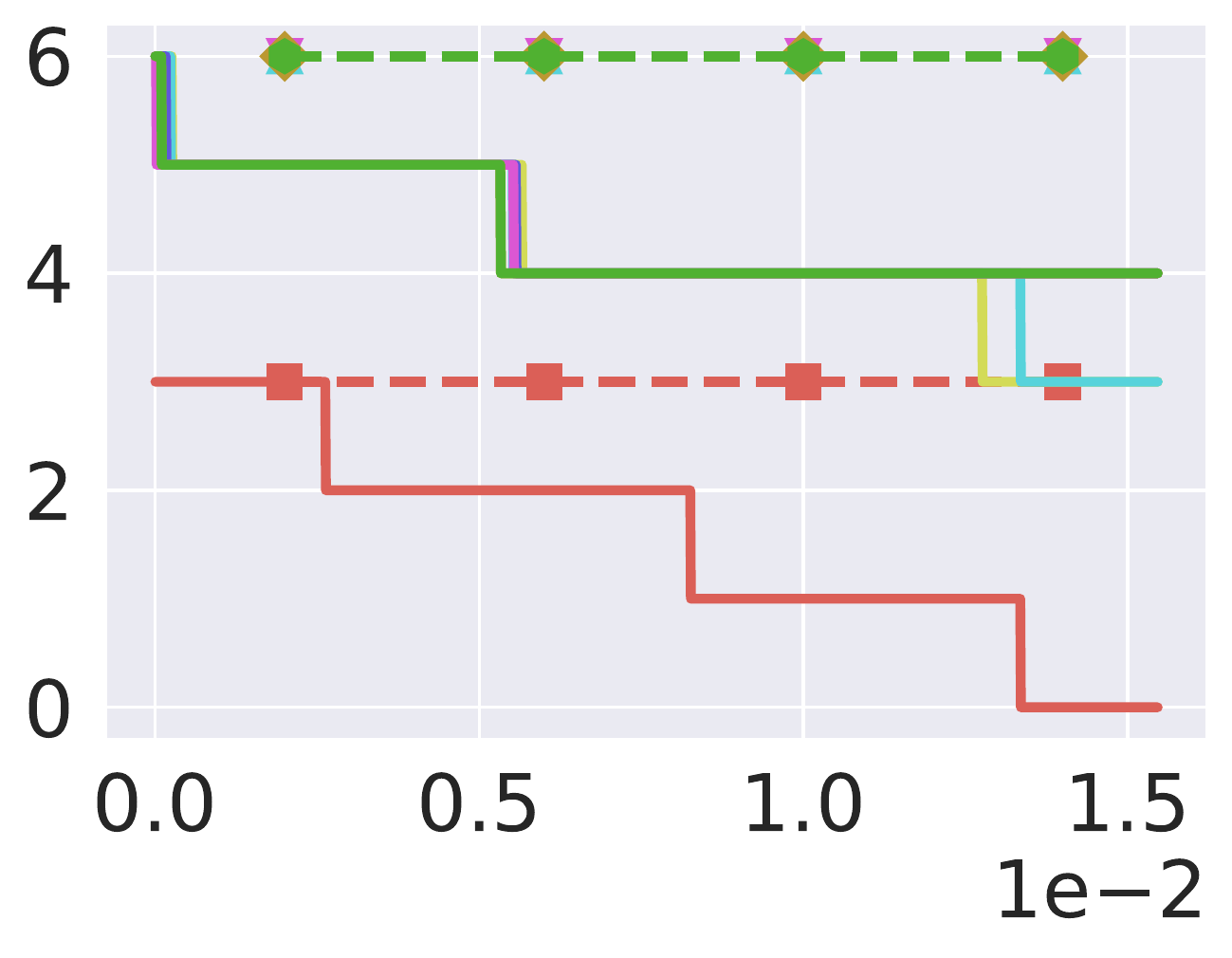}&
\includegraphics[height=\utilheightd]{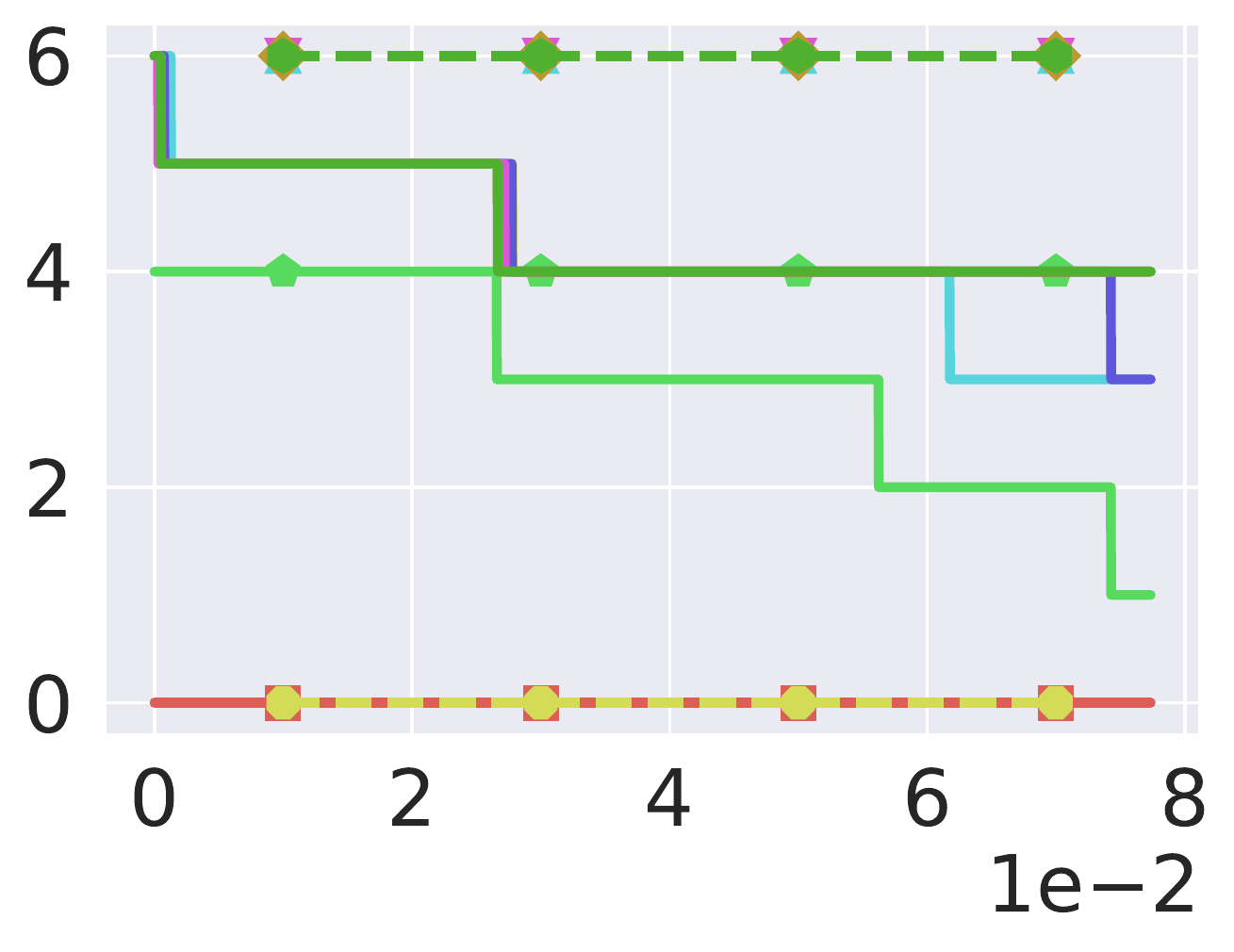}&
\includegraphics[height=\utilheightd]{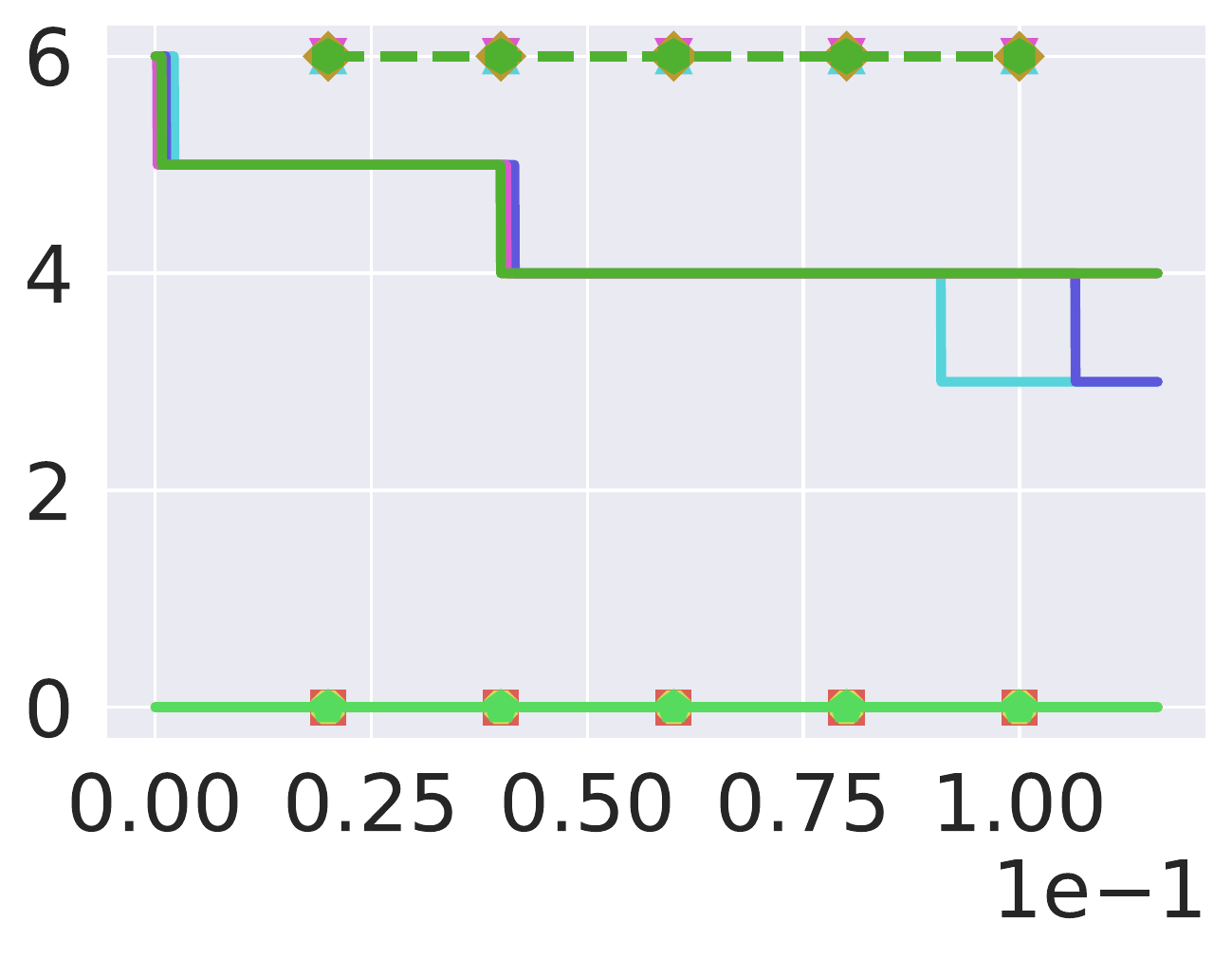}&
\includegraphics[height=\utilheightd]{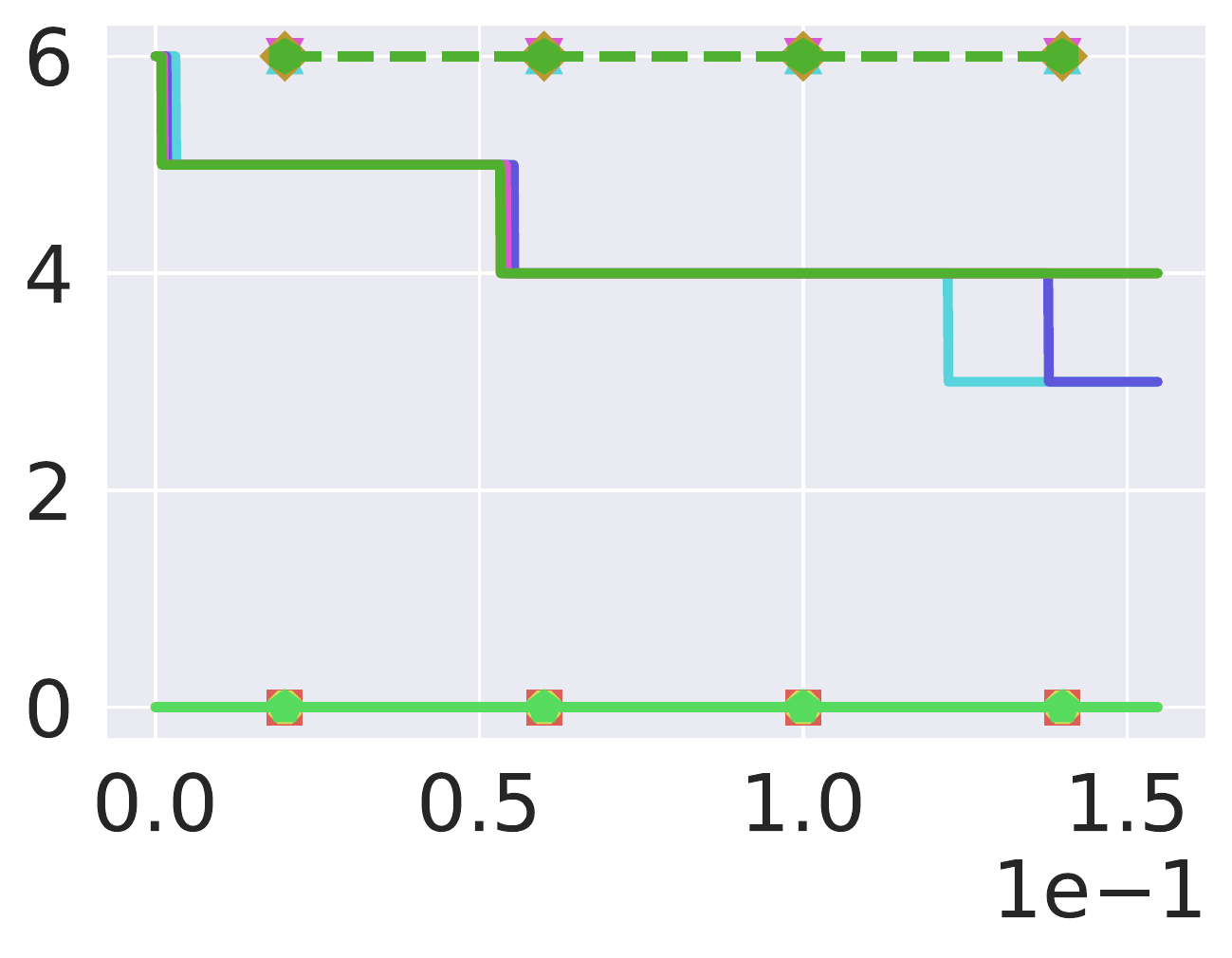}\\[-1.2ex]
\rownamed{\makecell{(Local) $\uj$}}&
\includegraphics[height=\utilheightaa]{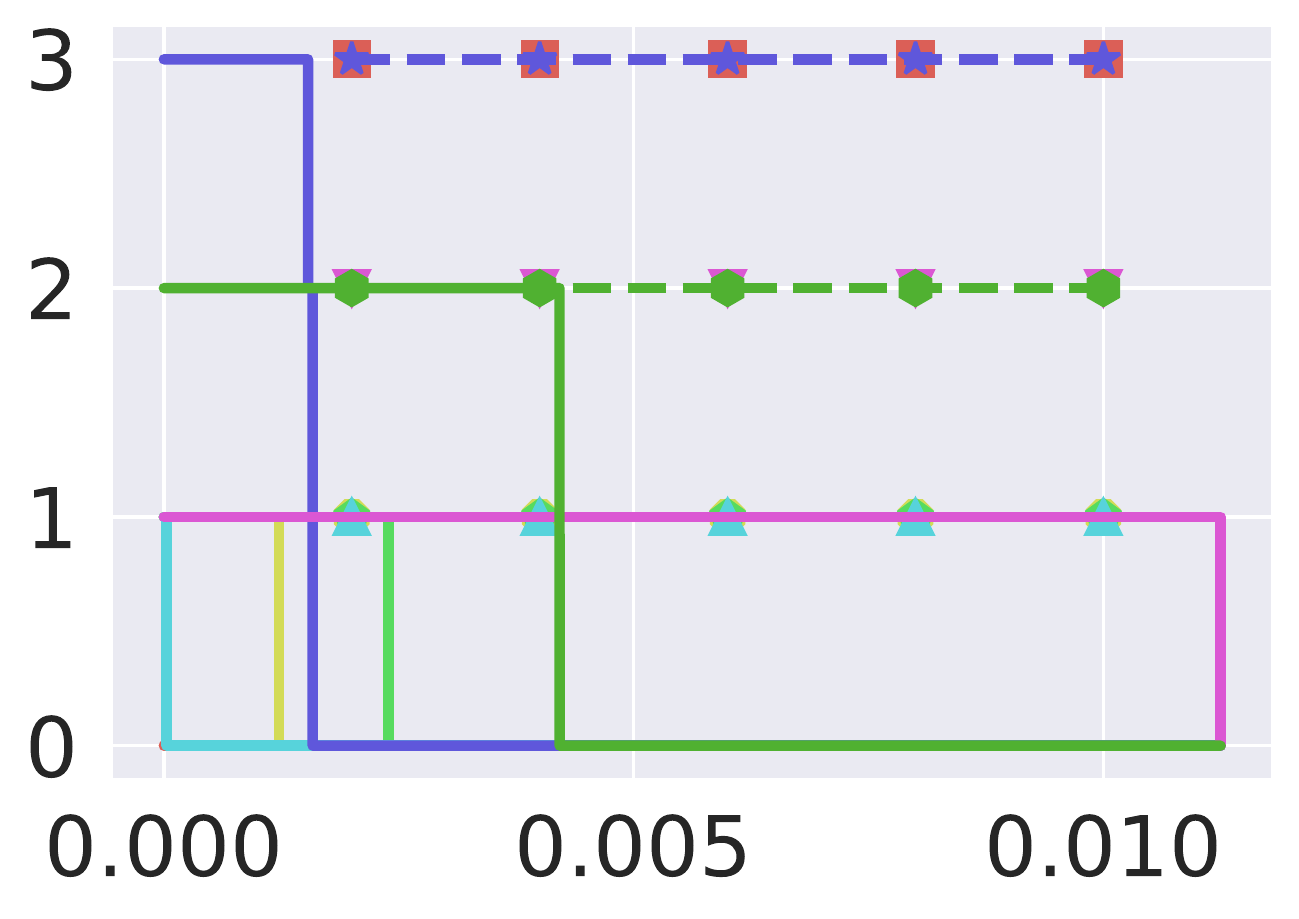}&
\includegraphics[height=\utilheightaa]{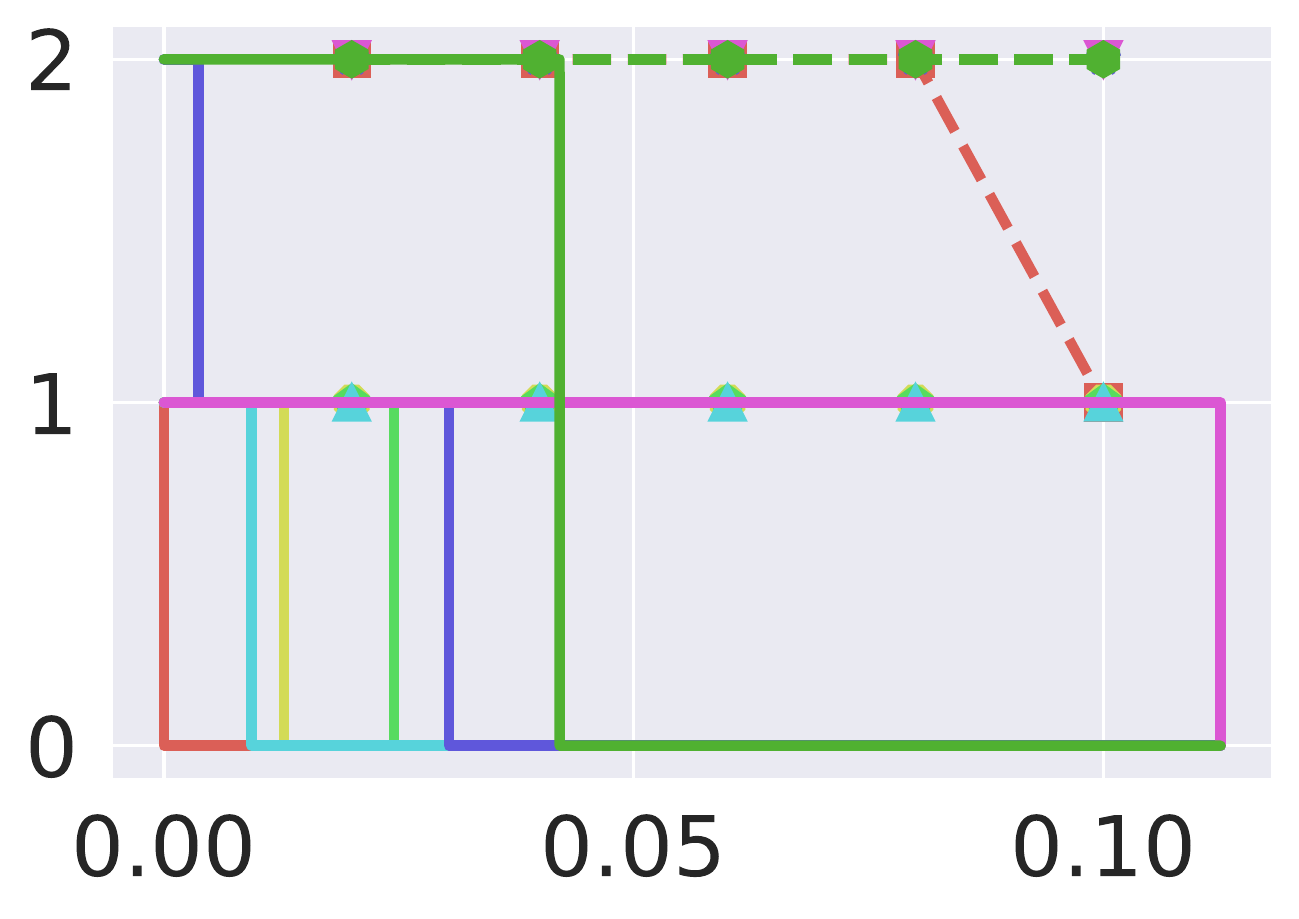}&
\includegraphics[height=\utilheightaa]{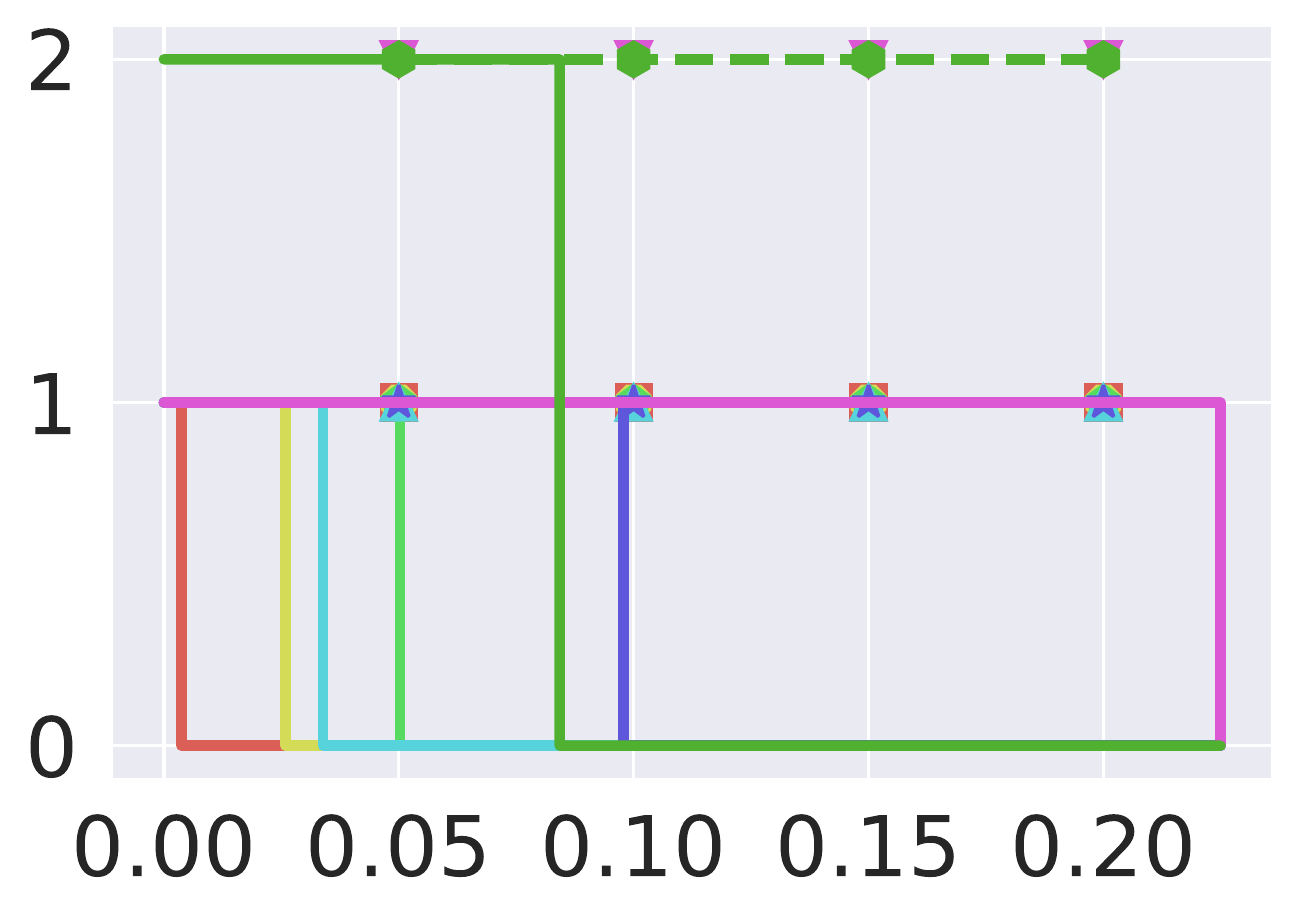}&
\includegraphics[height=\utilheightaa]{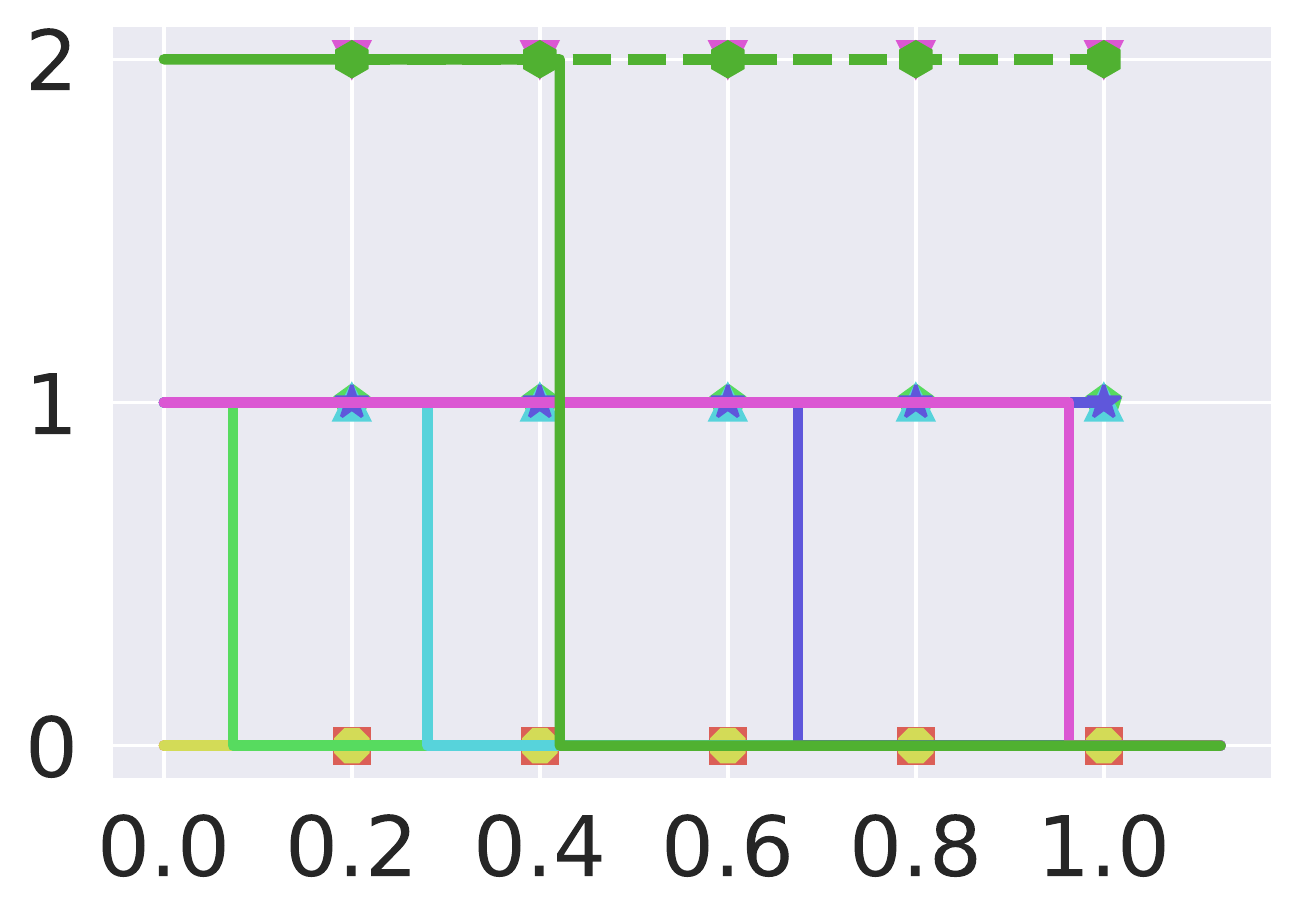}&
\includegraphics[height=\utilheightaa]{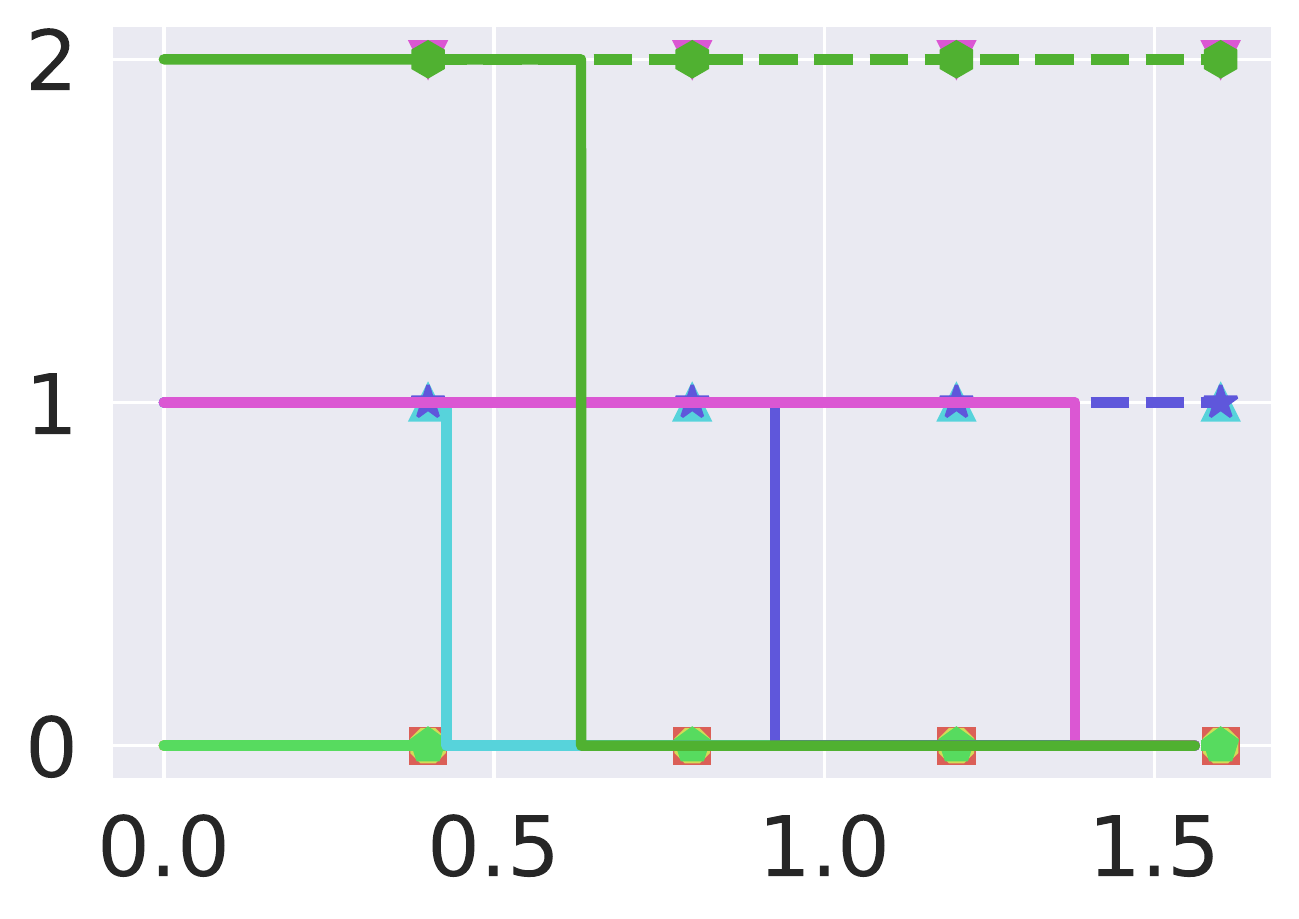}&
\includegraphics[height=\utilheightaa]{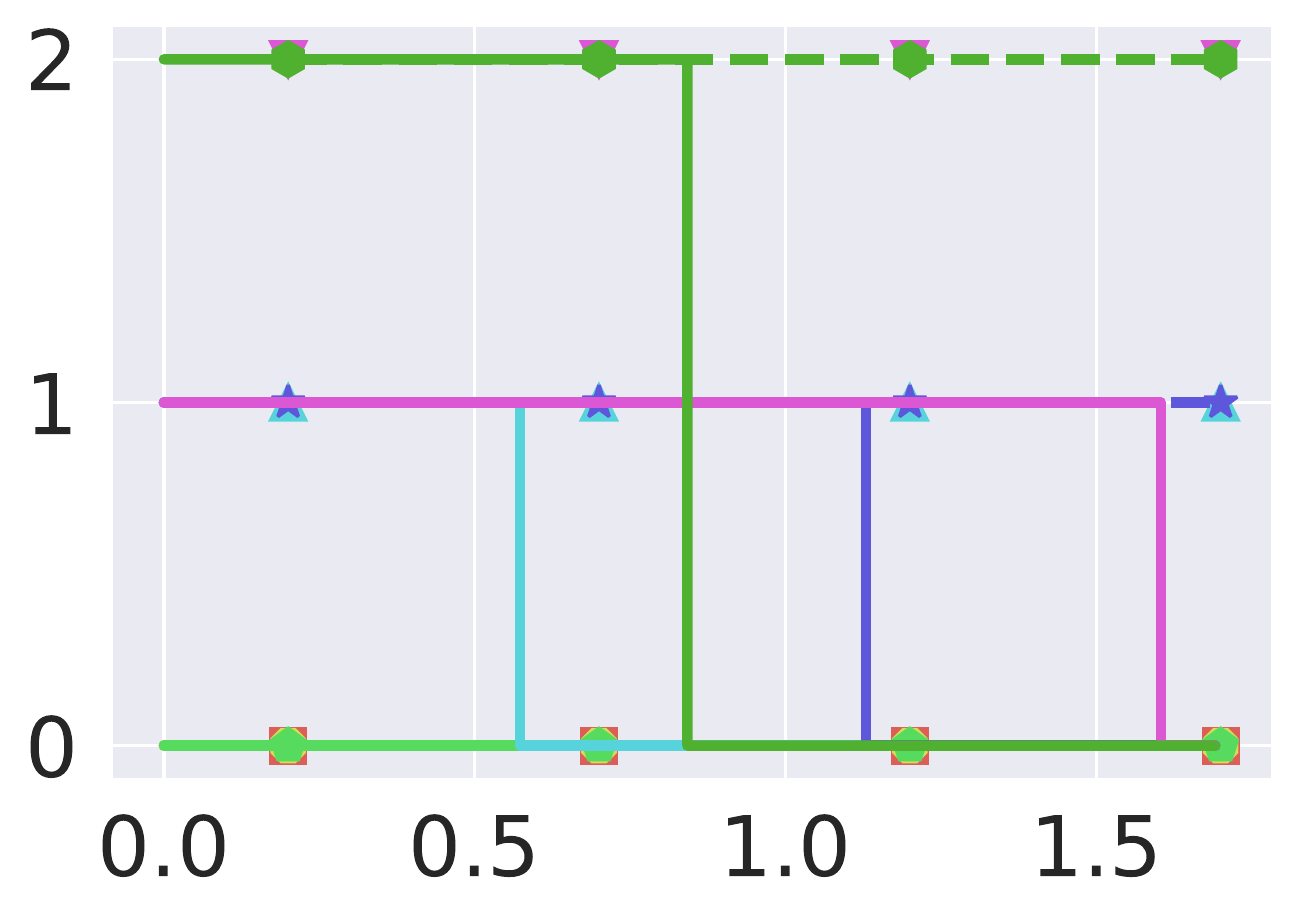}\\[-1.2ex]
        & \makecell{\tiny{attack $\eps$}}
        & \makecell{\tiny{attack $\eps$}}
        & \makecell{\tiny{attack $\eps$}}
        & \makecell{\tiny{attack $\eps$}}
        & \makecell{\tiny{attack $\eps$}}
        & \makecell{\tiny{attack $\eps$}}
\end{tabular}
}
\vspace{-3mm}
\caption{\small Freeway}\label{tab:bound-freeway}
\end{subtable}

\begin{subtable}[]{\linewidth}
\centering
\resizebox{\linewidth}{!}{%
\begin{tabular}{@{}p{3mm}@{}c@{}c@{}c@{}c@{}c@{}c@{}}
        & \makecell{\tiny{$\sigma=0.001$}}
        & \makecell{\tiny{$\sigma=0.005$}}
        & \makecell{\tiny{$\sigma=0.01$}}
        & \makecell{\tiny{$\sigma=0.03$}}
        & \makecell{\tiny{$\sigma=0.05$}}
        & \makecell{\tiny{$\sigma=0.1$}}
        \vspace{-2pt}\\
\rownamec{\makecell{(Global) $\uje$}}&
\includegraphics[height=\utilheightcp]{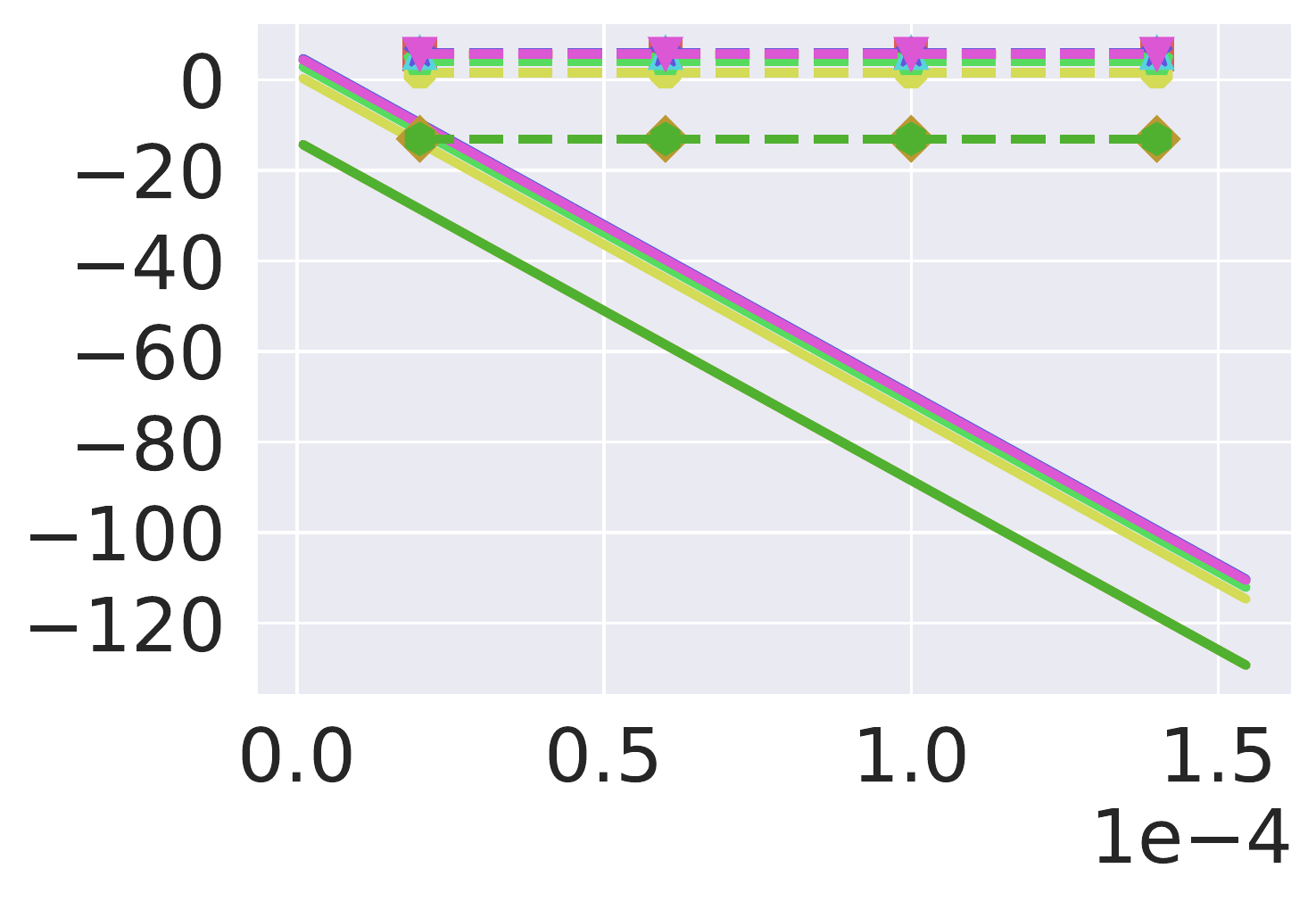}&
\includegraphics[height=\utilheightcp]{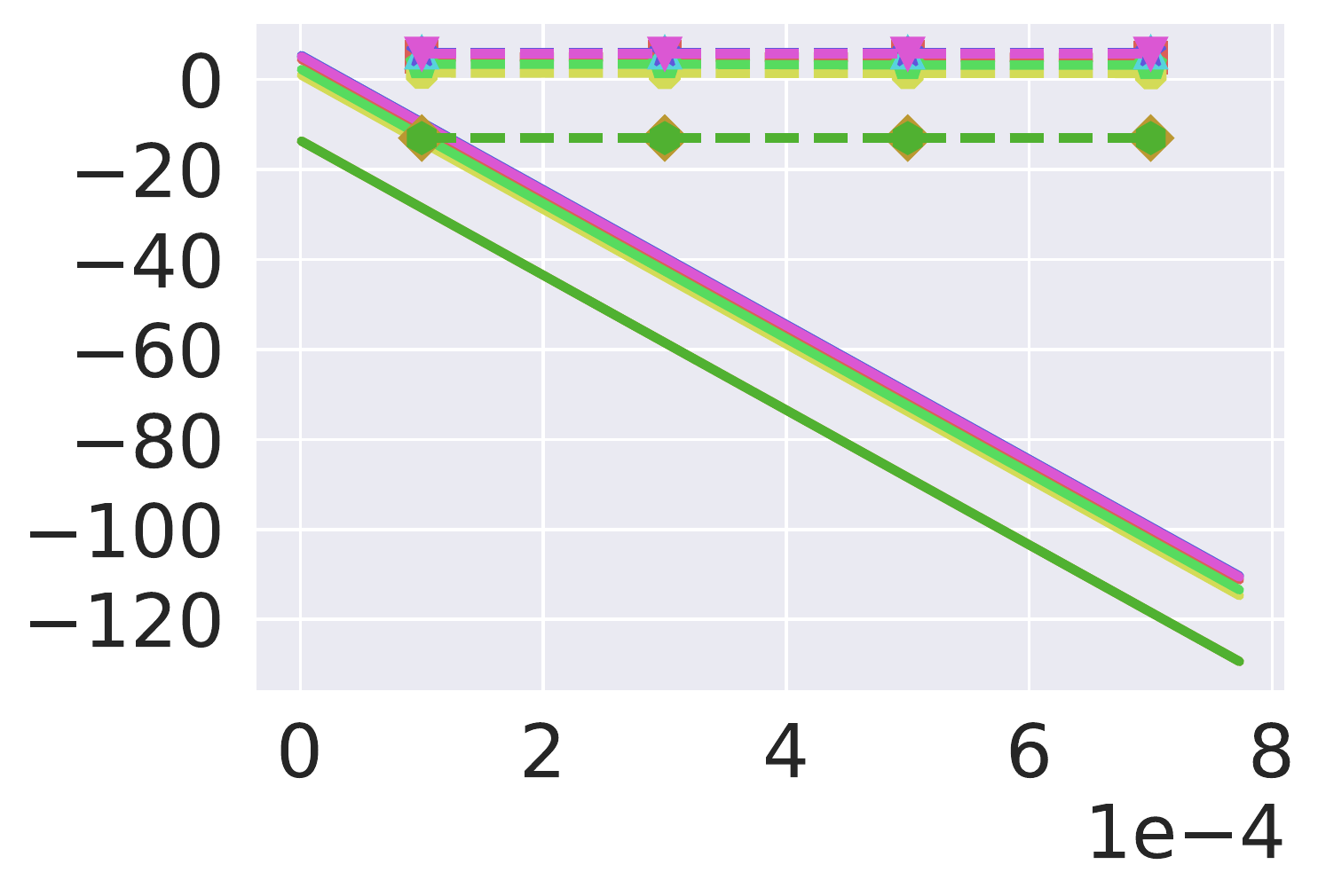}&
\includegraphics[height=\utilheightcp]{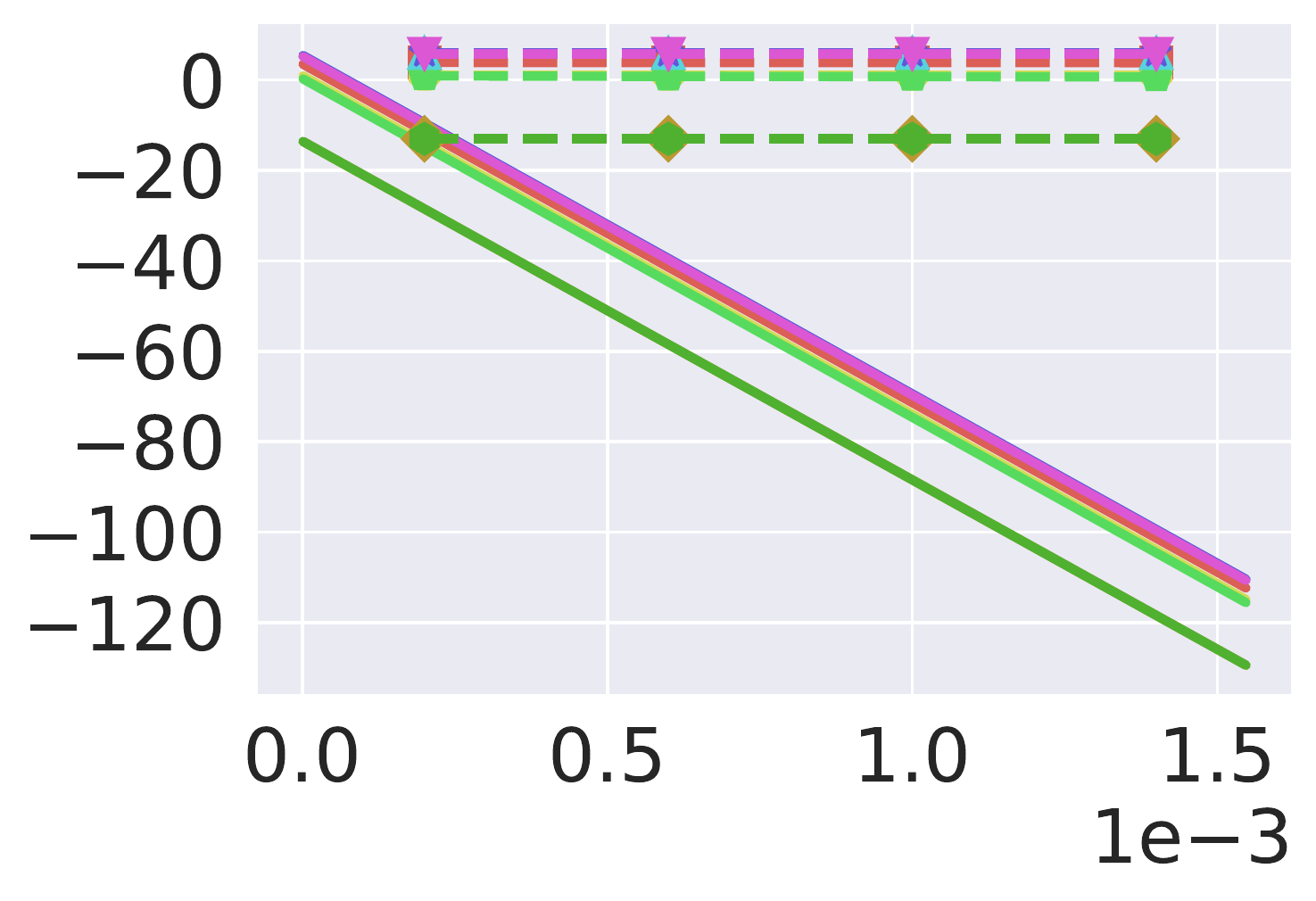}&
\includegraphics[height=\utilheightcp]{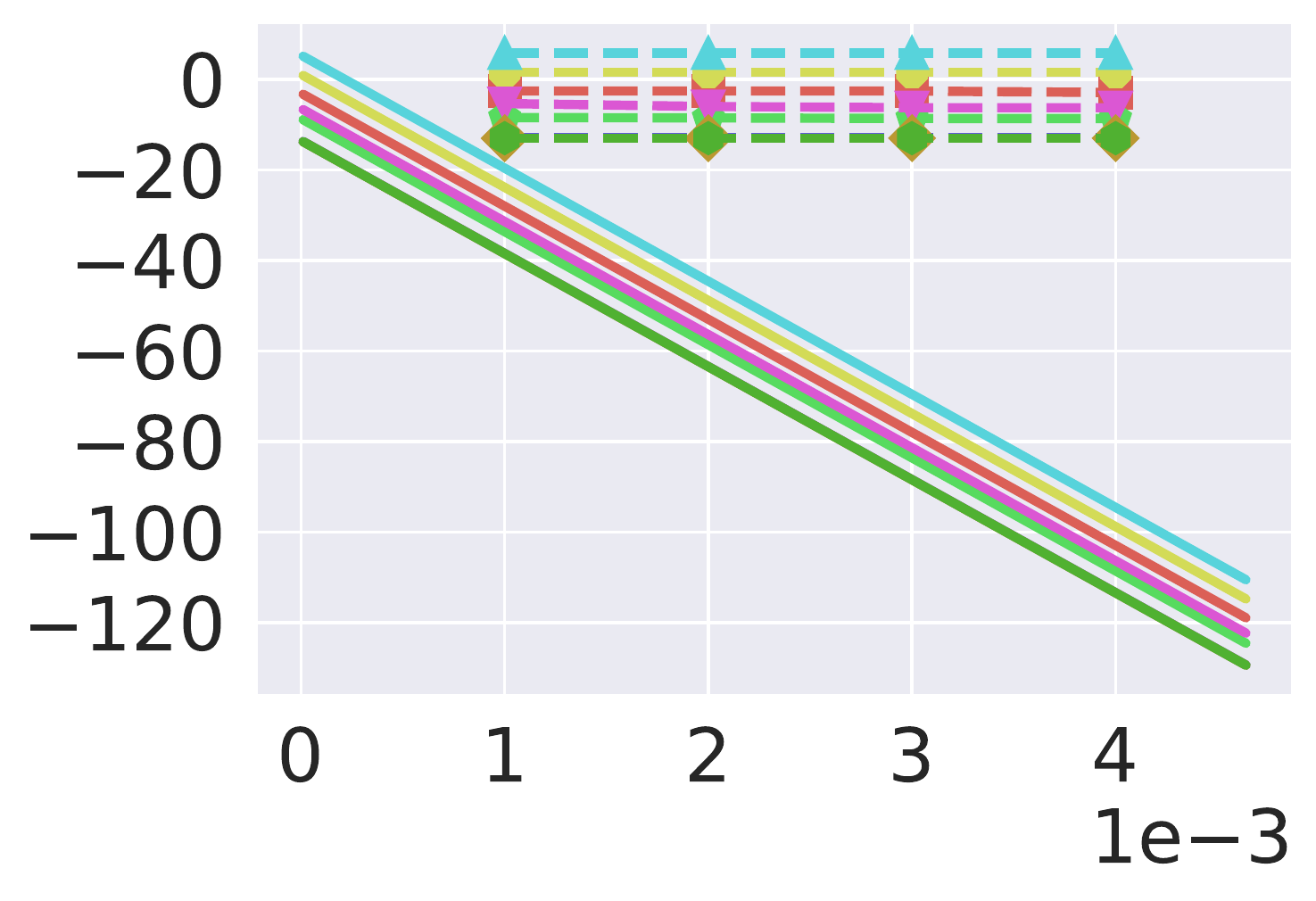}&
\includegraphics[height=\utilheightcp]{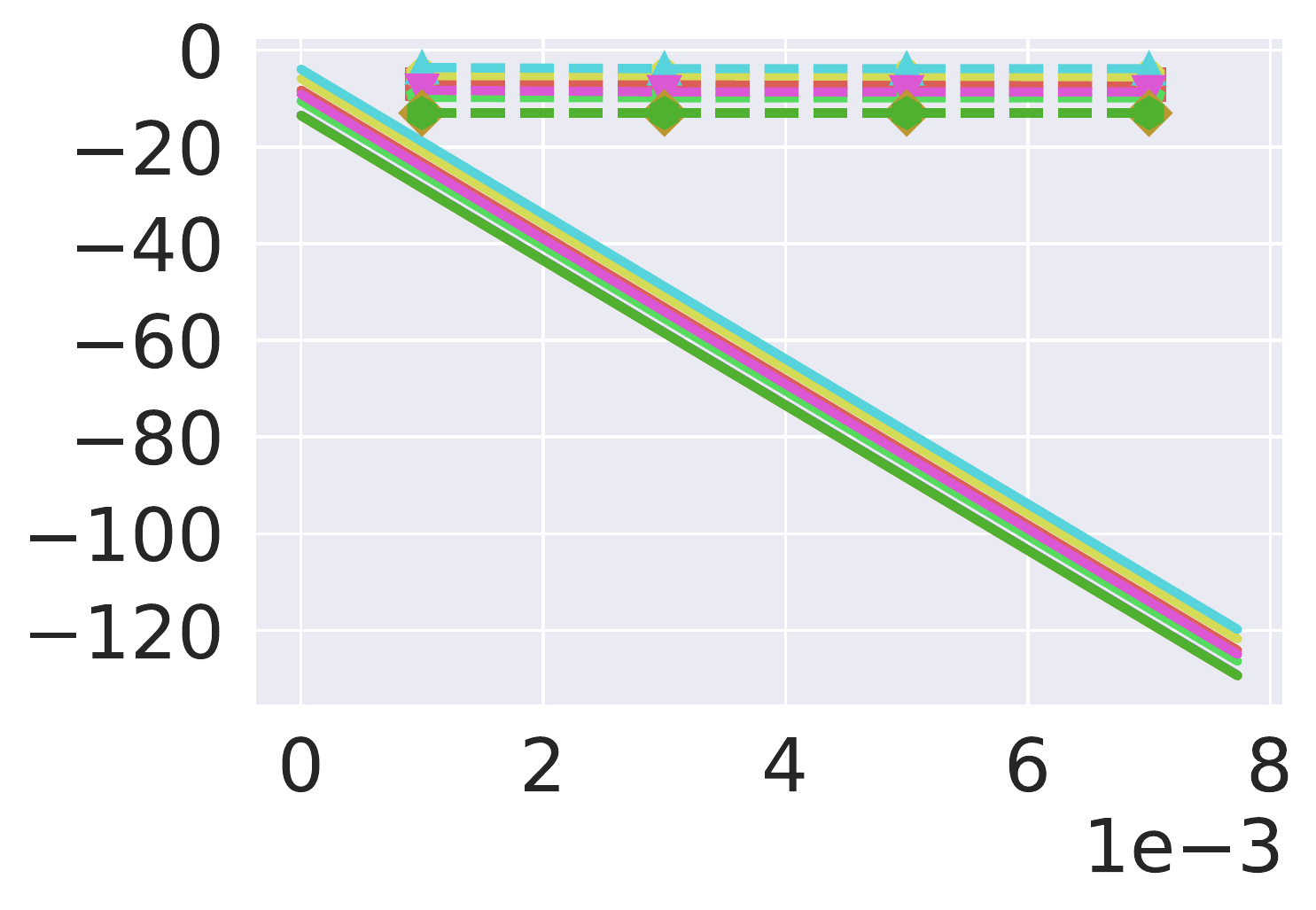}&
\includegraphics[height=\utilheightcp]{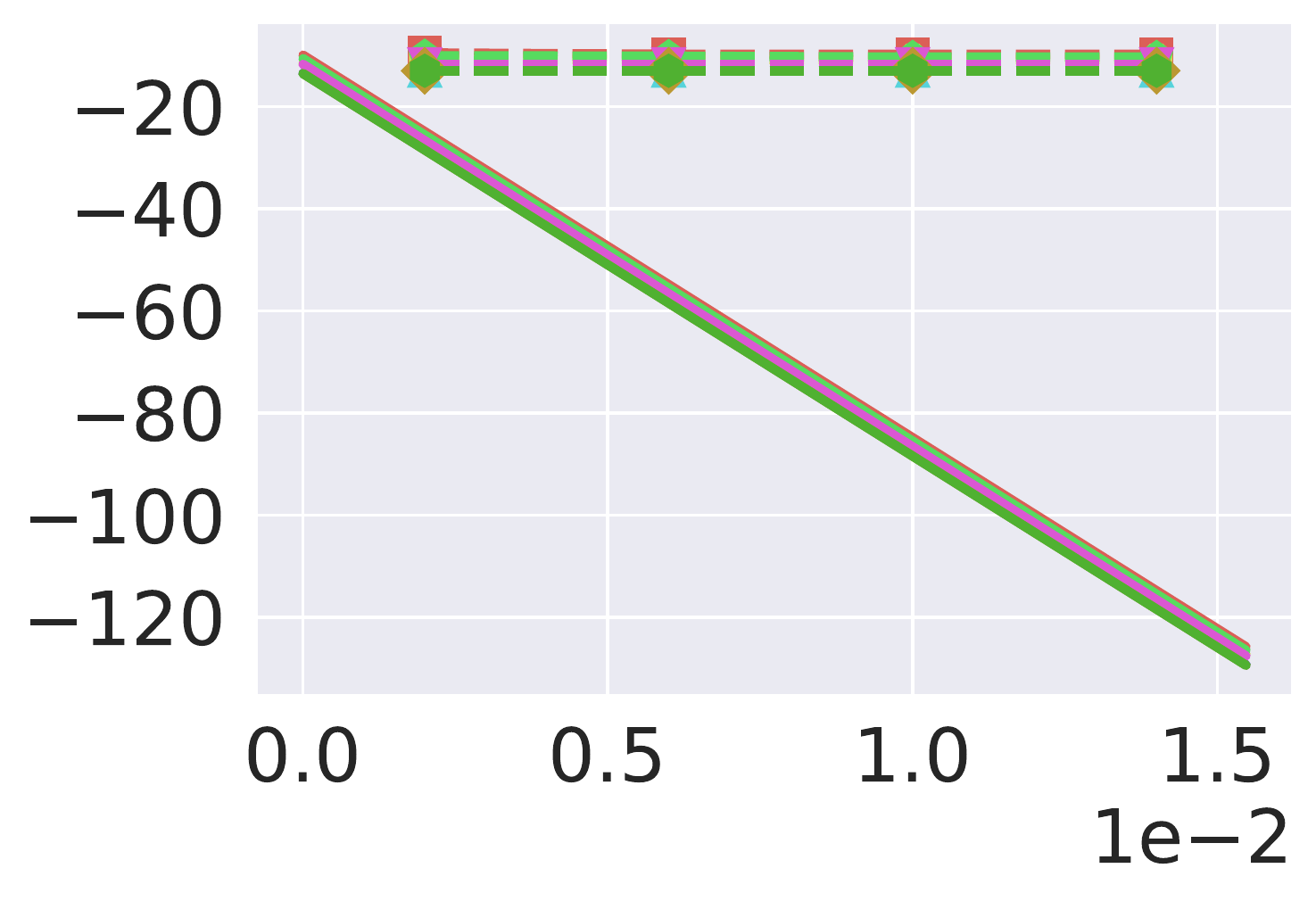}\\[-1.2ex]
\rownamed{\makecell{(Global) $\ujp$}}&
\includegraphics[height=\utilheightdp]{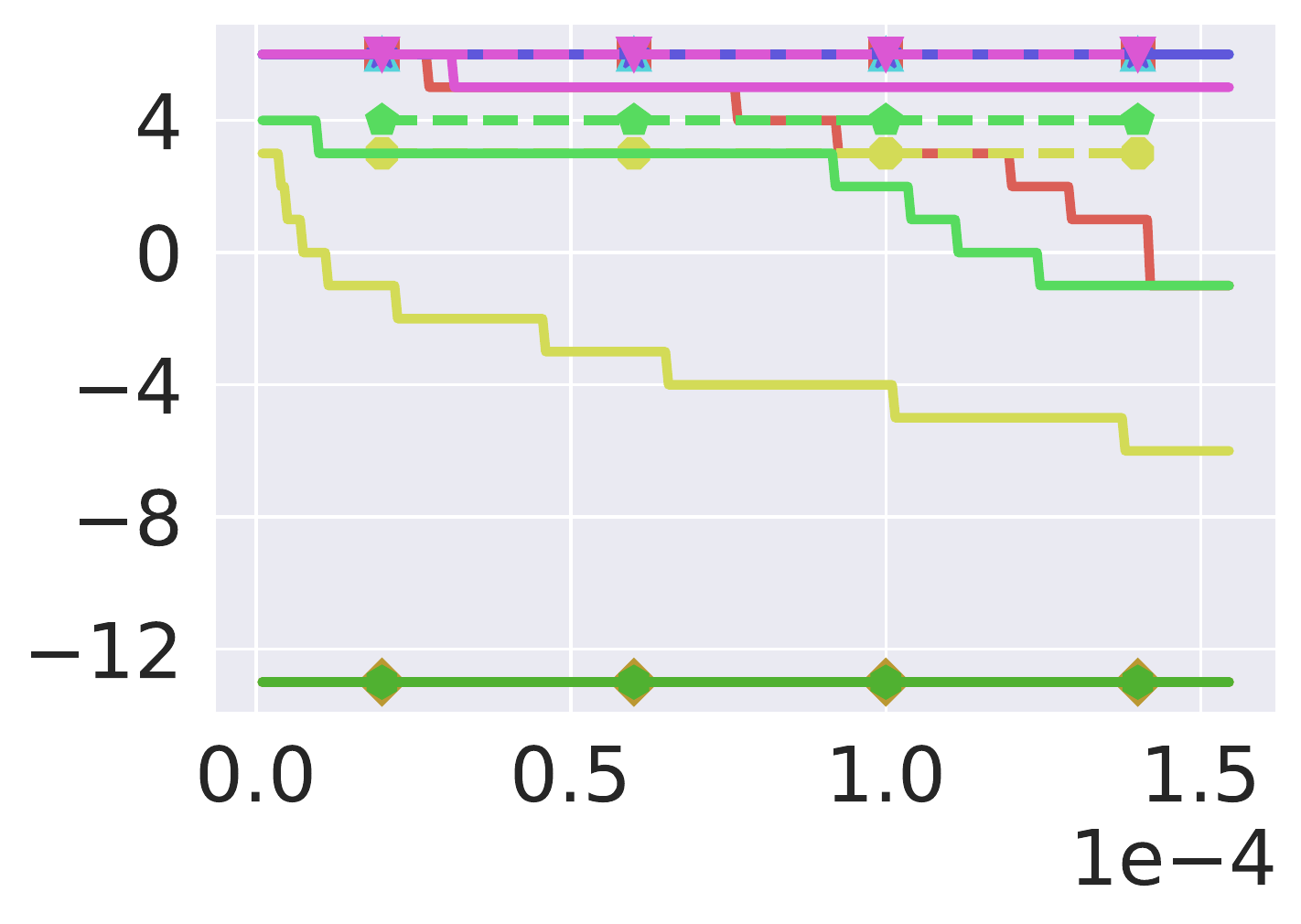}&
\includegraphics[height=\utilheightdp]{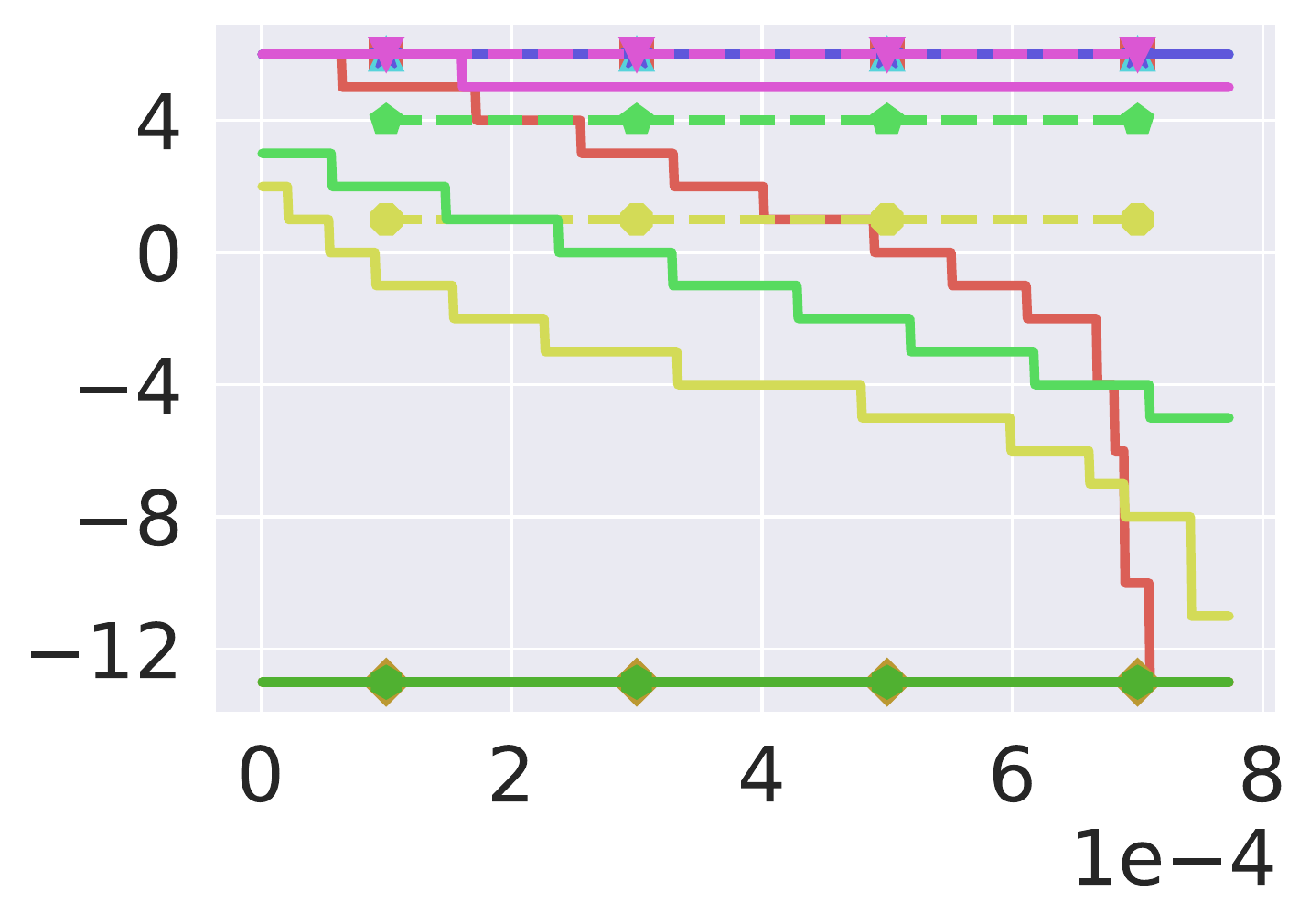}&
\includegraphics[height=\utilheightdp]{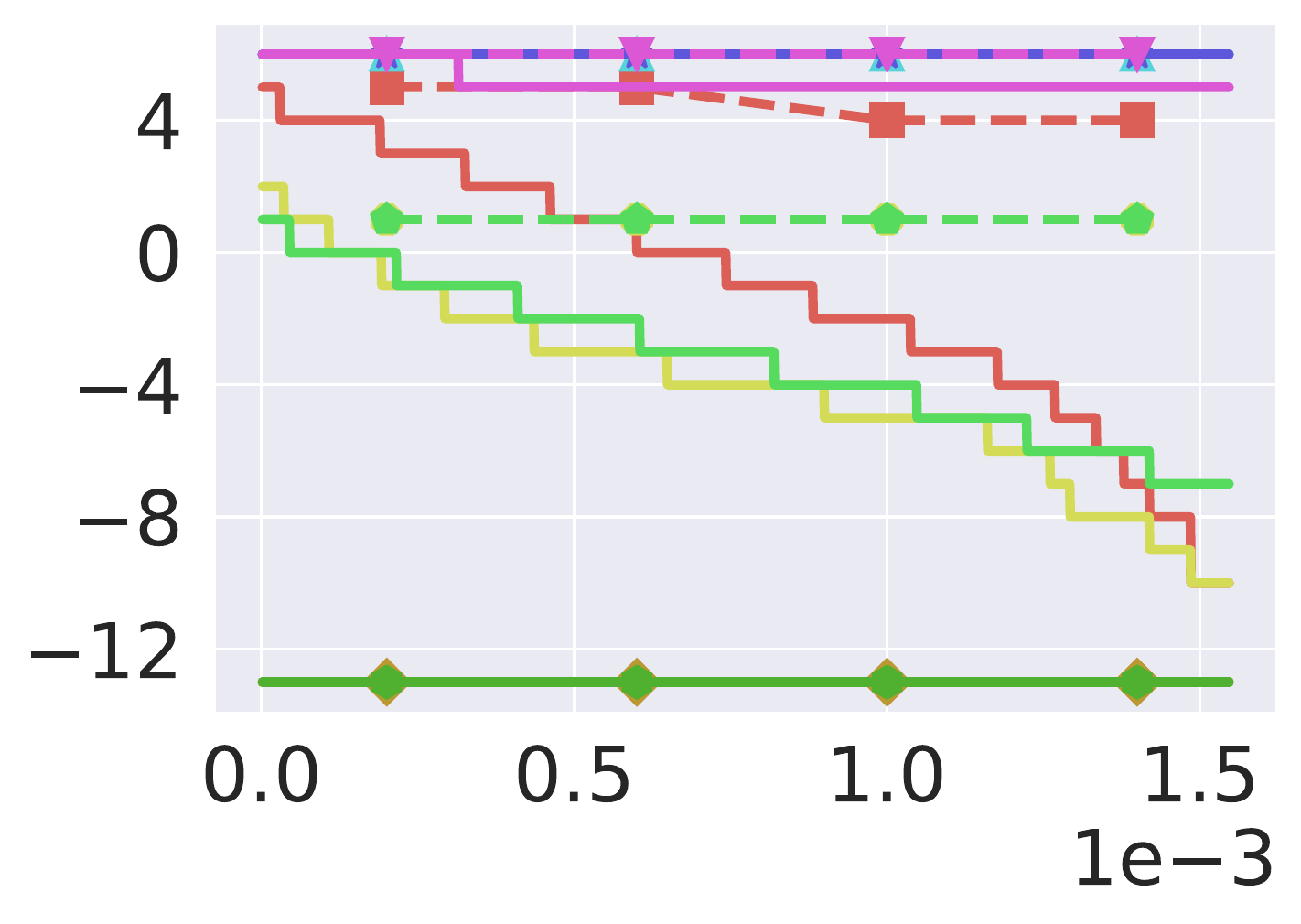}&
\includegraphics[height=\utilheightdp]{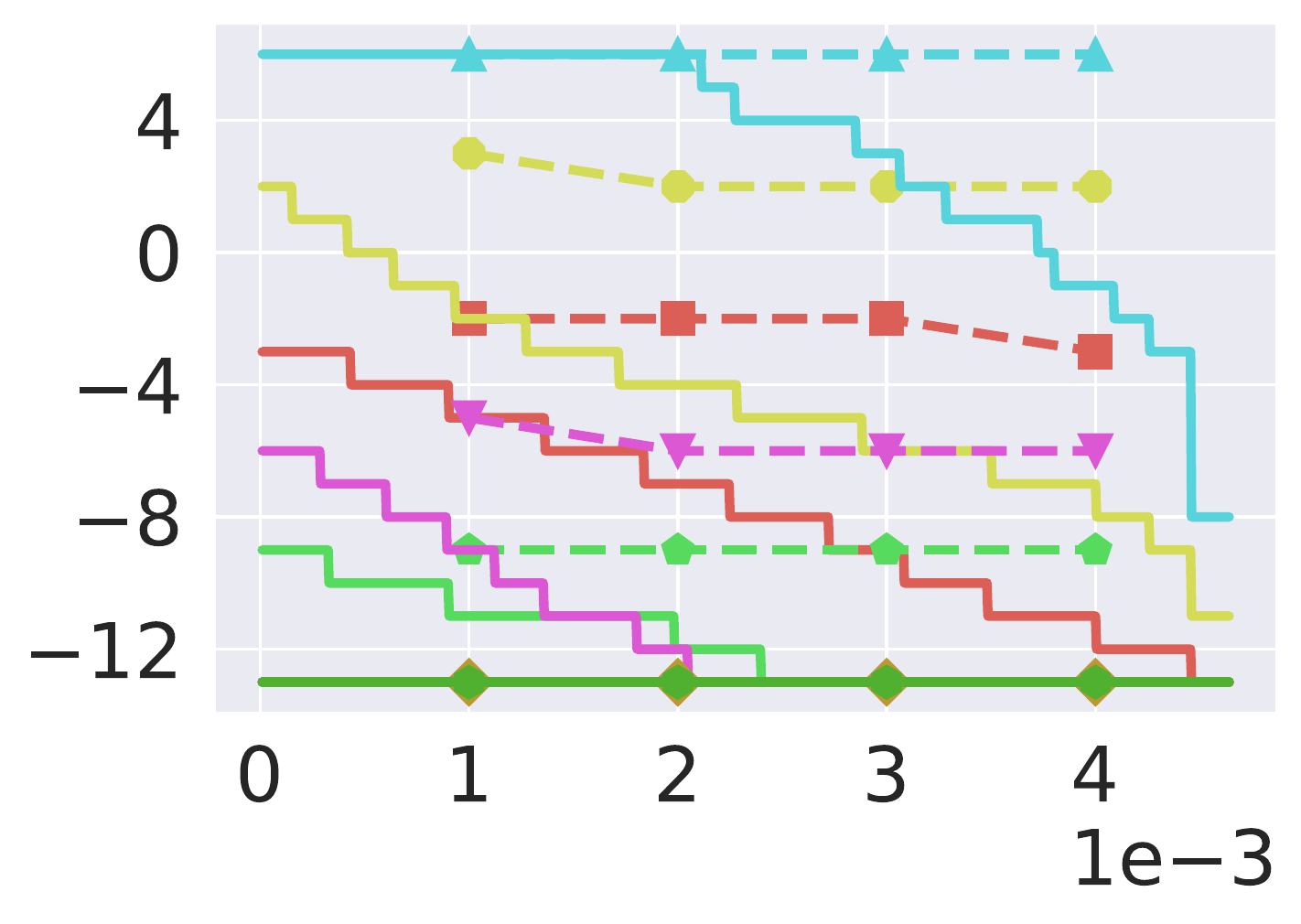}&
\includegraphics[height=\utilheightdp]{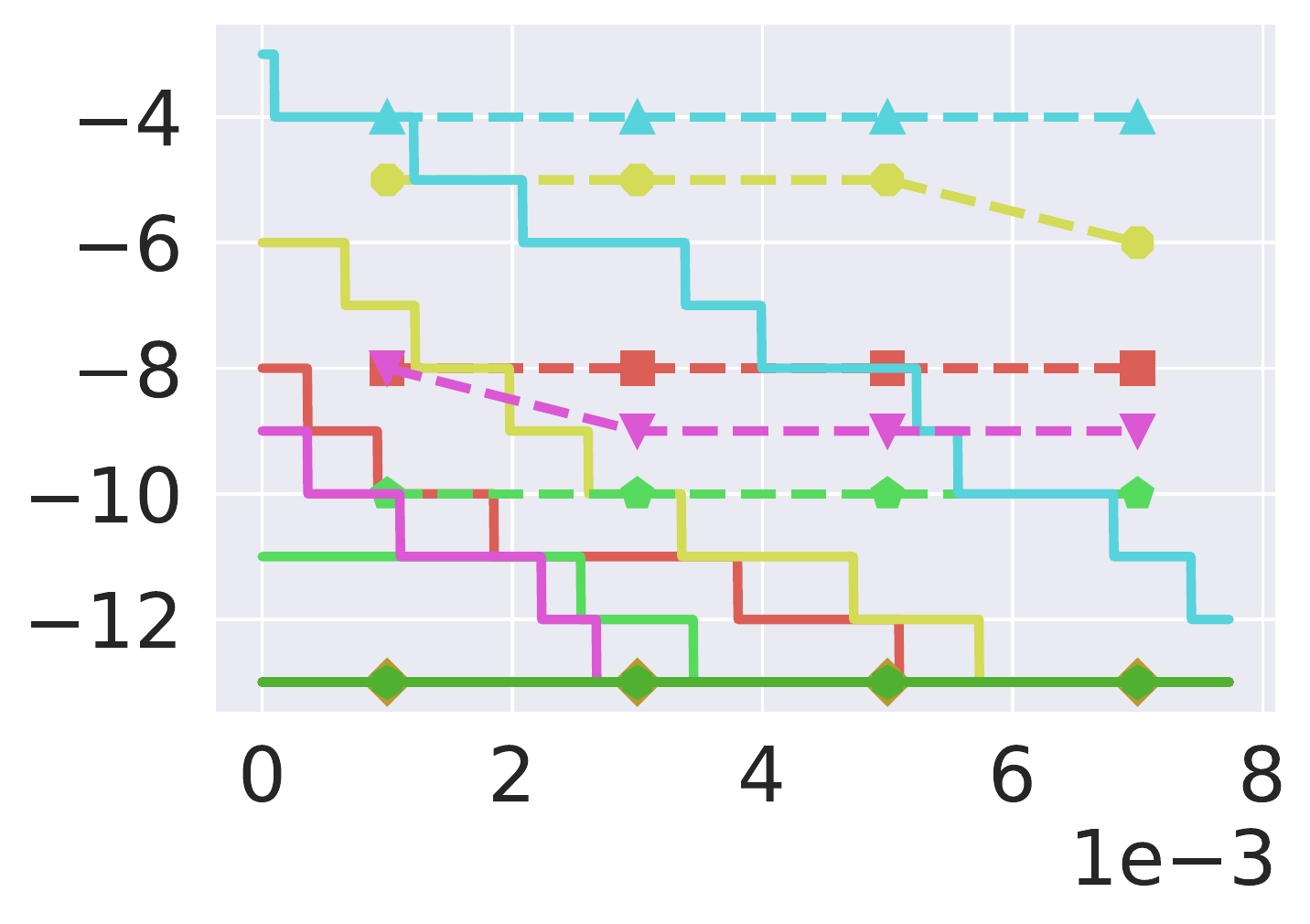}&
\includegraphics[height=\utilheightdp]{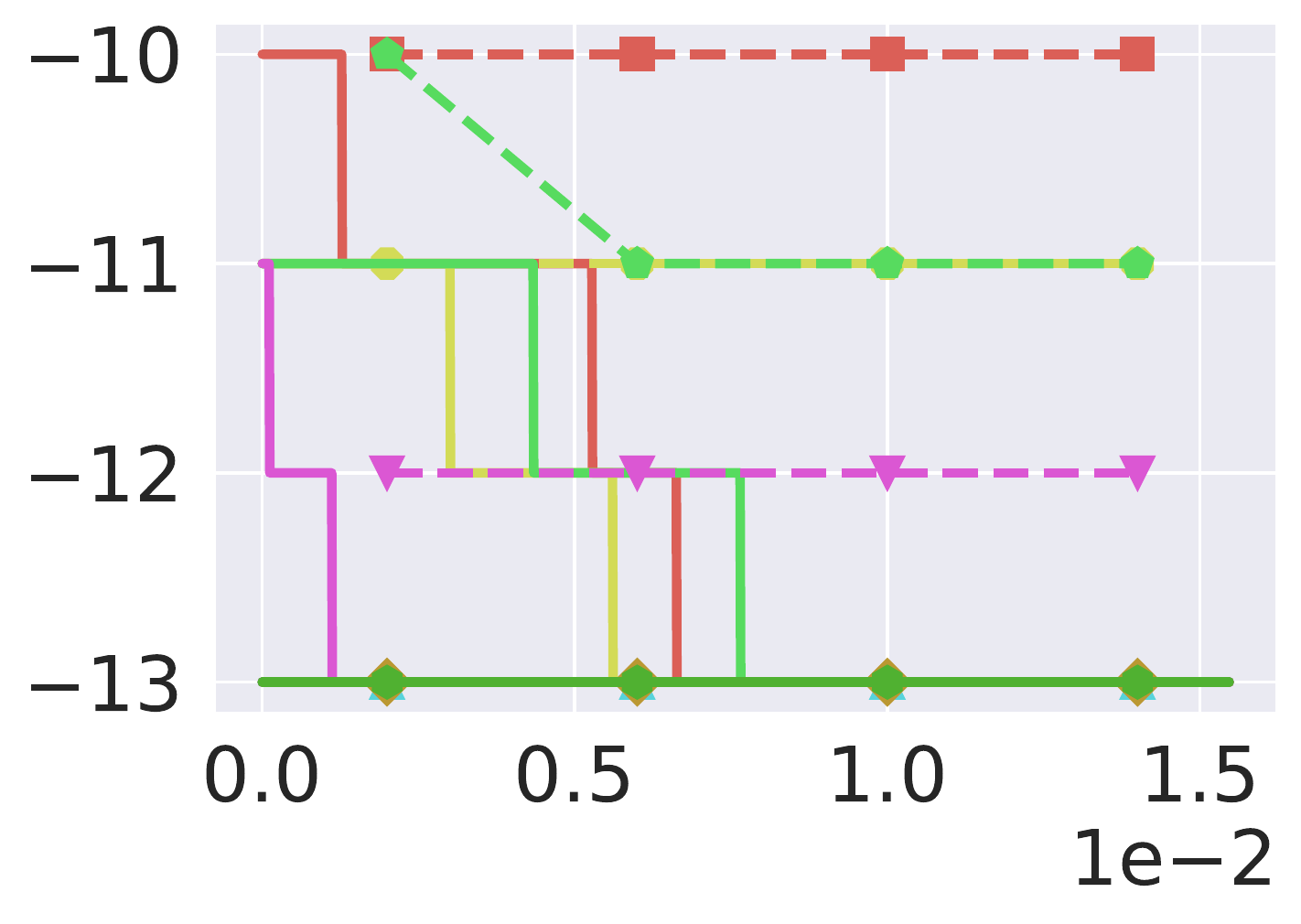}\\[-1.2ex]
\rownamed{\makecell{(Local) $\uj$}}&
\includegraphics[height=\utilheightaap]{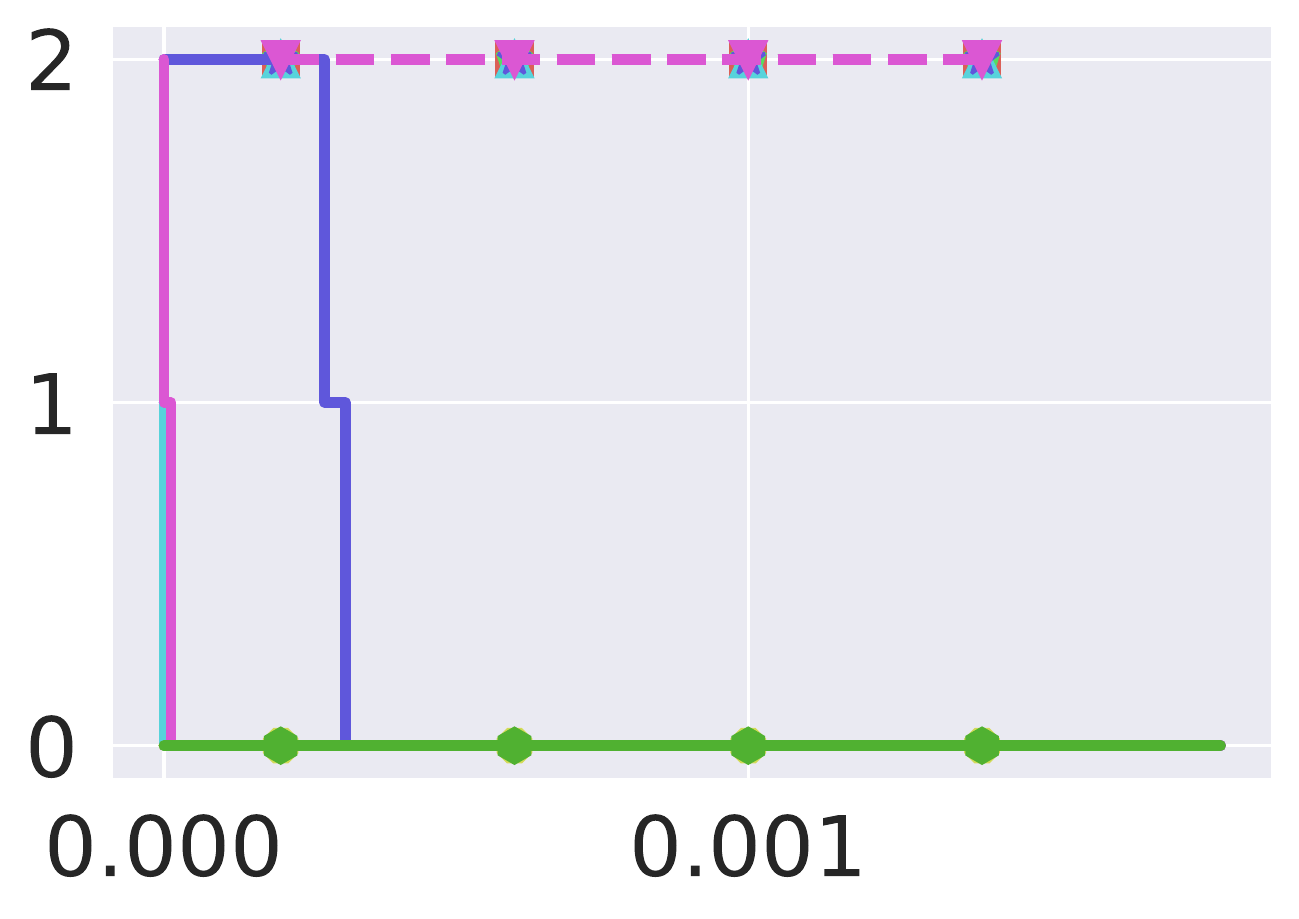}&
\includegraphics[height=\utilheightaap]{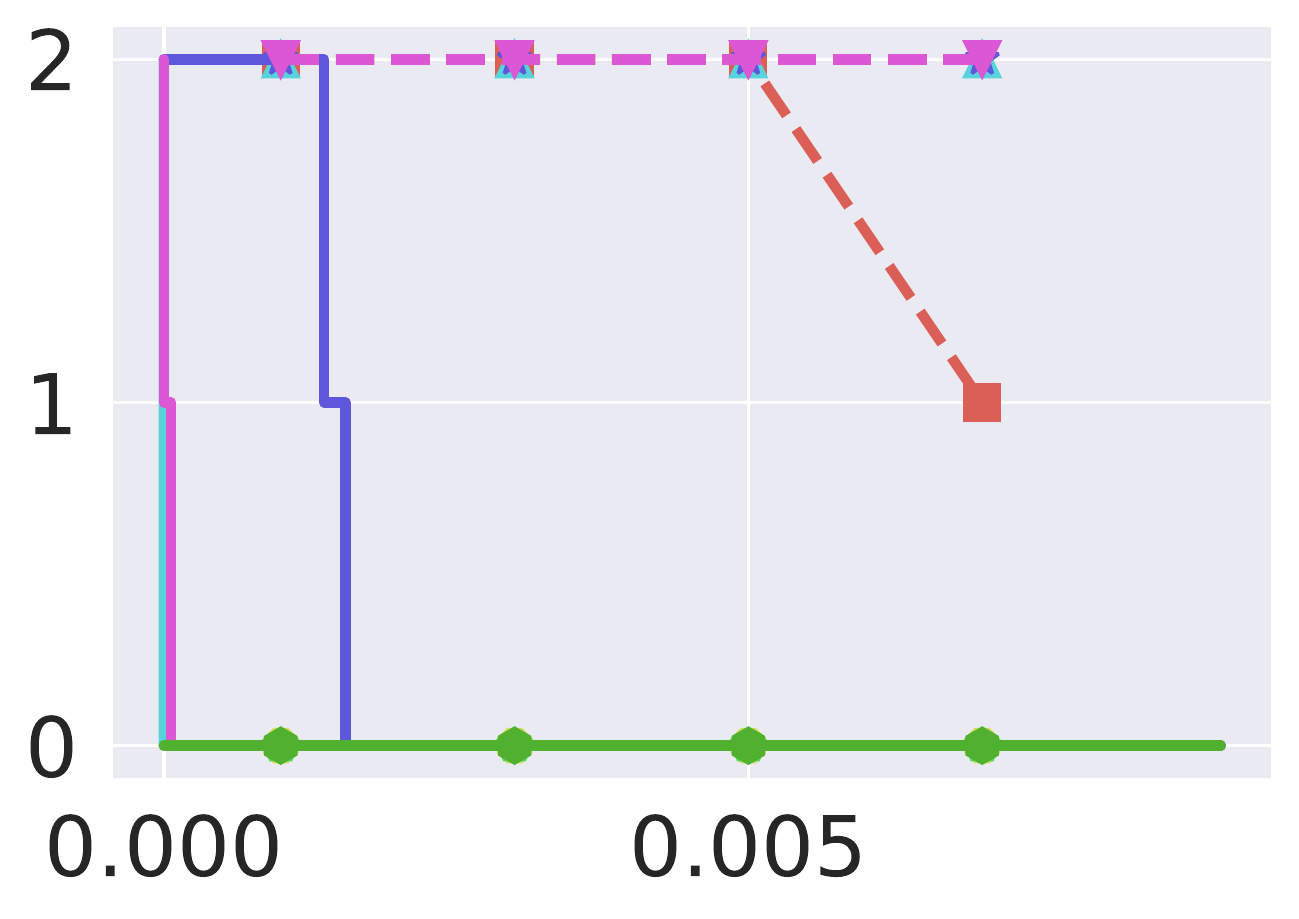}&
\includegraphics[height=\utilheightaap]{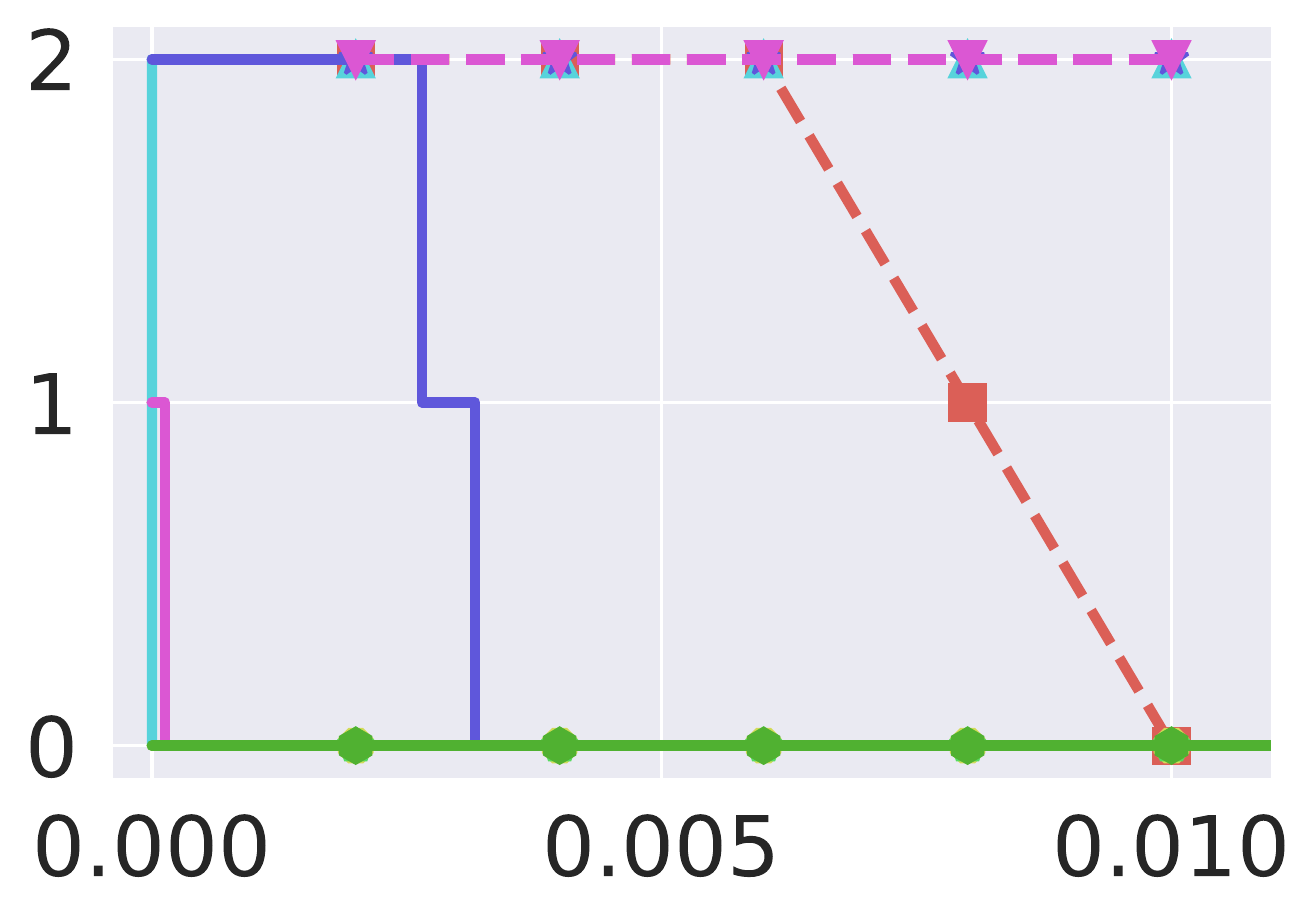}&
\includegraphics[height=\utilheightaap]{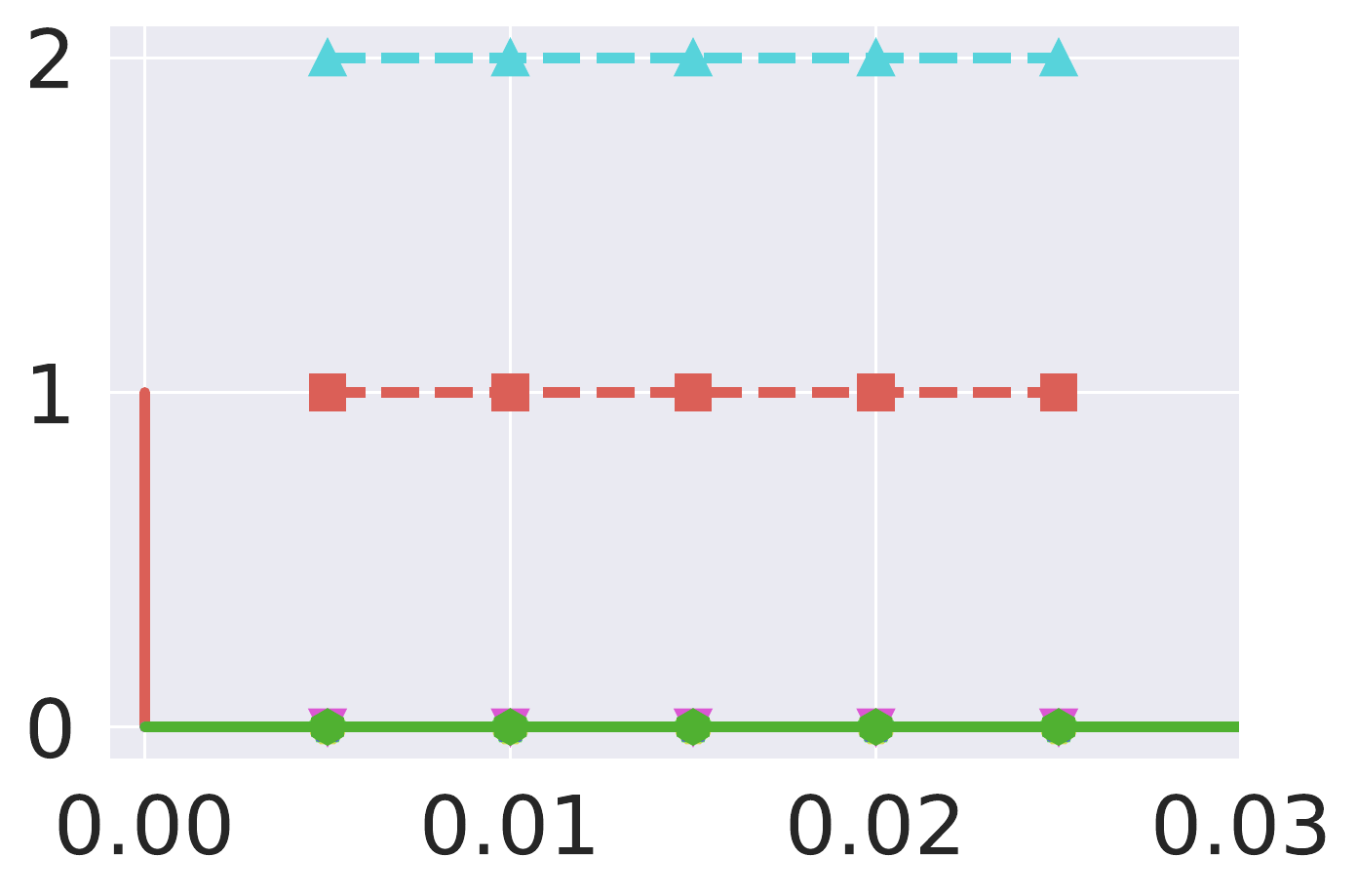}&
\includegraphics[height=\utilheightaap]{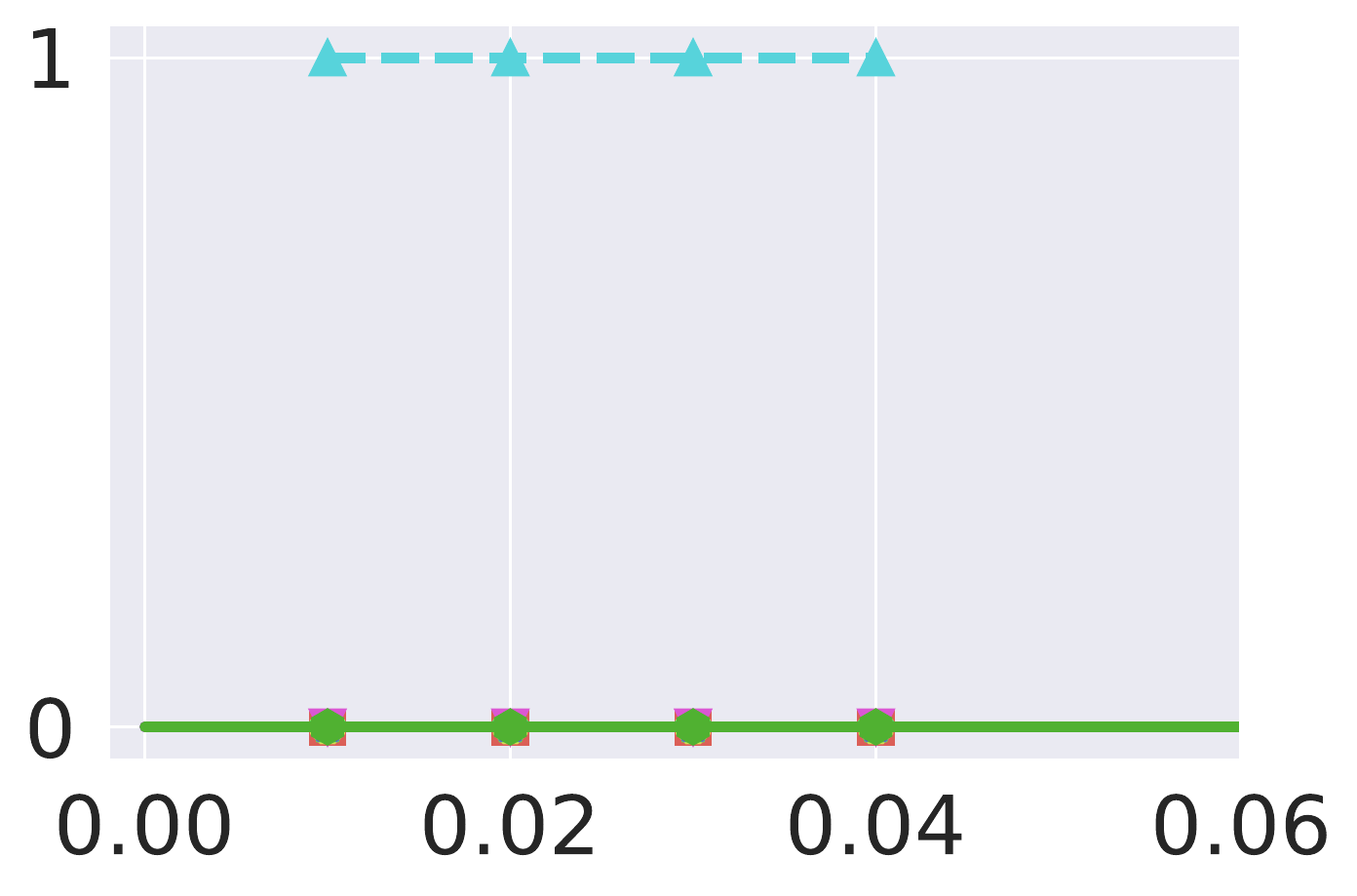}&
\includegraphics[height=\utilheightaap]{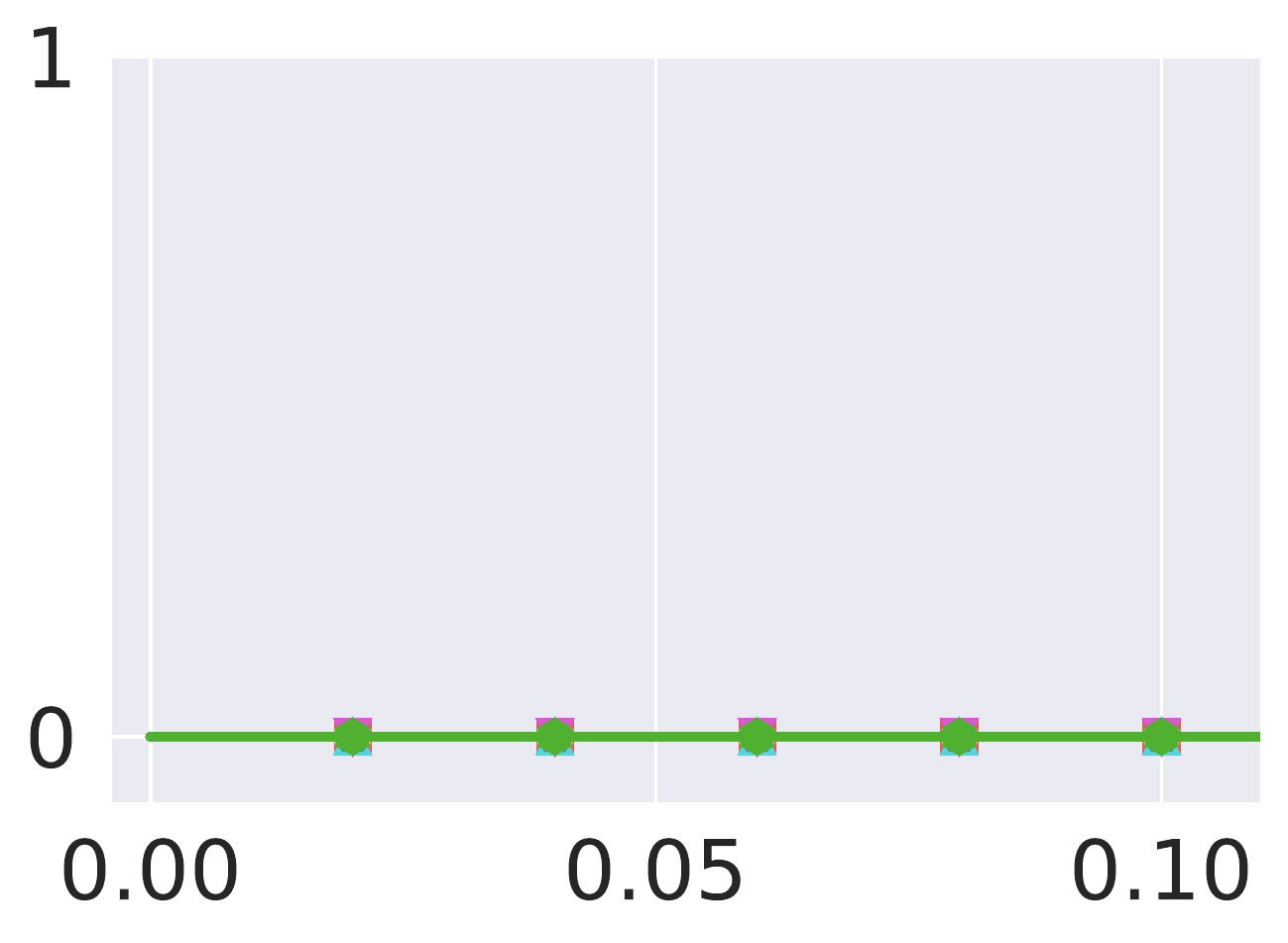}\\[-1.2ex]
        & \makecell{\tiny{attack $\eps$}}
        & \makecell{\tiny{attack $\eps$}}
        & \makecell{\tiny{attack $\eps$}}
        & \makecell{\tiny{attack $\eps$}}
        & \makecell{\tiny{attack $\eps$}}
        & \makecell{\tiny{attack $\eps$}}
\end{tabular}
}
\vspace{-3mm}
\caption{\small Pong}\label{tab:bound-pong}
\end{subtable}

}
\vspace{-3mm}
\caption{\small Robustness certification for cumulative reward, including \textit{expectation bound} $\uje$, \textit{percentile bound} $\ujp$ ($p=50\%$), and \textit{absolute lower bound} $\uj$. Each column corresponds to one smoothing variance. Solid lines represent the certified reward bounds, and dashed lines show the empirical performance under PGD.
% shows the change of the lower bounds w.r.t. the attack magnitude $\eps$.
% The solid curves represent the certified lower bounds, while the dotted lines represent the empirical values.
}%
\label{fig:all-bounds}
\vspace{-3em}
\end{figure}
}

\textbf{Tightness of the certification $\uje$, $\ujp$, and $\uj$.}\quad
We compare the empirical cumulative rewards achieved under PGD attacks with our certified lower bounds.
First, the empirical results are consistently lower bounded by our certifications, validating the correctness of our bounds.
Regarding the tightness, the improved $\ujp$ is much tighter than the loose $\uje$, supported by discussions in~\Cref{sec:cert-glb}.
The tightness of $\uj$ can be reflected by the zero gap between the certification and the empirical result under a wide range of attack magnitude $\eps$.
% The absolute lower bound $\uj$ is also tight---the certification and the empirical result has zero gap under a wide range of attack magnitude $\eps$.
Furthermore, the certification 
for SA-MDP (CVX,PGD) are quite tight for Freeway under large attack magnitudes, and RadialRL demonstrates its superiority on Freeway.
Additionally, different methods may achieve the same empirical results under attack, yet their certifications differ tremendously, indicating the importance of the robustness \textit{certification}.
% in addition to the \textit{empirical} evaluation. 
% This further demonstrates the significance of our certification in reliably distinguishing the robustness of different methods.

\vspace{-1mm}
\textbf{Game Properties.}\quad
% Comparing the certification of \textit{Freeway} and \textit{Pong}, we see that t
The certified robustness of any method on Freeway is much higher than that on Pong, indicating that Freeway is a more stable game than Pong, also shown  in~\citet{mnih2015human}.

\vspace{-2mm}
\section{Related Work}
\label{sec:related}
\vspace{-2mm}

% \noindent\textbf{Empirically Robust \RL}fd
We briefly review several RL methods that demonstrate empirical  robustness.
% , representatives of which have been evaluated in this paper.
% \textit{Randomization methods}~\citep{tobin2017domain,akkaya2019solving} were first proposed to encourage exploration. 
% through large-scale exploration
% This type of method was later systematically studied for its potential to improve model robustness.
% NoisyNet~\citep{fortunato2017noisy} adds parametric noise to the network's weight during training, providing better resilience to both training-time and test-time attacks~\citep{behzadan2017whatever,behzadan2018mitigation}.
% also reducing the transferability of adversarial examples.
% and enabling quicker recovery with fewer number of transitions during phase transition.
\citet{kos2017delving} and \citet{behzadan2017whatever} show that \textit{adversarial training} can increase the agent's resilience.
% Under the \textit{adversarial training} framework,
% \citet{kos2017delving} and \citet{behzadan2017whatever} show that re-training with random noise and FGSM perturbations increases the agent's resilience.
% \citet{pattanaik2018robust} leverage attacks using an engineered loss function specifically designed for RL.
% to significant increase the robustness to parameter variations.
% RS-DQN~\cite{fischer2019online} in an \textit{imitation learning} based approach that trains a robust student-DQN in parallel with a standard DQN in order to incorporate the constrains such as SOTA adversarial defenses~\cite{madry2017towards,mirman2018differentiable}. 
% We did not evaluate on their method due to no access to open-source code and missing details for reproducing the experiments.
SA-DQN~\citep{zhang2021robust} leverages \textit{regularization} to
% adds regularizers to the loss function for training the Q-networks to 
encourage the top-$1$ action to stay unchanged under perturbation.
Radial-RL~\citep{oikarinen2020robust} minimizes an adversarial loss function that incorporates the upper bound of the perturbed loss.
% , computed using certified bounds from verification algorithms.
% ~\citep{gowal2018effectiveness,weng2018towards}.
CARRL~\citep{everett2021certifiable} computes the  lower bounds of Q values under perturbation for action selection, but it is only suitable for low-dimensional environments.
% and selects actions according to the worst-case bound, but it relies on linear bounds~\citep{weng2018towards} and is only suitable for low-dimensional environments.
Most of these works only provide empirical robustness, with~\citet{zhang2021robust} and~\citet{fischer2019online} additionally provide robustness certificates at state level.
To the best of our knowledge, this is the \emph{first} work providing robustness certification for the cumulative reward of RL methods.
Broader discussions and comparisons of related work are in~\Cref{append:related}.

\textbf{Conclusions.}\quad
To provide the robustness certification for RL methods, we propose a general framework \sysname, aiming to provide certification based on two criteria. Our evaluations show that certain empirically robust RL methods are certifiably robust for specific RL environments.

\subsubsection*{Acknowledgments}
    This work is partially supported by the NSF grant No.1910100, NSF CNS 20-46726 CAR, Alfred P. Sloan Fellowship, and Amazon Research Award.

\textbf{Ethics Statement.}\quad
In this paper, we prove the first framework for certifying the robustness of RL algorithms against evasion attacks. 
We aim to certify and therefore potentially improve the trustworthiness of machine learning models, and we do not expect any ethics issues raised by our work.

\textbf{Reproducibility Statement.}\quad
All theorem statements are substantiated with rigorous proofs in our Appendix. We have upload the source code as the supplementary material for reproducibility purpose.

% \begin{NoHyper}
\bibliography{iclr2022_conference}
% \end{NoHyper}
\bibliographystyle{iclr2022_conference}

\clearpage

\appendix
\onecolumn

% \begin{center}
% {\Large Appendix}
% \end{center}

% \bo{add summary here}

\appendix

The appendices are organized as follows:
\begin{itemize}
    \item
In \Cref{append:proofs}, we present the detailed proofs for lemmas and theorems in~\Cref{sec:cert-rad} and~\Cref{sec:cert-cum}, laying down the basis for the design of our certification strategies.
    \item
In \Cref{append:methods}, we present some additional details of our three certification strategies: \staters to certify the per-state action, and \glbrs and \adasearch to certify the cumulative reward. We include the detailed algorithm description and the complete pseudocode for each algorithm.
Our implementation is publicly available at \url{https://github.com/AI-secure/CROP}.
% We also explain the inspiration and implication of the algorithms.
    \item
In \Cref{append:discuss-cert}, we provide a discussion on the \textit{advantages} and \textit{limitations} of different certification methods, and also present several promising direct \textit{extensions} for future work.
We also provide more \textit{detailed analysis} that help with the understanding of our algorithms.
    \item
In \Cref{append:exp}, we show the experimental details, including the game environment, details descriptions of the RL methods we evaluate, rationales for our evaluation settings and experimental designs, detailed evaluation setup for each of our certification algorithms in the main paper, as well as the setup for a new game CartPole.
    \item
\looseness=-1
In \Cref{append:results}, we present additional evaluation results and discussions for the environments and algorithms evaluated in the main paper, from the perspectives of different metrics and different ranges of parameters. We also provide the running time statistics and specifically illustrate the periodic patterns in the Pong game.
In addition, we present the evaluation results on a standard control \env CartPole and another autonomous driving \env Highway.
    \item
In \Cref{append:related}, we provide a broader discussion of the related works, spanning the evasion attacks in RL, Robust RL, as well as robustness certification for RL. 
\end{itemize}

\section{Proofs}
\label{append:proofs}

\subsection{Proof of \texorpdfstring{\Cref{lem:lipschitz}}{Lemma 1}}
\label{append:proof-lip-local}

We recall~\Cref{lem:lipschitz}:
\lipcontq*

    \begin{proof}
        To prove~\Cref{lem:lipschitz}, we leverage the technique in the proof for Lemma 1 of~\citet{salman2019provably} in their Appendix A.
        
        For each action $a\in\cal A$, our smoothed value function is 
    \vspace{8mm}
    \begin{small}
    \begin{align*}
        \wtilde Q^\pi(s,a)
        % := (Q^\pi\circ \cal N(0,\sigma^2 I))(s,a) 
        := \underset{\Delta\sim \gN(0,\sigma^2I_N)}{\E}\qpi(s+\Delta,a)
        =\frac{1}{(2 \pi)^{\nicefrac{N}{2}}\sigma^N} \int_{\mathbb{R}^{n}} \qpi(t,a) \exp \left(-\frac{1}{2\sigma^2}\|s-t\|^{2}\right) \,d t.
    \end{align*}
    \end{small}
    \vspace{5mm}
    
    Taking the gradient w.r.t. $s$, we obtain
    \vspace{8mm}
    \begin{small}
    \begin{align*}
        \nabla_s \wqpi(s,a) 
        = \frac{1}{(2 \pi)^{\nicefrac{N}{2}}\sigma^N} \int_{\mathbb{R}^{n}} \qpi(t,a) \frac{1}{\sigma^2}(s-t)\exp \left(-\frac{1}{2\sigma^2}\|s-t\|^{2}\right) \,d t.
    \end{align*}
    \end{small}
    \vspace{5mm}
    
    For any unit direction $u$, we have 
    \vspace{8mm}
    \begin{small}
    \begin{align*}
        u \cdot \nabla_s \wqpi(s,a) 
        & \leq \frac{1}{(2 \pi)^{\nicefrac{N}{2}}\sigma^N} \int_{\mathbb{R}^{n}}  \frac{\vmax-\vmin}{\sigma^2}|u(s-t)|\exp \left(-\frac{1}{2\sigma^2}\|s-t\|^{2}\right) \,d t\\
        & = \frac{\vmax-\vmin}{\sigma^2} \cdot  \int_{\mathbb{R}^{n}}  \frac{1}{(2 \pi)^{\nicefrac{1}{2}}\sigma} |s_i-t| \exp \left(-\frac{1}{2\sigma^2}|s_i-t|^{2}\right)\,dt \\
        &\quad \cdot \prod_{j\neq i}  \int_{\mathbb{R}^{n}} \frac{1}{(2 \pi)^{\nicefrac{1}{2}}\sigma} \exp \left(-\frac{1}{2\sigma^2}|s_j-t|^{2}\right)\,dt \\
        &= \frac{ V_{\max} - V_{\min} }{\sigma}\sqrt{\frac{2}{\pi}}
    \end{align*}
    \end{small}
    \vspace{5mm}
    
    Thus, $\wqpi$ is $L$-\lip continuous with $L=\frac{ V_{\max} - V_{\min} }{\sigma}\sqrt{\nicefrac{2}{\pi}}$ w.r.t. the state input.
    \end{proof}

\subsection{Proof of \texorpdfstring{\Cref{thm:radius}}{Theorem 1}}
\label{append:proof-radius}

We recall~\Cref{thm:radius}:
% \thmrad*

\begin{theoremappend}
\looseness=-1 Let $Q^\pi: \cal S \times \cal A\rightarrow [\vmin,\vmax]$ be a trained value network, 
$\wtilde Q^\pi$ be the smoothed function with (\ref{eq:smooth-q}).
% \linyi{\textbf{[old]}: At time step $t$, we can certify a radius $r_t$ for state $s_t$, such that $r_t = \bar\eps(s_t)$, \ie, 
% the induced action $\tpi(s_t+\delta_t)$ under any state perturbation $\delta_t$ will be the same as the action given clean state observation $\tpi(s_t)$ as long as $\eps < r_t$.
% The radius $r_t$ can be calculated as below:}
At time step $t$ with state $s_t$, we can compute the lower bound $r_t$ of maximum perturbation magnitude $\bar\eps(s_t)$~(\ie, $r_t \le \bar\eps(s_t)$, $\bar\eps$ defined in~\Cref{def:per-state}) for locally smoothed policy $\tilde \pi$:
% \begin{align*}
%     R(s)
%     &:= \frac{1}{2}\left( \eta_{a_1}(s) - \eta_{a_2}(s) \right),
% \end{align*}
\vspace{7mm}
\begin{small}
\begin{align}
        r_t = \frac{\sigma}{2}\left(
    \Phi^{-1}\left(\frac{\wtilde Q^\pi(s_t,a_1)-V_{\min}}{V_{\max}-V_{\min}}\right)-
    \Phi^{-1}\left(\frac{\wtilde Q^\pi(s_t,a_2)-V_{\min}}{V_{\max}-V_{\min}}\right)
    \right), \label{eq:R-t}
\end{align}
\end{small}
\vspace{5mm}

where $\Phi^{-1}$ is the inverse CDF function, $a_1$ is the action with the highest $\wqpi$ value at state $s_t$, and $a_2$ is the runner-up  action. We name the lower bound $r_t$ as certified radius for the state $s_t$.
\end{theoremappend}

We first present a lemma that can help in the proof of~\Cref{thm:radius}.
\begin{lemma}
\label{lem:helper}
    Let $\Phi$ be the CDF of a standard normal distribution,
    the mapping
    $\eta_a(s):=\sigma \cdot \Phi^{-1}\left(\frac{\wtilde Q^\pi(s,a)-V_{\min}}{V_{\max}-V_{\min}}\right)$ is $1$-\lip continuous.
\end{lemma}
The lemma can be proved following the same technique as the proof for~\Cref{lem:lipschitz} in~\Cref{append:proof-lip-local}. The detailed proof can be referred to in the proof for Lemma 2 of~\citet{salman2019provably} in their Appendix A.
We next show how to leverage~\Cref{lem:helper} to prove~\Cref{thm:radius}.

    \begin{proof}[Proof for~\Cref{thm:radius}]
        Let the perturbation be $\delta_t$, 
        based on the \lip continuity of the mapping $\eta$, we have
        \vspace{3mm}
        
        \begin{align}
            \label{eq:proof-sum-1}
            \eta_{a_1}(s_t) - \eta_{a_1}(s_t+\delta_t) \leq \norm{\delta_t}_2,\\
            \label{eq:proof-sum-2}
            \eta_{a_2}(s_t+\delta_t) - \eta_{a_2}(s_t) \leq \norm{\delta_t}_2.
        \end{align}
        \vspace{3mm}
        
        Suppose that under perturbation $\delta_t$, the action selection would be misled in the sense that the smoothed value for the original action $a_1$ is lower than that of another action $a_2$, \ie, $\wqpi(s_t+\delta_t,a_1)\leq \wqpi(s_t+\delta_t,a_2)$.
        Then, based on the monotonicity of $\eta$, we have
        \vspace{3mm}
        
        \begin{align}
            \label{eq:proof-sum-3}
            \eta_{a_1}(s_t+\delta) \leq \eta_{a_2}(s_t+\delta).
        \end{align}
        \vspace{3mm}
        
        Summing up (\ref{eq:proof-sum-1}), (\ref{eq:proof-sum-2}), and (\ref{eq:proof-sum-3}), we obtain
        \vspace{3mm}
        
        \begin{align*}
            \norm{\delta_t}_2 
            &\geq \frac{1}{2}\left(\eta_{a_1}(s_t) -  \eta_{a_2}(s_t) \right)\\
            &= \frac{\sigma}{2}\left(
    \Phi^{-1}\left(\frac{\wtilde Q^\pi(s_t,a_1)-V_{\min}}{V_{\max}-V_{\min}}\right)-
    \Phi^{-1}\left(\frac{\wtilde Q^\pi(s_t,a_2)-V_{\min}}{V_{\max}-V_{\min}}\right)
    \right)
        \end{align*}
        \vspace{5mm}
        
        which is a lower bound of the \textit{maximum perturbation magnitude} $\bar\eps (s_t)$ that can be tolerated at state $s_t$.
        Hence, when $r_t$ takes the value of the computed lower bound, it satisfies the condition that $r_t \leq \bar\eps (s_t)$.
    \end{proof}

\subsection{Proof of \texorpdfstring{\Cref{lem:global-lip}}{Lemma 2}}
\label{append:proof-global-lip}
We recall~\Cref{lem:global-lip}:

\lemlipf*

\begin{proof}
    The proof can be done in a similar fashion as the proof for~\Cref{lem:lipschitz} in~\Cref{append:proof-lip-local}.
    Compared with the smoothed value network where the expectation is taken over the sampled states, here, similarly, the smoothed \namef is derived by taking the expectation over sampled \sigrtjs. 
    The difference is that the output range of the Q-network is $[\vmin,\vmax]$, while the output range of the \namef is $[\jmin,\jmax]$. 
    Thus, the smoothed \namef $\wtilde F$ is $\frac{(J_{\max}-J_{\min})}{\sigma}\sqrt{\nicefrac{2}{\pi}}$-\lip continuous.
\end{proof}

\subsection{Proof of \texorpdfstring{\Cref{thm:exp-bound}}{Theorem 2}}
\label{append:proof-exp-bound}

We recall the definition of $\wtilde F_\pi$ as well as ~\Cref{thm:exp-bound}:
\vspace{8mm}
\begin{small}
\begin{align*}
    \wtilde F_\pi\left(\vert_{t=0}^{H-1}\delta_t\right)
              := \underset{\zeta \sim \cal N (0,\sigma^2I_{H\times N})}{\E} F_\pi\left(\vert_{t=0}^{H-1}(\delta_t+\zeta_t)\right).
\end{align*}
\end{small}
\vspace{5mm}

\expbound*

        \begin{proof}
        
            We note the following equality
            % \begin{small}
            % \begin{align}
            % \label{eq:long-eq}
            %      \wtilde F(\delta_0\vert\cdots\vert\delta_{H-1}, \pi)
            % = \E \left[F(\delta_0+\Delta_0\vert\cdots\vert\delta_{H-1}+\Delta_{H-1}, \pi)\right]
            % = \E \left[F(\delta_0\vert\cdots\vert\delta_{H-1}, \pi')\right]
            % = \bb E\left[J_\eps^H(\pi')\right],
            % \end{align}
            % \end{small}
            \vspace{8mm}
            \begin{small}
            \begin{align}
            \label{eq:long-eq}
                 \wtilde F_\pi(\vert_{t=0}^{H-1}\delta_t)
            \stackrel{(a)}{=} \E \left[F_\pi\left(\vertt{0}{H-1}(\delta_t+\zeta_t)\right)\right]
            \stackrel{(b)}{=} \E \left[F_{\pi'}\left(\vertt{0}{H-1}\delta_t\right)\right]
            \stackrel{(c)}{=} \bb E\left[J_\eps^H(\pi')\right],
            \end{align}
            \end{small}
            \vspace{5mm}
            
            where (a) comes from the definition of the smoothed \namef $\wtilde F_\pi$, (b) is due to the definition of the $\sigma$-randomized policy $\pi'$, and (c) arises from the definition of the \namepj $J_\eps^H$.
            Thus, the expected \namepj $\bb E\left[J_\eps^H(\pi')\right]$ is equivalent to the smoothed \namef $\wtilde F_\pi\left(\vertt{0}{H-1}\delta_t\right)$.
            Furthermore, since the distance between the all-zero $\vertt{0}{H-1}\mathbf{0}$ and the adversarial perturbations $\vertt{0}{H-1}\delta_t$ is bounded by $\eps\sqrt{H}$, leveraging the \lip smoothness of $\wtilde F$ in~\Cref{lem:global-lip}, we obtain the lower bound of the expected \namepj $\bb E\left[J_\eps^H(\pi')\right]$ as $\wtilde F_\pi\left(\vertt{0}{H-1}\mathbf{0}\right) - L\eps\sqrt{H}$.
        \end{proof}

\subsection{Proof of \texorpdfstring{\Cref{thm:perc-bound}}{Theorem 3}}
\label{append:proof-perc-bound}

We recall the definition of $\wtilde F_\pi^p$ as well as~\Cref{thm:perc-bound}:
        \vspace{8mm}
        \begin{small}
        \begin{align}
        \label{eq:wtilde-f-re}
            \wtilde F_{\pi}^p\left(\vertt{0}{H-1}\delta_t\right) := \mathrm{sup}_y\left\{ y\in \sR \mid \mathbb{P}\left[F_\pi\left(\vertt{0}{H-1}(\delta_t+\zeta_t)\right) \leq y\right] \leq p \right\}.
        \end{align}
        \end{small}
\vspace{5mm}        

\percbound*
    \begin{proof}
        To prove~\Cref{thm:perc-bound}, we leverage the technique in the proof for Lemma 2 of~\citet{chiang2020detection} in their Appendix B.
        
        For brevity, we abbreviate $\delta:=\vert_{t=0}^{H-1}\delta_t$ and redefine the plus operator such that  $\delta+\zeta:=\vert_{t=0}^{H-1}(\delta_t+\zeta_t)$.
        Then, similar to~\Cref{lem:helper}, we have the conclusion that
        \vspace{8mm}
        \begin{small}
        \begin{align*}
            \delta \mapsto \sigma \cdot \Phi^{-1} \left( \bb P\left[ F_\pi(\delta+\zeta) \leq \wtilde F_\pi^{p'}(\mathbf{0}) \right] \right)
        \end{align*}
        \end{small}
        \vspace{5mm}
        
        is $1$-\lip continuous, where $\zeta\sim \cal N(0,\sigma^2 I_{H\times N})$.
        
        Thus, under the perturbations $\delta_t\in\epsball$ for $t=0\ldots H-1$, we have
        \vspace{8mm}
        \begin{small}
        \begin{align*}
            \Phi^{-1}\left( \bb P\left[ F_\pi(\delta+\zeta) \leq \wtilde F_\pi^{p'}(\mathbf{0})\right] \right)
            &\leq 
            \Phi^{-1}\left( \bb P\left[ F_\pi(\zeta) \leq \wtilde F_\pi^{p'}(\mathbf{0})\right] \right) + \frac{\norm{\delta}_2}{\sigma}\\
            &\leq 
            \Phi^{-1}\left( \bb P\left[ F_\pi(\zeta) \leq \wtilde F_\pi^{p'}(\mathbf{0})\right] \right) + \frac{\eps\sqrt{H}}{\sigma}
            &\pushright{\textit{(Since $\norm{\delta}_2\leq \eps\sqrt{H}$)}}\\
            &= 
            \Phi^{-1}(p') + \frac{\eps\sqrt{H}}{\sigma}
            &\pushright{\textit{(By definition of $\wtilde F_\pi^p$)}}\\
            &= 
            \Phi^{-1}(p)
            &\pushright{\textit{(By definition of $p'$)}}.
        \end{align*}
        \end{small}
        \vspace{5mm}
        
        Since $\Phi^{-1}$ monotonically increase, this implies that 
        $
            \bb P\left[ F_\pi(\delta+\zeta) \leq \wtilde F_\pi^{p'}(\mathbf{0})\right] \leq p
        $.
        According to the definition of $\wtilde F_\pi^p$ in (\ref{eq:wtilde-f-re}), we see that $\wtilde F_\pi^{p'}(\mathbf{0})\leq \wtilde F_{\pi}^p\left(\delta\right)$, \ie, $\ujp \leq$ the $p$-th percentile of $J_\eps(\pi')$. Hence, the theorem is proved.
    \end{proof}

\subsection{Proof of \texorpdfstring{\Cref{thm:extend-radius}}{Theorem 4}}
\label{append:proof-extend-rad}

We recall~\Cref{thm:extend-radius}:

\thmextendrad*

    \begin{proof}
        Replacing $a_2$ with $a_{k+1}$ in the proof for~\Cref{thm:radius} in~\Cref{append:proof-radius} directly leads to~\Cref{thm:extend-radius}.
    \end{proof}

\section{Additional Details of Certification Strategies}
\label{append:methods}

In this section, we cover the concrete details regarding the implementation of our three certification strategies, as a complement to the high-level ideas introduced back in~\Cref{sec:cert-rad} and~\Cref{sec:cert-cum}.

\subsection{Detailed Algorithm of \staters}
\label{append:staters}

\begin{algorithm}[H]
\algsetup{linenosize=\tiny}
  \footnotesize
% \KwIn{Q action-value function $Q$ with network parameter $\theta$, current state $s$, noise scale $\sigma$; parameter: nubmer of samples $n$}
\setlength{\columnsep}{15pt}
\begin{multicols}{2}
\KwIn{state $s$, trained value network $Q^\pi$ with range $[V_{\min},V_{\max}]$; parameters for smoothing: sampling times $m$, smoothing variance $\sigma^2$, one-sided confidence parameter $\alpha$}
\KwOut{smoothed value network $\wtilde Q^\pi$, selected action $a$, certification indicator $cert$, certified radius $r$ constant $L$}
\SetKwFunction{funca}{pre-processing}
\SetKwFunction{funcb}{take\_one\_step}
\SetKwProg{Pn}{Procedure}{:}{}
\SetKwProg{Fn}{Function}{:}{}
% \KwOut{action $a$}
\DontPrintSemicolon

%%%%%%%%%%%%%%%%%%%%%%%%%%%%%%%%%%%%%%%%%%%%%%%%%%%%%%%%%%%%%%%%%%
%%%%%%%%%%%%%%%%%%%%%%%%%%%%%%%%% for computing only r
%%%%%%%%%%%%%%%%%%%%%%%%%%%%%%%%%%%%%%%%%%%%%%%%%%%%%%%%%%%%%%%%%%
% \Pn{\funca{}} {
%     % \Comment*[l]{\footnotesize Preprocessing: given a trained $Q$ network, compute its output range and the \lip constant of its smoothed version}
%     Sample the input states and pass them to the given network $Q$ to compute the minimum value $V_{\min}$ and maximum value $V_{\max}$\;
%     $L=\frac{ V_{\max} - V_{\min} }{\sigma}\sqrt{\nicefrac{2}{\pi}}$\;
% }\;

% \Fn{\funcb{$s,\eps$}} {
    \Comment*[l]{\footnotesize Step 1: smoothing}
    Generate noise samples $\delta_i\sim\cal N(0,\sigma^2 I)$ for $1\leq i\leq m$\;
    \For {each action $a \in \cal A$} {\Comment*[r]{\scriptsize clipping and averaging}
        $\wtilde Q^\pi(s,a)\leftarrow \frac{1}{m}\sum_{i=1}^{m} \texttt{clip}( Q^\pi(s+\delta_i, a), \min=V_{\min}, \max=V_{\max})$
    }
    $a_1,a_2\leftarrow $ best action and runner-up action given by $\wqpi$\;
    
    \columnbreak
    
    \Comment*[l]{\footnotesize Step 2: certification}
    $\Delta =  \left( V_{\max} - V_{\min}\right) \sqrt{\frac{1}{2m} \ln \frac{1}{\alpha}}$ \\\Comment*[r]{\scriptsize confidence interval}
    
    \uIf {$\wqpi(s,a_1) \geq \wqpi(s,a_2) + 2\Delta$} {
        \Comment*[r]{\scriptsize certification success}
        $cert\leftarrow$ \texttt{True}\;
        $r \leftarrow \frac{\sigma}{2}(
                                            \Phi^{-1}(\frac{\wqpi(s,a_1) - \Delta -V_{\min}}{V_{\max}-V_{\min}})
                                            -
                                            \Phi^{-1}(\frac{\wqpi(s,a_2) + \Delta -V_{\min}}{V_{\max}-V_{\min}})
                                      )$\;
    }
    \Else {
        \Comment*[r]{\scriptsize certification failure}
        $cert\leftarrow$ \texttt{False}\;
        $r \leftarrow $ \texttt{undefined}\;
    }
    \KwRet $\wqpi$, $a_1$, $cert$, $r$\;
% }
\end{multicols}
\caption{\small \staters: Local smoothing for certifying per-state action
}\label{alg:smooth-q}
\end{algorithm}

We present the concrete algorithm of \staters in~\Cref{alg:smooth-q} for the procedures introduced in~\Cref{sec:cert-rad-algo}.
For each given state $s_t$, we first perform Monte Carlo sampling~\citep{cohen2019certified,lecuyer2019certified} to achieve local smoothing. 
Based on the smoothed value function $\wqpi$, we then compute the robustness certiﬁcation for per-state action, \ie, the certified radius $r_t$ at the given state $s_t$, following~\Cref{thm:radius}.

\textbf{Detailed Inference Procedure.}\quad
During inference, we invoke the model with $m$ samples of Gaussian noise at each time step, and use the averaged Q value on these $m$ noisy samples to obtain the greedy action selection and compute the certification.
This procedure is similar to \textsc{Predict} in~\citet{cohen2019certified}, but since RL involves multiple step decisions, we do not take the ``abstain'' decision as in \textsc{Predict} in~\citet{cohen2019certified}; instead, \staters will take the greedy action at all steps no matter whether the action can be certified or not.

\textbf{Estimation of the Algorithm Parameters $\vmin,\vmax$.}\quad
Our estimate is obtained via sampling the trajectories and calculating the Q values associated with the state-action pairs along these trajectories. Since we perform clipping using the obtained bounds, the certification is sound. The cost for estimating $V_{\rm min}$ and $V_{\rm max}$ is essentially associated with the number of sampled trajectories, the number of steps in each trajectory, the number of sampled Gaussian noise $m$ per step, and the cost of doing one forward pass of the Q network; thus, the estimation can be done with low cost. 
% not high given that they can be directly obtained via sampling as concretely described above. 
Furthermore, we can balance the trade-off between the accuracy and efficiency of the estimation, and for any configuration of $V_{\rm min}$ and $V_{\rm max}$, the certification will invariably be sound (reason above).

\subsection{Detailed Algorithm of \glbrs}
\label{append:glbrs}
\begin{algorithm}[t]
\algsetup{linenosize=\tiny}
\footnotesize
\DontPrintSemicolon
\setlength{\columnsep}{15pt}
\begin{multicols}{2}
\KwIn{initial state distribution $d_0$, trained value network $\qpi$, game cumulative reward range $[\jmin,\jmax]$, number of steps in an episode $H$, perturbation magnitude at each state $\eps$, percentile $p$; parameters for smoothing: sampling times $m$, smoothing variance $\sigma^2$, one-sided confidence parameter $\alpha$}
\KwOut{Expectation bound $\uje$, $p$-th percentile bound $\ujp$}
\Comment*[l]{\footnotesize Step 1: smoothing}
\For {$i=1$ to $m$} {
    $s_0\sim d_0, J^C_i\leftarrow 0$ \Comment*[r]{\scriptsize initialization}
    Generate macro-state noise $\delta_t\sim \cal N(0, \sigma^2 I)$ for $0\leq t < H$\;
    \For {$t=0$ to $H-1$} {
        $a_t\leftarrow \argmax_{a}Q(s_{t}+\delta_t, a)$\;
        Execute action $a_t$ and observe reward $re_t$ and next state $s_{t+1}$ \Comment*[r]{\scriptsize take a step}
        $J^C_i\leftarrow J^C_i + re_t$ \Comment*[r]{\scriptsize accumulate the reward}
    }
}

\Comment*[l]{\footnotesize Step 2.1: certifying expectation bound}
$\wtilde F\leftarrow \frac{1}{m}\sum_{i=1}^{m} J^C_i$ \Comment*[r]{\scriptsize smoothed actual reward}
$\Delta_{\rm{conf}} \leftarrow (R_{\max}-R_{\min})\sqrt{\frac{\ln(1/\alpha)}{2m}}$ \\\Comment*[r]{\scriptsize confidence interval}
$\Delta_{\rm{lip}} \leftarrow \frac{R_{\max}-R_{\min}}{\sigma}\sqrt{\frac{2}{\pi}} \cdot \eps \sqrt{H}$ \\\Comment*[r]{\scriptsize bound given by \lip continuity}
$\uje \leftarrow \wtilde F - \Delta_{\rm{conf}} - \Delta_{\rm{lip}}$\;

\;
\Comment*[l]{\footnotesize Step 2.2: certifying percentile bound}
$k \leftarrow \texttc{ComputeOrderStats}(\eps,\sigma,p,m,H)$\;
$\ujp \leftarrow $ $k$-th smallest value in $\{J^C_i\}_{i=1}^{m}$\;
\;
\KwRet $\uje, \ujp$
\end{multicols}
\vspace{1em}
\caption{\small \glbrs: Global smoothing for certifying cumulative reward}\label{alg:global-rand}
\end{algorithm}
    We present the concrete algorithm of \glbrs in~\Cref{alg:global-rand} for the procedures introduced in~\Cref{sec:cert-glb}.
    Similarly to \staters, the algorithm also consists of two parts: performing smoothing and computing certification, where we compute both the \ebound $\uje$ and the \pbound $\ujp$.

    % As a result, we can sample to estimate $\tilde F(0,\cdots,0) = J(\pi')$~(using Hoeffding's Inequality to get a high confidence bound).
    % Then, we can leverage the smoothness property of $\tilde F$~\cite{salman2019provably,chiang2020detection} to obtain a lower bound of $\tilde F(\delta_1', \delta_2', \dots,\delta_H')$, which is a valid certification of the cumulative reward for the randomized policy $\pi'$.

        % For a given game, we assume its reward range $[\jmin,\jmax]$ is known a priori.

        \textbf{Step 1: Global smoothing.}\quad
        We adopt Monte Carlo sampling~\citep{cohen2019certified,lecuyer2019certified} to estimate the smoothed \namef $\wtilde F$ by sampling multiple \sigrtjs via drawing $m$ noise sequences. For each noise sequence $\zeta\sim \gN(0,\sigma^2I_{H\times N})$,
        we apply noise $\zeta_t$ to the input state $s_t$  sequentially,
        and obtain the sum of the reward $J_i^C=\sum_{t=0}^{H-1}re_t$ as the return for this \sigrtj.
        We then aggregate the smoothed perturbed return values 
        % compute the two bounds below based on
        $\{J_i^C\}_{i=1}^{m}$ 
        via mean smoothing and percentile smoothing.
        
        \textbf{Step 2: Certification for perturbed cumulative reward.}\quad
        First, we compute the \textit{\ebound} $\uje$ using~\Cref{thm:exp-bound}.
        Since the smoothed \namef $\wtilde F$ is obtained based on $m$ sampled noise sequences, we use Hoeffding's inequality~\citep{hoeffding1994probability} to compute the lower bound of the random variable $\wtilde F\left(\vertt{0}{H-1}\mathbf{0}\right)$
        % its lower bound 
        with a confidence level $\alpha$.
        We then calculate the lower bound of $\wtilde F\left(\vertt{0}{H-1}\delta_t\right)$
        % its lower bound 
        under all possible $\ell_2$-bounded perturbations $\delta_t\in\epsball$ leveraging the smoothness of $\wtilde F$.

        We then compute the \textit{\pbound} $\ujp$ using~\Cref{thm:perc-bound}.
        We let ${J_i^C}$ be sorted increasingly, 
        % we extend Chiang \etal~\cite{chiang2020detection}
        and perform normal approximations~\citep{stein1972bound} to compute the largest empirical order statistic $J^C_{k}$ such that $\bb P\left[ \ujp \geq J^C_{k}\right] \geq 1-\alpha$. The empirical order statistic $J^C_{k}$ is then used as the proxy of $\ujp$ under $\alpha$ confidence level.
        We next provide detailed explanations for \texttc{ComputeOrderStats}, which aims to compute the order $k$ using binomial formula plus normal approximation.

\textbf{Estimation of $\wtilde F_\pi$} (defined in~\Cref{lem:global-lip})\textbf{.}\quad
For computing the expectation bound (\ie, the lower bound of $\bb E_{\Delta}\left[ J_\eps(\pi')\right]$), an intermediate step is to estimate $\widetilde F_\pi$ (as shown in line 8 of~\Cref{alg:global-rand}). 

We emphasize that the \textit{accuracy} of the estimation does not influence the \textit{soundness} of the lower bound calculation. 
The reason is given below. 
The number of sampled randomized trajectories $m$ controls the trade-off between the estimation \textit{efficiency} and \textit{accuracy}, as well as the \textit{tightness} of the derived lower bound. 
Concretely, a small $m$ would provide high efficiency, low estimation accuracy, and loose lower bound; but the lower bound is always \textit{sound}, since our algorithm explicitly accounts for the \textit{inaccuracy} associated with $m$ via leveraging the Hoeffding’s inequality in line 9 in~\Cref{alg:global-rand}.

\textbf{Details of \texttc{ComputeOrderStats}.}\quad
We consider the sorted sequence $J_1^C\leq J_2^C\leq \cdots\leq J_m^C$. We additionally set $J_0^C=-\infty$ and $J_{m+1}^C=\infty$.
Our goal is to find the largest $k$ such that $\bb P\left[ \ujp \geq J^C_{k}\right] \geq 1-\alpha$. We evaluate the probability explicitly as follows:
\vspace{8mm}
\begin{small}
\begin{align*}
    \bb P\left[ \ujp \geq J^C_{k}\right] 
    = \sum_{i=k}^{m} \bb P\left[ J_i^C \leq \ujp < J_{i+1}^C \right]
    = \sum_{i=k}^{m} \pmat{m}{i}(p')^i (1-p')^{m-i}.
\end{align*}
\end{small}
\vspace{5mm}

Thus, the condition $\bb P\left[ \ujp \geq J^C_{k}\right] \geq 1-\alpha$ is equivalent to 
\vspace{8mm}
\begin{small}
\begin{align}
    \label{eq:lower-order}
    \sum_{i=0}^{k-1} \pmat{m}{i}(p')^i (1-p')^{m-i} \leq  \alpha.
\end{align}
\end{small}
\vspace{5mm}

Given large enough $m$, the LHS of (\ref{eq:lower-order}) can be approximated via a normal distribution with mean equal to $mp'$ and variance equal to $mp'(1-p')$.
Concretely, we perform binary search to find the largest $k$ that satisfies the constraint.

We finally explain the upper bound of $\eps$ that can be certified for each given smoothing variance.
In practical implementation, for a given sampling number $m$ and confidence level parameter $\alpha$, if $p'$ is too small, then the condition (\ref{eq:lower-order}) may not be satisfied even for $k=1$. This implies that the existence of an upper bound of $\eps$ that can be certified for each smoothing parameter $\sigma$, recalling that $p':=\Phi\left(\Phi^{-1}(p)-\nicefrac{\eps\sqrt{H}}{\sigma}\right)$.

\textbf{Detailed Inference and Certification Procedures.}\quad
We next provide detailed descriptions of the inference and certification procedures, respectively.

\textit{Inference procedure.}\quad 
We deploy the $\sigma$-randomized policy $\pi'$ as defined in~\Cref{def:sig-rand} in~\Cref{sec:cert-glb}. That is, for each observed state, we sample one time of Gaussian noise $\Delta_t$ to add to $s_t$, and take action according to this single randomized observation $a_t = \pi(s_t+\Delta_t)$. 
Thus, in deployment time, the $\Delta$-randomized policy $\pi'$ executes on \textit{one} rollout---at each step it takes one observation and chooses one action (following the procedure described above).

\textit{Certification procedure.}\quad
We sample $m$ randomized trajectories in \glbrs instead of sampling $m$ noisy states per time step as in \staters. For each sampled randomized trajectory, at each time step in the trajectory, we invoke the model with $1$ sample of Gaussian noise (as shown in~(\ref{eq:global-f})). Given the $m$ randomized trajectories and the cumulative reward for each of them, we compute the certification using~\Cref{thm:exp-bound} and~\Cref{thm:perc-bound}.

\textbf{The Algorithm Parameters $\jmin,\jmax$.}\quad
We use the default values of  $J_{\rm min}, J_{\rm max}$ in the game specifications rather than estimate them, which are independent with the trained models. Thus there is no computational cost at all for obtaining the two parameters. As mentioned under~\Cref{thm:exp-bound} in~\Cref{sec:cert-glb}, using these default values may induce a loose bound in practice, so we further propose the percentile smoothing that aims to eliminate the dependency of our certification on $J_{\rm min}$ and $J_{\rm max}$, thus achieving a tighter bound.

\subsection{Detailed Algorithms of \adasearch}
\label{append:adasearch}

\begin{algorithm}[t]
\algsetup{linenosize=\scriptsize}
\scriptsize
\DontPrintSemicolon
\setlength{\columnsep}{15pt}
\begin{multicols}{2}
\KwIn{Enivronment $\cal E=(\cal S, \cal A, R, \Gamma,d_0)$,
% initial state distribution $d_0$, deterministic transition function $\Gamma(s,a)$, deterministic immediate reward function $R(s,a)$, 
trained value network $Q^\pi$ with range $[\vmin, \vmax]$; parameters for randomized smoothing: sampling times $m$, smoothing variance $\sigma^2$, one-sided confidence parameter $\alpha$}
\KwOut{a map $M$ that maps an attack magnitude $\eps$ to the corresponding certified lower bound of reward $J$}
\SetKwFunction{funca}{\textsc{GetActions}}
\SetKwFunction{funcb}{\textsc{Expand}}
\SetKwProg{Pn}{Procedure}{:}{}
\SetKwProg{Fn}{Function}{:}{}
\Comment*[l]{\scriptsize Initialize global variables}
$p\_que \leftarrow \emptyset$ \Comment*[r]{\tiny initialize an empty priority queue containing tuples of (state $s$, action $a$, radius $r$, reward $J$), sorted by increasing $r$}
$M \leftarrow \emptyset$\; % \Comment*[r]{\tiny initialize an empty map that maps a radius $r$ to the certified lower bound of reward $J$}
$\jglobal \leftarrow \infty$ \Comment*[r]{\tiny initialize global minimum reward}
$\Delta =  \left( V_{\max} - V_{\min}\right) \sqrt{\frac{1}{2m} \ln \frac{1}{\alpha}}$ \Comment*[r]{\tiny confidence bound}\;

\Fn{\funca{$s,\eps_{\rm lim}, \jcur$}}{
    Generate noise samples $\delta_i\sim \cal N(0, \sigma^2I)$ for $1\leq i\leq m$\; 
    \For {each action $a\in \cal A$} {
        $\wtilde Q^\pi(s,a)\leftarrow \frac{1}{m} \sum_{i=1}^{m}\texttt{clip}( Q^\pi(s+\delta_i, a), \min=V_{\min}, \max=V_{\max})$\;
    }
    $a^\star\leftarrow \argmax_{a\in\cal A} \wtilde Q^\pi(s,a)$\;
    $a\_list\leftarrow \emptyset$\;
    \For {each action $a\in \cal A$} {
        \If {$\Gamma(s,a)=\bot$} {
            \textbf{continue}\;
        }
        % $re\leftarrow \jcur + R(s,a)$\;
        \uIf {$\wtilde Q^\pi(s,a^\star) \geq \wtilde Q^\pi(s,a)+2\Delta$} {
            $r\leftarrow \frac{\sigma}{2}(
                                                \Phi^{-1}(\frac{\wtilde Q^\pi(s,a^\star)-\Delta-V_{\min}}{V_{\max}-V_{\min}})-
                                                \Phi^{-1}(\frac{\wtilde Q^\pi(s,a)+\Delta-V_{\min}}{V_{\max}-V_{\min}})
                                            )$\;
        }
        \Else {
            $r\leftarrow $ 0\;
        }
        
        \uIf{$r\leq \eps_{\rm lim}$} { \Comment*[r]{\tiny take possible actions}
            $a\_list\leftarrow a\_list \cup \{a\}$ 
        }
        \Else { \Comment*[r]{\tiny store impossible actions in queue for later expansion}
            $p\_que.\texttt{push}((s, a, r, \jcur))$ 
        }
    }
    \KwRet{$a\_list$}\;
}
\;
\Pn{\funcb{$s,\eps_{\rm lim}, \jcur$}}{
    \If {$\jcur \geq \jglobal$} {
        \KwRet{$0$} \Comment*[r]{\tiny pruning}
    }
    $a\_list \leftarrow \funca(s,\eps_{\rm lim}, \jcur)$\;
    \If {$a\_list = \emptyset$} {
        $\jglobal \leftarrow \texttt{min}(\jglobal, \jcur)$\;
        \KwRet{$0$}\;
    }
    \For {$a \in a\_list$} {
        $s'\leftarrow \Gamma(s,a)$\;
        $ret\leftarrow \funcb(s',\eps_{\rm lim}, \jcur+R(s,a))$
    }
}
\;
$s_0\sim d_0$ \Comment*[r]{\tiny initialize initial state}
\funcb($s_0, \eps_{\rm lim}=0,\jcur=0$) \Comment*[r]{\tiny expand initial trajectory}
\While {True} {
    \If {$p\_que = \emptyset$} {
        \textbf{break}\;
    }
    \Comment*[r]{\tiny pop out the first element}
    $(s,a,r,J)\leftarrow p\_que.\texttt{pop}()$ \;
    \Comment*[r]{\tiny examine the next first element}
    $(\_,\_,r',\_)\leftarrow p\_que.\texttt{top}()$ \;
    \Comment*[r]{\tiny derive the critical $\eps$'s}
    $\eps \leftarrow r, \eps'\leftarrow r'$\;
    \Comment*[r]{\tiny obtain a pair of mapping}
    $M[\eps] \leftarrow \jglobal$ \;
    \Comment*[r]{\tiny expand the tree from the new node}
    \funcb$(\Gamma(s,a), \eps', J+R(s,a))$ 
}
\end{multicols}
\vspace{1em}
\caption{\adasearch: Adaptive search for certifying cumulative reward}\label{alg:adaptive-search-full}
\end{algorithm}%
% \end{minipage}
% }

% {
% \LinesNotNumbered
% \begin{algorithm}
% \LinesNotNumbered
% \caption{\adasearch: Adaptive search for certifying cumulative reward}\label{alg:adaptive-search}
% \clipbox{0pt {\depth} 0pt {\baselineskip}}{\usebox{\tempbox}}\hfill
% \raisebox{\depth}{\clipbox{0pt 1ex 0pt {\height}}{\usebox{\tempbox}}}
% \end{algorithm}
% }
    % \input{pseudo-code/adaptive-search-main}

    In the following, we will explain the \textit{trajectory exploration and expansion}, the \textit{growth of perturbation magnitude}, and \textit{optimization tricks} in details.

    \textbf{Trajectory Exploration and Expansion.}\quad 
    \adasearch organizes all possible trajectories in the form of a search tree and progressively grows it.
    Each node of the tree represents a state, and the depth of the node is equal to the time step of the corresponding state in the trajectory.
    The root node (at depth $0$) represents the initial state $s_0$.
    For each node, leveraging \Cref{thm:extend-radius}, we compute a non-decreasing sequence $\{r^k(s)\}_{k=0}^{|\gA|-1}$ corresponding to required perturbation radii for $\pi$ to choose each alternative action~(the subscript $t$ is omitted for brevity).
    Suppose the current $\eps$ satisfies $r^i(s) \le \eps < r^{i+1}(s)$. We grow $(i+1)$ branches from current state $s$ corresponding to the original action and $i$ alternative actions since $\eps \ge r^j(s)$ for $1\le j\le i$.
    For nodes on the newly expanded branch, we repeat the same procedure to expand the tree with depth-first search~\citep{tarjan1972depth} until the terminal state of the game is reached or the node depth reaches $H$.
    As we expand, we keep the record of cumulative reward for each trajectory and update the lower bound $\uj$ when reaching the end of the trajectory if necessary.
    
    \textbf{Perturbation Magnitude Growth.}\quad
    When all trajectories for perturbation magnitude $\eps$ are explored, we need to increase $\eps$ to seek for certification under larger perturbations.
    Luckily, since the action space is discrete, we do not need to examine every $\eps \in \sR^+ $~(which is infeasible) but only need to examine the next $\eps$ where the chosen action in some step may change.
    We leverage \textit{priority queue}~\citep{van1977preserving} to effectively find out such next ``critical'' $\eps$.
    Concretely, along the trajectory exploration, at each tree node, we search for the possible actions and store actions corresponding to $\{r^k(s)\}_{k=i+1}^{|\gA|-1}$ into the priority queue, since these actions are exactly those need to be explored when $\eps$ grows.
    After all trajectories for $\eps$ are fully explored, we pop out the head element from the priority queue as the next node to expand and the next perturbation magnitude $\eps$ to grow.
    We repeat this process until the priority queue  becomes empty or the perturbation magnitude $\eps$ reaches the predefined threshold.
    
    \textbf{Additional Optimization.}\quad
    We adopt some additional optimization tricks to reduce the complexity of the algorithm.
    First, for environments with no negative reward, we perform pruning to limit the tree size---if the cumulative reward leading to the current node already reaches the recorded lower bound, we can perform pruning, since the tree that follows will not serve to update the lower bound.
    This largely reduces the potential search space.
    We additionally adopt the memorization~\citep{michie1968memo} technique which is commonly applied in search algorithms.
    
    We point out a few more potential improvements.
    \textit{First}, with more specific knowledge of the game mechanisms, the search algorithm can be further optimized. Take the Pong game as an example, given the horizontal speed of the ball, we can compress the time steps where the ball is flying between the two paddles, thus reducing the computation.
    \textit{Second}, empirical attacks may be efficiently incorporated into the algorithm framework to provide upper bounds that help with pruning.

    \textbf{Time Complexity.}\quad
    The time complexity of \adasearch is $O(H |S_{\mathrm{explored}}| \times (\log |S_{\mathrm{explored}}| + |A| T))$, where $|S_{\mathrm{explored}}|$ is the number of explored states throughout the search procedure, which is no larger than cardinality of state set, $H$ is the horizon length, $|A|$ is the cardinality of action set, and $T$ is the time complexity of performing local smoothing. The main bottleneck of the algorithm is the large number of possible states, which is in the worst case exponential to state dimension. However, to provide a sound worst-case certification agnostic to game properties, exploring all possible states may be inevitable.

    % We leave the complete algorithm with the add-on tricks to~\Cref{sec:append-adasearch}.
    
    \textbf{Detailed Inference and Certification Procedures and Important Modules in \Cref{alg:adaptive-search-full}.}\quad
    We next provide detailed descriptions of the inference and certification procedures,  respectively, along with other important modules.
    
    \textit{Inference procedure.}\quad
    In \adasearch, we deploy the locally smoothed policy $\tilde\pi$ as defined in (\ref{eq:smooth-q}) in~\Cref{sec:cert-act-thm}. Concretely, for each observed state $s_t$, we sample a batch of Gaussian noise to add to $s_t$, and take action according to the computed mean Q value on the batch of randomized observations. 
    % Thus, in deployment time, the locally smoothed policy $\tilde\pi$ executes on one rollout---at each step it takes one observation and chooses one action (following the procedure described above).
    (Actually, the \textit{inference} procedure per step is exactly the same as in \staters described in \Cref{append:staters}.) 
    
    \textit{Certification procedure.}\quad
    We obtain the \textit{certification} via~\Cref{thm:extend-radius} and the adaptive search algorithm, which outputs a collection of pairs $\{(\eps_{i}, \uj_{\eps_i})\}_{i=1}^{|C|}$ sorted in ascending order of $\eps_i$, where $|C|$ is the length of the collection. 
    The interpretation for this collection of pairs is provided below. 
    For all $\eps’$, let $i$ be the largest integer such that $ \eps_i \leq \eps’ < \eps_{i+1} $, then as long as the perturbation magnitude $\eps \le \eps'$, the cumulative reward $J_\eps(\tpi) \ge \uj_{\eps_i}$.
    This is supported by the fact that all certified radii associated with all nodes in the entire expanded tree compose a discrete set of finite cardinality; and the perturbation value between two adjacent certified radii in the returned collection will not lead to a different tree from the tree corresponding to the largest smaller perturbation magnitude.
    This actually is also the inspiration for the certification design.

    \textit{Important modules.}\quad
    The function \texttt{\textsc{GetAction}} computes the possible actions at a given state $s$ under the limit $\eps$, while the procedure \texttc{Expand} accomplishes the task of expanding upon a given node/state.
    The main part of the algorithm involves a loop that repeatedly selects the next element from the priority queue, \ie, a node associated with an $\eps$ value, to expand upon.

\textbf{Additional Clarifications.}\quad
We clarify that the algorithm only requires access to the environment such that it can obtain the reward and next state via \textit{interacting} with the environment $\cal E$ (\ie, taking action $a$ at state $s$ and obtain next action $s’$ and reward $r$), but does not require access to an oracle transition function $\Gamma$. We further emphasize that access to the environment that supports back-tracking is already sufficient for conducting our adaptive search algorithm.

\section{Discussion on the Certification Methods}
\label{append:discuss-cert}

We discuss the \textit{advantages} and \textit{limitations} of our certification methods, 
% the pitfalls when applying our certification strategies, 
as well as possible direct \textit{extensions}, hoping to pave the way for future research along similar directions.
We also provide more \textit{detailed analysis} that help with the understanding of our algorithms.

\textbf{\staters}.\quad
The algorithm provides state-wise robustness certification in terms of the stability/consistency of the per-state action. 
It treats each time step \textit{independently}, smooths the given state at the given time step, and provides the corresponding certification. 
Thus, one potential \textit{extension} is to expand the time window from one time step to a consecutive sequence of several time steps, and provide certification for the sequences of actions for the given window of states.
% One pitfall in the implementation is the estimation of $\vmin$ and $\vmax$, which is expected to be as accurate as possible for valid input states.

\looseness=-1
\textbf{\glbrs.}\quad
As explained in~\Cref{sec:cert-cum}, the \ebound $\uje$ is too loose to have any practical usage. 
In comparison, \pbound $\ujp$ is much tighter and practical.
However, one limitation of $\ujp$ is that there exists an upper bound of the attack magnitude $\eps$ that can be certified for each $\sigma$, as explained in~\Cref{append:glbrs}. For attack magnitudes that exceed the upper bound, we can obtain no useful information via this certification (though the upper bound is usually sufficiently large).

\textit{One limitation of \glbrs.}\quad
\Cref{lem:global-lip} requires knowing the noise added to each step beforehand so as to generate $m$ randomized trajectories given the known noise sequence $\{\delta_t\}_{t=0}^{H-1}$. Thus, it cannot be applied against an adaptive attacker.
We clarify that our \staters and \adasearch does not have such limitation.

\textbf{\adasearch.}\quad
The algorithm provides the absolute lower bound $\uj$ of the cumulative reward for any finite-horizon trajectory with a given initial state.
% under $\ell_2$-bounded perturbations. 
The advantage of the algorithm is that $\uj$ is an absolute lower bound that bounds the worst-case situation (apart from the exceptions due to the probabilistic confidence $\alpha$), rather than the statistical lower bounds $\uje$ and $\ujp$ that characterize the statistical properties of the random variable $J_\eps$.

\textit{One potential \textit{pitfall} in understanding $\uj$.}\quad 
What \adasearch certifies is the lower bound for the trajectory with \textit{a given initial state} $s_0\sim d_0$, rather than all possible states in $d_0$. This is because our search starts from a root node of the tree, which is set to the fixed given state $s_0$.

\looseness=-1
\textit{Confidence of \adasearch.}\quad
Regarding the \textit{confidence} of our probabilistic certification, we clarify that we consider the independent multiple-test. Concretely, due to the independence of decision making errors, the confidence is $(1-\alpha)^N$, where $N$ is the maximum number of possible attacked states explored by \adasearch. 
Formally, 
\vspace{25pt}
\begin{align}
    \Pr[\text{\adasearch certification holds}] & = \prod_{\substack{\text{all } s \text{ explored }\\\text{by \adasearch}}} \Pr[\text{not make error on }s | \text{not make error on }s_{pre}] \notag
    \\ & \overset{\text{(i)}}{=} \prod_{\text{all attackable } s} (1-\alpha) \times \prod_{\text{all unattackable } s} 1 \notag
    \\ & \overset{\text{(ii)}}{=} (1-\alpha)^N.
\end{align}
\vspace{10pt}

In the above equation, we can see that (i) leverages the independence which in turn gives a bound of form $(1-\alpha)^N$. This is because in each step, the event of “certification does not hold” is \textit{independent} since we sample Gaussian noise \textit{independently}. Leveraging such independence, we obtain a confidence lower bound of $(1-\alpha)^N$.
From (ii), we see that the confidence is only related to the number of possible attacked states $N$. This is because for unattackable states, the certification deterministically holds since there is no attack at current step. Therefore, we only need to count the confidence intervals from all possibly attackable steps instead of all states explored.
We remark that in practice, the attacker usually has the ability to perturb only a limited number of steps, \ie, $N$ is typically small. Therefore, the confidence is non-trivial.

\looseness=-1
\textit{Main limitation of \adasearch.}\quad
The main \textit{limitation} is the high time complexity of the algorithm (details see~\Cref{append:adasearch}). The algorithm has exponential time complexity in worst case despite the existence of several optimizations. Therefore, it is not suitable for environments with a large action set, or when the horizon length is set too long.

\textbf{Extension to Policy-based Methods.}\quad
Though we specifically study the robustness certification in Q-learning in this paper,
our two certification criteria and three certification strategies can be readily extended to policy-based methods. The intuition is that, instead of smoothing the value function in the Q-learning setting, we directly smooth the policy function in policy-based methods. With the smoothing module replaced and the theorems updated, other technical details in the algorithms for certifying the per-state action and the cumulative reward would then be similar.

\section{Additional Experimental Details}
\label{append:exp}

\subsection{Details of the Atari Game Environment}
\label{append:env}
We experiment with two Atari-2600 environments in OpenAI Gym~\citep{brockman2016openai} on top of the Arcade Learning Environment~\citep{bellemare13arcade}.
The \textit{states} in the environments are high dimensional color images ($210\times 160\times 3$) and the \textit{actions} are discrete actions that control the agent to accomplish certain tasks.
Concretely, we use the \texttt{NoFrameskip-v4} version for our experiments, where the randomness that influences the environment dynamics can be fully controlled by setting the random seed of the environment at the beginning of an episode.

\subsection{Details of the RL Methods}
\label{append:impl-rl}

\textcolor{black}{We introduce the details of the nine RL methods we evaluated. 
Specifically, for each method, we first discuss the algorithm, and then present the implementation details for it.}

%\fan{introduction of the RL methods}

%\zijian{finish introduction for each RL methods}

\textcolor{black}{\textbf{StdTrain}~\citep{mnih2013playing}: StdTrain~(naturally trained DQN model) is an algorithm based on Q-learning, in which DQN is used to reprent Q-value function and TD-loss is used to optimize DQN. Our StdTrain model is implemented with Double DQN~\citep{van2016deep} and Prioritized Experience Replay~\citep{schaul2015prioritized}. }

\textcolor{black}{\textbf{GaussAug}~\citep{behzadan2017whatever}: The algorithm of GaussAug is similar with AtdTrain, in which appropriate Gaussian random noises added to states during training. As for implementation, our GaussAug model is based on the same architecture with the same set of \textit{basic} training techniques of StdTrain, and adds Gaussian random noise with $\sigma=0.005$, which has the best testing performance under attacks compared with other $\sigma$, to all the frames during training. }

\textcolor{black}{\textbf{AdvTrain}~\citep{behzadan2017whatever}: Instead of using the original observation to train Double DQN~\citep{van2016deep}, AdvTrain generates adversarial perturbations and applies the perturbations to part of or all of the frames when training. In our case, we add the perturbations generated by $5$-step PGD attack to 50\% of the frames to make a balance between between stable training and effectiveness of adversarial training. }

\textcolor{black}{\textbf{SA-MDP (PGD) and SA-MDP (CVX)}~\citep{zhang2021robust}: Instead of utilizing adversarial noise to train models directly like AdvTrain, SA-MDP (PGD) and SA-MDP (CVX) regularize the loss function with the help of adversarial noise during training in order to train empirically robust RL agents. In our experiment settings, which is the same as SA-MDP (PGD) and SA-MDP (CVX)~\citep{zhang2021robust}, models consider adversarial noise with $\ell_\infty$-norm $\eps=\nicefrac{1}{255}$ when solving the maximization for the regularizer and minimize the original TD-loss concurrently. }

\textcolor{black}{\textbf{RadialRL}~\citep{oikarinen2020robust}: With the similar idea of adding regularization during training of SA-MDP (PGD) and SA-MDP (CVX), RadialRL compute the worst-case loss instead of regularizing of the bound of the difference of policy distribution under original states and adversarial states. In our case, the RadialRL model uses a linearly increasing perturbation magnitude $\eps$ (from $0$ to $\nicefrac{1}{255}$) to compute the adversarial loss.}

\textcolor{black}{\textbf{CARRL}~\citep{everett2021certifiable}: Unlike other methods which try to improve robustness of RL agents during training, CARRL enhances RL agents' robustness during testing time. It computes the lower bound of each action's Q value at each step and take the one with the highest lower bound conservatively. 
Given that the lower bound is derived via neural network verification methods~\citep{gowal2018effectiveness,weng2018towards}, CARRL can only be applied to low dimensional environments~\citep{everett2021certifiable}.
Thus, we only evaluate CARRL on CartPole and Highway, where we set the bound of the $\ell_2$ norm of the perturbation be $\eps=0.1$ and $\eps=0.05$ respectively when computing the lower bound of Q value, which demonstrate the best empirical performance under noise and empirical attacks.}

\textcolor{black}{\textbf{NoisyNet}~\citep{fortunato2017noisy}: Instead of the conventional exploration heuristics for DQN and Dueling agents which is $\eps$-greedy algorithms, NoisyNet adds parametric noise to DQN's weights, which is claimed to aid efficient exploration and yield higher scores than StdTrain. Concretely, our NoisyNet shares the same basic architecture with StdTrain, and samples the network weights during training and testing, where the weights are sampled from a Gaussian distribution with $\sigma=0.5$.}

\textcolor{black}{\textbf{GradDQN}~\citep{pattanaik2018robust}: GradDQN is a variation of AdvTrain, which utilizes conditional value of risk (CVaR) as optimization criteria for robust control instead standard expected long term return. This method maximizes the expected return over worst $\alpha$ percentile of returns, which can prevent the adversary from possible bad states. In our case, we generate the adversarial states with 10-step attack by calculating the direction of gradient and sampling in that direction from \textit{beta}\(1,1\) distribution in each step. }

%\zijian{In latex, below is the old version for RL method introduction}

% StdTrain implements Double DQN~\cite{van2016deep} and Prioritized Experience Replay~\cite{schaul2015prioritized}.
%StdTrain~(naturally trained DQN model) is implemented with Double DQN~\cite{van2016deep} and Prioritized Experience Replay~\cite{schaul2015prioritized}.
%The other five methods are implemented based on the same architecture with the same set of \textit{basic} training techniques.
%GaussAug adds Gaussian random noise with $\sigma=0.005$ to all the frames.
%AdvTrain generates adversarial perturbations using $5$-step PGD attack, and applies the perturbations to 50\% of the frames when training Freeway and 10\% of the frames when training Pong.
%SA-MDP (PGD) and SA-MDP (CVX) consider adversarial noise with $\ell_\infty$-norm $\eps=\nicefrac{1}{255}$ when solving the maximization for the regularizer.
% add regularizers derived based on either projected gradient descent or convex relaxations.
%RadialRL uses a linearly increasing perturbation magnitude $\eps$ (from $0$ to $\nicefrac{1}{255}$) to compute the worst-case loss.
\textcolor{black}{For StdTrain, SA-MDP (PGD), SA-MDP (CVX), RadialRL, we directly evaluate the trained models provided by the authors\footnote{\scriptsize StdTrain, SA-MDP (PGD), SA-MDP (CVX) from \url{https://github.com/chenhongge/SA_DQN} and RadialRL from \url{https://github.com/tuomaso/radial_rl} under Apache-2.0 License}.}
% for GaussAug and AdvTrain, we train the models ourselves.

\subsection{Rationales for Settings and Experimental Designs}
\label{append:rationale}

\looseness=-1
\textbf{Finite Horizon Setting.}\quad
In this paper, we focus on a finite horizon setting, mainly for the evaluation purpose---we would need to compute the (certified and empirical) cumulative rewards from finite rollouts.
We clarify that our criteria, theorem, and algorithm are applicable to the infinite horizon case as well.

\textbf{The Role of Practical Attacks in Our Evaluation.}\quad
The goal of this paper is to provide certification that comes with theoretical guarantees, and therefore following the standard certified robustness literatures~\citep{cohen2019certified,li2020sok,salman2019provably,jeong2020consistency}, there is no need to evaluate the empirical attacks, since the certification always holds as long as the attacks satisfy certain conditions regardless of the actual attack algorithms.
The attacks we evaluated in our paper were only for the demonstration purpose and in the hope of providing an example.

\subsection{Detailed Evaluation Setup for Atari Games}
\label{append:exp-setup}

We introduce the detailed evaluation setups for the three certification methods corresponding to the experiments in~\Cref{sec:exp}.

\textbf{Evaluation Setup for \staters.}\quad
We report results averaged over $10$ episodes and set the length of the horizon  $H=500$.
At each time step, we sample $m=10,000$ noisy states for smoothing.
When applying Hoeffding's inequality, we adopt the confidence level parameter $\alpha=0.05$.
Since the input state observations for the two Atari games are in image space, 
we rescale the input states such that each pixel falls into the range $[0,1]$.
When adding Gaussian noise to the rescaled states, we sample Gaussian noise of zero mean and different variances. Concretely, the standard deviation $\sigma$ is selected among $\{0.001, 0.005, 0.01, 0.03, 0.05, 0.1, 0.5, 0.75, 1.0,1.5,2.0,4.0\}$.
We evaluate different parameters for different environments.

\textbf{Evaluation Setup for \glbrs.}\quad
We sample $m=10,000$ \sigrtjs, each of which has length $H=500$,
and conduct experiments with the same set of smoothing parameters as in the setup of \staters.
When accumulating the reward, we set the discount factor $\gamma=1.0$ with no discounting.
We take $\alpha=0.05$ as the confidence level when applying Hoeffding's inequality in the \ebound $\uje$, and $\alpha=0.05$ as the confidence level for \texttc{ComputeOrderStats} in the \pbound $\ujp$.
For the validation of tightness, concretely, we carry out a $10$-step PGD attack with $\ell_2$ radius within the perturbation bound at all time steps during testing to evaluate the tightness of our certification.

\textbf{Evalution Setup for \adasearch.}\quad
We set the horizon length as $H=200$ and sample $m=10,000$ noisy states with the same set of smoothing variance as in the setup of \staters.
Similarly, we adopt $\gamma=1.0$ as the discount factor and $\alpha=0.05$ as the confidence level when applying Hoeffding's inequality.
When evaluating Freeway, we modify the reward mechanism so that losing one ball incurs a zero score rather than a negative score. This modification enables the pruning operation and therefore facilitates the search algorithm, yet still respects the goal of the game.
We also empirically attack the policy $\tpi$ to validate the tightness of the certification via the same PGD attack as described above.
For each set of experiments, we run the attack one time with the same initial state as used for computing the lower bound in~\Cref{alg:adaptive-search-full}.

\subsection{Experimental Setup for CartPole}
We additionally experiment with CartPole-v0 in OpenAI Gym~\citep{brockman2016openai} on top of the Arcade Learning Environment~\citep{bellemare13arcade}. We introduce the experimental setup below.

\textbf{Details of the CartPole Environment.}\quad
The \textit{state} in the CartPole game is a low dimensional vector of length $4$ and the $\textit{action}\in\{\textit{move right},\textit{move left}\}$. 
The goal of the game is to balance the rod on the cart, and one reward point is earned at each timestep when the rod is upright.

\textbf{Implementation Details of the RL Methods on CartPole.}\quad
In addition to the eight RL Methods evaluated on Pong and Freeway as shown in~\Cref{sec:exp}, we evaluate another RL algorithm CARRL~\citep{everett2021certifiable}, which is claimed to be robust in low dimensional games like CartPole. This method relies on linear bounds output by the Q network and selects the action whose lower bound of the Q value is the highest. 
More detailed introduction to the algorithm and description of the implementation details are provided in~\Cref{append:impl-rl}.

\textbf{Evaluation Setup for \staters on CartPole.}\quad
We report results averaged over 10 episodes and set the length of the horizon $H=200$. At each time step, we sample $m=10,000$ noisy states for smoothing. When applying Hoeffding’s inequality, we adopt the confidence level parameter $\alpha=0.05$. 
% Because the input state observations for CartPole game is a vector with length of 4, which is different from that in Atari games in image space, we do not rescale states. 
We do not perform rescaling on the state observations.
When adding Gaussian noise to states, we sample Gaussian noise of zero mean and different variances. Concretely, the standard deviation $\sigma$ is selected among $\{0.001, 0.005, 0.01, 0.03, 0.05, 0.1\}$.

\textbf{Evaluation Setup for \glbrs on CartPole.}\quad
We sample $m=10,000$ \sigrtjs, each of which has length $H=200$,
and conduct experiments with the same set of smoothing parameters as in the setup of \staters. All other experiment settings are the same as Pong and Freeway in~\Cref{append:exp-setup}.

\textbf{Evalution Setup for \adasearch on CartPole.}\quad
Since the reward of CartPole game is quite dense, we set the horizon length as $H=10$ and sample $m=10,000$ noisy states with the same set of smoothing variance as in the setup of \staters. All other experiment settings are the same as Pong and Freeway in~\Cref{append:exp-setup}.

\subsection{Experimental Setup for Highway}
In addition to the previous environments in OpenAI Gym~\citep{brockman2016openai} on the Arcade Learning Environment~\citep{bellemare13arcade}, we also evaluate our methods in the Highway environment \citep{highway-env}. Specifically, we test out algorithms in the highway-fast-v0 environment.

\textbf{Details of the highway Environment.}\quad
The \textit{state} in the highway-fast-v0 environment is a $5 \times 5$ matrix, 
where each line represents a feature vectors of either the ego vehicle or other vehicles closest to the ego vehicle.
The feature vector for each vehicle has the form of $[x,y,vx,vy,1]$, corresponding to the two-dimensional positions and velocities and an additional indicator value.
The \textit{actions} are high-level control actions among \{\textit{lane left}, \textit{idle}, \textit{lane right}, \textit{faster}, \textit{slower}\}.
The \textit{reward} mechanism in this environment is associated with the status of the ego vehicle; more concretely, it gives higher reward for the vehicle when it stays at the rightmost lane, moves at high speed, and does not crash into other vehicles.

{\textbf{Evaluation Setup for \staters on highway.}\quad
We report results averaged over 10 episodes and set the length of the horizon $H=30$, which is the maximum length of an episode in the environment configuration. 
At each time step, we sample $m=10,000$ noisy states for smoothing. When applying Hoeffding’s inequality, we adopt the confidence level parameter $\alpha=0.05$. We do not perform rescaling on the state observations.
When adding Gaussian noise to states, we sample Gaussian noise of zero mean and different variances. Concretely, the standard deviation $\sigma$ is selected among $\{0.005, 0.05, 0.1, 0.5, 0.75, 1.0\}$.}

{\textbf{Evaluation Setup for \glbrs on highway.}\quad
We sample $m=10,000$ \sigrtjs (each of which has length $H=30$)
and conduct experiments with the same set of smoothing parameters as in the setup of \staters. All other experiment settings are the same as Pong and Freeway in~\Cref{append:exp-setup}.}

{\textbf{Evalution Setup for \adasearch on highway.}\quad
Since the reward of highway game is also quite dense, we set the horizon length as $H=10$ as in the CartPole environment and sample $m=10,000$ noisy states with the same set of smoothing variance as in the setup of \staters. All other experiment settings are the same as Pong and Freeway in~\Cref{append:exp-setup}.}

\section{Additional Evaluation Results and Discussions}
\label{append:results}

\subsection{Robustness Certification for Per-state Action -- Certified Ratio}
\label{append:cert-ratio}
\begin{figure}
\newlength{\utilheightb}
\settoheight{\utilheightb}{\includegraphics[width=.165\linewidth]{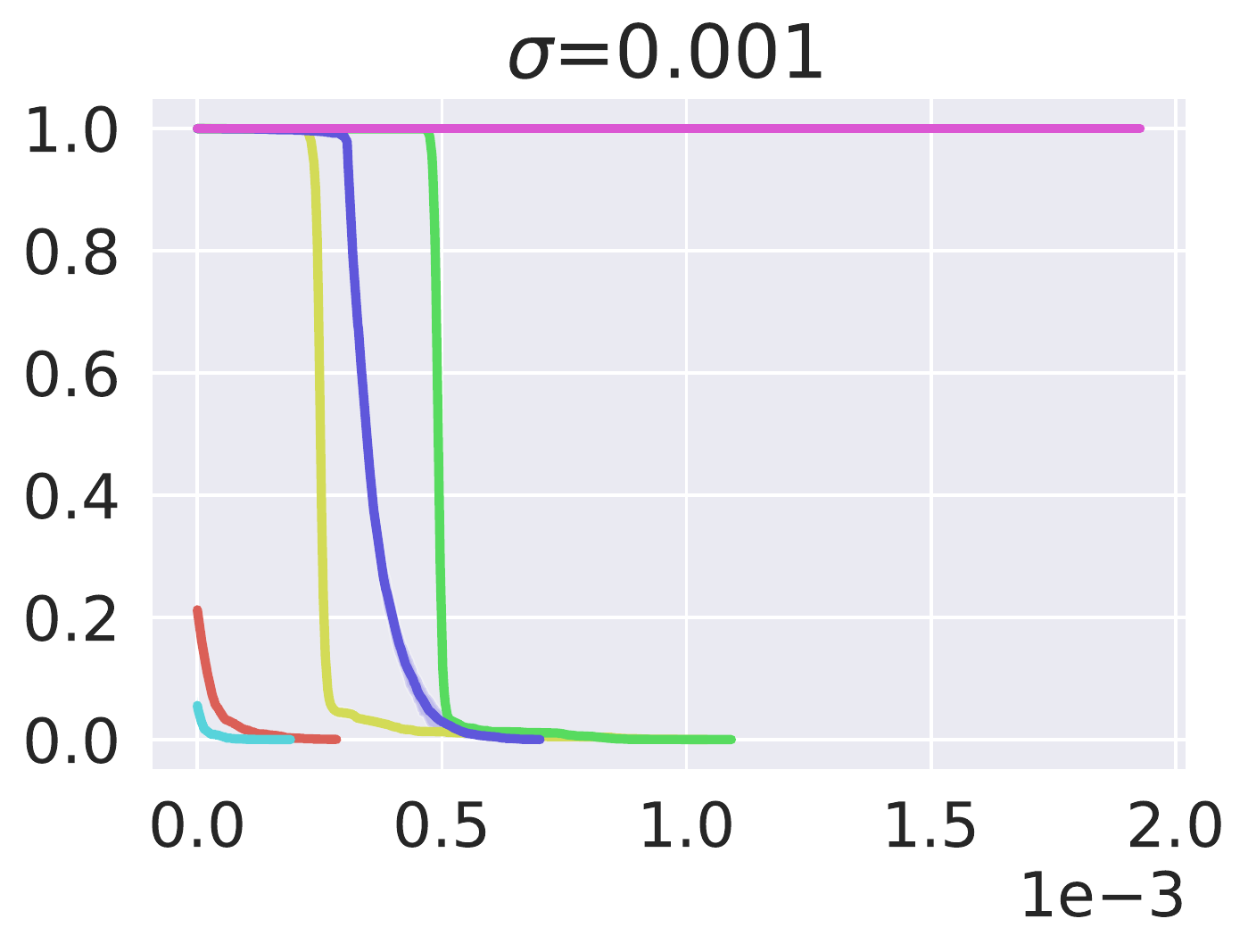}}%

\newlength{\legendheightratio}
\setlength{\legendheightratio}{0.3\utilheightb}%

\newcommand{\rowname}[1]% #1 = text
{\rotatebox{90}{\makebox[\utilheightb][c]{\tiny #1}}}

\centering

{
\renewcommand{\tabcolsep}{10pt}

\begin{subtable}[]{\linewidth}
\begin{tabular}{l}
\includegraphics[height=\legendheightratio]{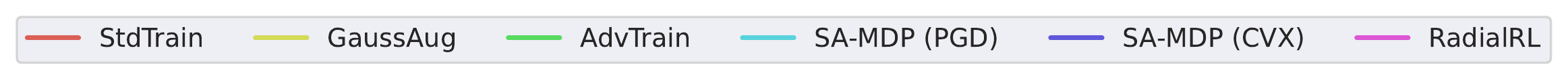}
\end{tabular}
\end{subtable}

\begin{subtable}[]{\linewidth}
\centering
\resizebox{\linewidth}{!}{%
\begin{tabular}{@{}p{5mm}@{}c@{}c@{}c@{}c@{}c@{}c@{}}
\rowname{\makecell{Freeway\\ratio \%}}&
\includegraphics[height=\utilheightb]{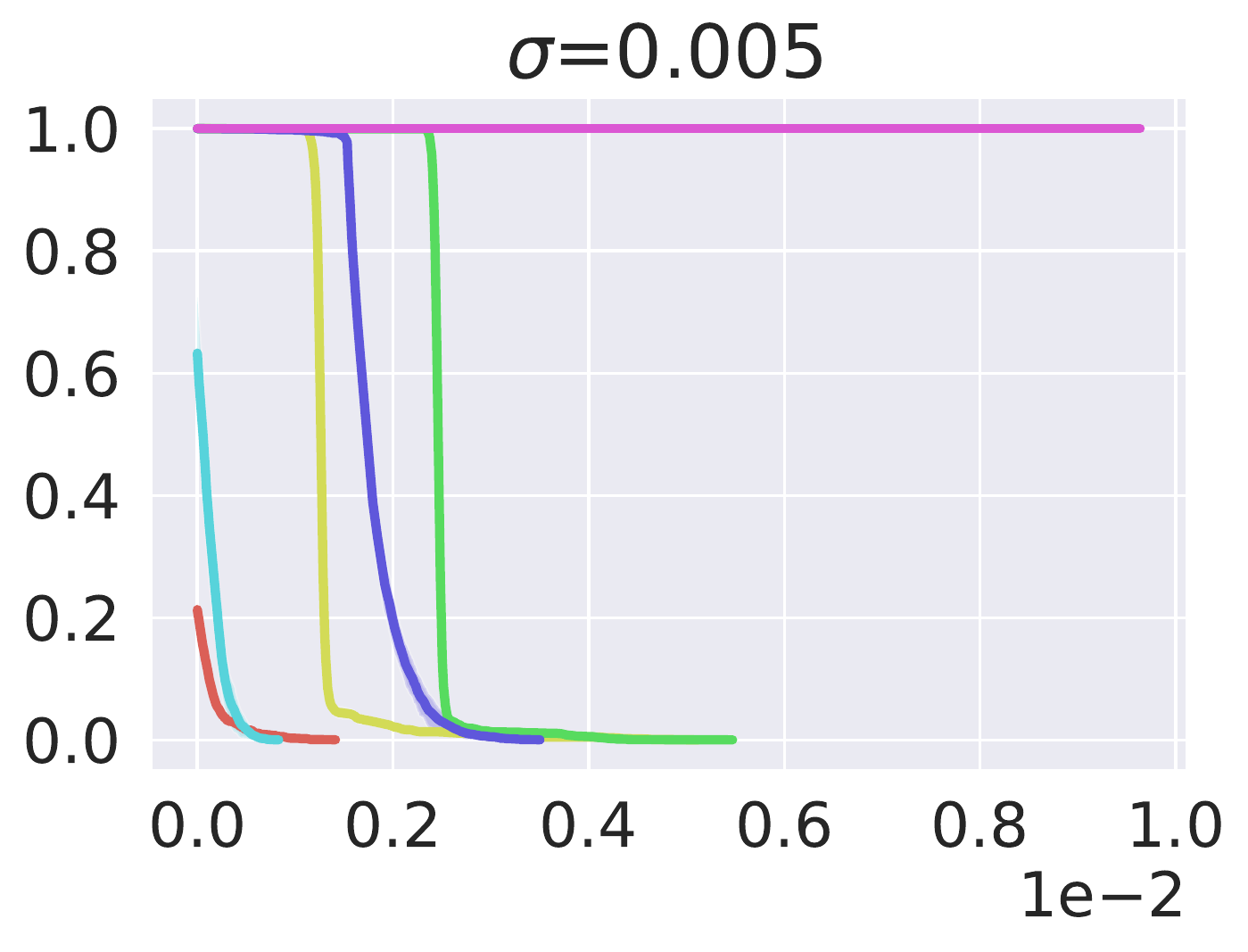}&
\includegraphics[height=\utilheightb]{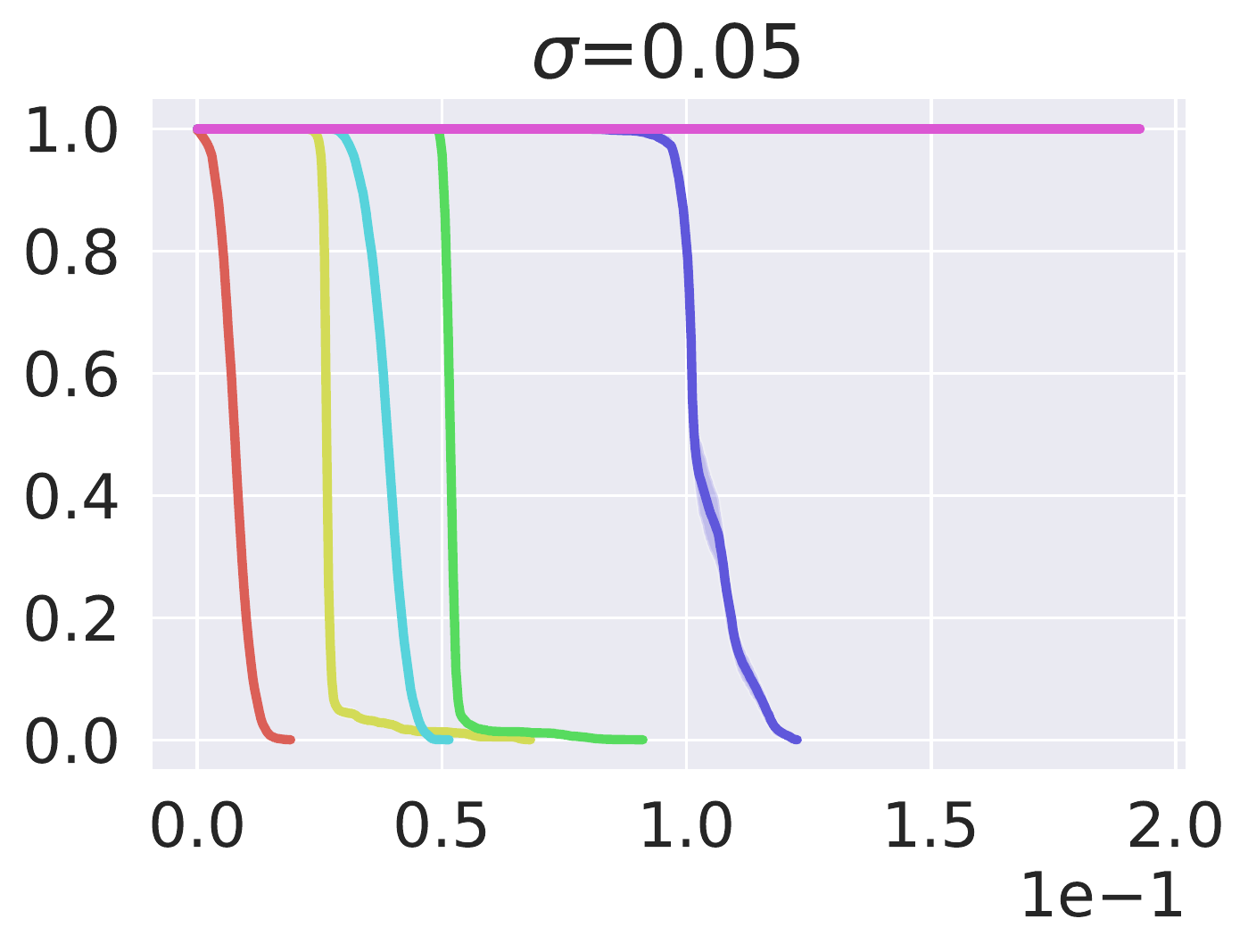}&
\includegraphics[height=\utilheightb]{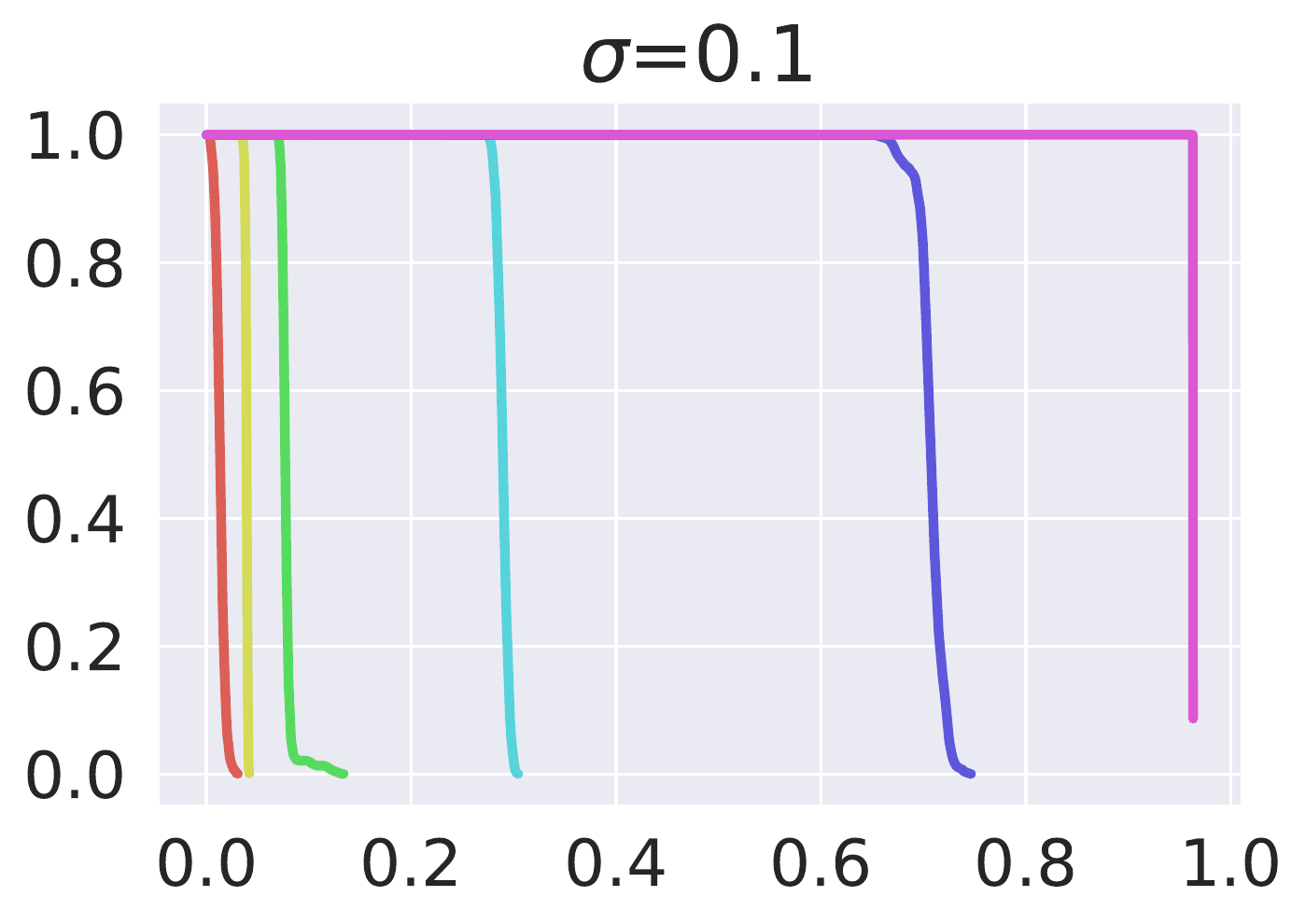}&
\includegraphics[height=\utilheightb]{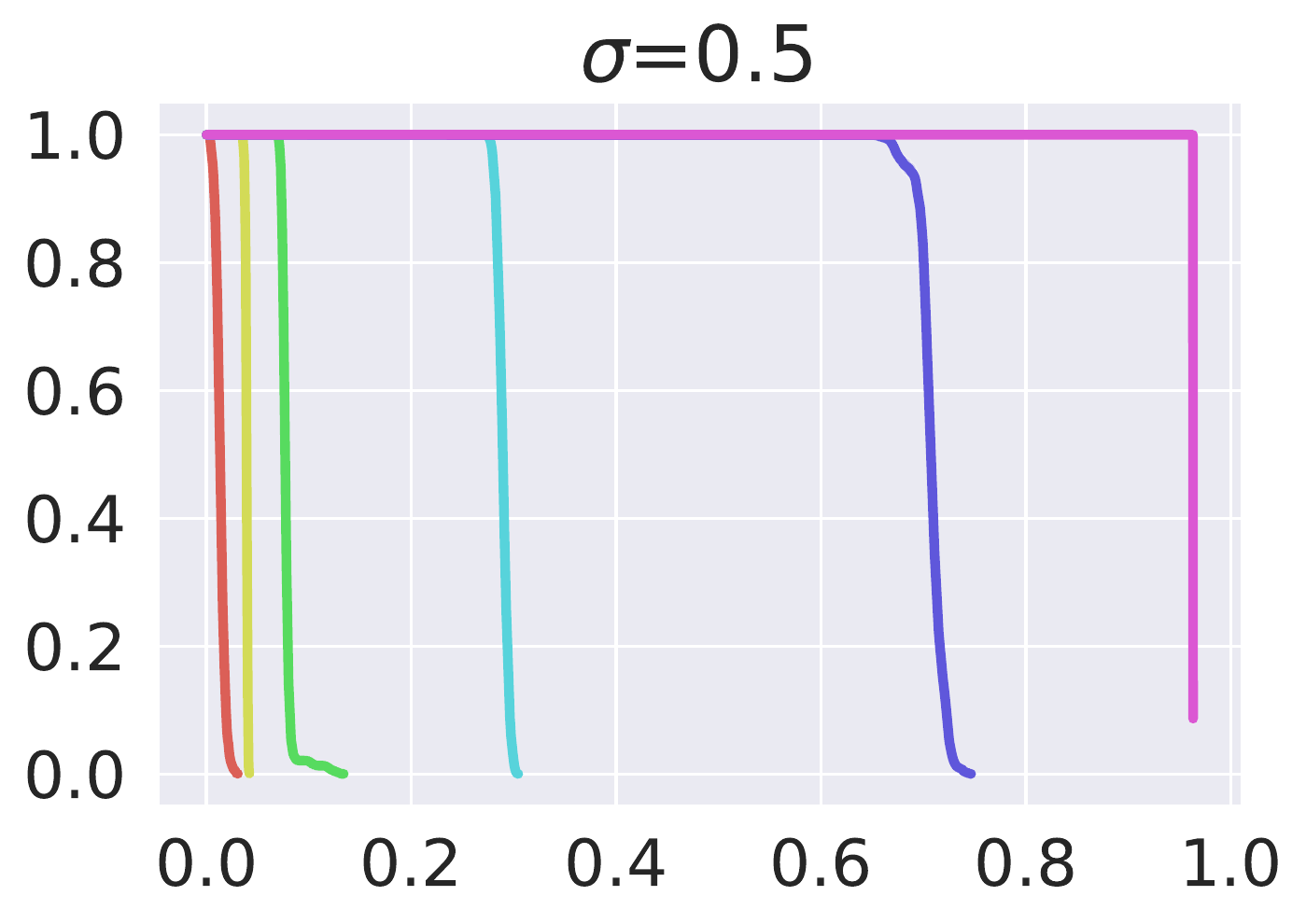}&
\includegraphics[height=\utilheightb]{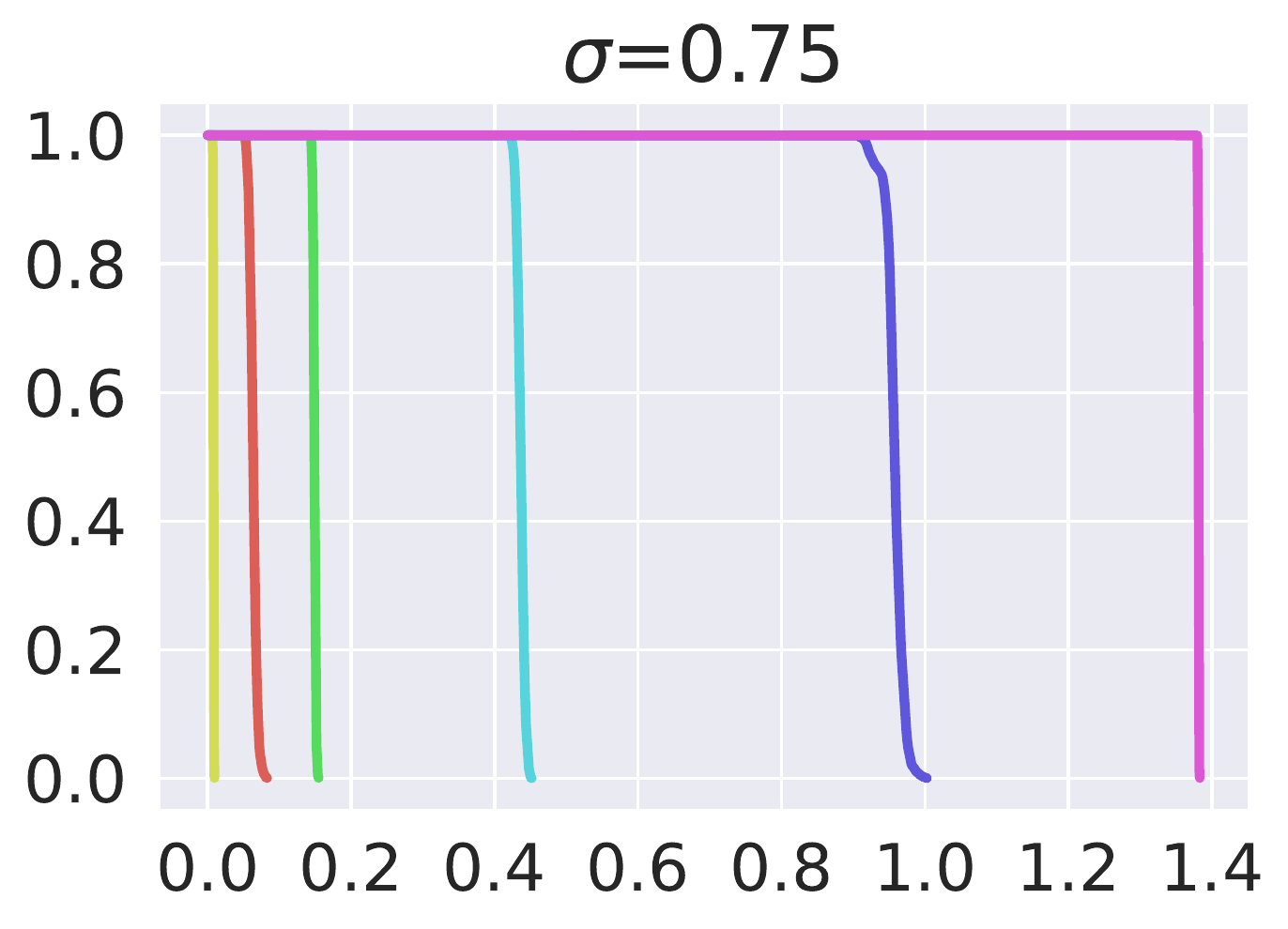}&
\includegraphics[height=\utilheightb]{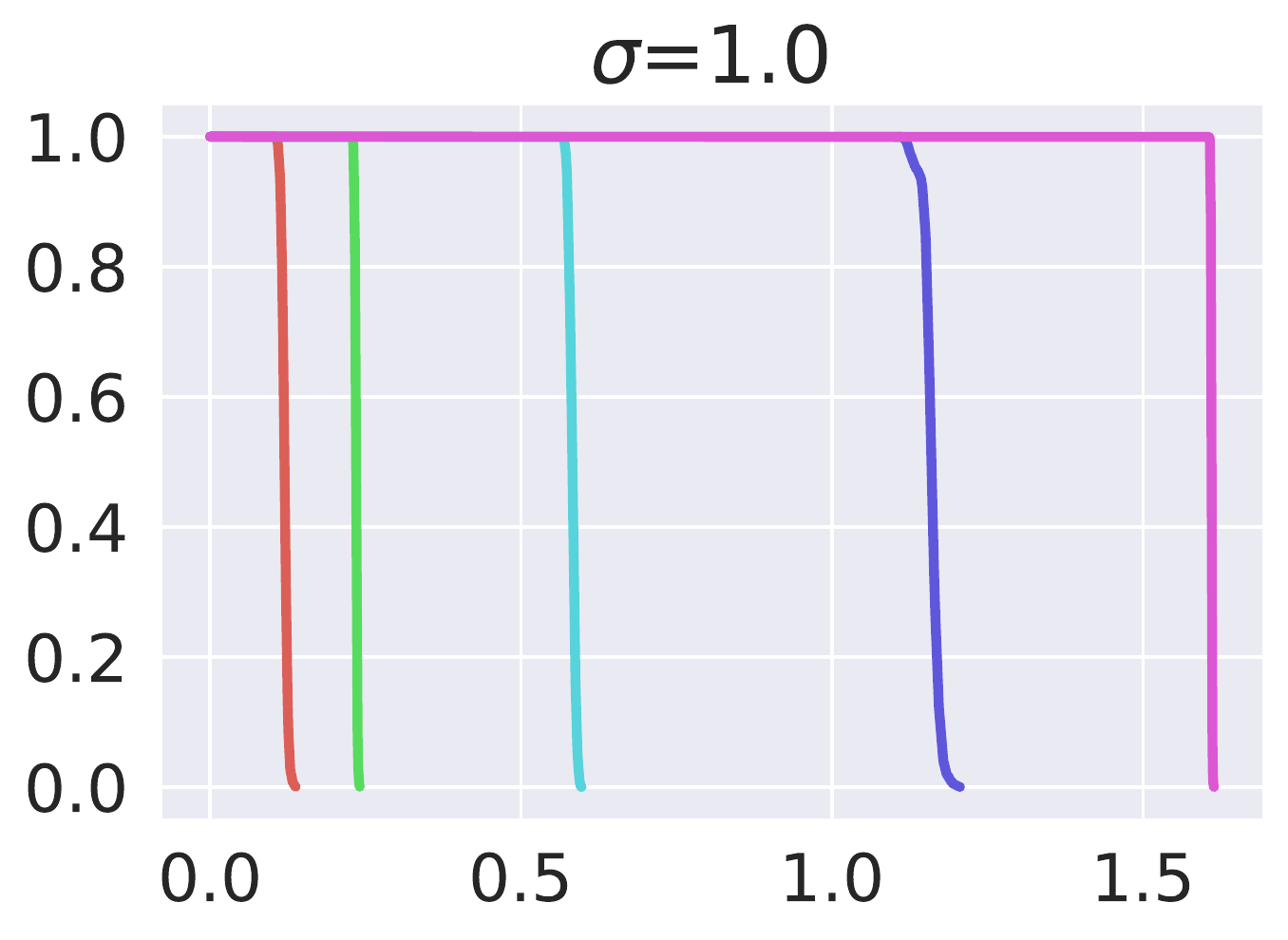}\\[-1.2ex]
\rowname{\makecell{Pong\\ratio \%}}&
\includegraphics[height=\utilheightb]{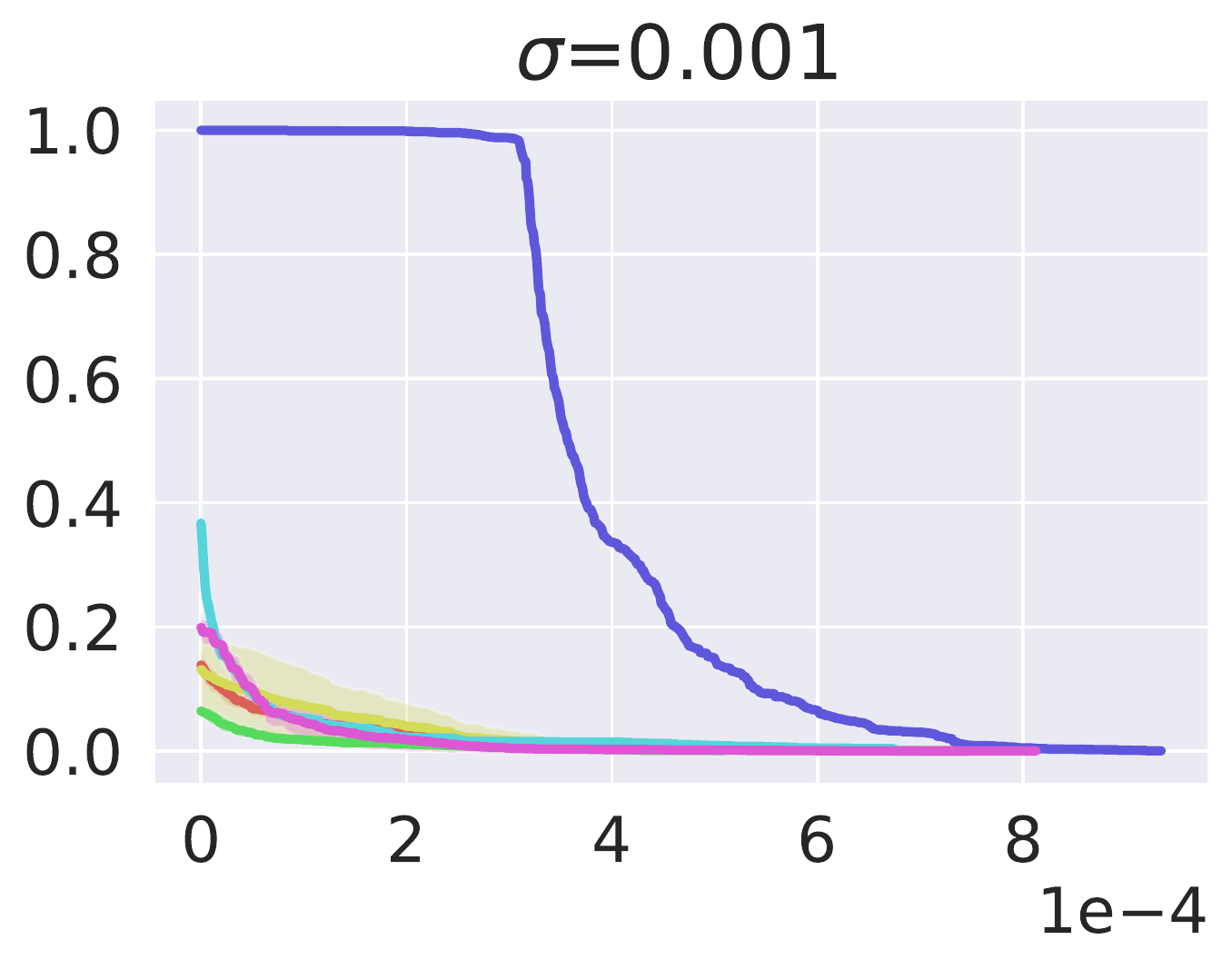}&
\includegraphics[height=\utilheightb]{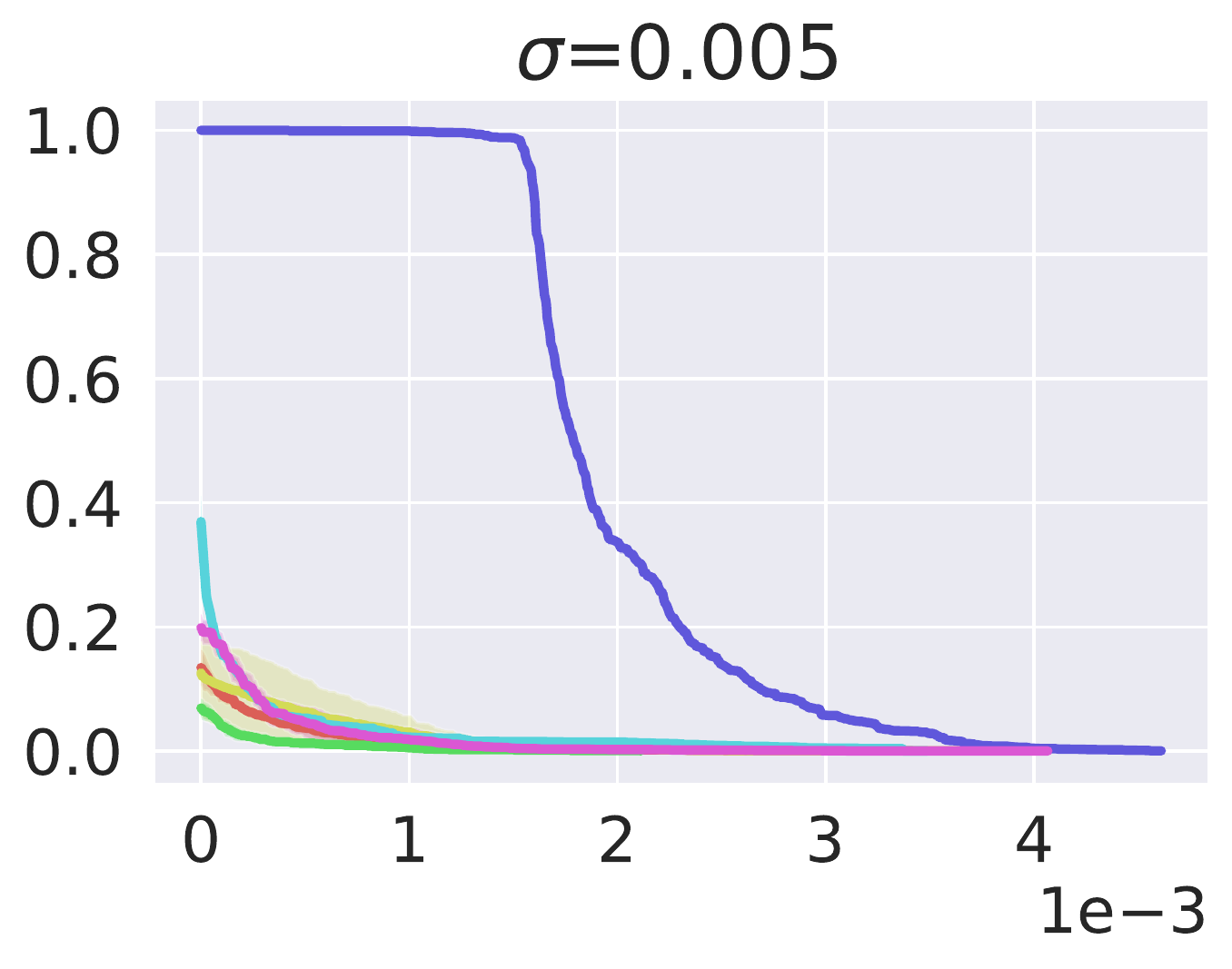}&
\includegraphics[height=\utilheightb]{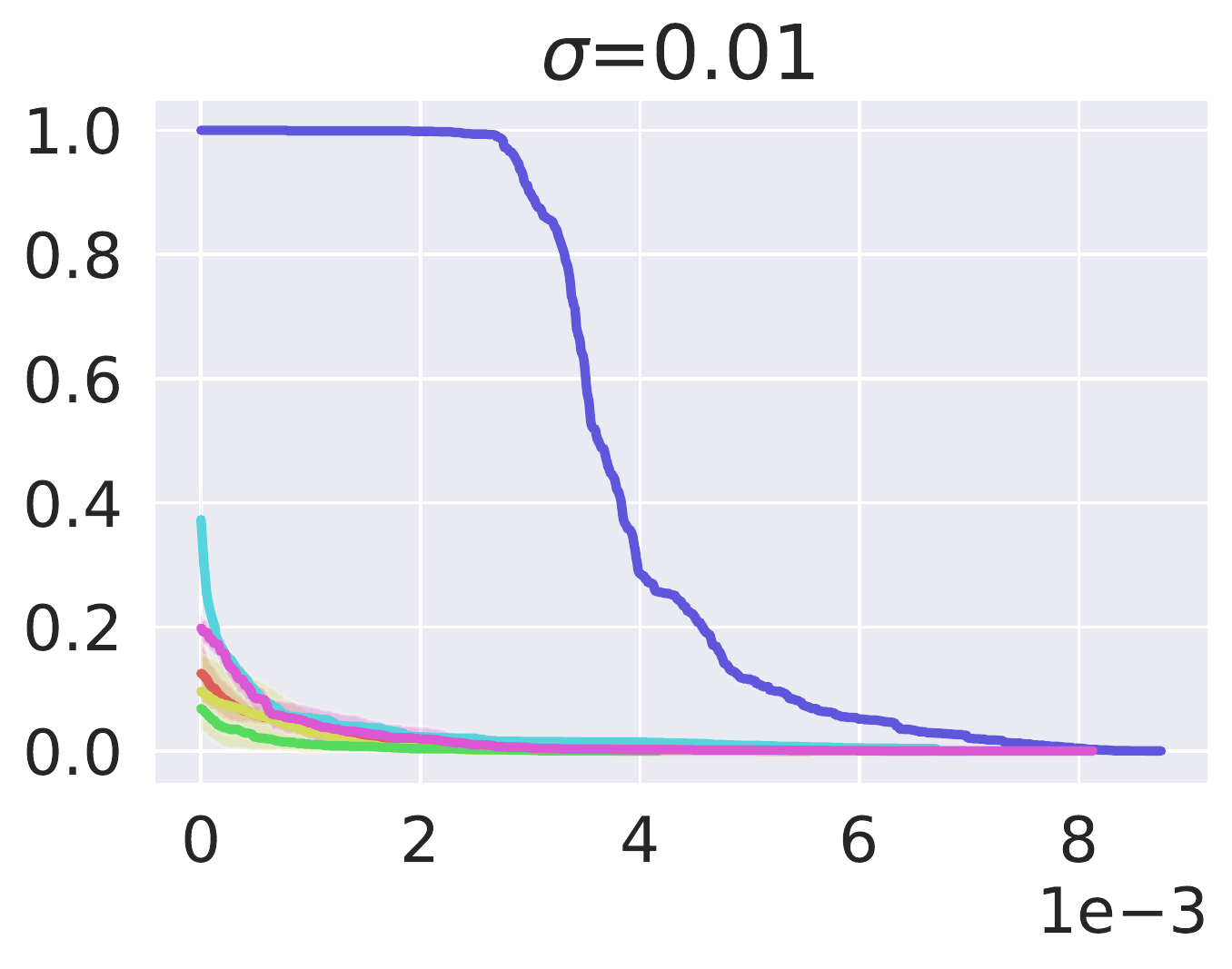}&
\includegraphics[height=\utilheightb]{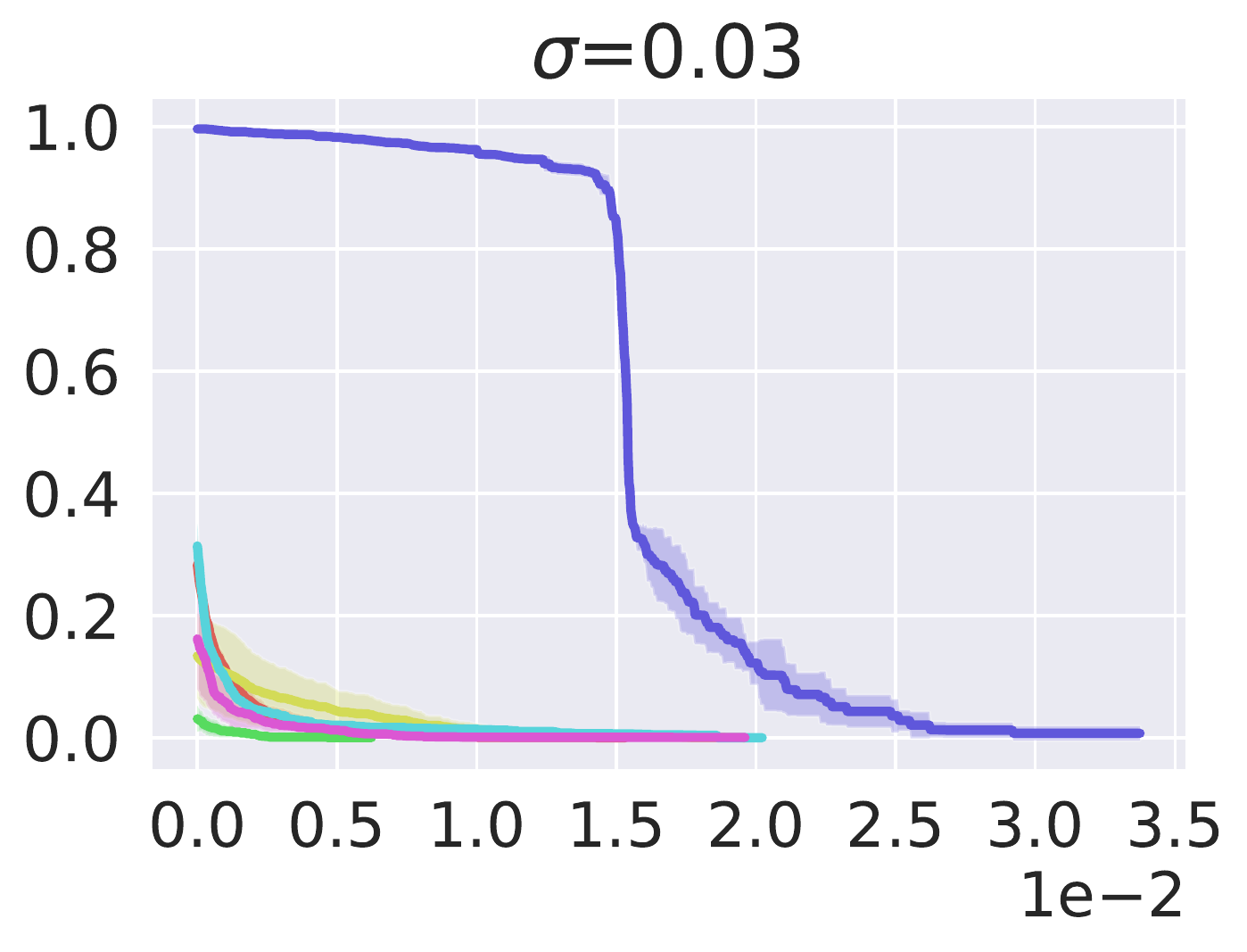}&
\includegraphics[height=\utilheightb]{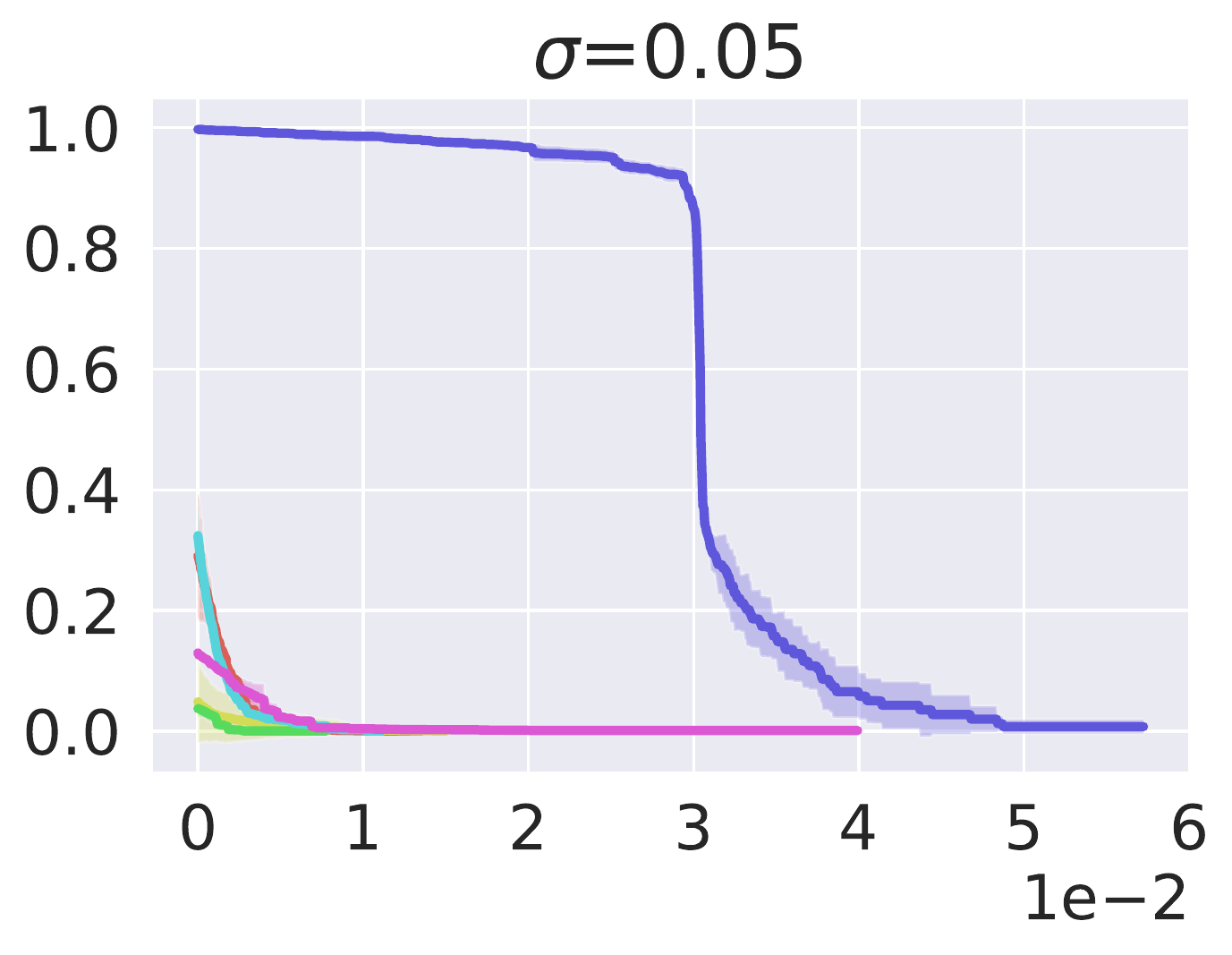}&
\includegraphics[height=\utilheightb]{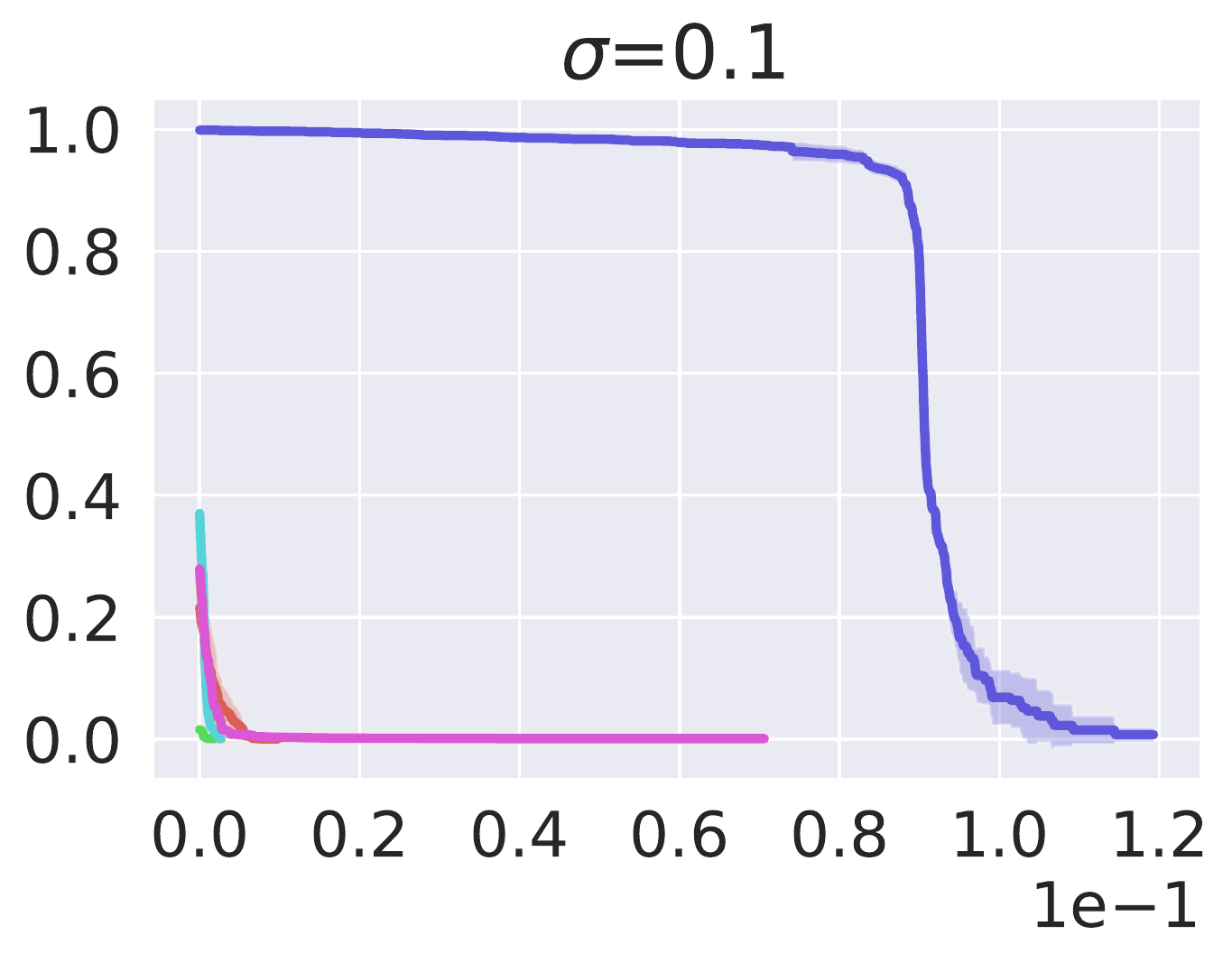}\\[-1.2ex]
        & \makecell{\tiny{radius $r$}}
        & \makecell{\tiny{radius $r$}}
        & \makecell{\tiny{radius $r$}}
        & \makecell{\tiny{radius $r$}}
        & \makecell{\tiny{radius $r$}}
        & \makecell{\tiny{radius $r$}}
\end{tabular}
}
% \caption{\small Certified ratio $\eta_R$ w.r.t. radius $r$}\label{tab:cert-ratio}
\end{subtable}
}
\caption{\small Robustness certification for \textit{per-state action} in terms of certified ratio $\eta_r$ w.r.t. certified radius $r$. 
Each column corresponds to one smoothing parameter $\sigma$. The shaded area represents the standard deviation.
}%
\label{fig:cert-ratio}
\vspace{-2mm}
\end{figure}

We present the robustness certification for per-state action in terms of the certified ratio $\eta_r$ w.r.t. the certified radius $r$ in~\Cref{fig:cert-ratio}.
From the figure, we see that RadialRL is the most certifiably robust method on Freeway, followed by SA-MDP (CVX) and SA-MDP (PGD);
while on Pong, SA-MDP (CVX) is the most robust.
These conclusions are highly consistent with those in~\Cref{sec:eval-cert-rad} as observed from~\Cref{fig:statewise}.
Comparing the curves for Freeway and Pong, we note that Freeway not only achieves larger certified radius \textit{overall}, it also more \textit{frequently} attains large certified radius.

\subsection{Results of Freeway under Larger Smoothing Parameter $\sigma$}
\label{append:freeway-large-sigma}

\begin{figure}

\newlength{\utilheightlarge}
\settoheight{\utilheightlarge}{\includegraphics[width=.15\linewidth]{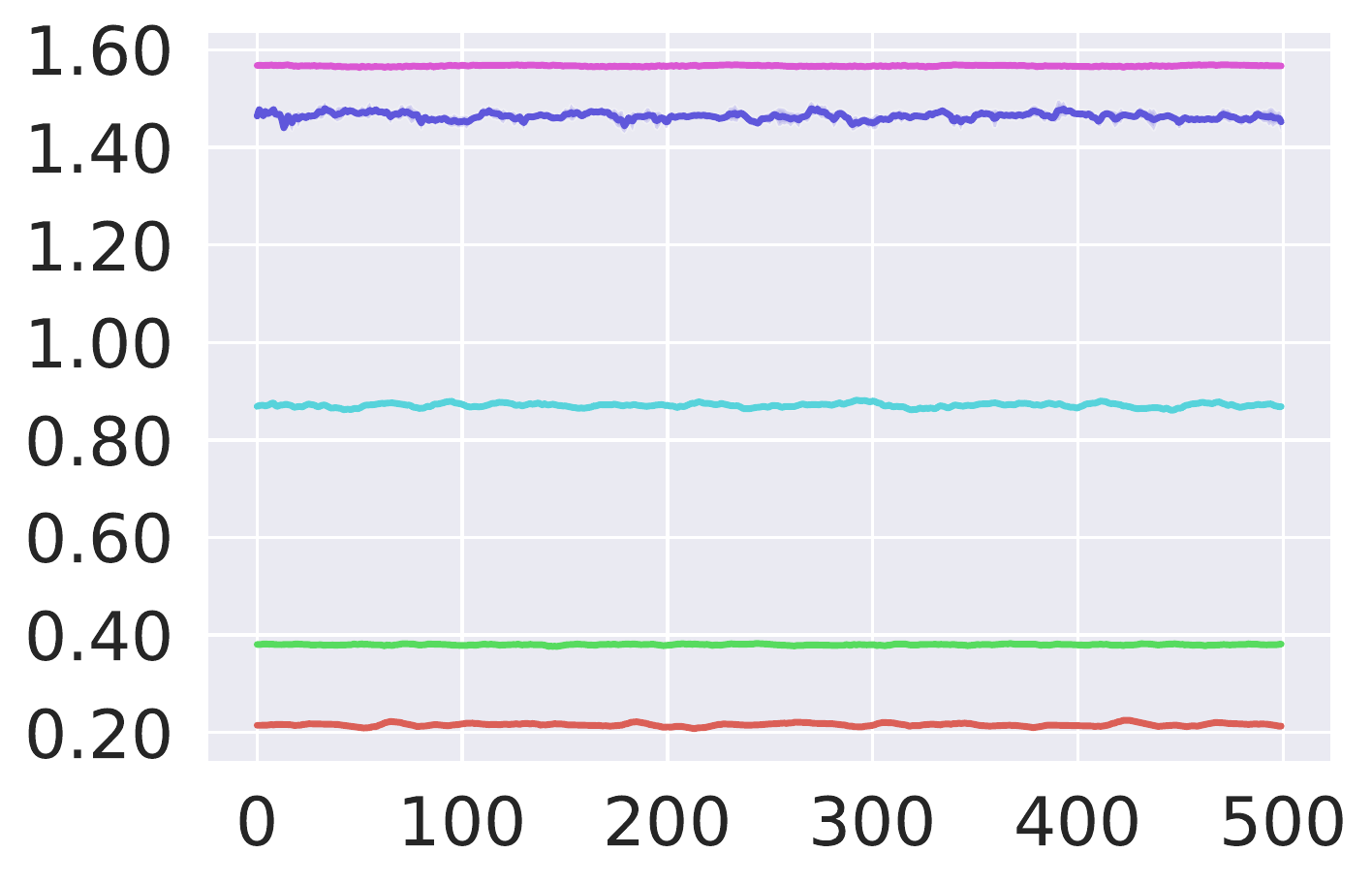}}%

\newlength{\utilheightlargeada}
\settoheight{\utilheightlargeada}{\includegraphics[width=.15\linewidth]{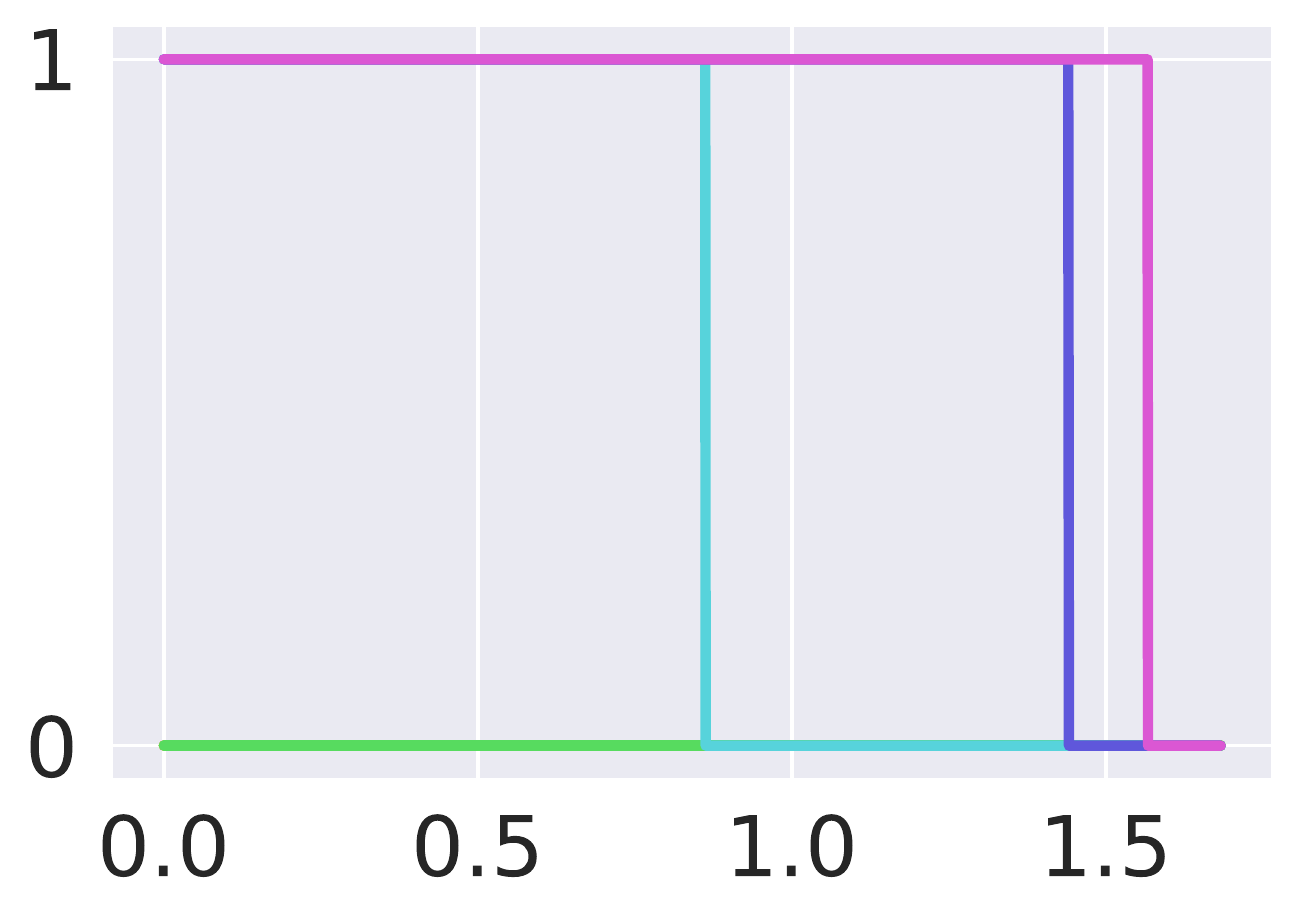}}%

\newlength{\legendheightlarge}
\setlength{\legendheightlarge}{0.4\utilheightlarge}%

\newcommand{\rowname}[1]% #1 = text
{\rotatebox{90}{\makebox[\utilheightlarge][c]{\tiny #1}}}

\centering

{
\renewcommand{\tabcolsep}{10pt}

\begin{subtable}[]{\linewidth}
\begin{tabular}{l}
\includegraphics[height=\legendheightlarge]{figures/legend_simp.pdf}
\end{tabular}
\end{subtable}

\begin{subtable}[]{0.48\linewidth}
\centering
% \resizebox{\linewidth}{!}{%
\begin{tabular}{@{}p{3mm}@{}c@{}c@{}c@{}}
        & \makecell{\tiny{$\sigma=1.5$}}
        & \makecell{\tiny{$\sigma=2.0$}}
        & \makecell{\tiny{$\sigma=4.0$}}\\
\rowname{\makecell{Radius $r$}}&
\includegraphics[height=\utilheightlarge]{figures/Freeway_stepwise_sigma-1.5.pdf}&
\includegraphics[height=\utilheightlarge]{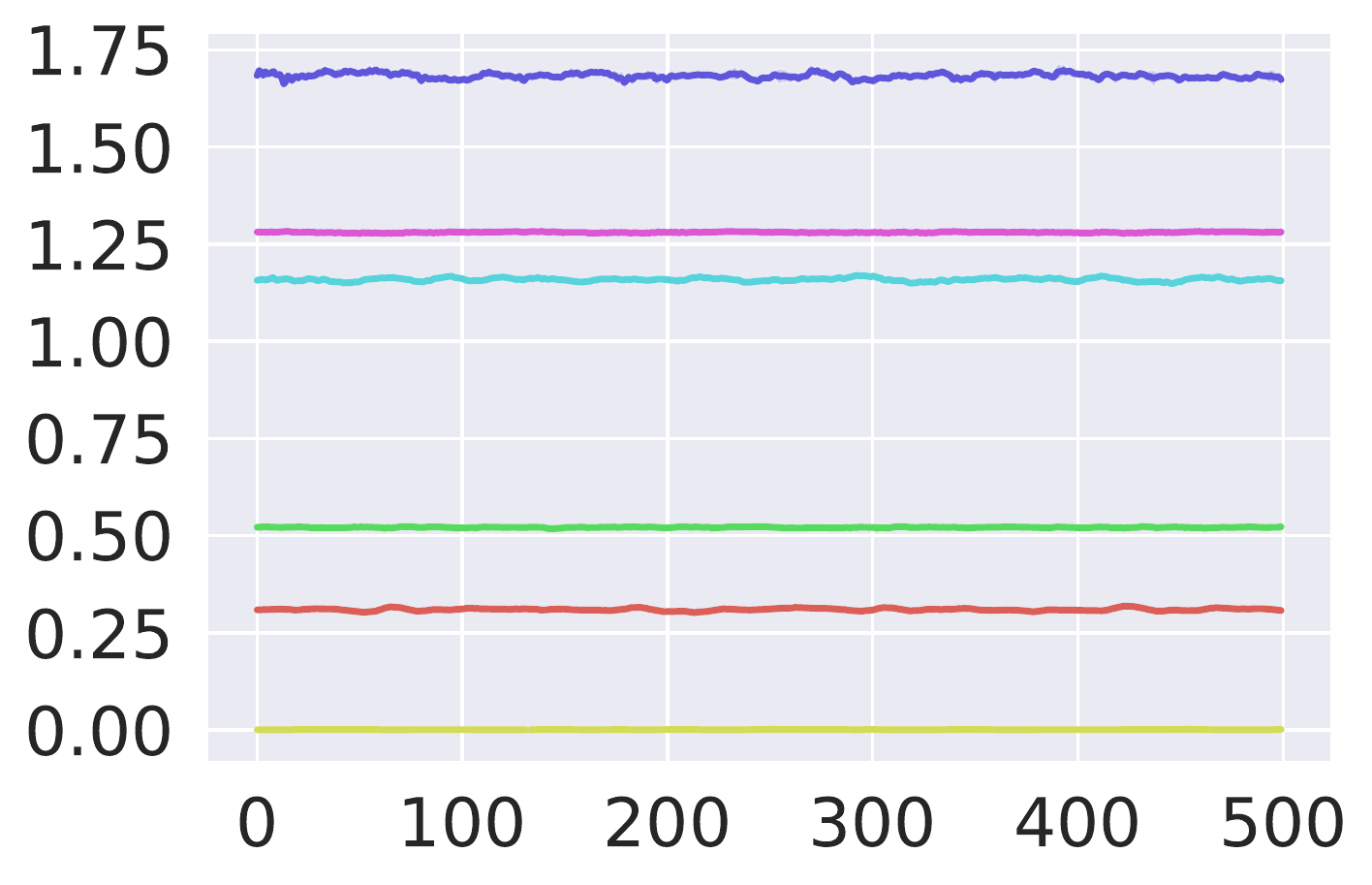}&
\includegraphics[height=\utilheightlarge]{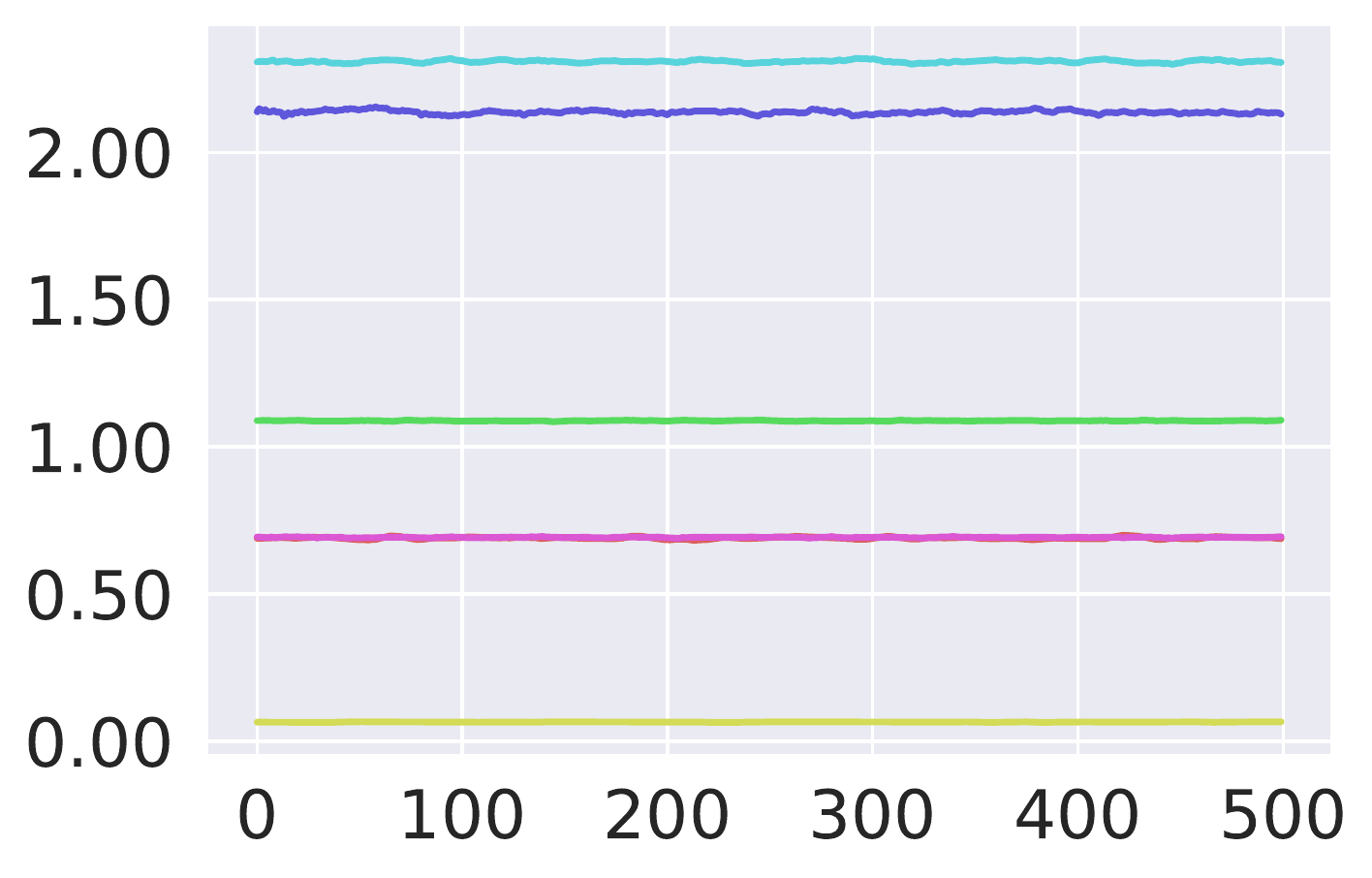}\\[-1.2ex]
        & \makecell{\tiny{time step $t$}}
        & \makecell{\tiny{time step $t$}}
        & \makecell{\tiny{time step $t$}}
\end{tabular}
% }
\caption{\small Certified radius $R_t$ along time steps}\label{tab:cert-rad-large}
\end{subtable}
\begin{subtable}[]{0.48\linewidth}
\centering
% \resizebox{\linewidth}{!}{%
\begin{tabular}{@{}p{3mm}@{}c@{}c@{}c@{}}
        & \makecell{\tiny{$\sigma=1.5$}}
        & \makecell{\tiny{$\sigma=2.0$}}
        & \makecell{\tiny{$\sigma=4.0$}}\\
\rowname{\makecell{$\uj$}}&
\includegraphics[height=\utilheightlargeada]{figures/Freeway_adasearch_sigma-1.5.pdf}&
\includegraphics[height=\utilheightlargeada]{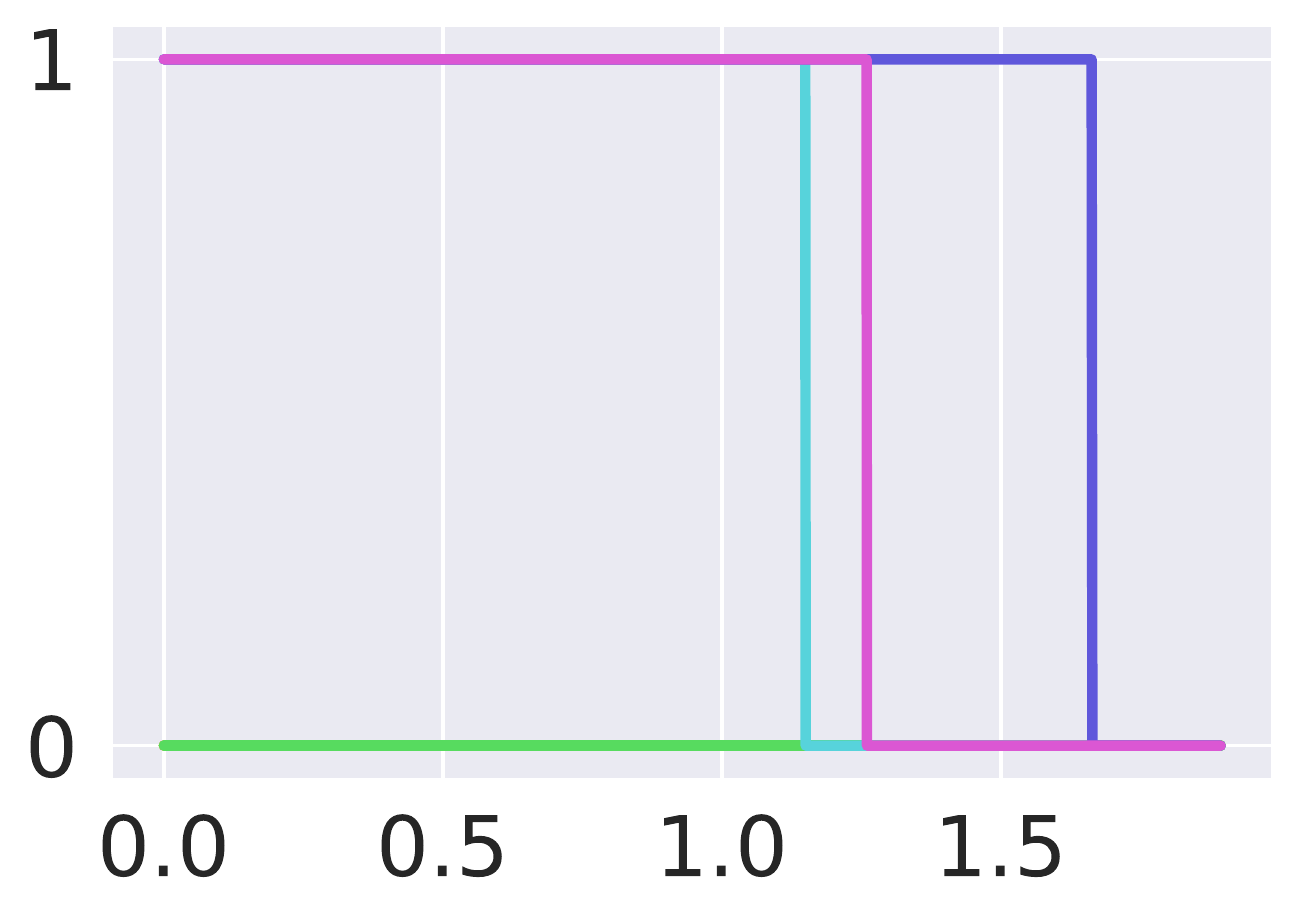}&
\includegraphics[height=\utilheightlargeada]{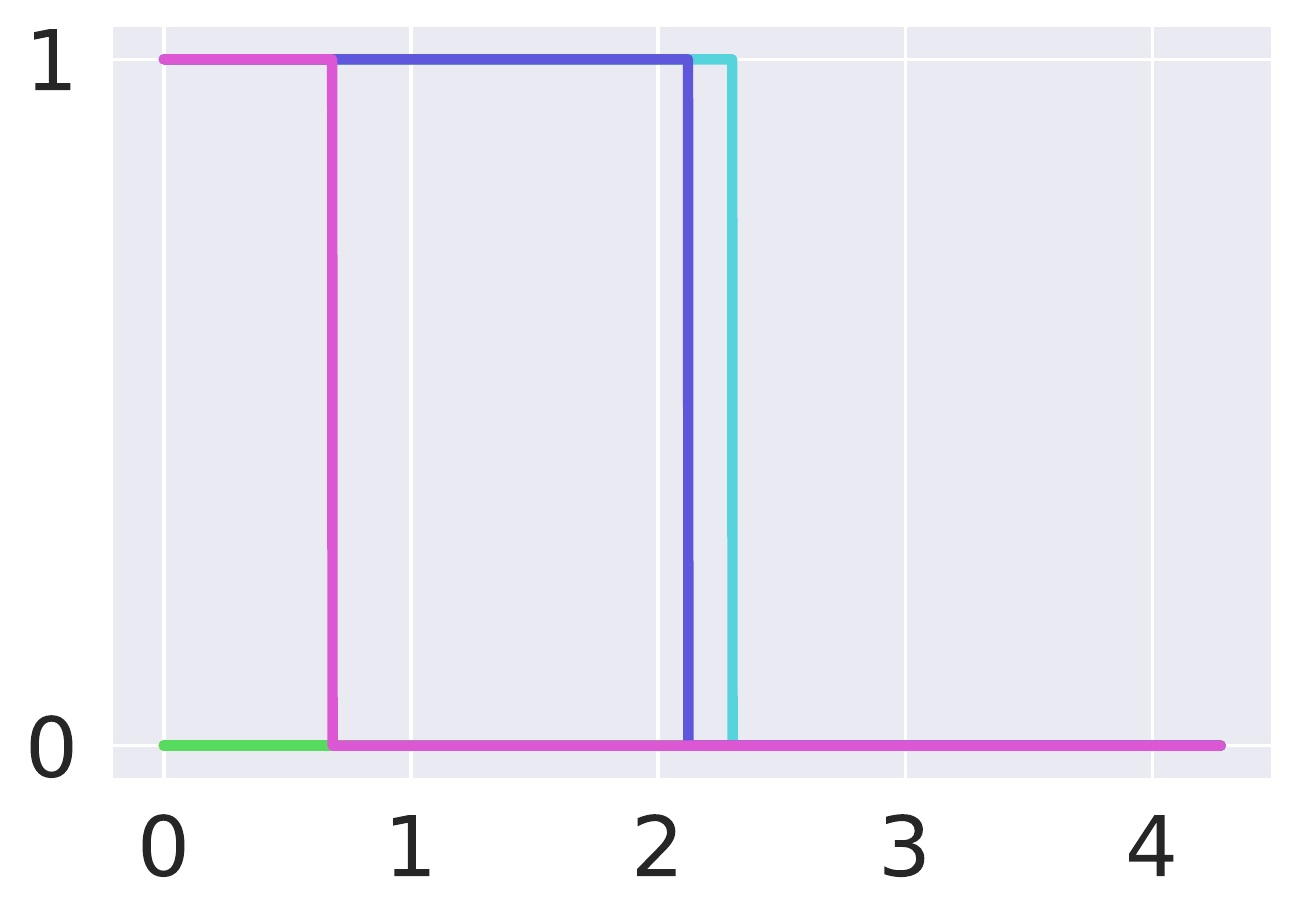}\\[-1.2ex]
        & \makecell{\tiny{attack $\eps$}}
        & \makecell{\tiny{attack $\eps$}}
        & \makecell{\tiny{attack $\eps$}}
\end{tabular}
% }
\caption{\small Absolute lower bound $\uj$ w.r.t. attack $\eps$}\label{tab:cert-uj-large}
\end{subtable}

}
\caption{\small Robustness certification for (a) \textit{per-state action} in terms of certified radius $r$ at all time steps, and (b) \textit{cumulative reward} in terms of absolute lower bound $\uj$ w.r.t. the attack magnitude $\eps$. The results are reported under large smoothing parameter $\sigma$ compared with those evaluated in~\Cref{fig:statewise} and~\Cref{fig:all-bounds} in the main paper.
% Each column corresponds to one smoothing variance $\sigma$. 
% We consider six methods: StdTrain, GaussAug, AdvTrain, RegPGD, RegCVX, and RadialRL, with the last three being SOTA. 
% The shaded area represents the standard deviation which is small.
% RadialRL is the most certifiably robust method on Freeway, while RegCVX is the most robust on Pong.
}%
\label{fig:statewise-large}
\vspace{-1em}
\end{figure}

Since the robustness of the methods RadialRL, SA-MDP (CVX), and SA-MDP (PGD) on Freeway have been shown to improve with the increase of the smoothing parameter $\sigma$ in~\Cref{sec:exp}, we subsequently evaluate their performance under even larger $\sigma$ values. We present the certification results in~\Cref{fig:statewise-large} and their benign performance in~\Cref{fig:clean-full}.
\newlength{\wrapheightb}
\settoheight{\wrapheightb}{\includegraphics[width=.40\textwidth]{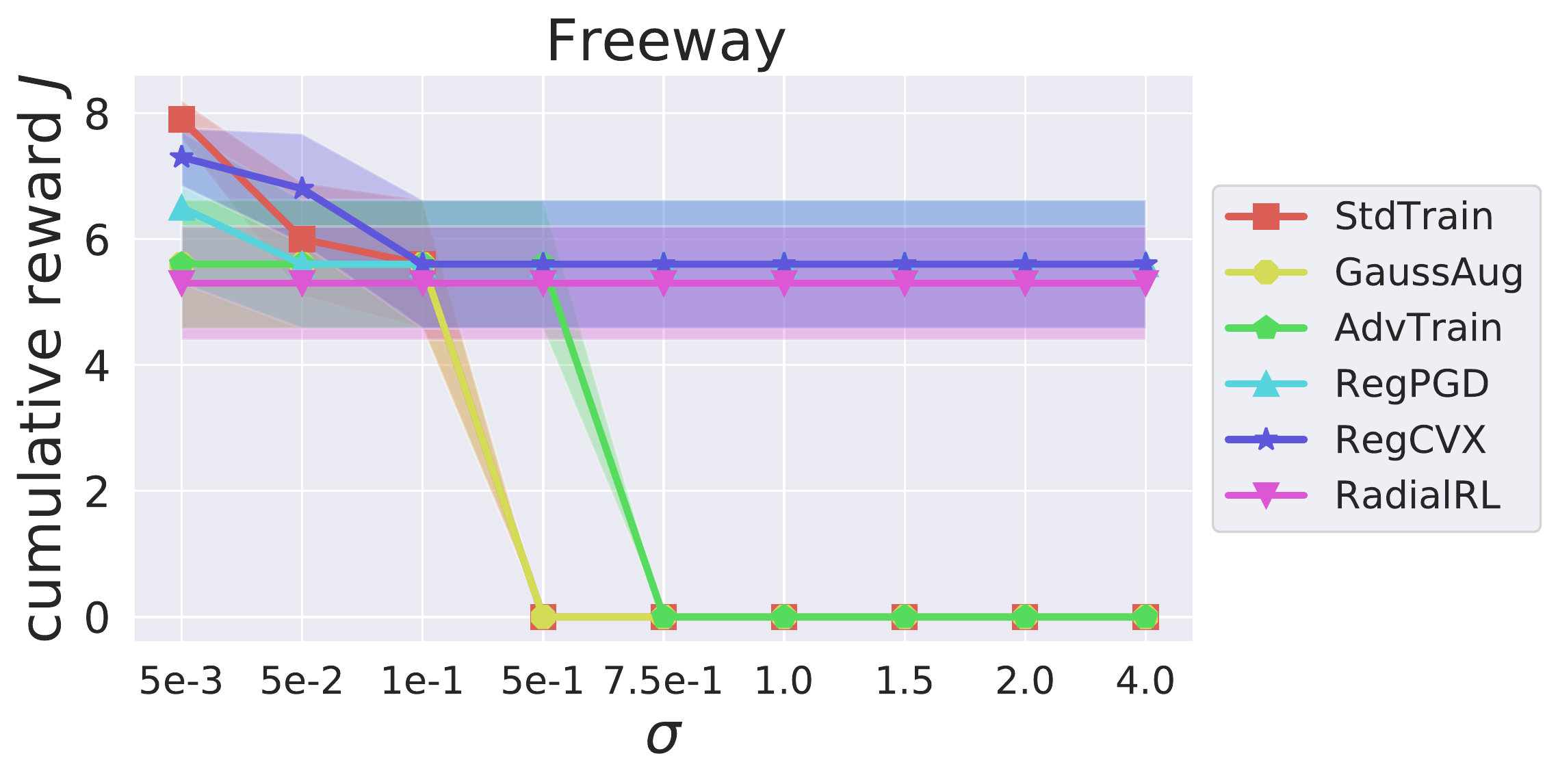}}%

\begin{wrapfigure}{r}{0.40\textwidth}
\vspace{-7mm}
\begin{center}
\includegraphics[height=\wrapheightb]{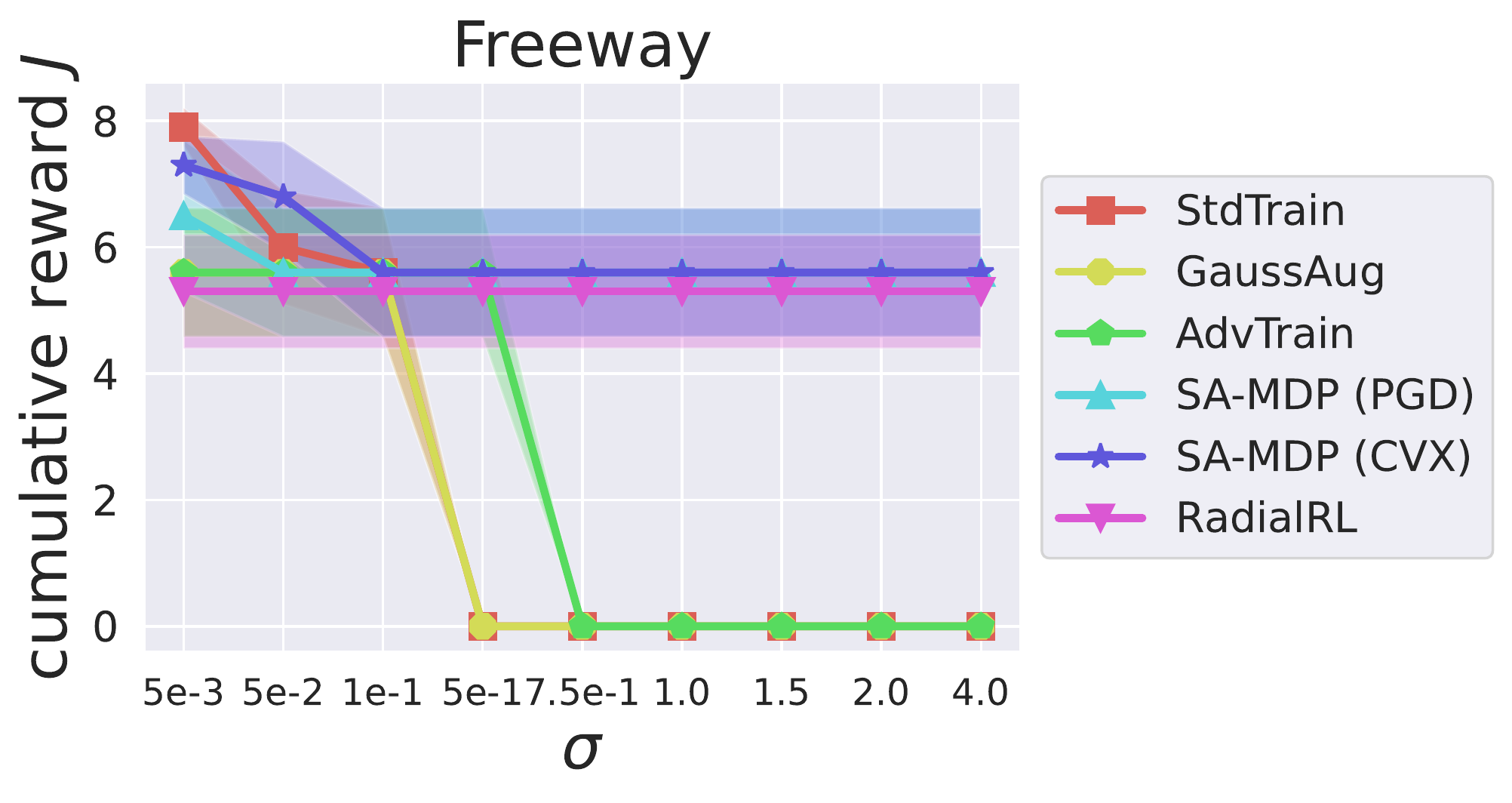}
\end{center}
\vspace{-5mm}
\caption{\small Benign performance of  locally smoothed policy $\tpi$ under a larger range of smoothing parameter $\sigma$ with clean state observations.}\label{fig:clean-full}
\end{wrapfigure}

%\newpage
First of all, as~\Cref{fig:clean-full} reveals, RadialRL, SA-MDP (CVX), and SA-MDP (PGD) can tolerate quite large noise---the benign performance of these methods does not drop even for $\sigma$ up to $4.0$.
For the remaining three methods, the magnitude of noise they can tolerate is much smaller, as already presented in~\Cref{fig:statewise-clean} in the main paper.
On the other hand, although all the three methods RadialRL, SA-MDP (CVX), and SA-MDP (PGD) invariably attain good benign performance as $\sigma$ grows, their certified robustness does not equally increase.
As~\Cref{fig:statewise-large} shows, the certified robustness of SA-MDP (PGD) increases the fastest, while that of RadialRL drops a little. 
The observation corresponds exactly to our discussion regarding the tradeoff between value function smoothness and the margin between the values to the top two actions in~\Cref{sec:cert-rad}.

\subsection{Periodic Patterns for Per-State Robustness}
\label{append:pong-periodic}

We specifically study the periodic patterns for per-state robustness in Pong and present the results in~\Cref{fig:pong-periodic}.
In the game frames, our agent controls the green paddle on the right. Whenever the opponent (the orange paddle on the left) misses catching a ball, we earn a point. 
We present two periods (frame 320-400 and frame 400-480) where
our agent each earns a point in the figure above. The total points increase from $2$ to $3$, and from $3$ to $4$, respectively.

We note that Pong is a highly periodic game, with each period consisting of the following four game stages corresponding to distinct stages of the \textit{certified robustness}:
1) A ball is initially fired and set flying towards the opponent. In this stage, the certified radius remains low, since our agent does not need to take any critical action.
2) The opponent catches the ball and reverses the direction of the ball. In this stage, the certified radius gradually increases, since the moves of our agent serve as the preparation to catch the ball.
3) The ball bounces from the ground, flies towards our agent, gets caught, and bounces back. In this stage, the certified radius increases rapidly, since every move is critical in determining whether we will be able to catch the ball, \ie, whether our point~(action value) will increase.
4) The opponent misses the ball, and the ball disappears from the game view. In this stage, the certified radius drastically plummets to a low value, since there is no clear signal (\ie, ball) in the game for the agent to take any determined action.

The illustration not only helps us understand the semantic meaning of the certified radius, but can further provide guidance on designing empirically robust RL algorithms that leverage the information regarding the high/low certified radius of the states.

\begin{figure}[]
    \centering
    \begin{tabular}{c|c}
\includegraphics[width=.475\linewidth]{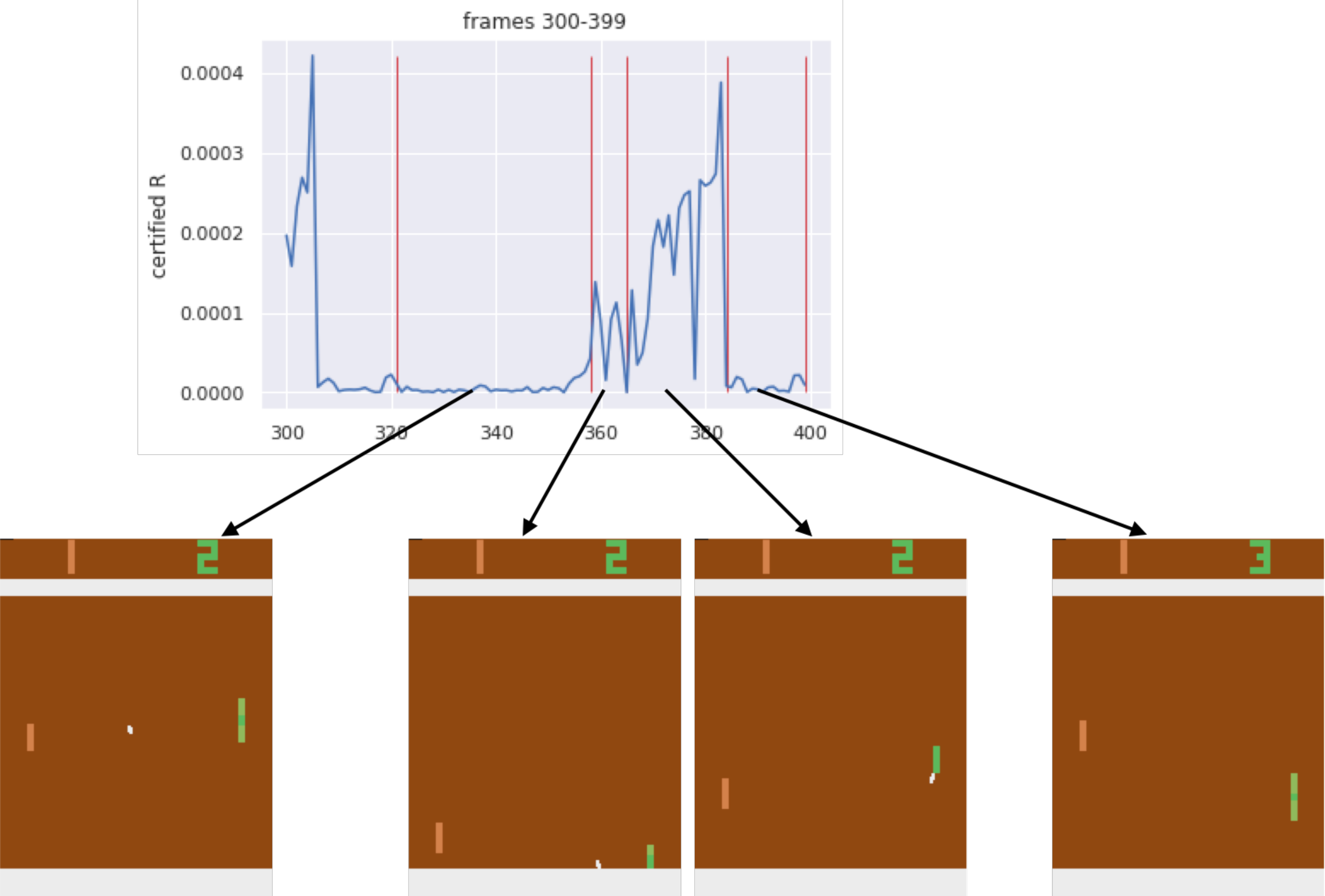}         &  \includegraphics[width=.475\linewidth]{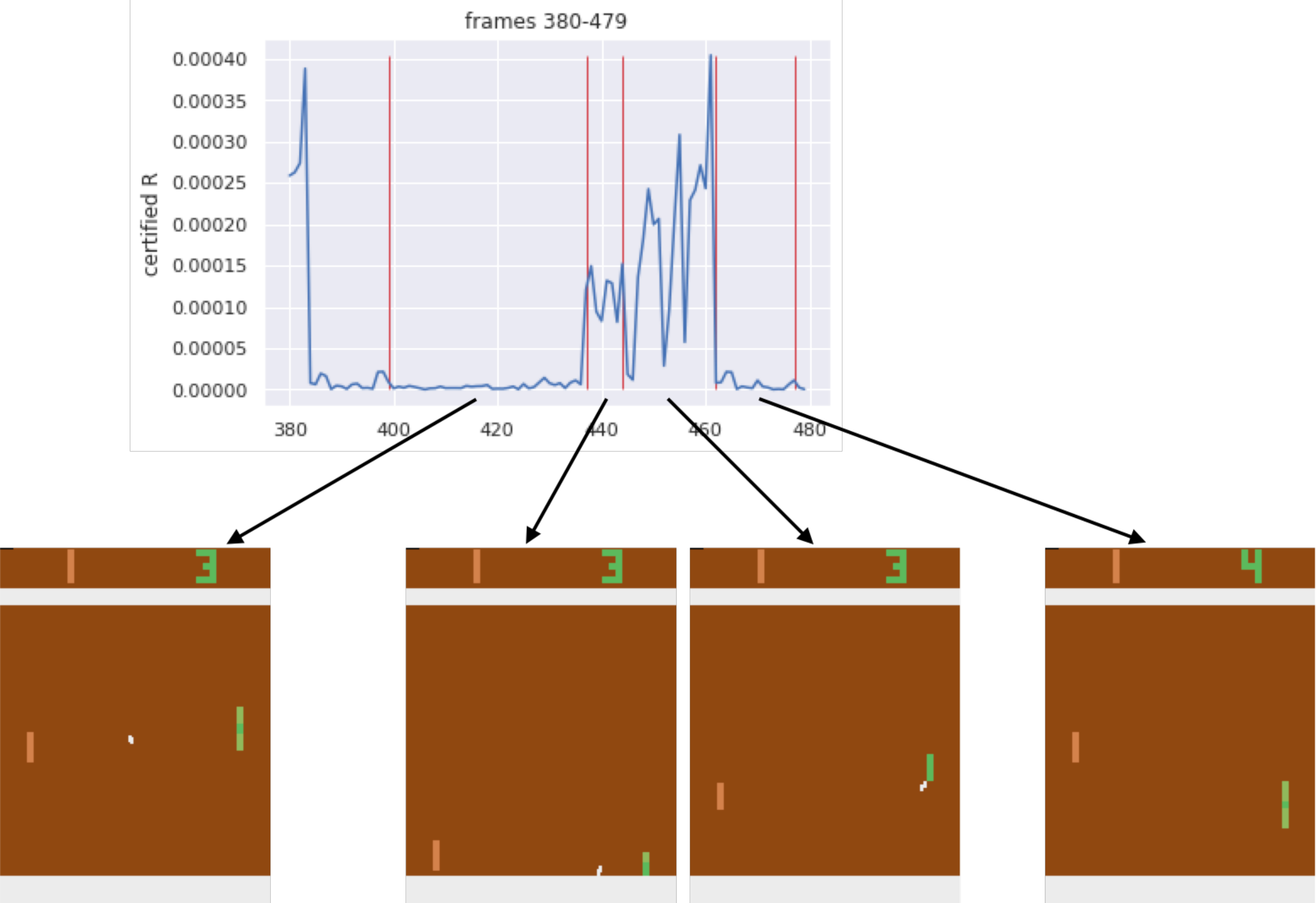}
    \end{tabular}
    \caption{\small Periodic patterns in Pong. 
    The figure shows two periods (left and right), including the certified radius $r$ w.r.t. the time steps (above), and
    the selected game frames corresponding to different stages in each period (below).
    Different periods are highly similar.}
    \label{fig:pong-periodic}
\end{figure}

\subsection{Detailed Discussions and Selection of Smoothing Parameter $\sigma$}
\label{append:discuss-sigma}

We first take Freeway as an example to concretely explain the impact of smoothing variance from the perspective of benign performance and certiﬁed results, and then introduce the guidelines for the selection of $\sigma$ based on different certification criteria.

On \uit{Freeway}, 
as $\sigma$ increases, the \textit{benign performance} of the locally smoothed policy $\tpi$ of StdTrain, GaussAug, and AdvTrain decrease rapidly in~\Cref{fig:statewise-clean}. 
Also, the \textit{certified radius} of these methods barely gets improved and remains low even as $\sigma$ increases.
For the \pbound $\ujp$, these three methods similarly suffer a drastic decrease of \textit{cumulative reward} as $\sigma$ grows.
This is however not the case for SA-MDP (PGD), SA-MDP (CVX), and RadialRL, for which not only the \textit{benign performance} do not experience much degradation for $\sigma$ as large as $1.0$, but the \textit{certified radius} even steadily increases. 
% In addition, in terms of the \pbound $\ujp$,
% as $\sigma$ grows,
% the upper bound of $\eps$ that can be certified increases rapidly, while $\ujp$ almost does not decrease.
In addition, in terms of the \pbound $\ujp$,
as $\sigma$ grows,
even we increase the attack $\eps$ simultaneously,
$\ujp$ almost does not decrease.
This indicates that larger smoothing variance can bring more robustness to SA-MDP (PGD), SA-MDP (CVX), and RadialRL.
% without much sacrifice.
% on benign performance or cumulative reward.
% On \underline{\textit{Pong}}, following the same rationale, we can see that a smoothing parameter in the range $0.01$-$0.03$ shall be suitable for almost all methods, balancing the tradeoff between benign performance and robustness.
% However, we note that the improvement of robustness is limited, given that the magnitude of attack that can be certified is quite small. This is aligned with our previous conclusion that models trained on Pong are overall less robust.
% \input{fig_desc/fig_pong_periodic}

\textbf{\staters.}\quad
% For the local smoothing method, as discussed in~\Cref{sec:cert-act-thm}, the smoothing parameter $\sigma$ reflects the tradeoff between value function smoothness and the margin between the values of top two actions, which can be observed from~\Cref{fig:statewise} (\resp,~\Cref{fig:statewise-large}) and~\Cref{fig:statewise-clean} (\resp,~\Cref{fig:clean-full}).
We discuss how to find the sweet spot that enables high benign performance and robustness simultaneously.
As $\sigma$ increases, the benign performance of different methods  generally decreases due to the noise added, while not all methods decrease at the same rate. 
On \underline{\textit{Freeway}}, StdTrain and GaussAug decrease the fastest, while SA-MDP (PGD), SA-MDP (CVX), and RadialRL do not degrade much even for $\sigma$ as large as $4.0$.
Referring back to~\Cref{fig:statewise} and~\Cref{fig:statewise-large}, we see that the certified radius of StdTrain and GaussAug remains low even though $\sigma$ increases, thus the best parameters for these two methods are both around $0.1$.

For SA-MDP (PGD) and SA-MDP (CVX), their certified radius instead steadily increases as $\sigma$ increases to $4.0$, indicating that larger smoothing variance can bring more robustness to these two methods without sacrificing benign performance, offering certified radius larger than $2.0$.
As to RadialRL, its certified radius reaches the peak at $\sigma=1.0$ and then decreases under larger $\sigma$. 
Thus, $1.0$ is the best smoothing parameter for RadialRL which offers a certified radius of $1.6$.
% On \underline{\textit{Pong}}, similarly we can see that a smoothing variance in the range $0.01-0.03$ shall be suitable for almost all methods.
Generally speaking, it is feasible to select an appropriate smoothing parameter $\sigma$ for all methods to increase their robustness, and robust methods will benefit more from this by admitting larger smoothing variances and achieving larger certified radius.

\textbf{\glbrs.}\quad
We next discuss how to leverage~\Cref{fig:all-bounds} to assist with the selection of $\sigma$.
For a given $\eps$ or a known range of $\eps$, we can compare between the lower bounds corresponding to different $\sigma$ and select the one that gives the highest lower bound.
For example, if we know that $\eps$ is small in advance, we can go for a relatively small $\sigma$ which retains a quite high lower bound. 
If we know that $\eps$ will be large, we will instead choose among the larger $\sigma$'s, since a larger $\eps$ will not fall in the certifiable ranges of smaller $\sigma$'s. 
According to this guideline, we are able to obtain lower bounds of quite high value
for RadialRL, SA-MDP (CVX), and SA-MDP (PGD) on Freeway under a larger range of attack $\eps$.
% However, on Pong, we do not have the same luck due to the lack of strong certified robustness for the methods evaluated here.

\textbf{\adasearch.}\quad
Apart from the conclusion that larger smoothing parameter $\sigma$ often secures higher lower bound $\uj$, we point out that
smaller smoothing variances may be able to lead to a higher lower bound than larger smoothing variances for a certain range of $\eps$.
This is almost always true for very small $\eps$ since smaller $\sigma$ is sufficient to deal with the weak attack without sacrificing much empirical performance, \eg, in~\Cref{fig:all-bounds},
when $\eps=0.001$, SA-MDP (CVX) on Freeway can achieve $\uj=3$ at $\sigma=0.005$ while only $\uj=1$ for large $\sigma$.
This can also happen to robust methods at large $\eps$, \eg,
when $\eps=1.2$, RadialRL on Freeway achieves $\uj=2$ at $\sigma=0.75$ while only $\uj=1$ at $\sigma=1.0$.
Another case is when $\sigma$ is large enough such that further increasing $\sigma$ will not bring additional robustness, \eg, in~\Cref{fig:statewise-large},
when $\eps=1.5$, RadialRL on Freeway achieves $\uj=1$ at $\sigma=1.5$ while only $\uj=0$ at $\sigma=4.0$.

\subsection{Computational Cost}
\label{append:comp-cost}

Our experiments are conducted on  GPU machines, including GeForce RTX 3090, GeForce RTX 2080 Ti, and GeForce RTX 1080 Ti. 
The running time of $m=10000$ forward passes with sampled states for per-state smoothing ranges from $2$ seconds to $9$ seconds depending on the server load.
Thus, for one experiment of $\staters$ with trajectory length $H=500$ and $10$ repeated runs, the running time ranges from $2.5$ hours to $12.5$ hours.
\glbrs is a little more time-consuming than \staters, but the running time is still in the same magnitude. 
For trajectory length $H=500$ and sampling number $m=10000$, most of our experiments finish within $3$ hours.
For \adasearch specifically, the running time depends on the eventual size of the search tree, and therefore differs tremendously for different methods and different smoothing variances, ranging from a few hours to $4$ to $5$ days.

% \section{The Complete \adasearch algorithm}
% \label{sec:append-adasearch}

%     \linyim{In \Cref{alg:adaptive-search-full}, we further expand \adasearch algorithm shown in \Cref{alg:adaptive-search} with the full set of optimization tricks. We introduced these tricks in \Cref{sec:tree-search-impl}.}

% \linyi{I guess the following appendices (G and H) need to be removed from submission? Update: commented out and you can recover though}

% \input{certify-algs}

% \input{tx_draft}

% \bibliography{refs}

%\input{fig_desc/fig_cartpole}

% 
\vspace{-1mm}
\subsection{Discussion on Evaluation Results of CartPole}
\label{append:cartpole}

% \looseness=-1
We provide the evaluation results for \sysname on CartPole, comparing nine RL algorithms: the six algorithms evaluated in our main paper (StdTrain, GaussAug, AdvTrain, SA-MDP (PGD), SA-MDP (CVX), and RadialRL), as well as three additional algorithms (CARRL, NoisyNet, and GradDQN). The detailed descriptions of these algorithms are provided in~\Cref{append:impl-rl}.
We present the evaluation results in~\Cref{tab:cartpole-figs}.
%\zijian{finish discussion on CartPole results}

\textcolor{black}{\textbf{Impact of Smoothing Parameter $\sigma$.}\quad
In CartPole game, we draw similar conclusions regarding the impact of smoothing variance from the results given by $\staters$, \glbrs and \adasearch. For small $\sigma$, different methods are indistinguishable. As $\sigma$ increases, CARRL demonstrates its advantage as $\sigma$ increases, but AdvTrain and GradDQN tolerate large noise. One of differences when evaluating with \glbrs and \adasearch is that RadialRL can hold a high performance and even sometimes can be better than CARRL. Another one is that the lower bound for GradDQN is lower than all other methods in \adasearch. }

\textcolor{black}{\textbf{Tightness of the certiﬁcation $\uje$, $\ujp$ and $\uj$.}\quad 
In CartPole, we also compare the empirical cumulative rewards achieved under PGD attacks with our certiﬁed lower bounds $\uje$, $\ujp$ and $\uj$. Demonstrated by \Cref{tab:cartpole-bounds}, first of all, the correctness of our bounds is validated because the empirical results are consistently lower bounded by our certiﬁcations. Compared with the loose expectation bound $\uje$, the improved percentile bound $\ujp$ is much tighter. Compared with these two methods, the absolute lower bound $\uj$ is even tighter, especially noticing the zero gap between the certiﬁcation and the empirical result when the attack magnitude is not too large. Finally, methods with the same empirical results under attack may achieve very different certified lower bounds $\uje$, $\ujp$ and $\uj$, showing the importance of certification.
}

%\vspace{2em}

\textcolor{black}{\textbf{Environment Properties.}\quad 
Different from Pong and Freeway, CartPole is an \env with low dimension \textit{states}, in contrast to the high dimensional Atari games evaluated in~\Cref{sec:exp}.
This gives rise to several outcomes. 
First, StdTrain can tolerate noise well in CartPole. Second, CARRL demonstrates its advantage in low dimensional game, which is consistent with the empirical observations in~\citet{everett2021certifiable}. }

{
\renewcommand{\thesubfigure}{\alph{subfigure}}

\begin{figure}

\newlength{\cputilheighta}
\settoheight{\cputilheighta}{\includegraphics[width=.160\linewidth]{figures/Freeway_stepwise_sigma-0.001.pdf}}%

\newlength{\cpwrapheighta}
\settoheight{\cpwrapheighta}{\includegraphics[width=.400\textwidth]{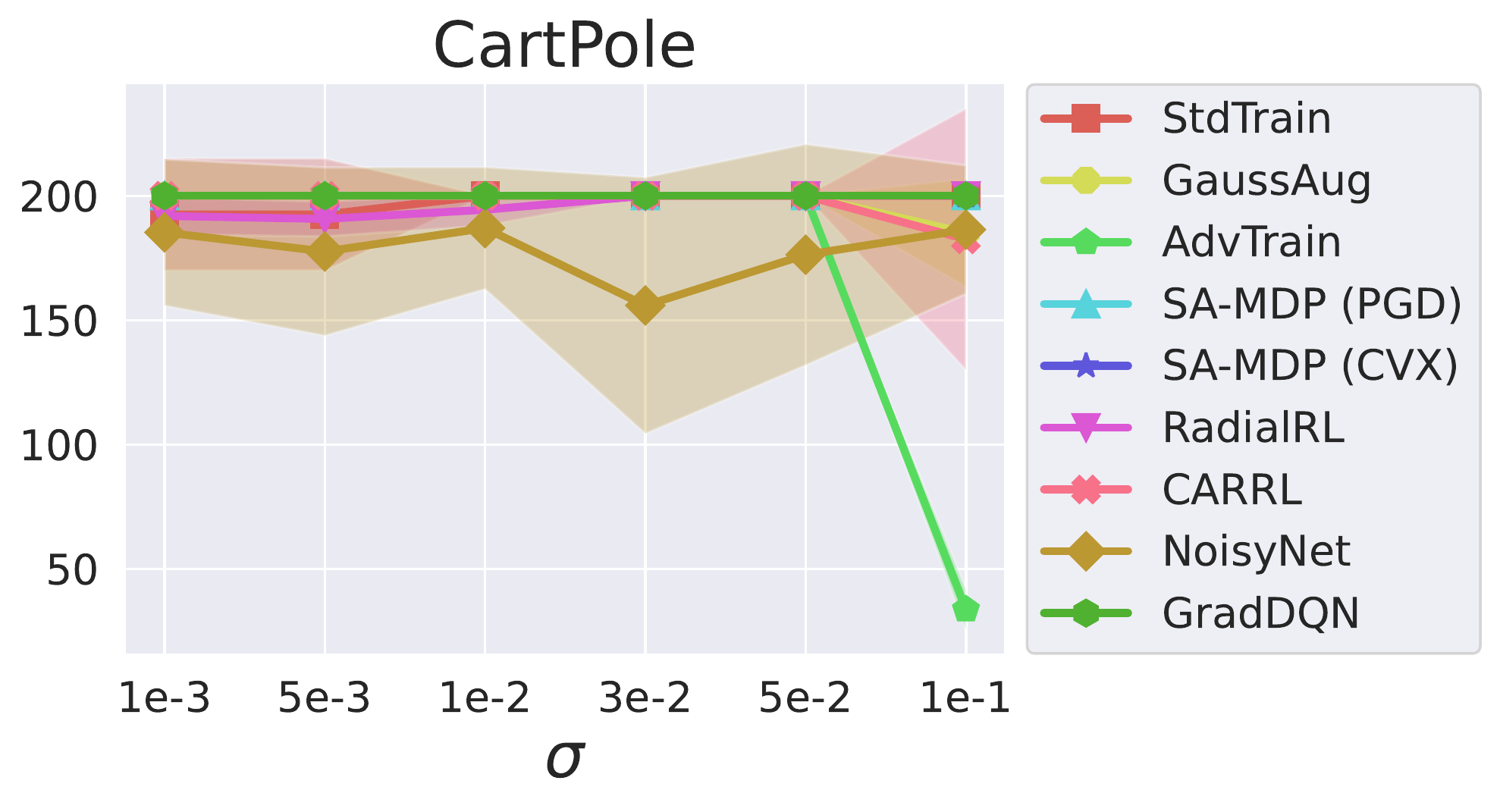}}%

\newlength{\cplegendheight}
\setlength{\cplegendheight}{0.3\cputilheighta}%

\newcommand{\cprowname}[1]% #1 = text
{\rotatebox{90}{\makebox[\cputilheighta][c]{\tiny #1}}}

\newlength{\cputilheightc}
\settoheight{\cputilheightc}{\includegraphics[width=.160\linewidth]{figures/Freeway_global_mean_sigma-0.001.pdf}}%

\newlength{\cputilheightd}
\settoheight{\cputilheightd}{\includegraphics[width=.165\linewidth]{figures/Freeway_global_median_sigma-0.001.pdf}}%

\newlength{\cputilheightaa}
\settoheight{\cputilheightaa}{\includegraphics[width=.162\linewidth]{figures/Freeway_adasearch_sigma-0.05.pdf}}%

\newlength{\cputilheightcp}
\settoheight{\cputilheightcp}{\includegraphics[width=.160\linewidth]{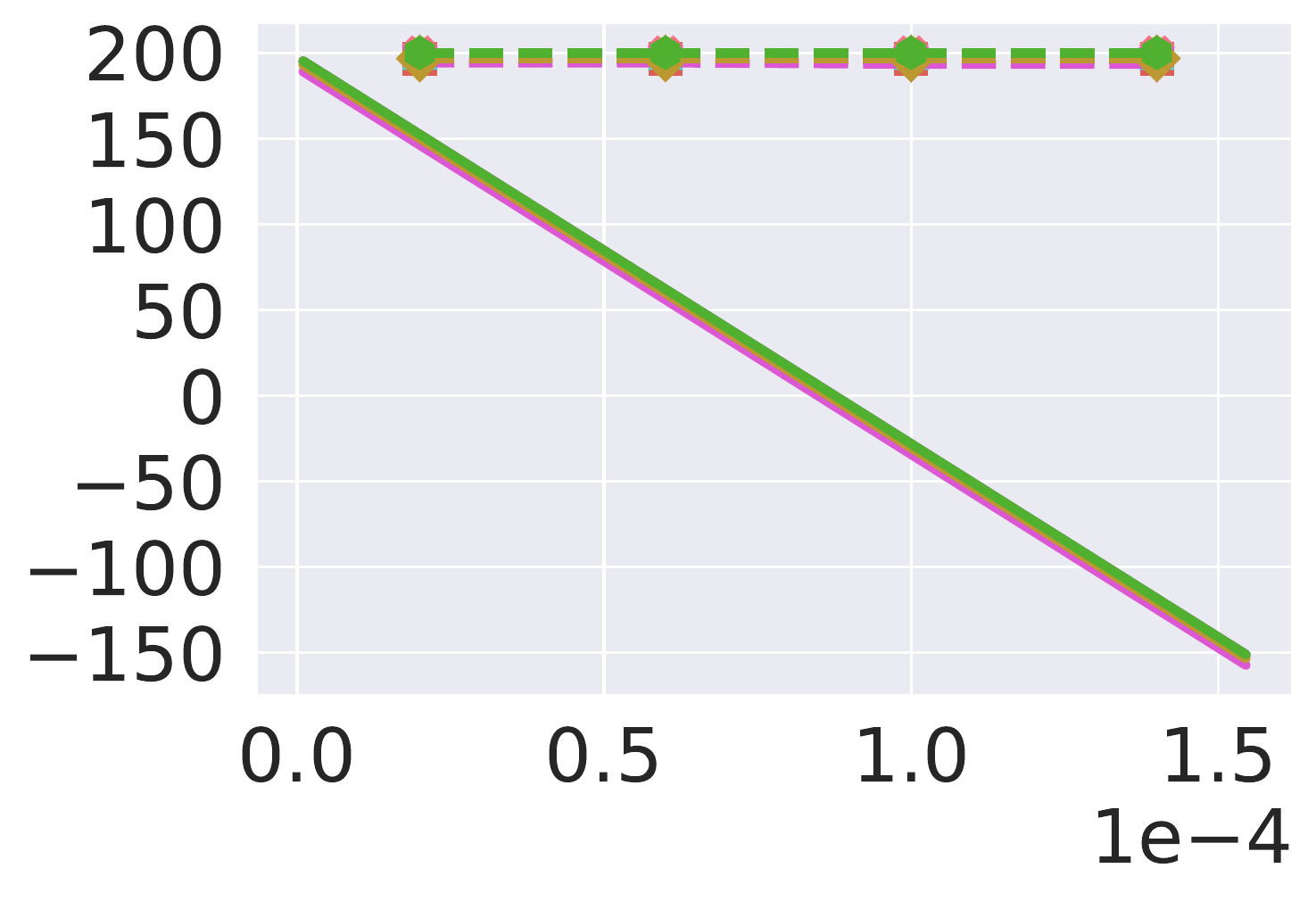}}%

\newlength{\cputilheightdp}
\settoheight{\cputilheightdp}{\includegraphics[width=.165\linewidth]{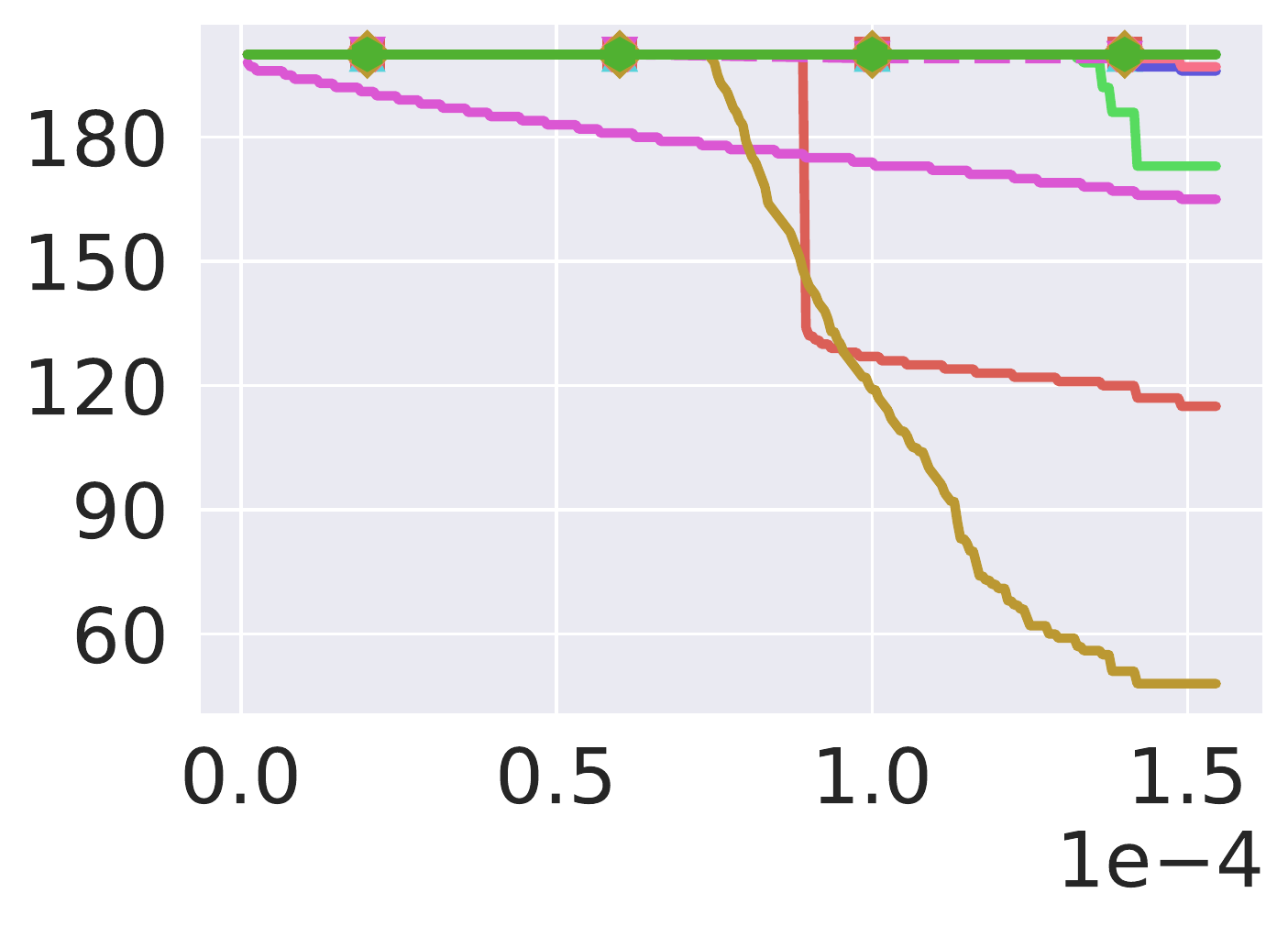}}%

\newlength{\cputilheightaap}
\settoheight{\cputilheightaap}{\includegraphics[width=.162\linewidth]{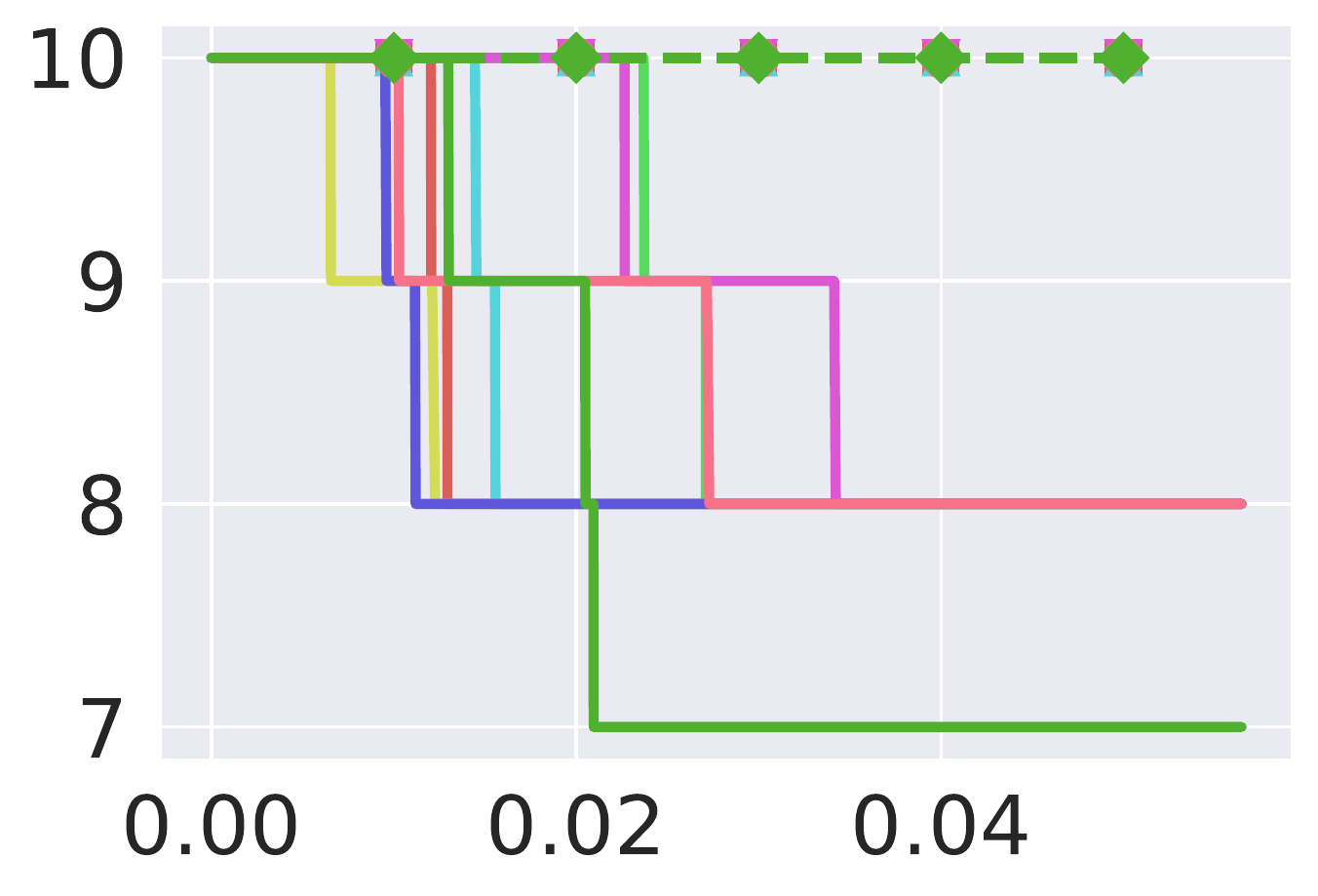}}%

% \newlength{\cputilheight}
% \settoheight{\cputilheight}{\includegraphics[width=.138\linewidth]{figures/twitch-DE: low degree.pdf}}%

% \newlength{\attackheightb}
% \settoheight{\attackheightb}{\includegraphics[width=.138\linewidth]{figures/twitch-DE: high degree.pdf}}%

\newlength{\cplegendheightb}
\setlength{\cplegendheightb}{0.3\cputilheightc}%

\newcommand{\cprownamec}[1]% #1 = text
{\rotatebox{90}{\makebox[\cputilheightc][c]{\tiny #1}}}

\newcommand{\cprownamed}[1]% #1 = text
{\rotatebox{90}{\makebox[\cputilheightd][c]{\tiny #1}}}

\centering

{
\renewcommand{\tabcolsep}{10pt}

\begin{subtable}[]{\linewidth}
\begin{tabular}{l}
\includegraphics[height=\cplegendheight]{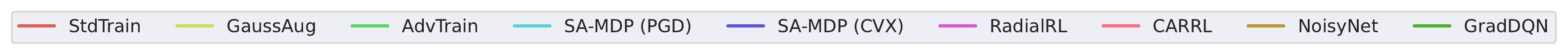}
\end{tabular}
\end{subtable}

\begin{subtable}[]{\linewidth}
\centering
% \resizebox{\linewidth}{!}{%
\begin{tabular}{@{}p{5mm}@{}c@{}c@{}c@{}c@{}c@{}c@{}}
\cprowname{\makecell{CartPole\\Radius $r$}}&
\includegraphics[height=\cputilheighta]{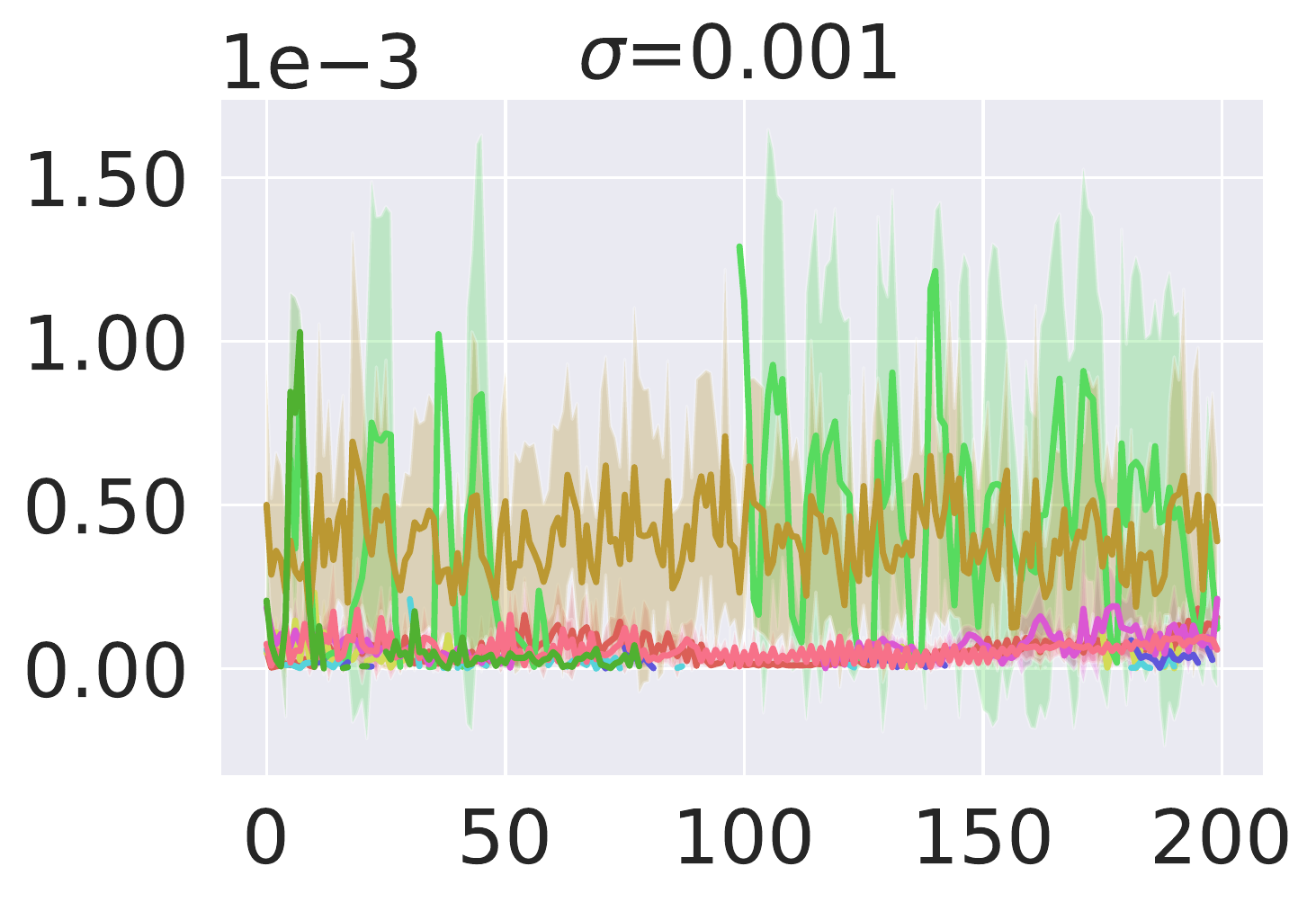}&
\includegraphics[height=\cputilheighta]{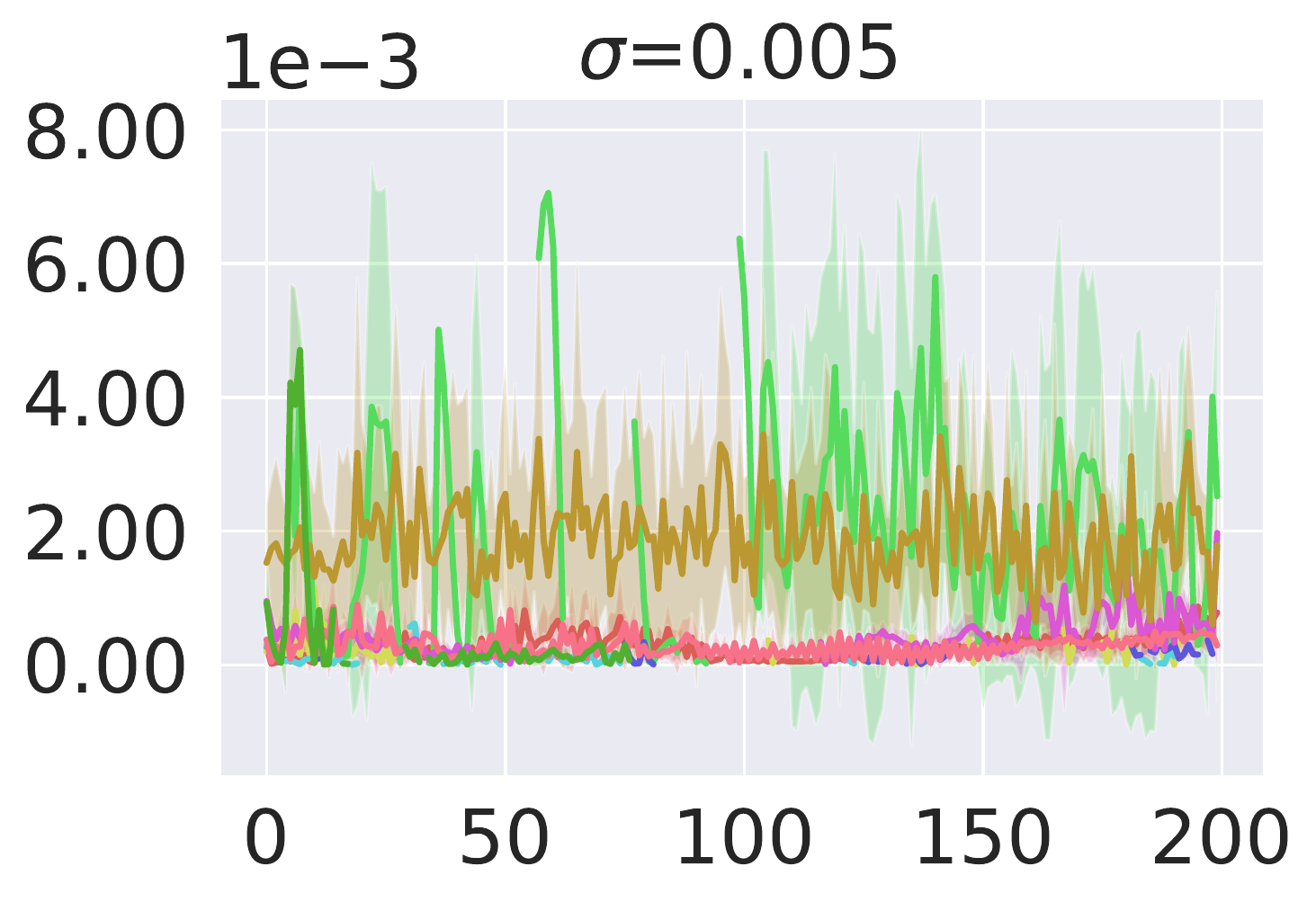}&
\includegraphics[height=\cputilheighta]{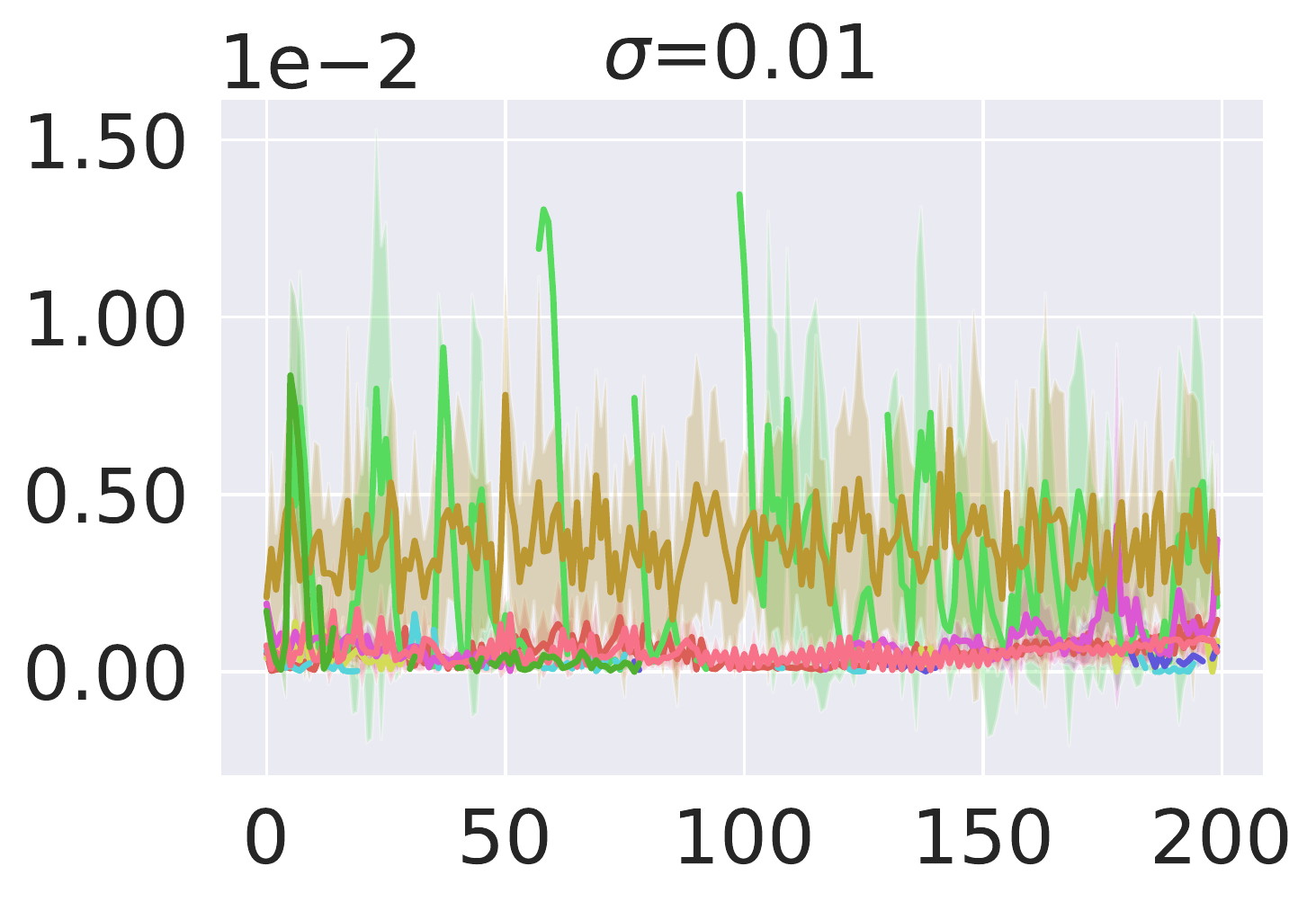}&
\includegraphics[height=\cputilheighta]{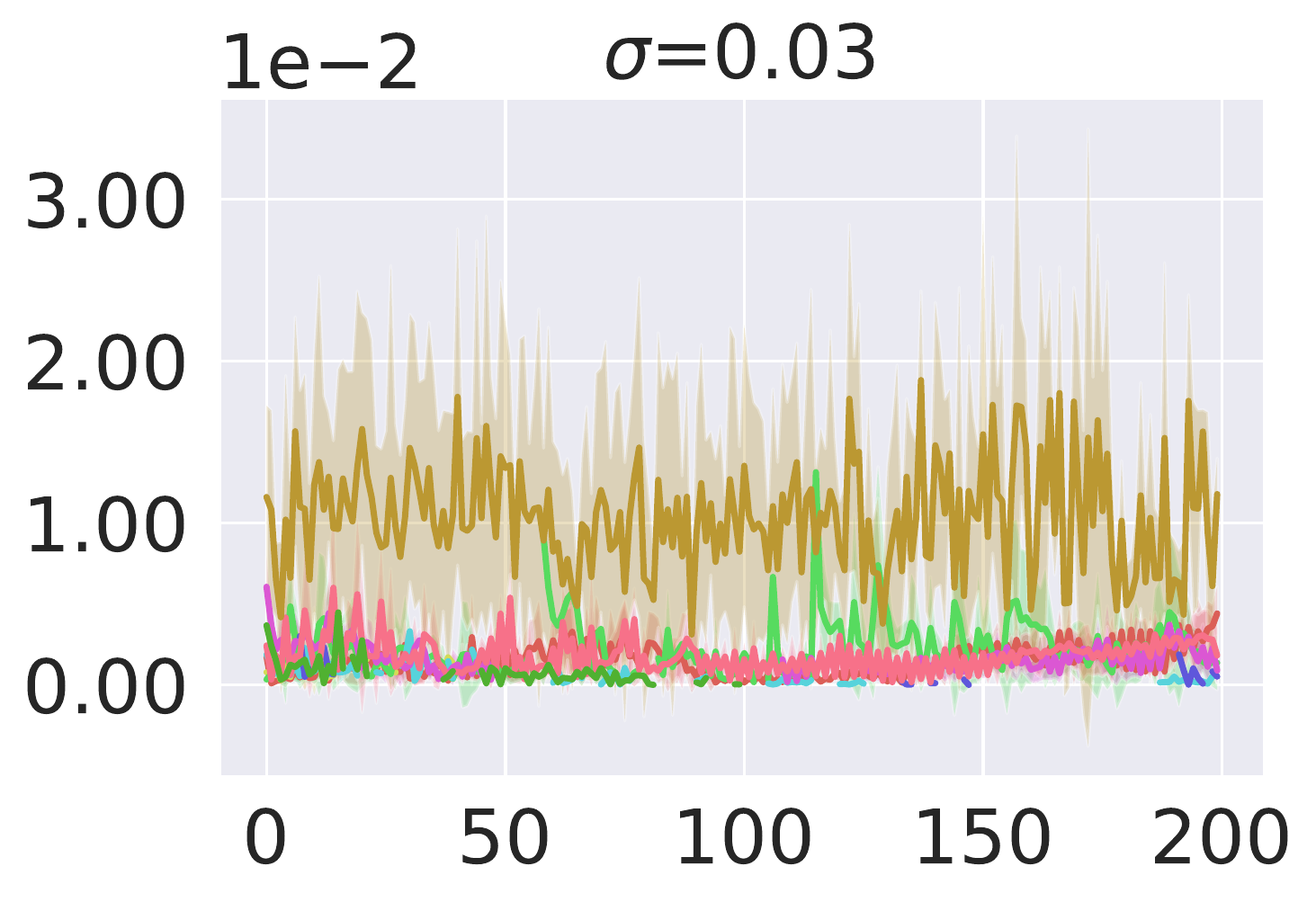}&
\includegraphics[height=\cputilheighta]{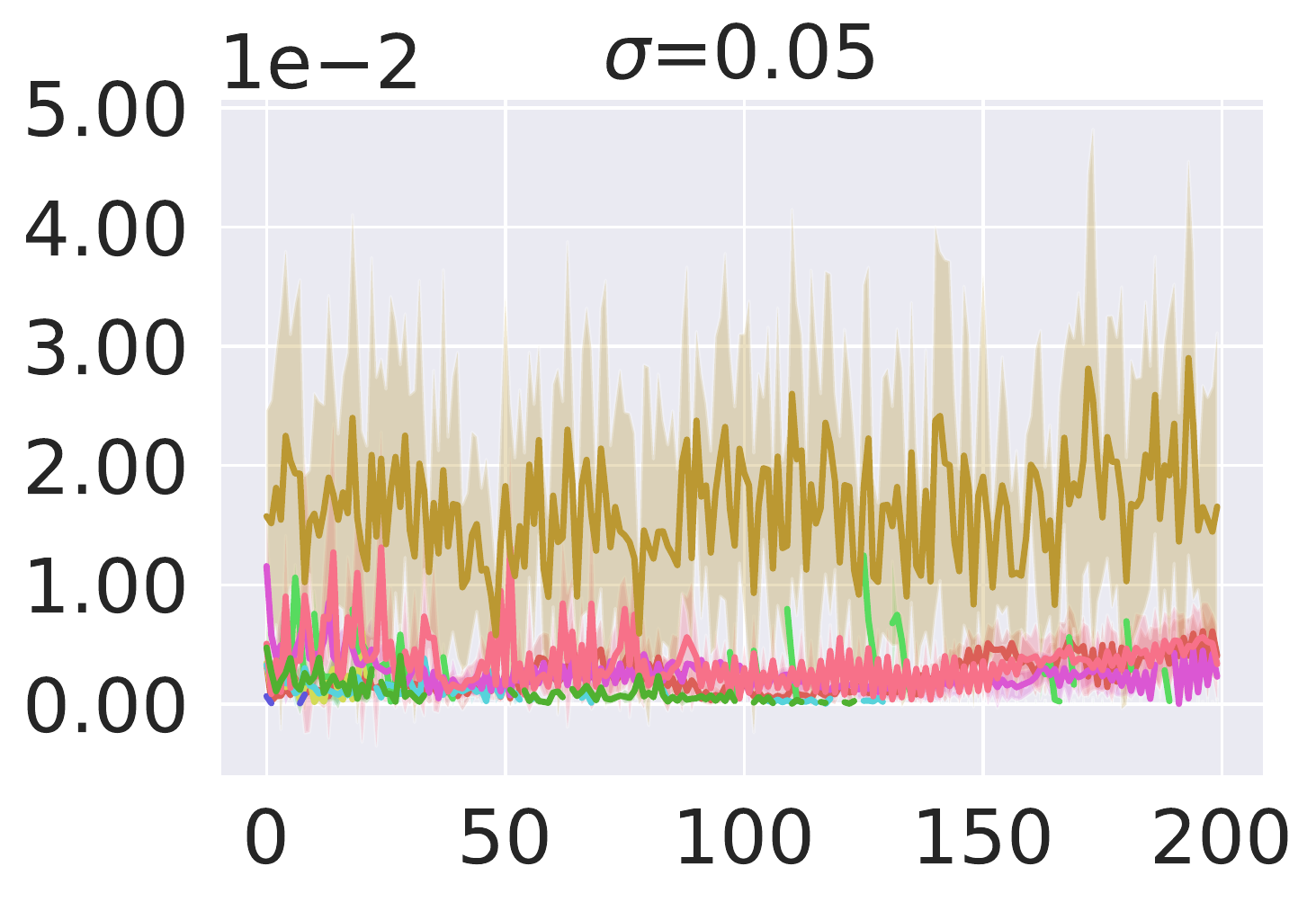}&
\includegraphics[height=\cputilheighta]{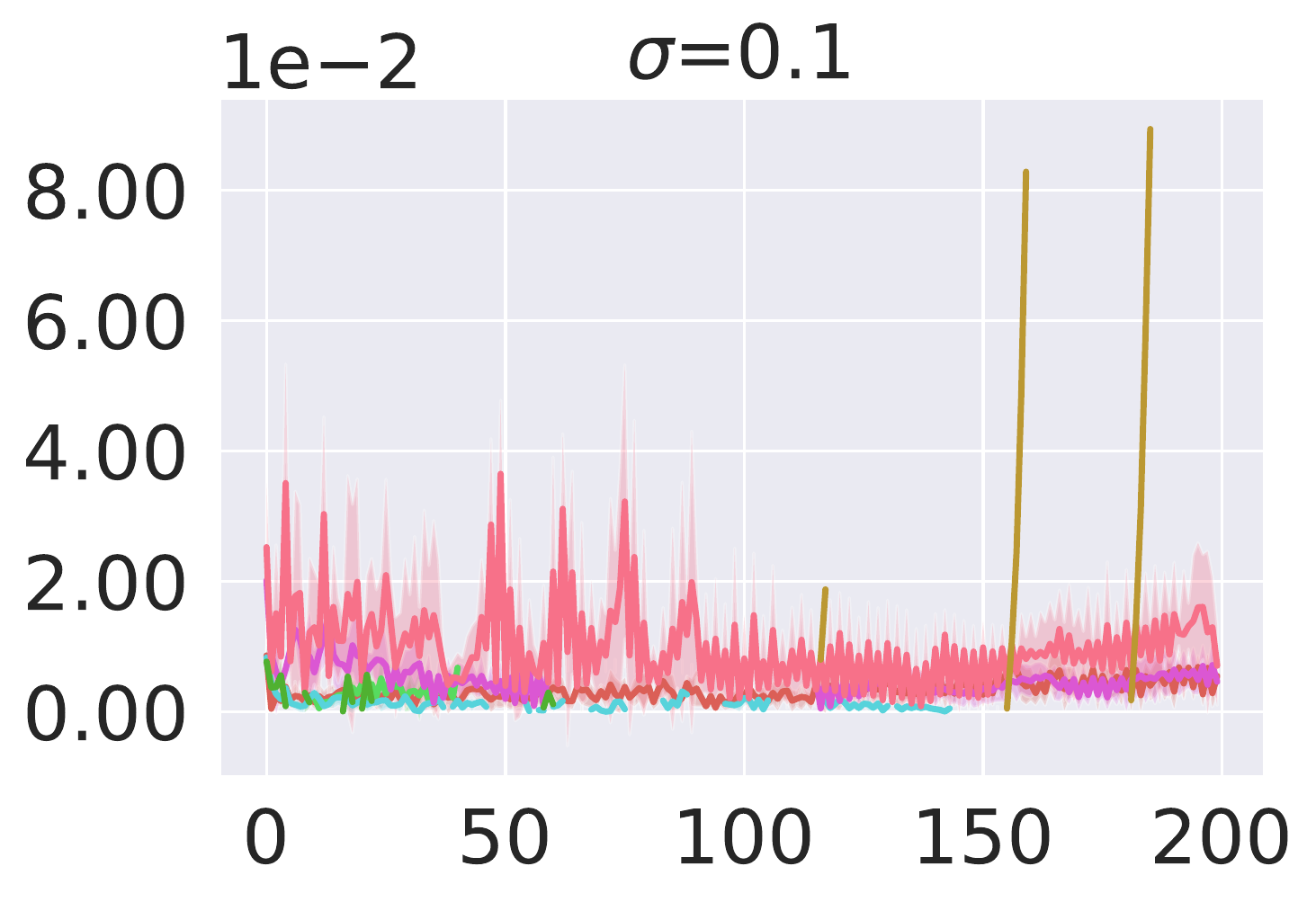}\\[-1.2ex]
        & \makecell{\tiny{time step $t$}}
        & \makecell{\tiny{time step $t$}}
        & \makecell{\tiny{time step $t$}}
        & \makecell{\tiny{time step $t$}}
        & \makecell{\tiny{time step $t$}}
        & \makecell{\tiny{time step $t$}}
\end{tabular}
\caption{\small Robustness certification for \textit{per-state action} in terms of certified radius $r$ at all time steps.
% We consider six methods: StdTrain, GaussAug, AdvTrain, RegPGD, RegCVX, and RadialRL, with the last three being SOTA. 
The shaded area represents the standard deviation.
% RadialRL is the most certifiably robust method on Freeway, while RegCVX is the most robust on CartPole.
}%
\label{fig:cartpole-statewise}
% }
% \caption{\small Certified radius $R_t$ along time steps}\label{tab:cert-rad}
\end{subtable}

\begin{subtable}[]{\linewidth}
\centering
\includegraphics[height=\cpwrapheighta]{figures/Cartpole_loact_clean.pdf}
\vspace{-1mm}
\caption{\small Benign performance of  locally smoothed policy $\tpi$ under different  smoothing variance $\sigma$.}\label{fig:cartpole-statewise-clean}
\end{subtable}

\vspace{1em}

\begin{subtable}[]{\linewidth}
\begin{tabular}{l}
\includegraphics[height=\cplegendheightb]{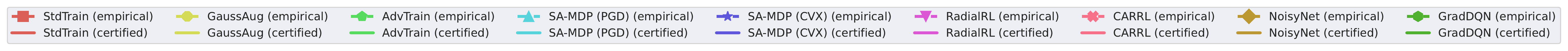}
\end{tabular}
\end{subtable}

\begin{subtable}[]{\linewidth}
\centering
\resizebox{\linewidth}{!}{%
\begin{tabular}{@{}p{3mm}@{}c@{}c@{}c@{}c@{}c@{}c@{}}
        & \makecell{\tiny{$\sigma=0.001$}}
        & \makecell{\tiny{$\sigma=0.005$}}
        & \makecell{\tiny{$\sigma=0.01$}}
        & \makecell{\tiny{$\sigma=0.03$}}
        & \makecell{\tiny{$\sigma=0.05$}}
        & \makecell{\tiny{$\sigma=0.1$}}
        \vspace{-1.7pt}\\
\cprownamec{\makecell{(Global) $\uje$}}&
\includegraphics[height=\cputilheightcp]{figures/CartPole_global_mean_sigma-0.001_cmp.pdf}&
\includegraphics[height=\cputilheightcp]{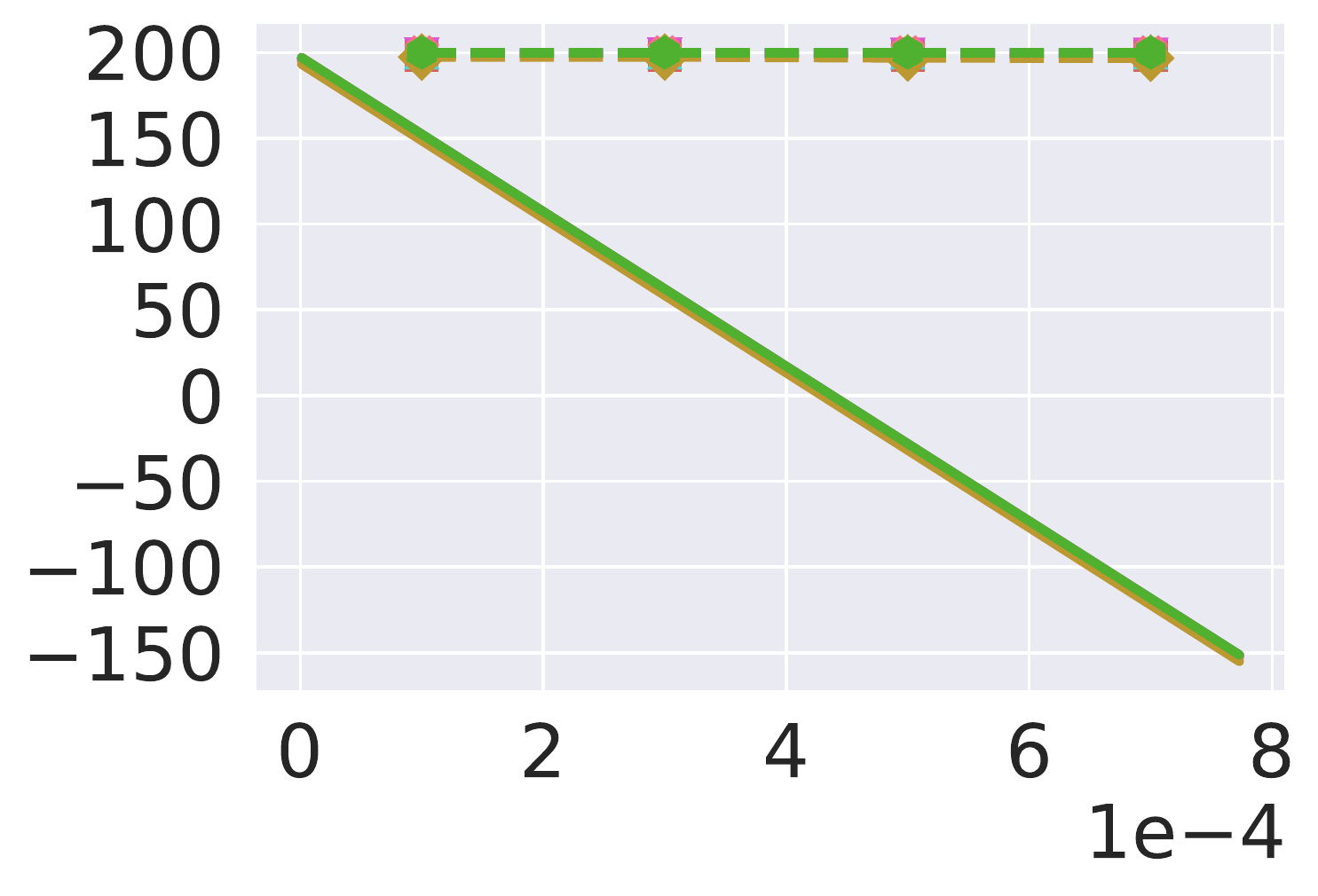}&
\includegraphics[height=\cputilheightcp]{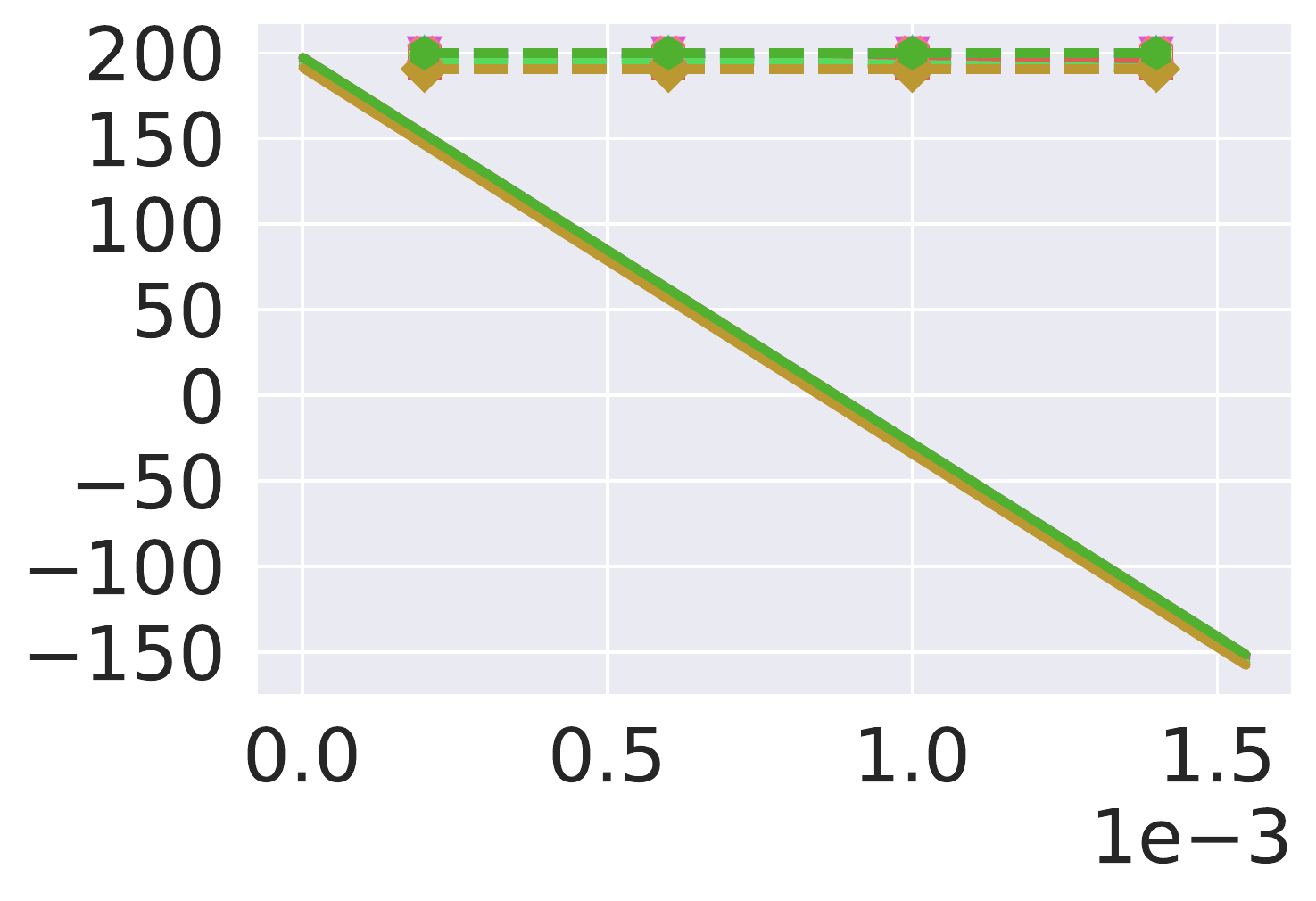}&
\includegraphics[height=\cputilheightcp]{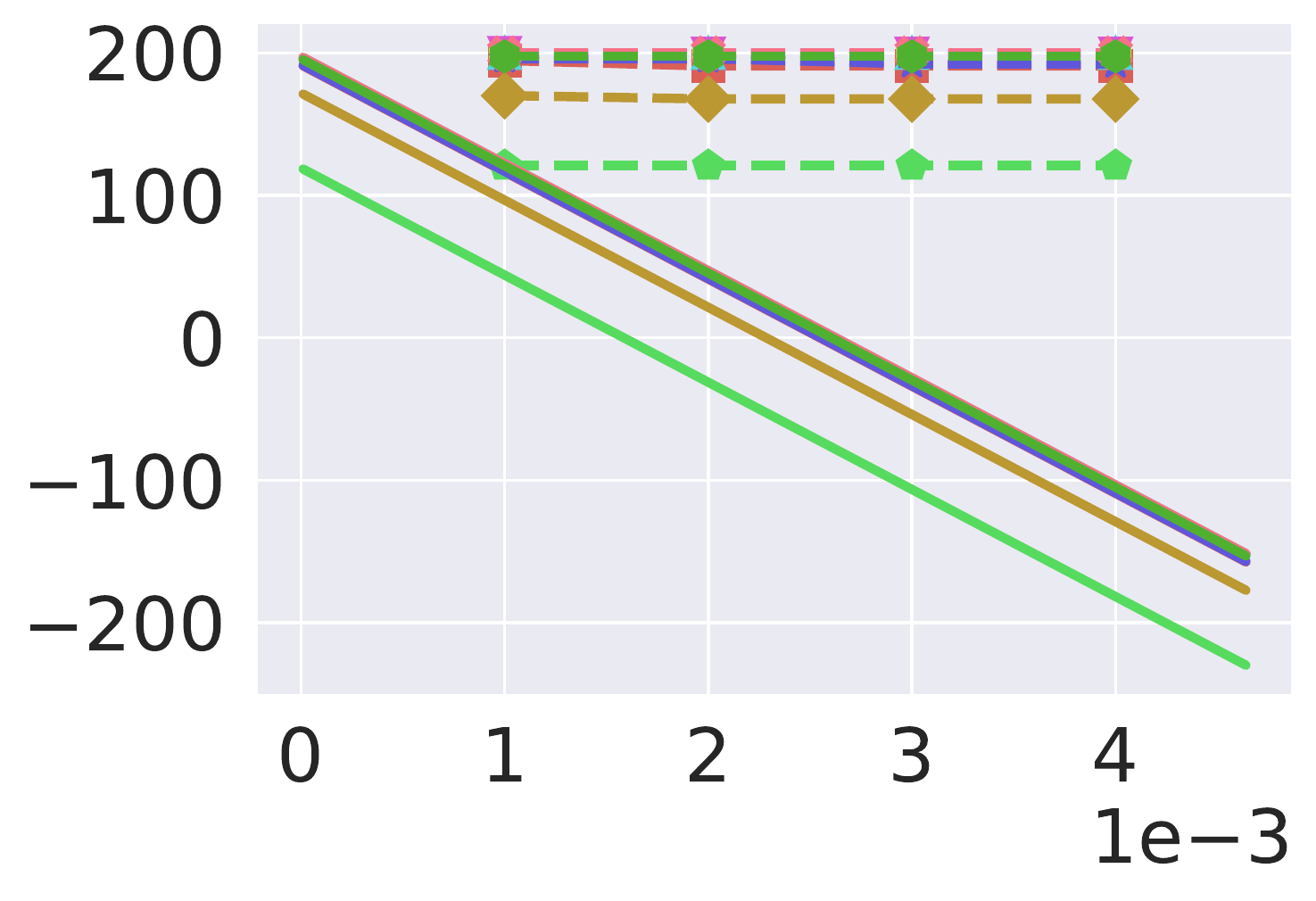}&
\includegraphics[height=\cputilheightcp]{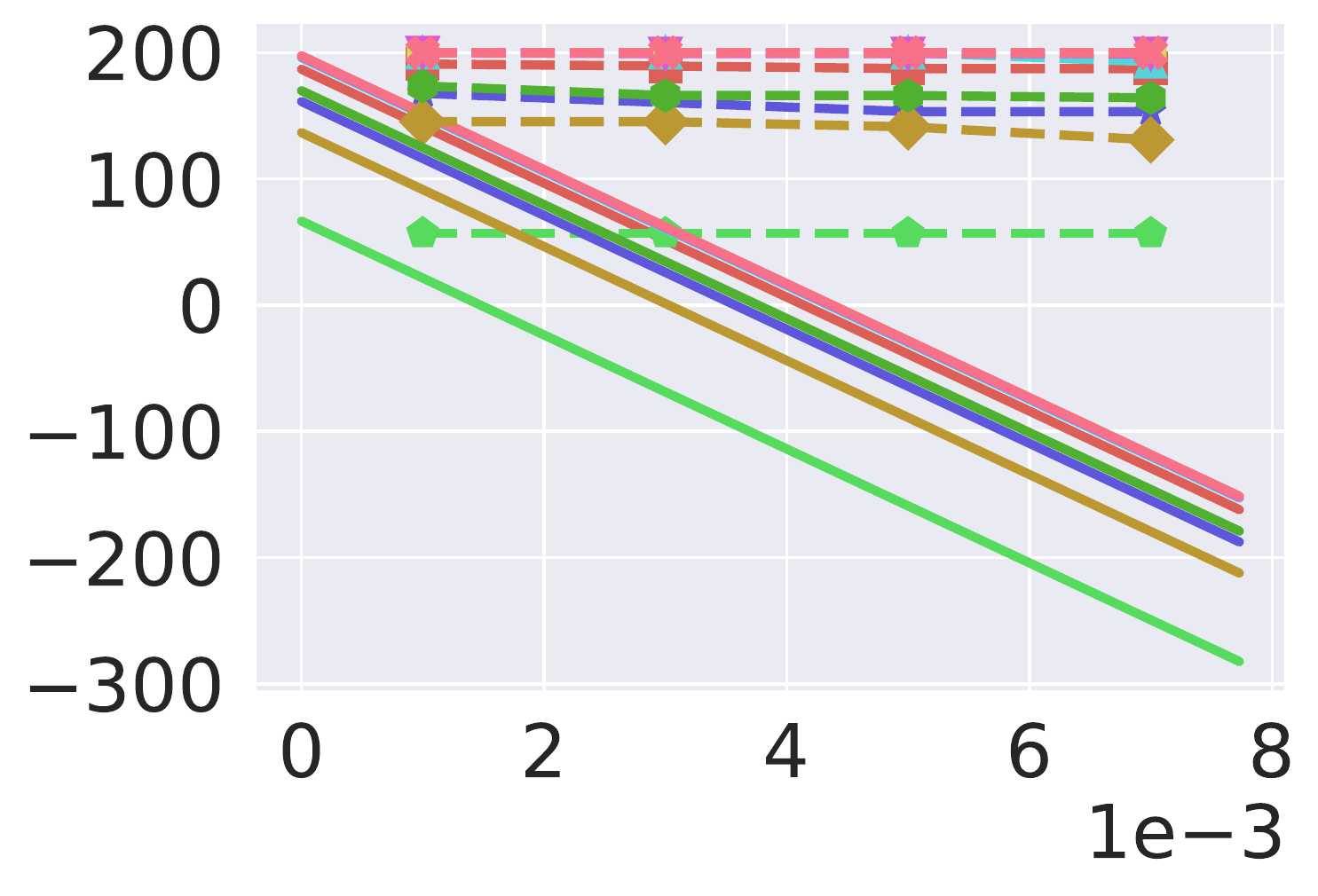}&
\includegraphics[height=\cputilheightcp]{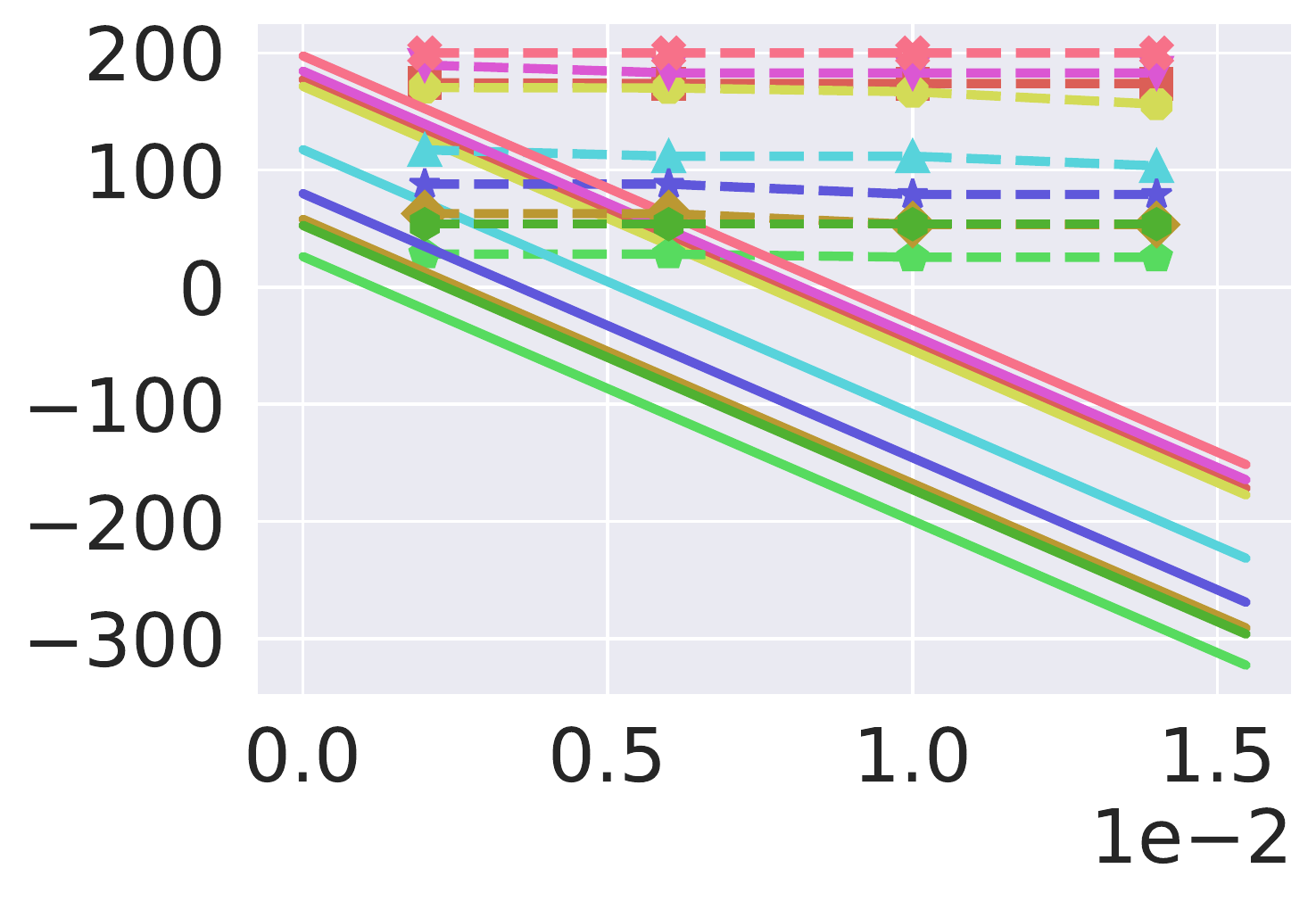}\\[-1.2ex]
\cprownamed{\makecell{(Global) $\ujp$}}&
\includegraphics[height=\cputilheightdp]{figures/CartPole_global_median_sigma-0.001_cmp.pdf}&
\includegraphics[height=\cputilheightdp]{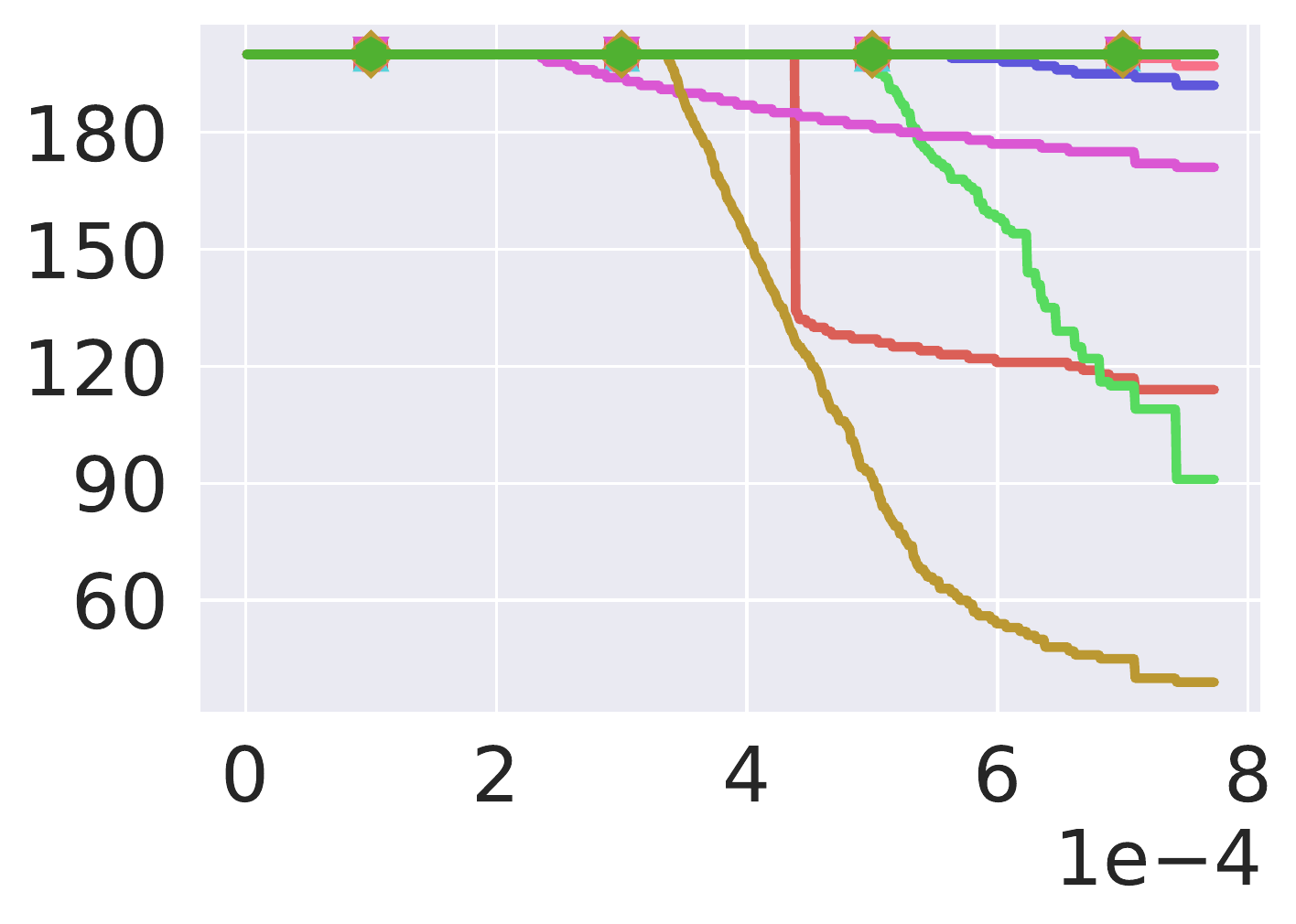}&
\includegraphics[height=\cputilheightdp]{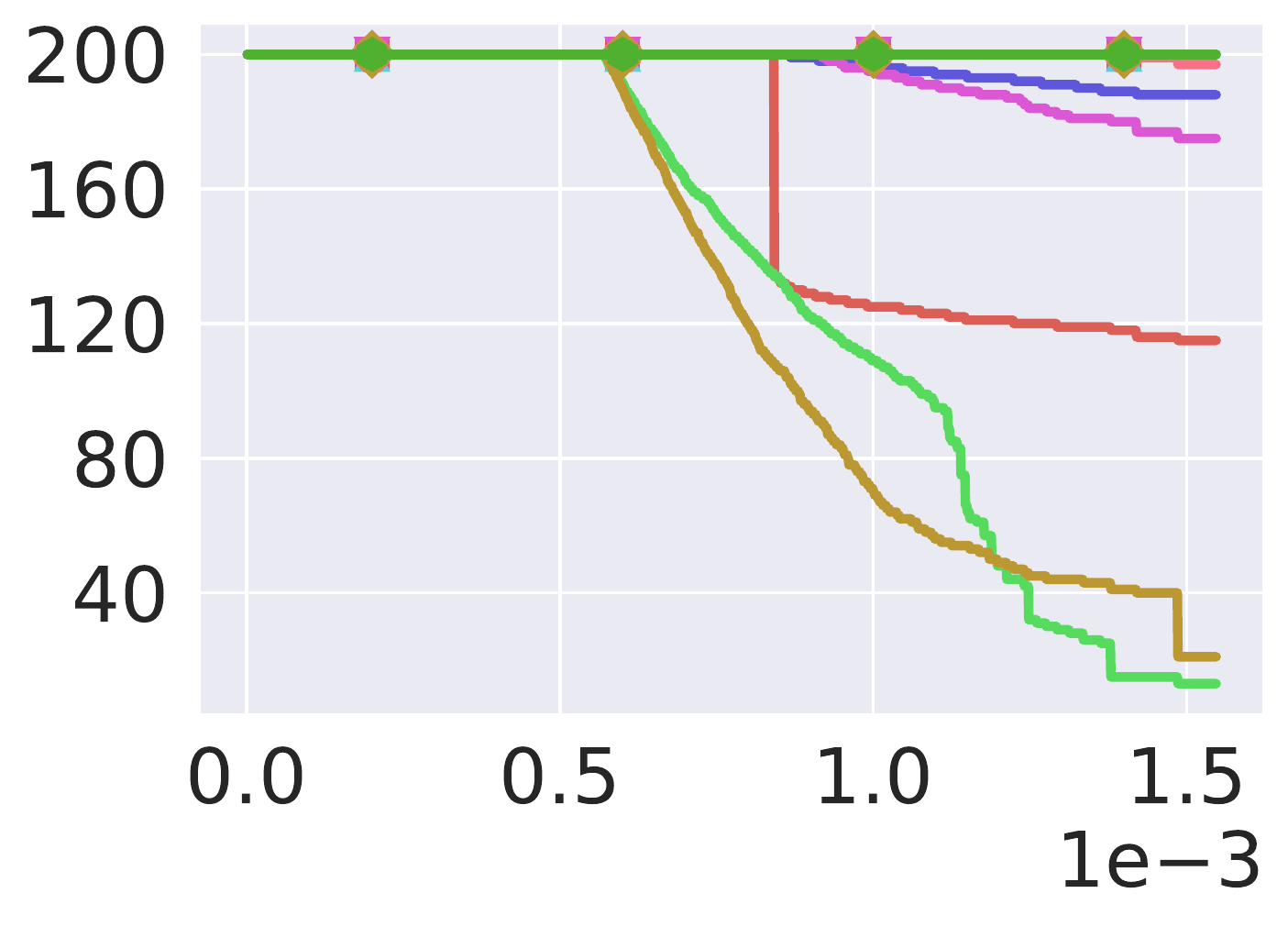}&
\includegraphics[height=\cputilheightdp]{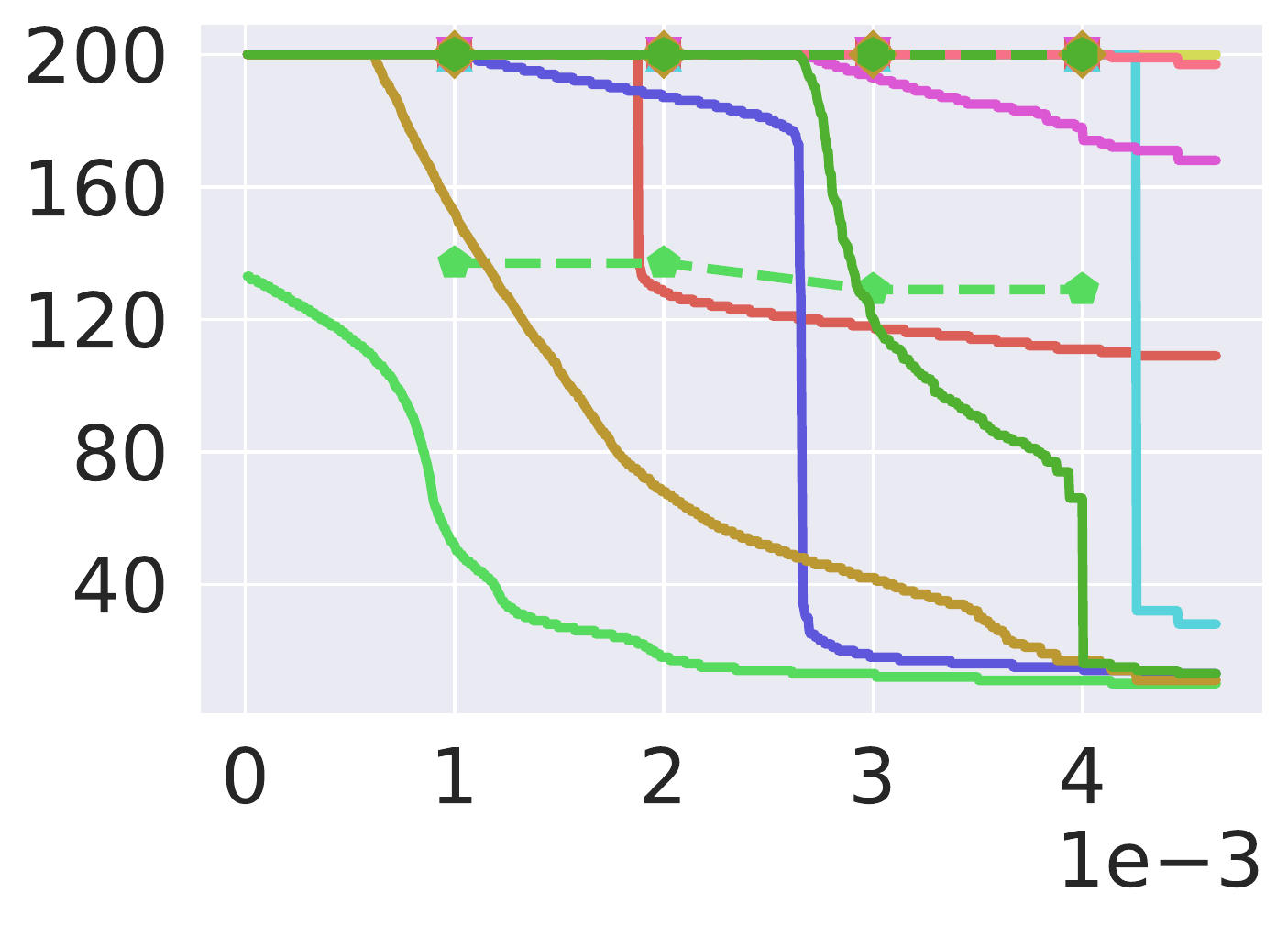}&
\includegraphics[height=\cputilheightdp]{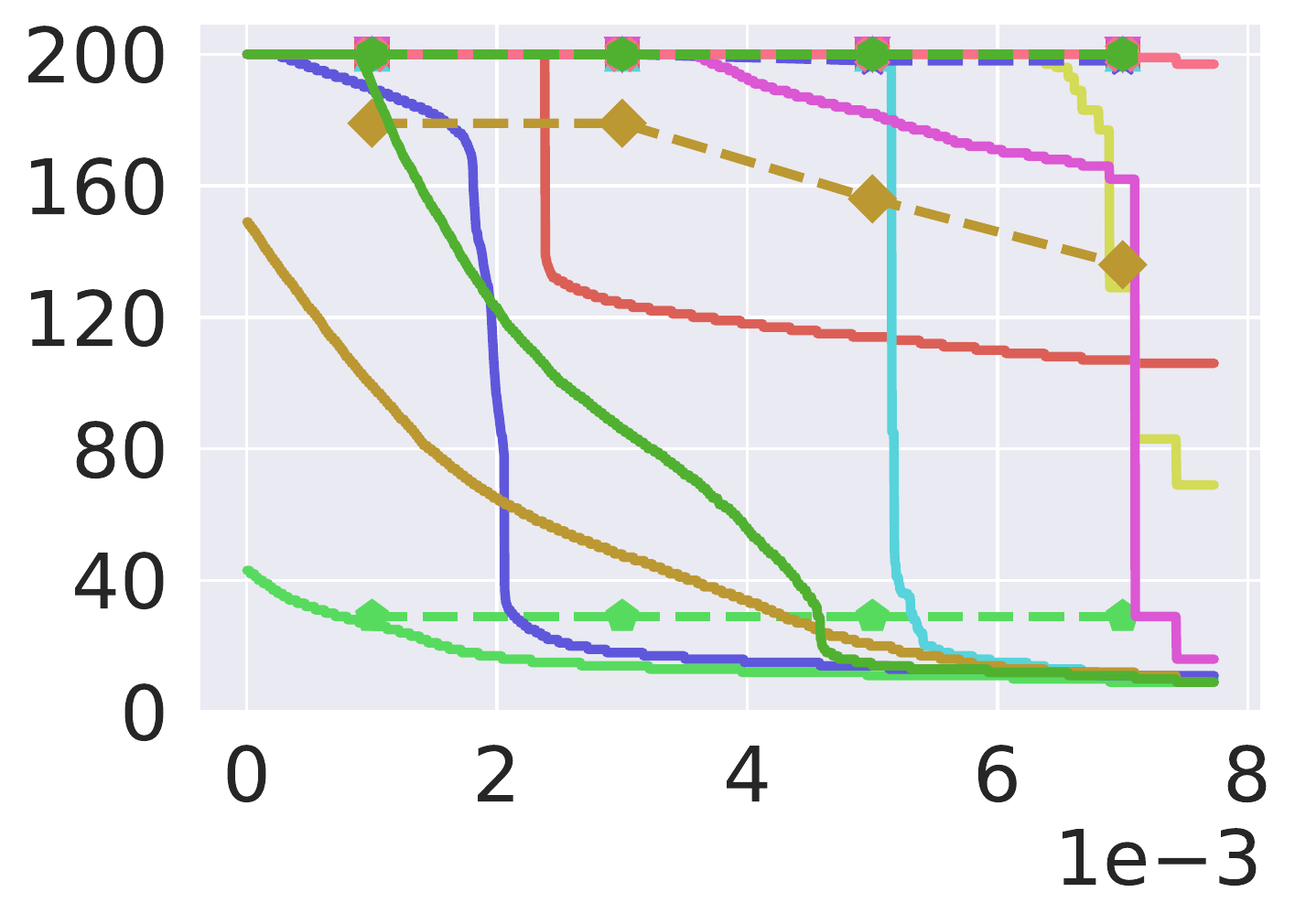}&
\includegraphics[height=\cputilheightdp]{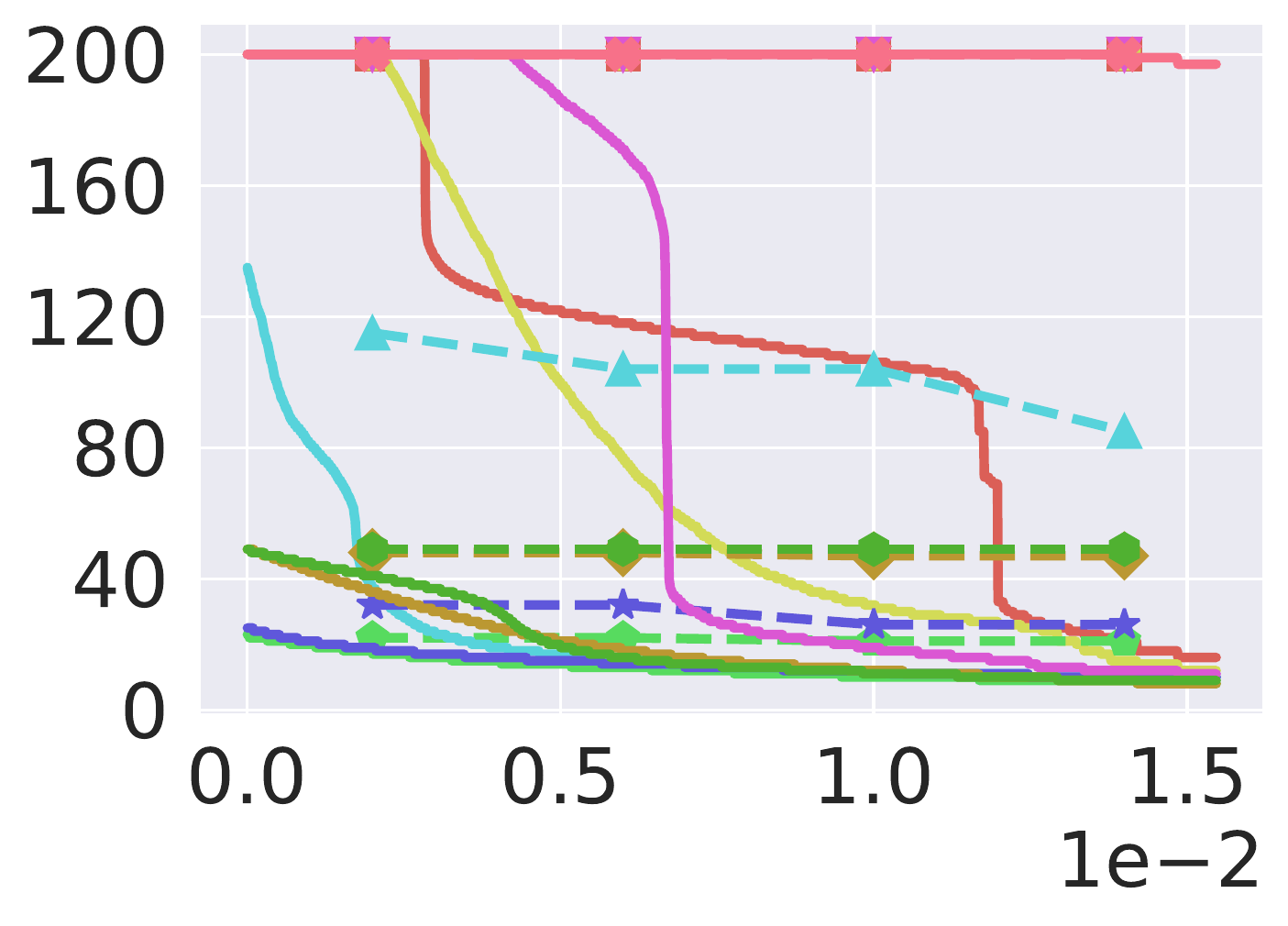}\\[-1.2ex]
\cprownamed{\makecell{(Local) $\uj$}}&
\includegraphics[height=\cputilheightaap]{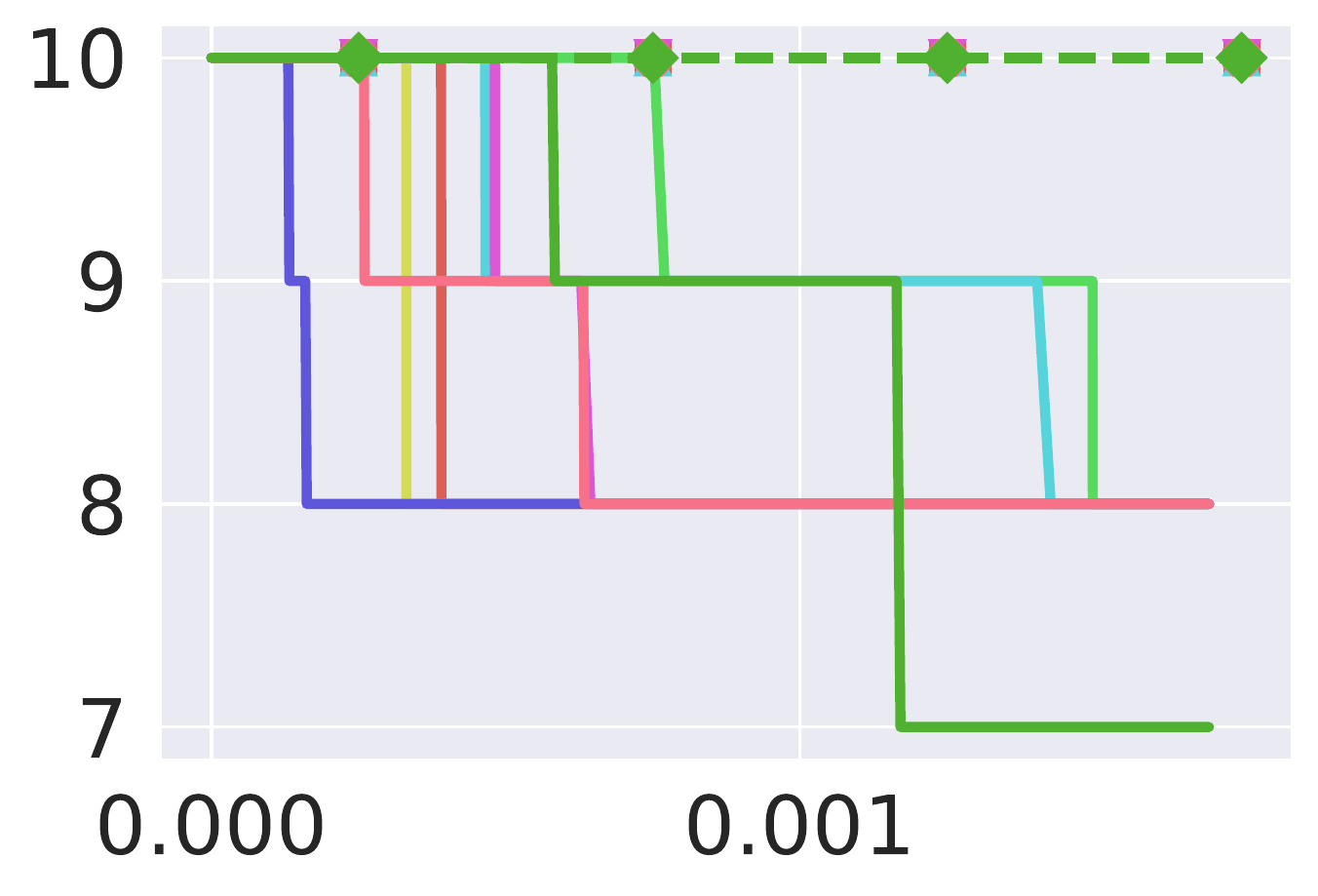}&
\includegraphics[height=\cputilheightaap]{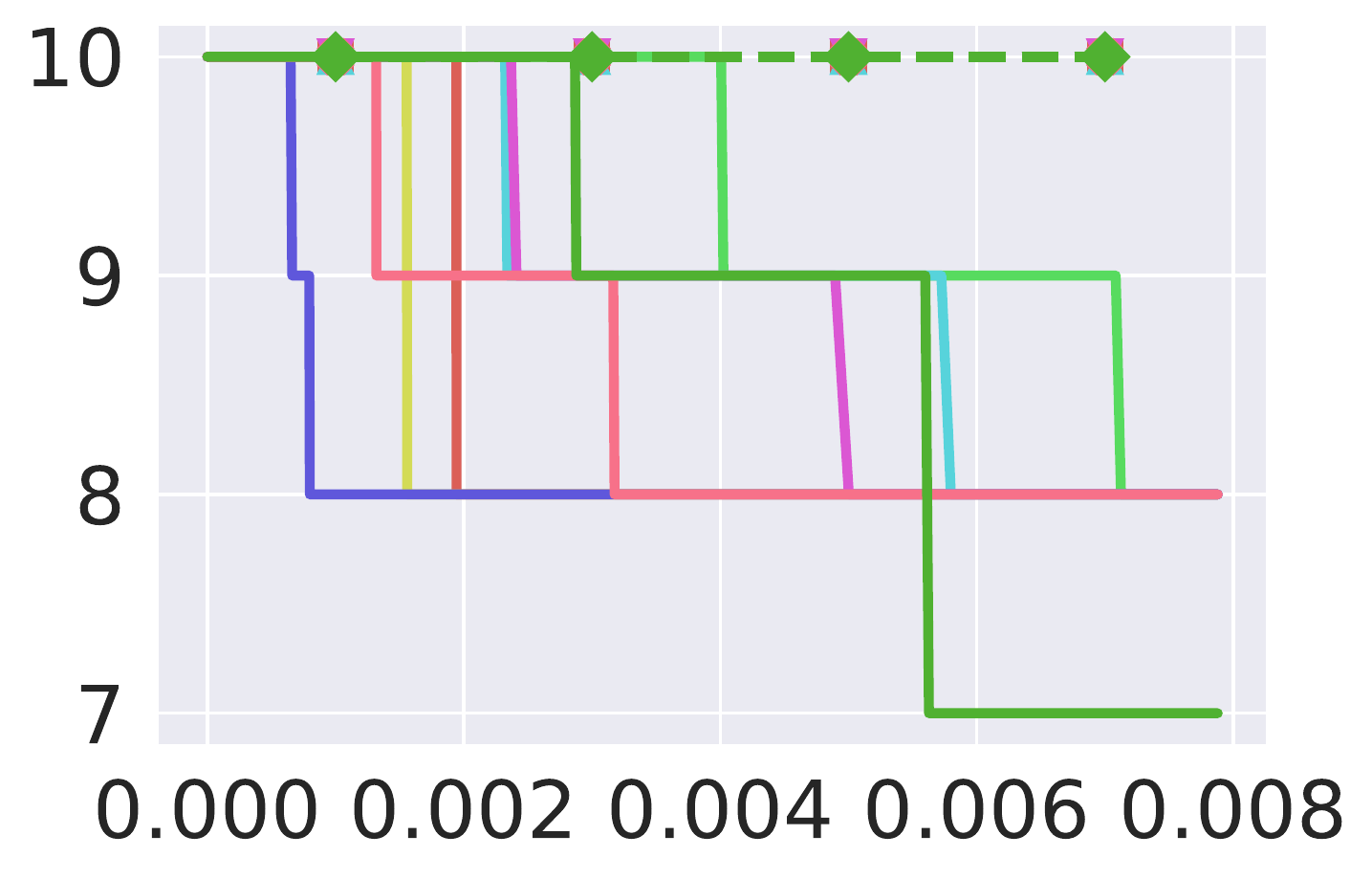}&
\includegraphics[height=\cputilheightaap]{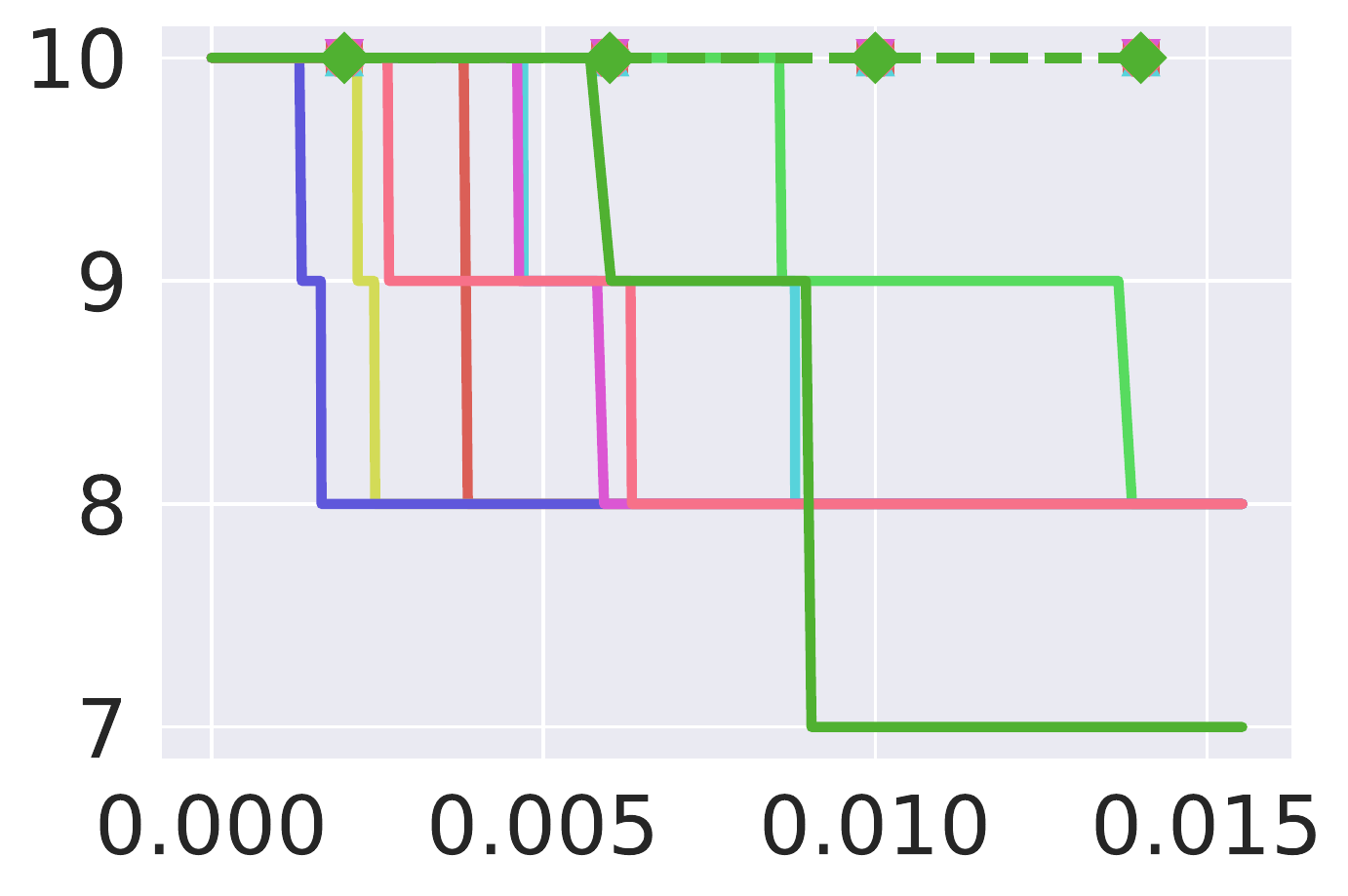}&
\includegraphics[height=\cputilheightaap]{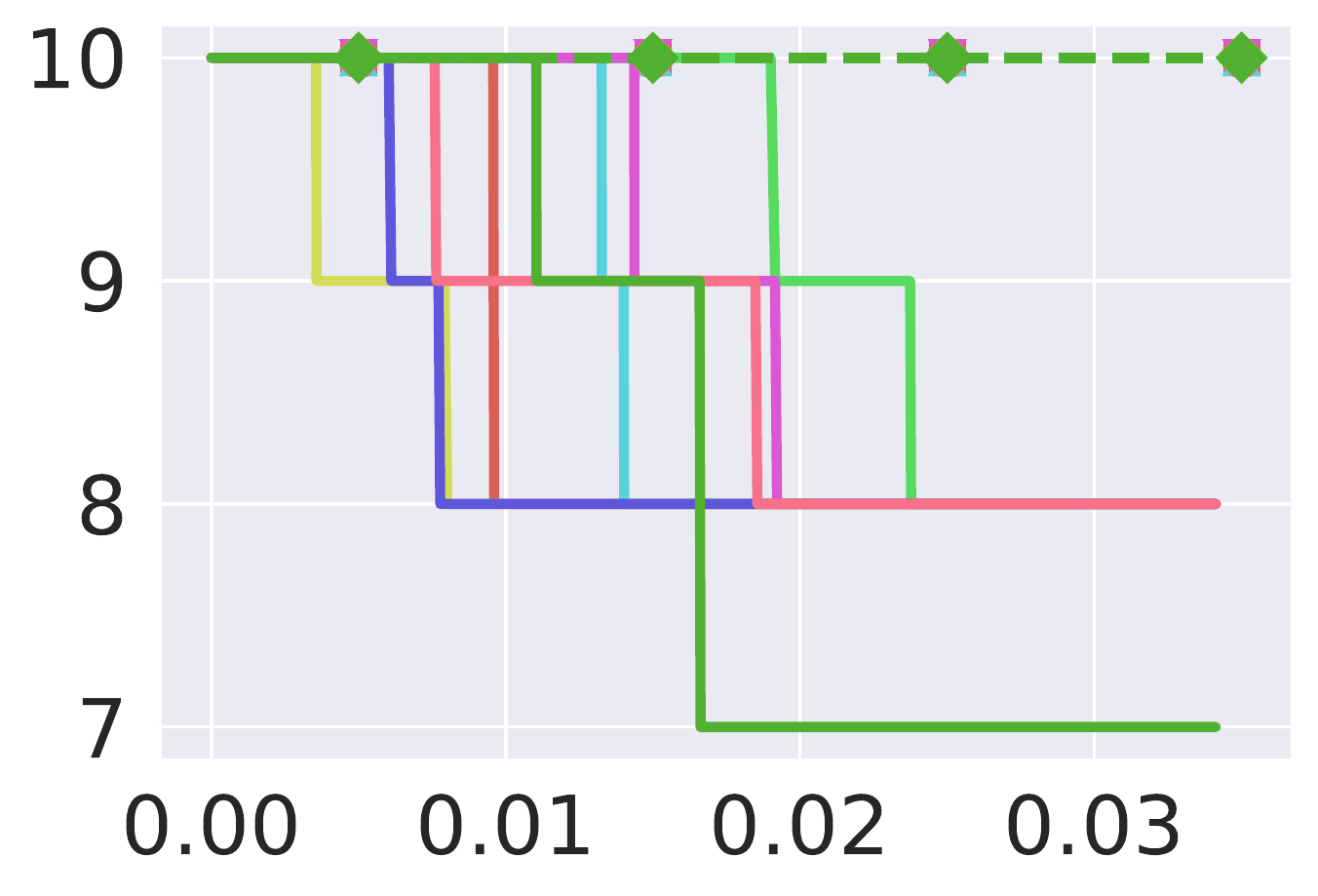}&
\includegraphics[height=\cputilheightaap]{figures/CartPole_adasearch_sigma-0.05_cmp.pdf}&
\includegraphics[height=\cputilheightaap]{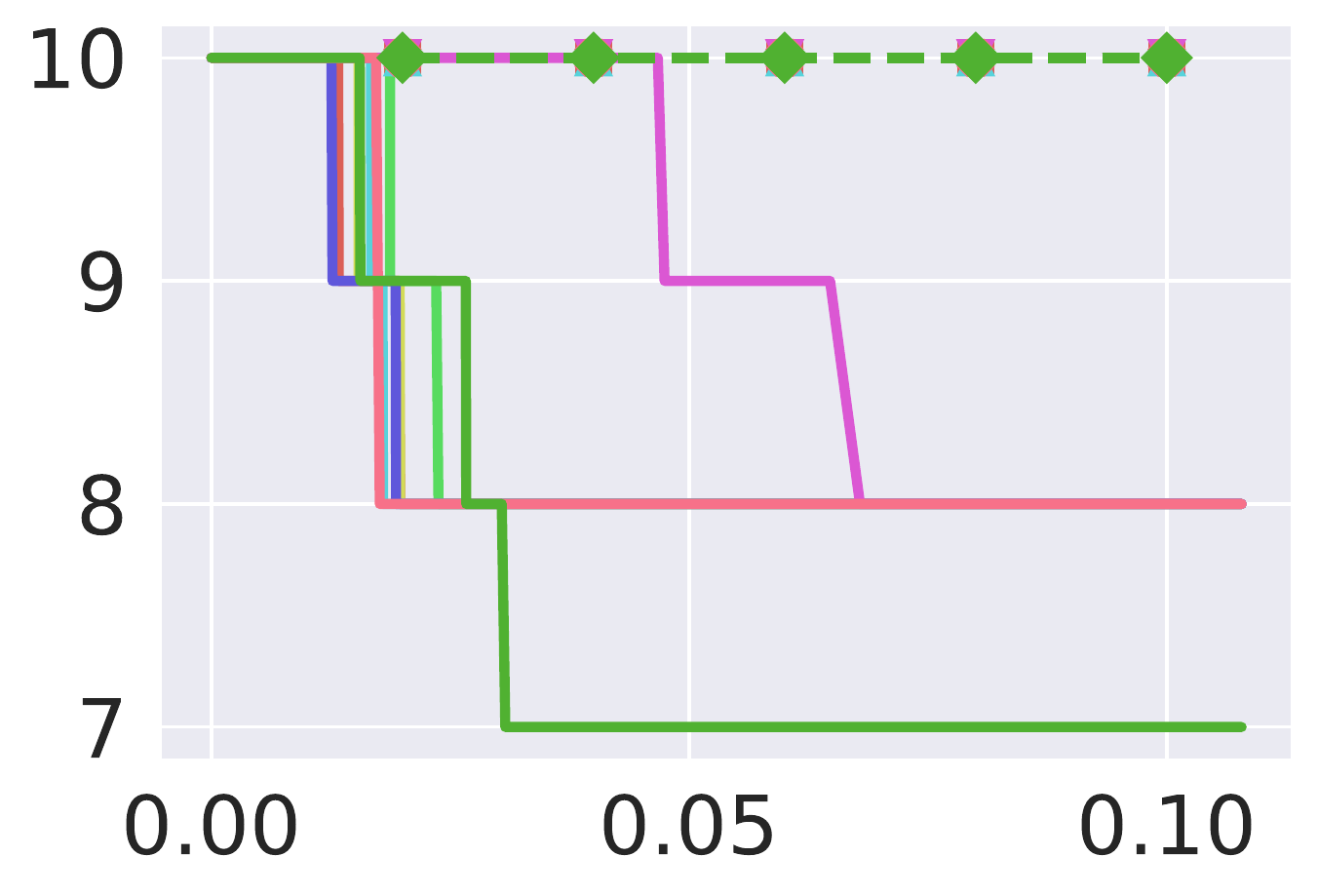}\\[-1.2ex]
        & \makecell{\tiny{attack $\eps$}}
        & \makecell{\tiny{attack $\eps$}}
        & \makecell{\tiny{attack $\eps$}}
        & \makecell{\tiny{attack $\eps$}}
        & \makecell{\tiny{attack $\eps$}}
        & \makecell{\tiny{attack $\eps$}}
\end{tabular}
}
\caption{\small Robustness certification as cumulative reward, including \textit{expectation bound} $\uje$, \textit{percentile bound} $\ujp$ ($p=50\%$), and \textit{absolute lower bound} $\uj$. 
% Each column corresponds to one smoothing variance. 
Solid lines represent the certified reward bounds of different methods, and dashed lines show the empirical performance under PGD attacks.}\label{tab:cartpole-bounds}
\end{subtable}
}

\caption{\small \textbf{Robustness certification on CartPole} in terms of (a-b): \textit{robustness of per-state action} and (c): \textit{lower bound of cumulative rewards}.
In (a) and (c), each column corresponds to one smoothing variance. 
% , including as cumulative reward, including \textit{expectation bound} $\uje$, \textit{percentile bound} $\ujp$ ($p=50\%$), and \textit{absolute lower bound} $\uj$  on (a) Freeway and (b) Pong. Each column corresponds to one smoothing variance. Solid lines represent the certified reward bounds of different methods, and dashed lines show the empirical performance under PGD attacks.
}\label{tab:cartpole-figs}
\end{figure}

}

{
\subsection{Discussion on Evaluation Results of Highway}
\label{append:highway}
We next present the evaluation results for \sysname on an autonomous driving environment Highway, whose state dimension is larger than CartPole but smaller than Atari games. We compare nine RL algorithms as introduced previously. We present the evaluation results in~\Cref{tab:highway-figs}.}

{\textbf{Impact of Smoothing Parameter $\sigma$.}\quad
In highway, CARRL demonstrates its advantage when $\sigma$ is small, which is supported by results in all three certifications $\staters$, \glbrs, and \adasearch. 
For action certification (\ie, $\staters$), SA-MDP (PGD) can tolerate a large range of $\sigma$, meaning that the action selection of model trained with SA-MDP (PGD) algorithm is more consistent than other methods. 
For reward certification (\ie, \glbrs and \adasearch), we derive the conclusion that RadialRL outperforms than other methods when $\sigma$ is large.} 

{\textbf{Tightness of the certiﬁcation $\uje$, $\ujp$ and $\uj$.}\quad
In highway, we compare the empirical cumulative rewards achieved under PGD attacks with our certiﬁed lower bounds $\uje$, $\ujp$, and $\uj$. The correctness of our bounds is validated because the empirical results are consistently lower bounded by our certiﬁcations, as shown in \Cref{tab:highway-bounds}. As for the tightness of reward certification methods, we draw similar conclusion as other environments---$\uj$ is the most tight, followed by $\ujp$; and $\uje$ is much looser than $\ujp$ and $\uje$. 
We also notice the non-negligible gap between certified lower bounds and the empirical rewards demonstrated by RadialRL and CARRL under large smoothing parameters. This may imply that there exists large scopes for improving the attack method, or there is potential to further tighten the certified lower bounds.
% the certified reward of some methods is low sometimes when the empirical reward of them is relatively high, which emphasizes the importance of robustness certification in RL.
}

% {\textbf{Environment Properties.}\quad
% Because of the state dimension is in middle range and the uniqueness of highway environment, we can observe some interesting phenomenon in the highway game. First, although SA-MDP (PGD) can tolerate a large range of when considering action stability, it does not have good reward certification. This suggests that we need to consider reward certification when evaluating RL algorithms, instead of just evaluating them with the methods similar to those targeting at certifying the robustness of image classifiers.}

{
\renewcommand{\thesubfigure}{\alph{subfigure}}

\begin{figure}

% \newlength{\cputilheighta}
\settoheight{\cputilheighta}{\includegraphics[width=.160\linewidth]{figures/Freeway_stepwise_sigma-0.001.pdf}}%

% \newlength{\cpwrapheighta}
\settoheight{\cpwrapheighta}{\includegraphics[width=.400\textwidth]{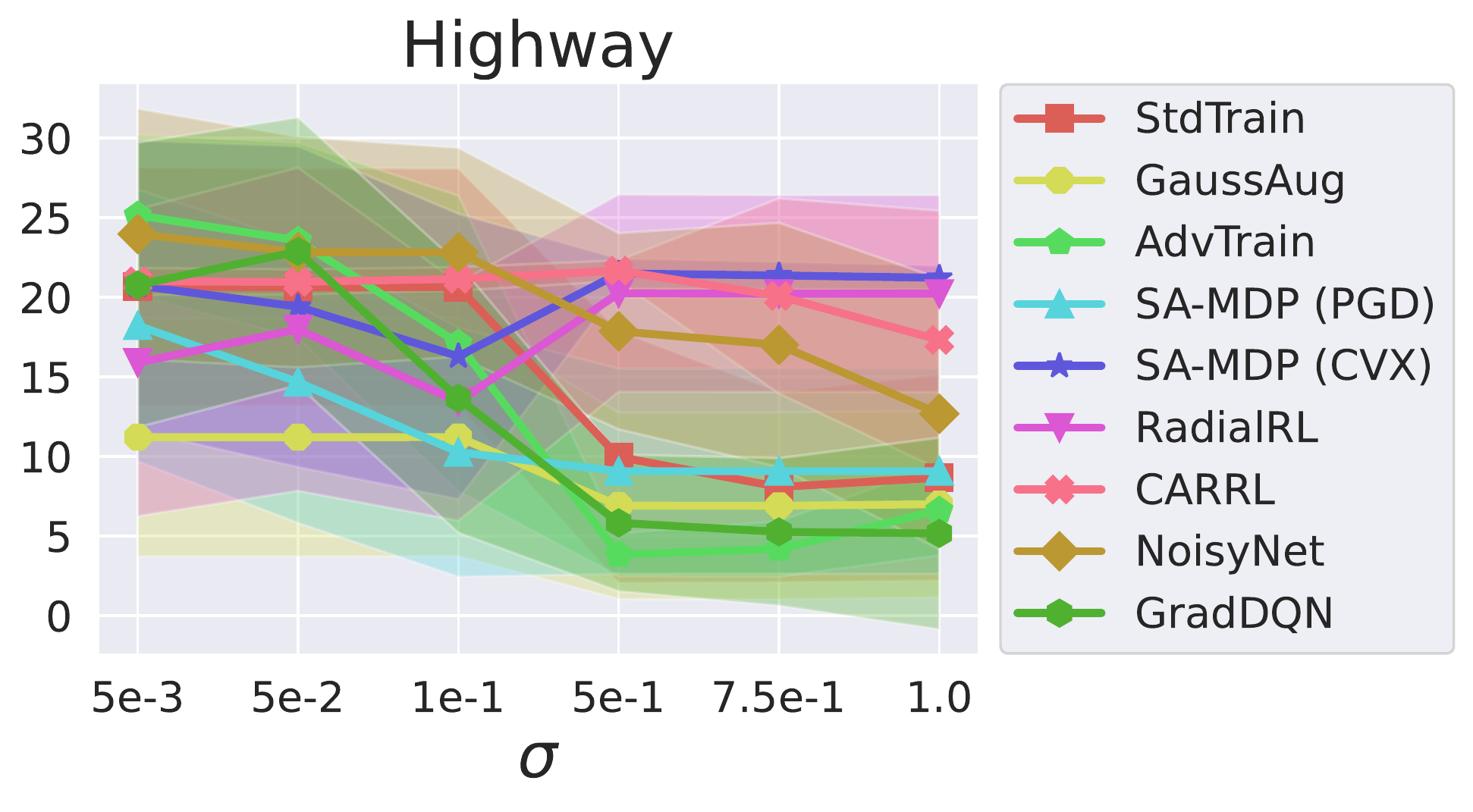}}%

% \newlength{\cplegendheight}
\setlength{\cplegendheight}{0.3\cputilheighta}%

\newcommand{\cprowname}[1]% #1 = text
{\rotatebox{90}{\makebox[\cputilheighta][c]{\tiny #1}}}

% \newlength{\cputilheightc}
\settoheight{\cputilheightc}{\includegraphics[width=.160\linewidth]{figures/Freeway_global_mean_sigma-0.001.pdf}}%

% \newlength{\cputilheightd}
\settoheight{\cputilheightd}{\includegraphics[width=.165\linewidth]{figures/Freeway_global_median_sigma-0.001.pdf}}%

% \newlength{\cputilheightaa}
\settoheight{\cputilheightaa}{\includegraphics[width=.162\linewidth]{figures/Freeway_adasearch_sigma-0.05.pdf}}%

% \newlength{\cputilheightcp}
\settoheight{\cputilheightcp}{\includegraphics[width=.160\linewidth]{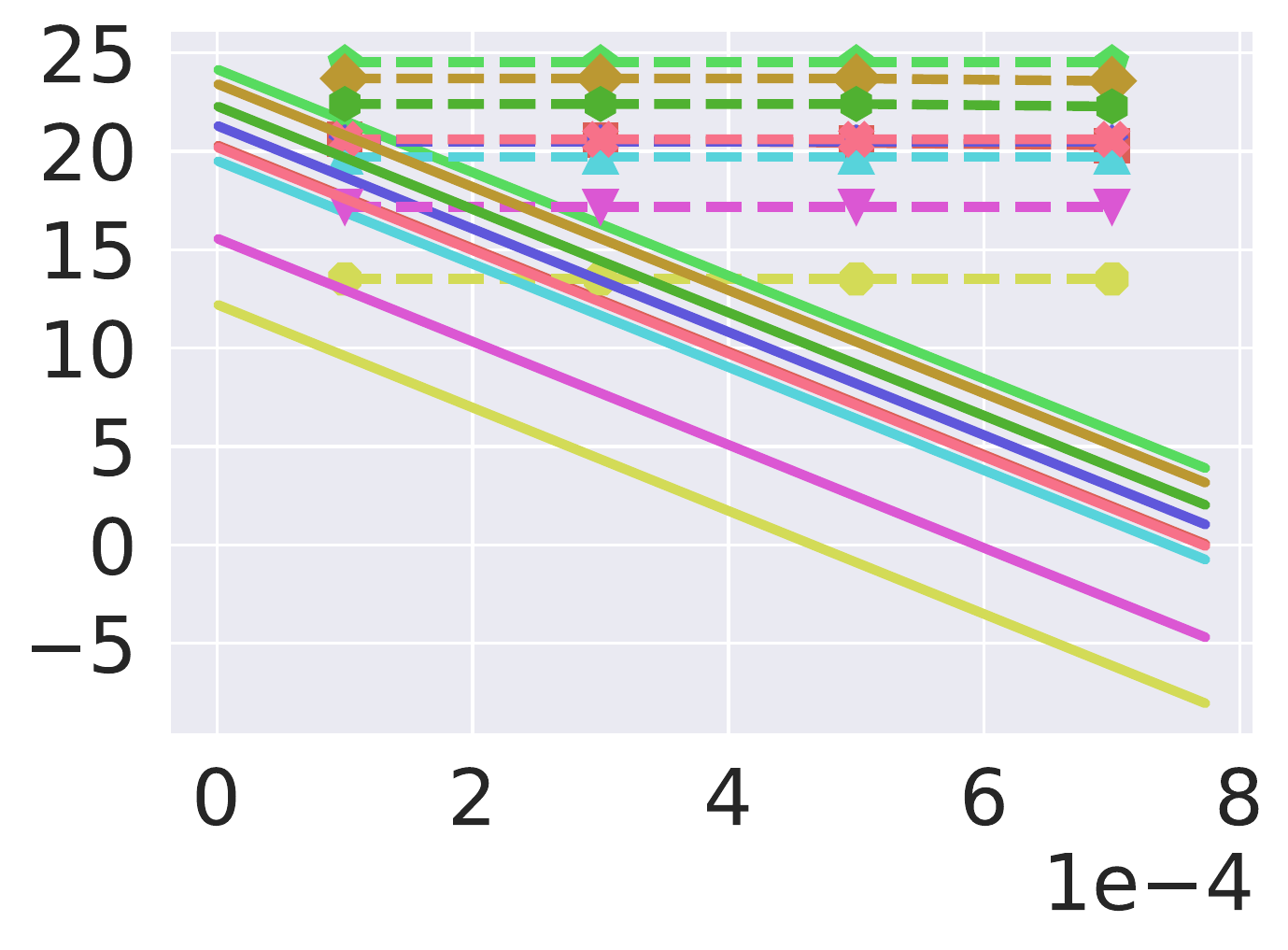}}%

% \newlength{\cputilheightdp}
\settoheight{\cputilheightdp}{\includegraphics[width=.165\linewidth]{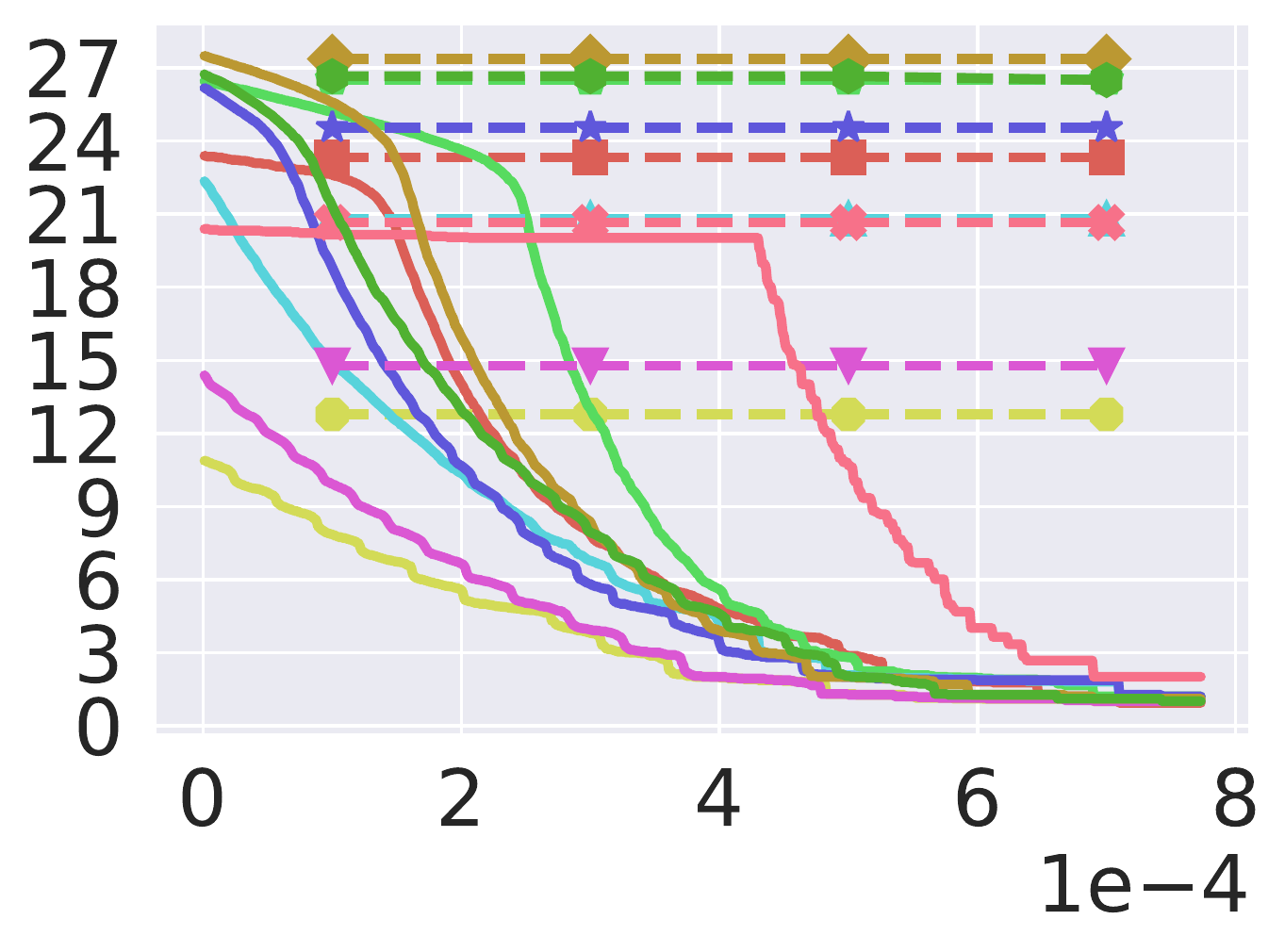}}%

% \newlength{\cputilheightaap}
\settoheight{\cputilheightaap}{\includegraphics[width=.162\linewidth]{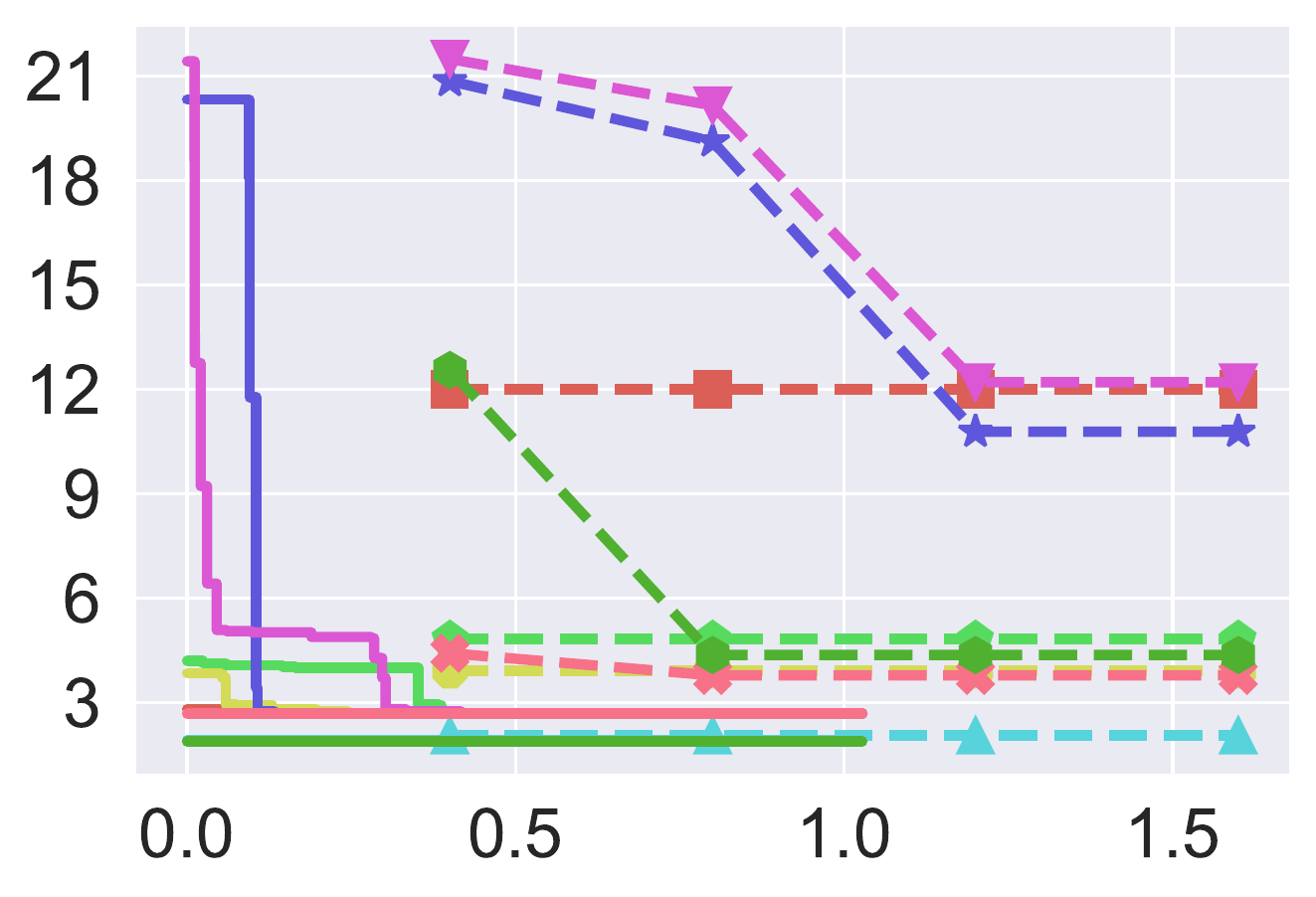}}%

% % \newlength{\cputilheight}
% % \settoheight{\cputilheight}{\includegraphics[width=.138\linewidth]{figures/twitch-DE: low degree.pdf}}%

% % \newlength{\attackheightb}
% % \settoheight{\attackheightb}{\includegraphics[width=.138\linewidth]{figures/twitch-DE: high degree.pdf}}%

% \newlength{\cplegendheightb}
\setlength{\cplegendheightb}{0.3\cputilheightc}%

\newcommand{\cprownamec}[1]% #1 = text
{\rotatebox{90}{\makebox[\cputilheightc][c]{\tiny #1}}}

\newcommand{\cprownamed}[1]% #1 = text
{\rotatebox{90}{\makebox[\cputilheightd][c]{\tiny #1}}}

\centering

{
\renewcommand{\tabcolsep}{10pt}

\begin{subtable}[]{\linewidth}
\begin{tabular}{l}
\includegraphics[height=\cplegendheight]{figures/Cartpole_legend.pdf}
\end{tabular}
\end{subtable}

\begin{subtable}[]{\linewidth}
\centering
% \resizebox{\linewidth}{!}{%
\begin{tabular}{@{}p{5mm}@{}c@{}c@{}c@{}c@{}c@{}c@{}}
\cprowname{\makecell{highway\\Radius $r$}}&
\includegraphics[height=\cputilheighta]{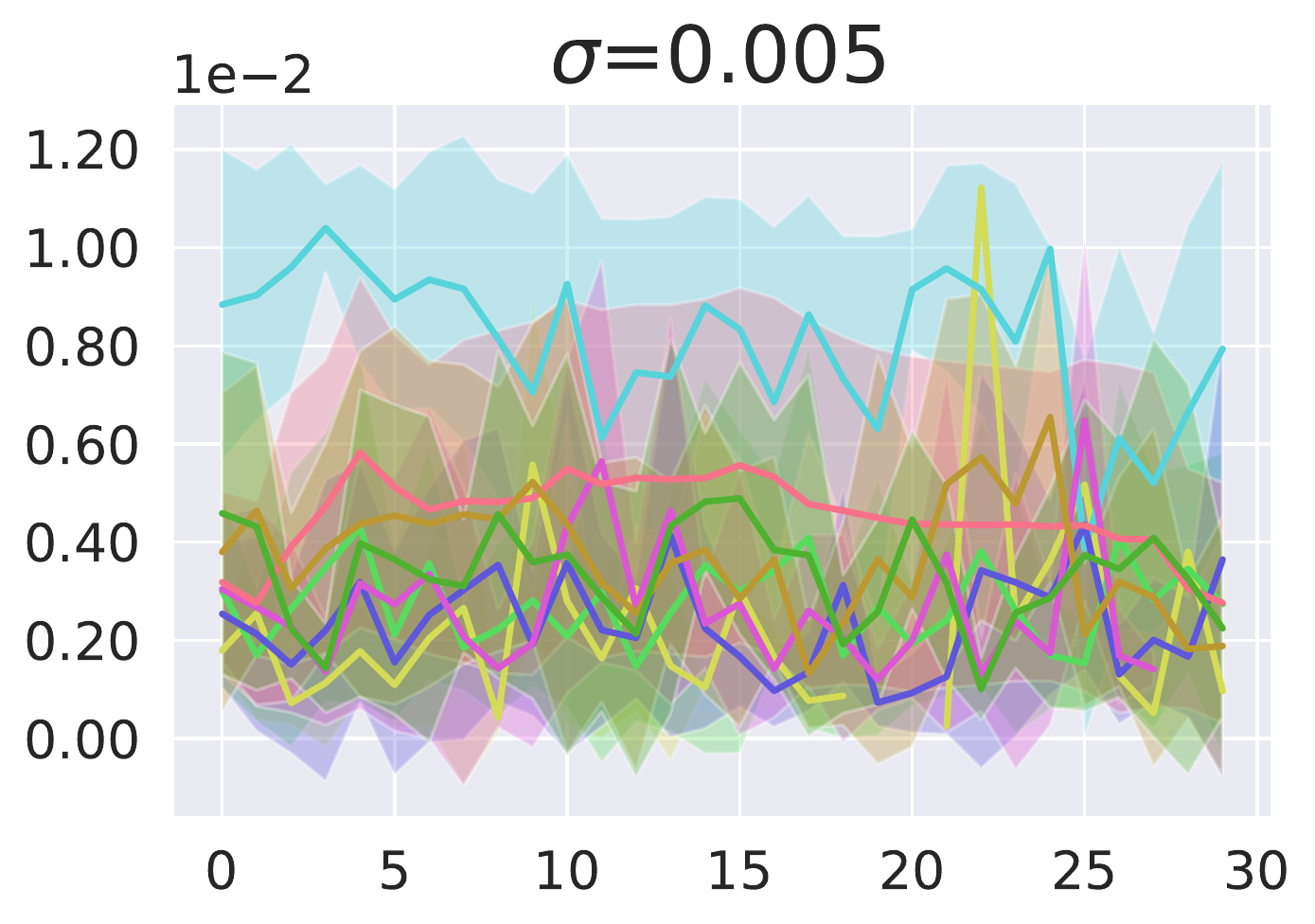}&
\includegraphics[height=\cputilheighta]{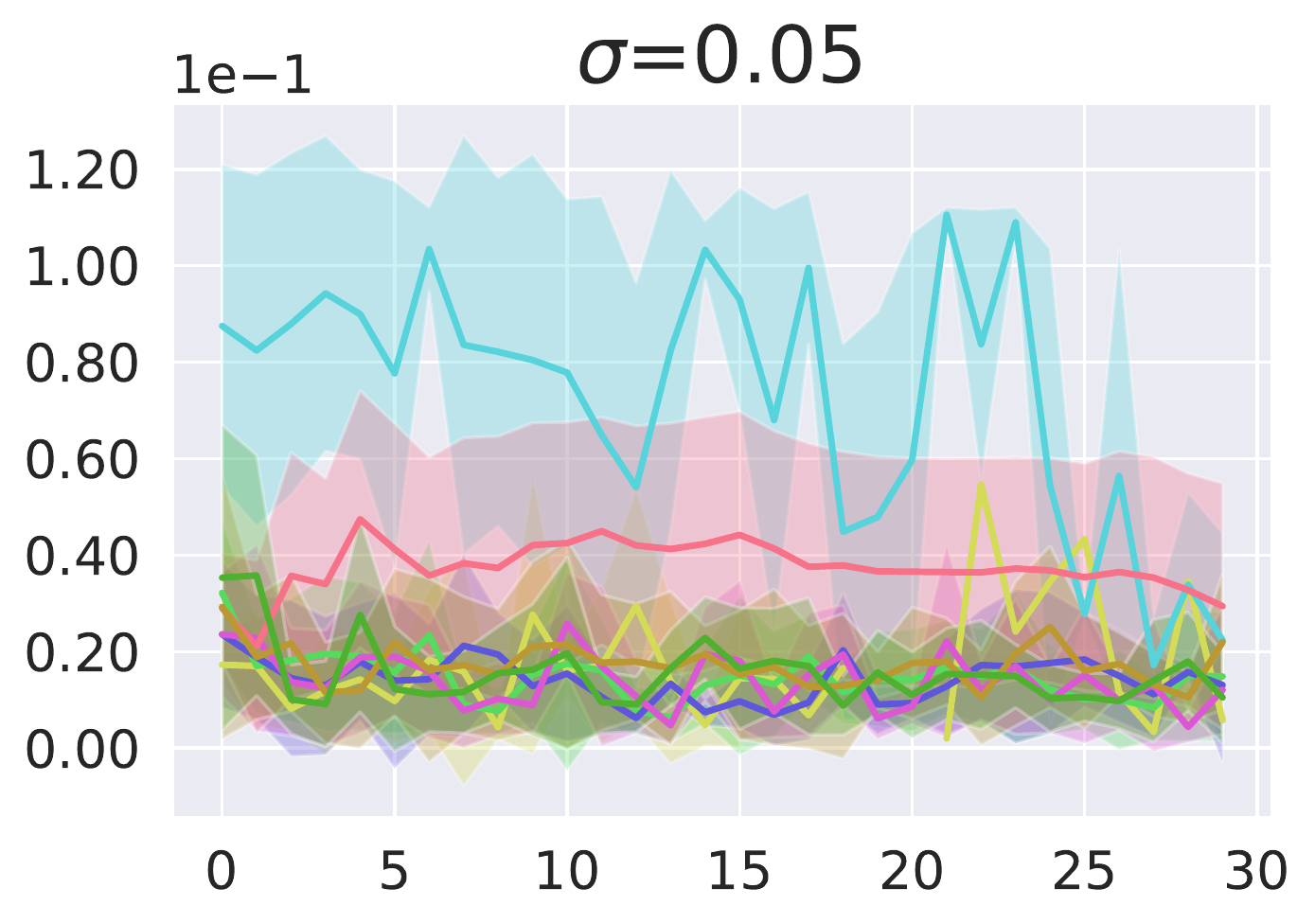}&
\includegraphics[height=\cputilheighta]{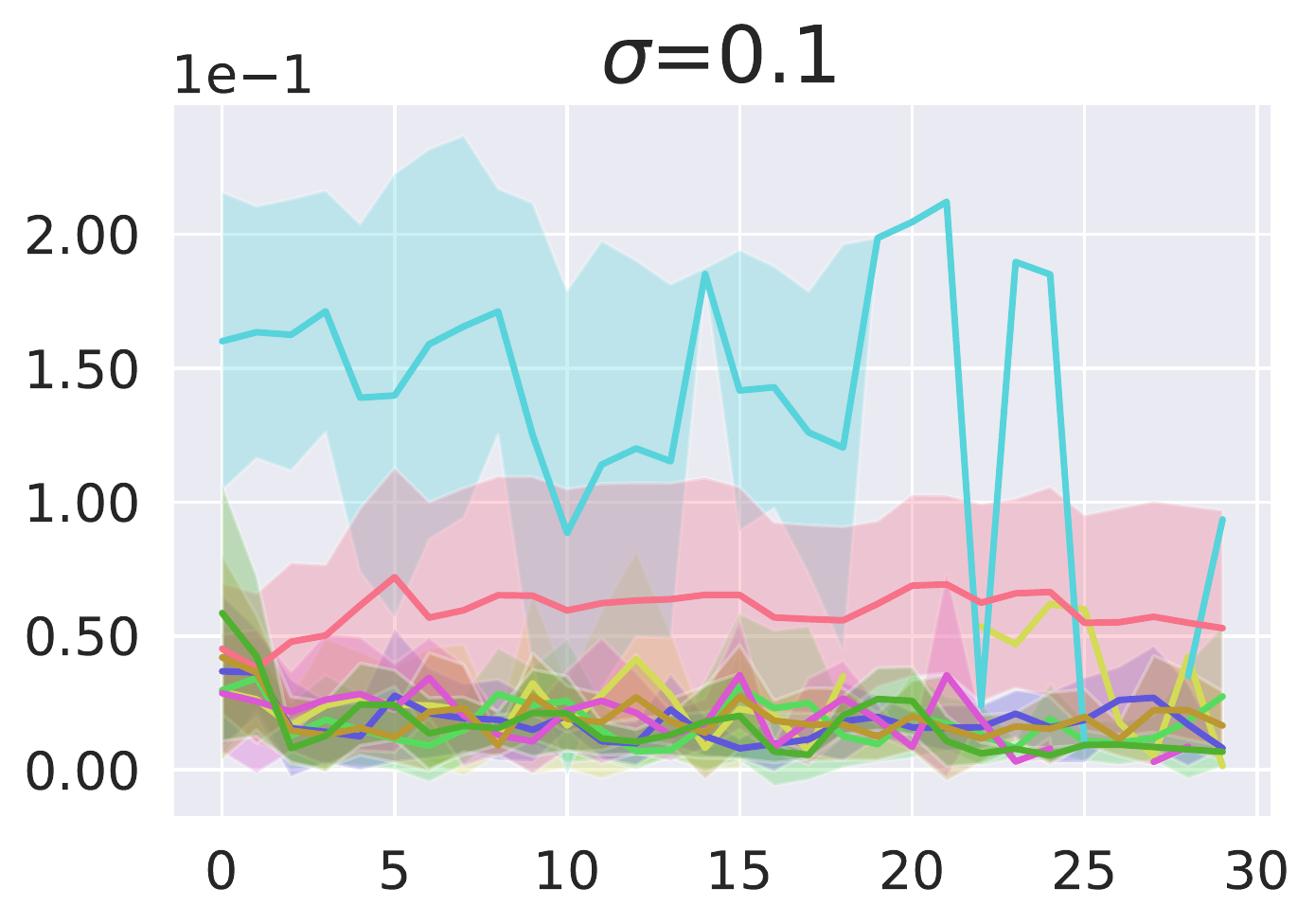}&
\includegraphics[height=\cputilheighta]{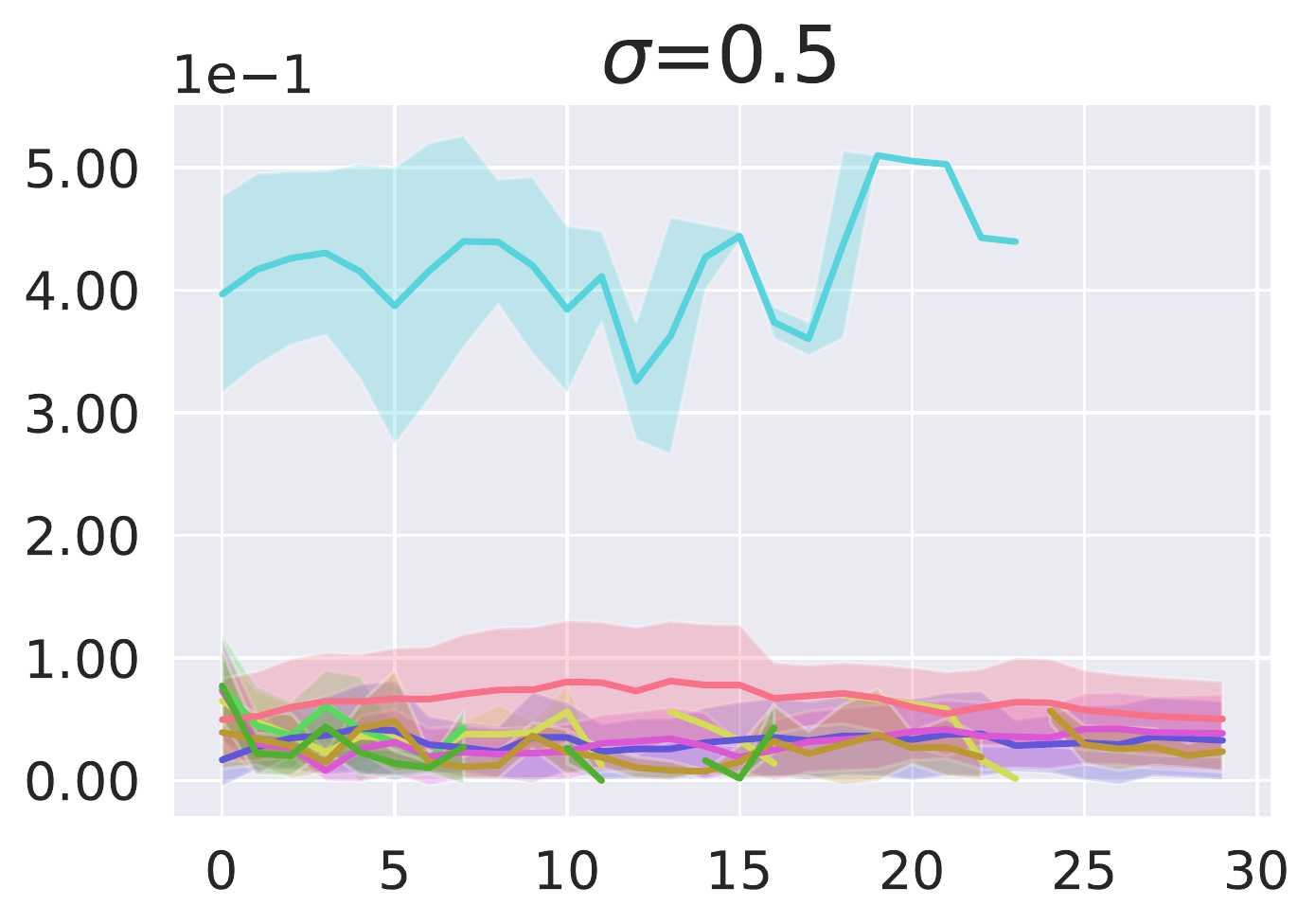}&
\includegraphics[height=\cputilheighta]{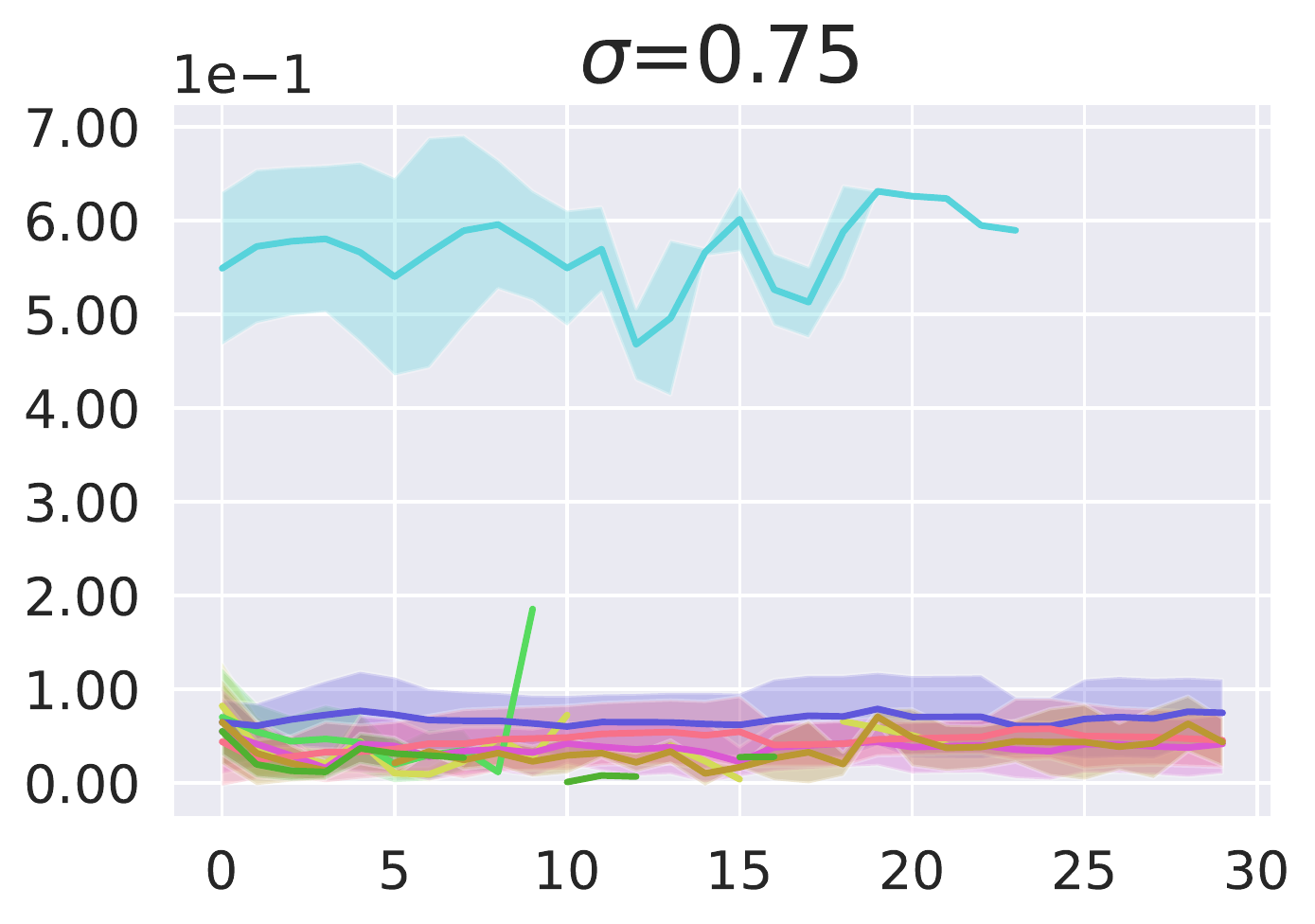}&
\includegraphics[height=\cputilheighta]{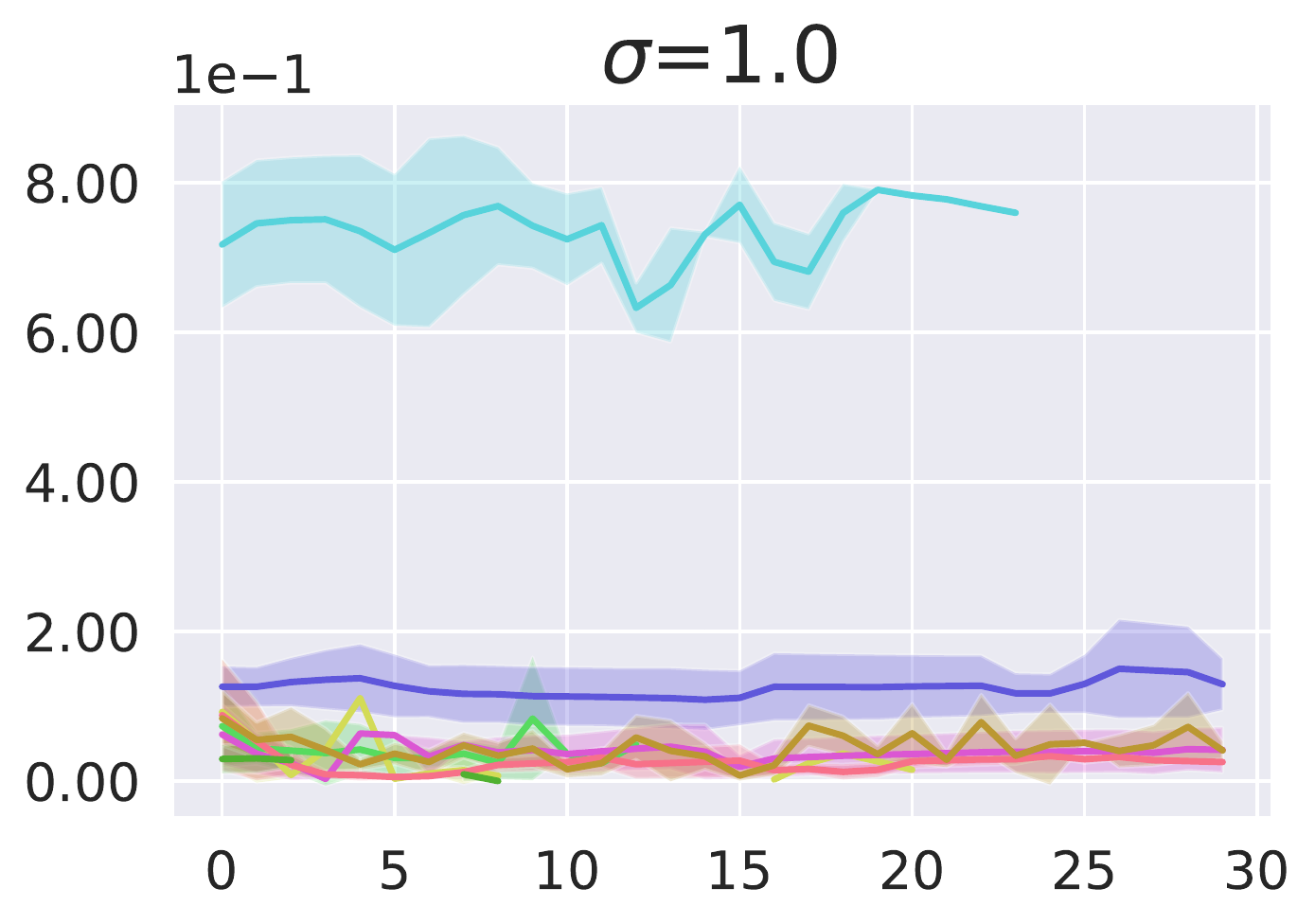}\\[-1.2ex]
        & \makecell{\tiny{time step $t$}}
        & \makecell{\tiny{time step $t$}}
        & \makecell{\tiny{time step $t$}}
        & \makecell{\tiny{time step $t$}}
        & \makecell{\tiny{time step $t$}}
        & \makecell{\tiny{time step $t$}}
\end{tabular}
\caption{\small Robustness certification for \textit{per-state action} in terms of certified radius $r$ at all time steps.
% We consider six methods: StdTrain, GaussAug, AdvTrain, RegPGD, RegCVX, and RadialRL, with the last three being SOTA. 
The shaded area represents the standard deviation.
% RadialRL is the most certifiably robust method on Freeway, while RegCVX is the most robust on CartPole.
}%
\label{fig:highway-statewise}
% }
% \caption{\small Certified radius $R_t$ along time steps}\label{tab:cert-rad}
\end{subtable}

\begin{subtable}[]{\linewidth}
\centering
\includegraphics[height=\cpwrapheighta]{figures/highway-fast_loact_clean.pdf}
\vspace{-1mm}
\caption{\small Benign performance of  locally smoothed policy $\tpi$ under different  smoothing variance $\sigma$.}\label{fig:highway-statewise-clean}
\end{subtable}

\vspace{1em}

\begin{subtable}[]{\linewidth}
\begin{tabular}{l}
\includegraphics[height=\cplegendheightb]{figures/CartPole_legend_marker.pdf}
\end{tabular}
\end{subtable}

\begin{subtable}[]{\linewidth}
\centering
\resizebox{\linewidth}{!}{%
\begin{tabular}{@{}p{3mm}@{}c@{}c@{}c@{}c@{}c@{}c@{}}
        & \makecell{\tiny{$\sigma=0.005$}}
        & \makecell{\tiny{$\sigma=0.05$}}
        & \makecell{\tiny{$\sigma=0.1$}}
        & \makecell{\tiny{$\sigma=0.5$}}
        & \makecell{\tiny{$\sigma=0.75$}}
        & \makecell{\tiny{$\sigma=1.0$}}
        \vspace{-1.7pt}\\
\cprownamec{\makecell{(Global) $\uje$}}&
\includegraphics[height=\cputilheightcp]{figures/highway-fast_global_mean_sigma-0.005.pdf}&
\includegraphics[height=\cputilheightcp]{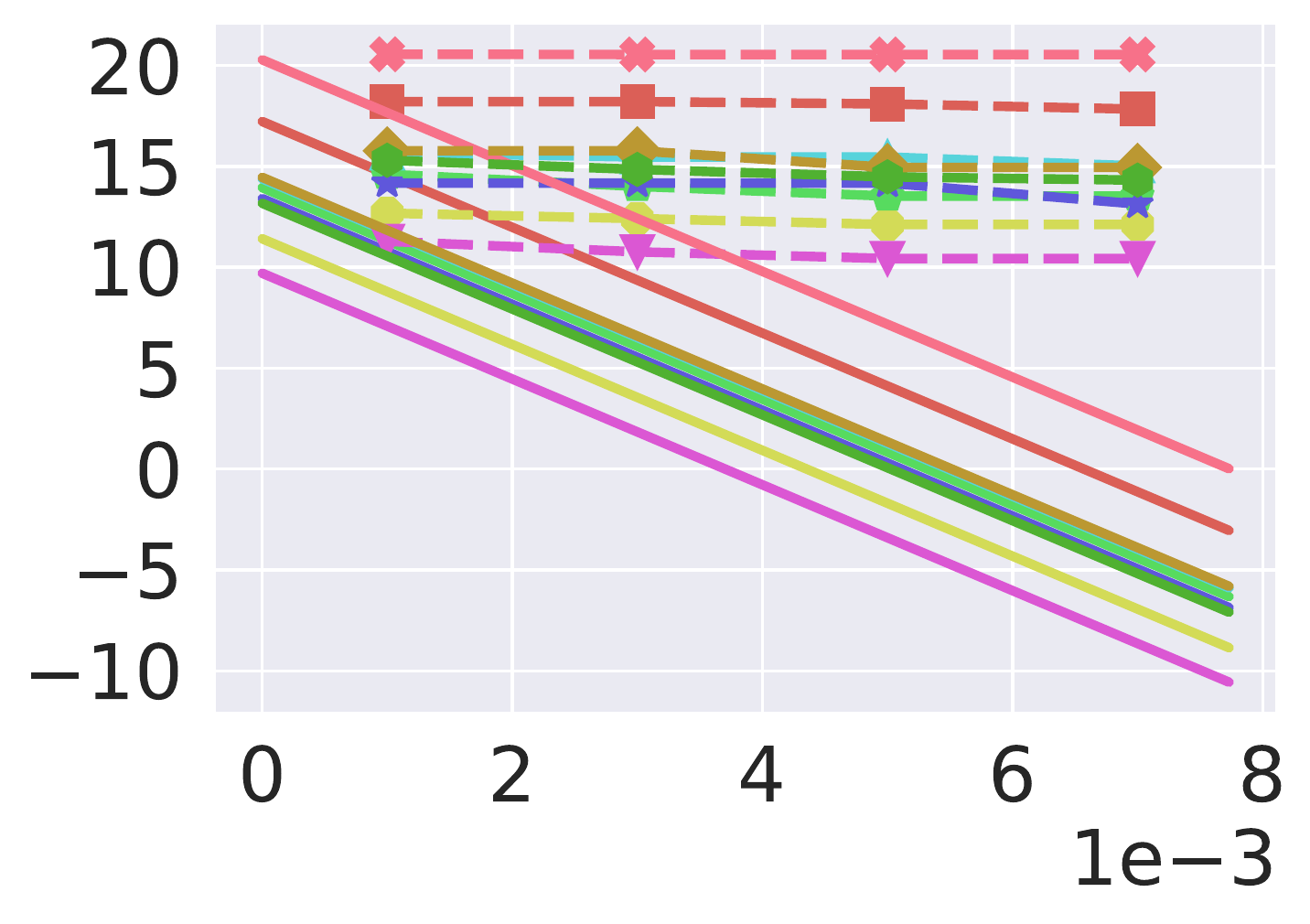}&
\includegraphics[height=\cputilheightcp]{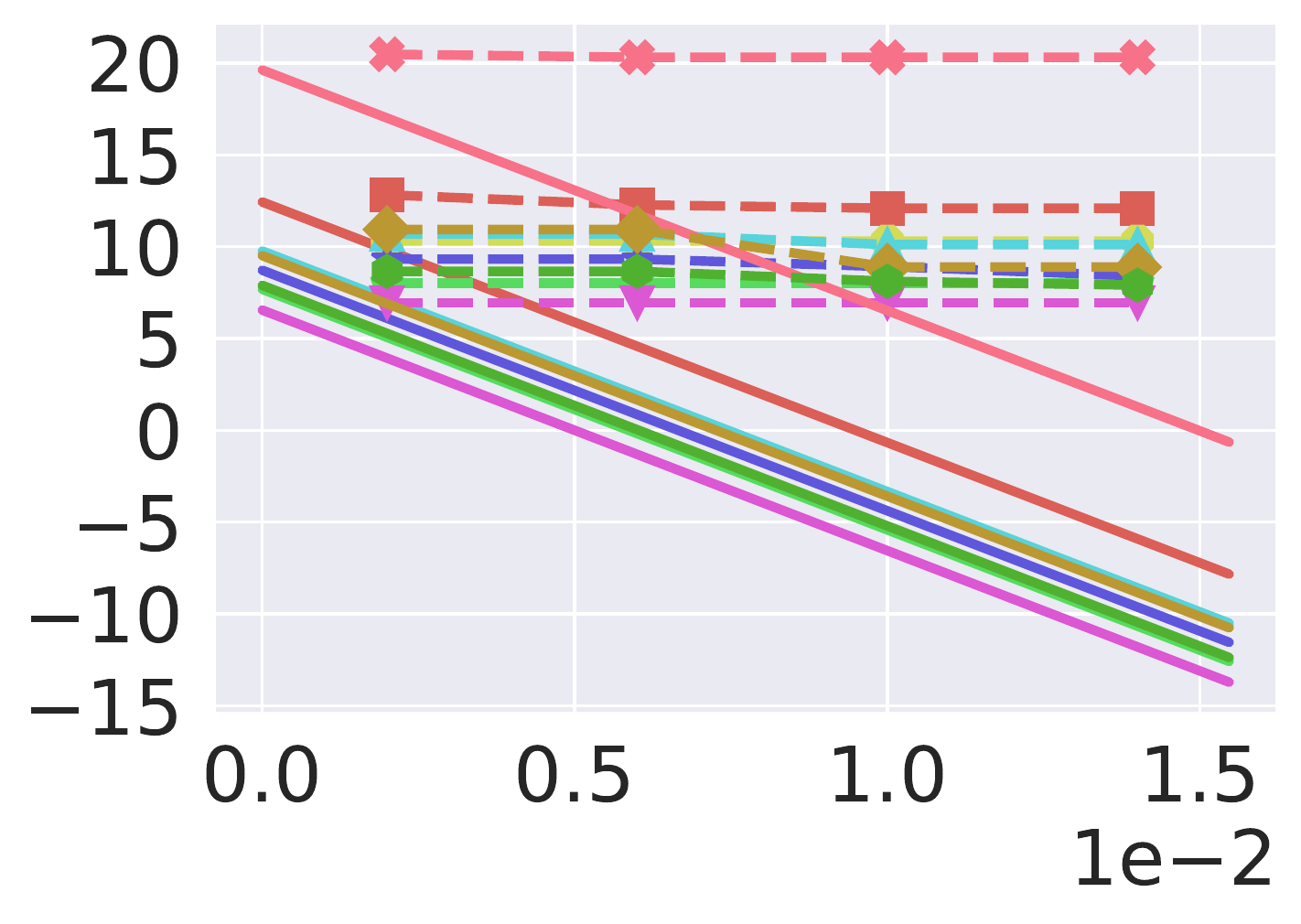}&
\includegraphics[height=\cputilheightcp]{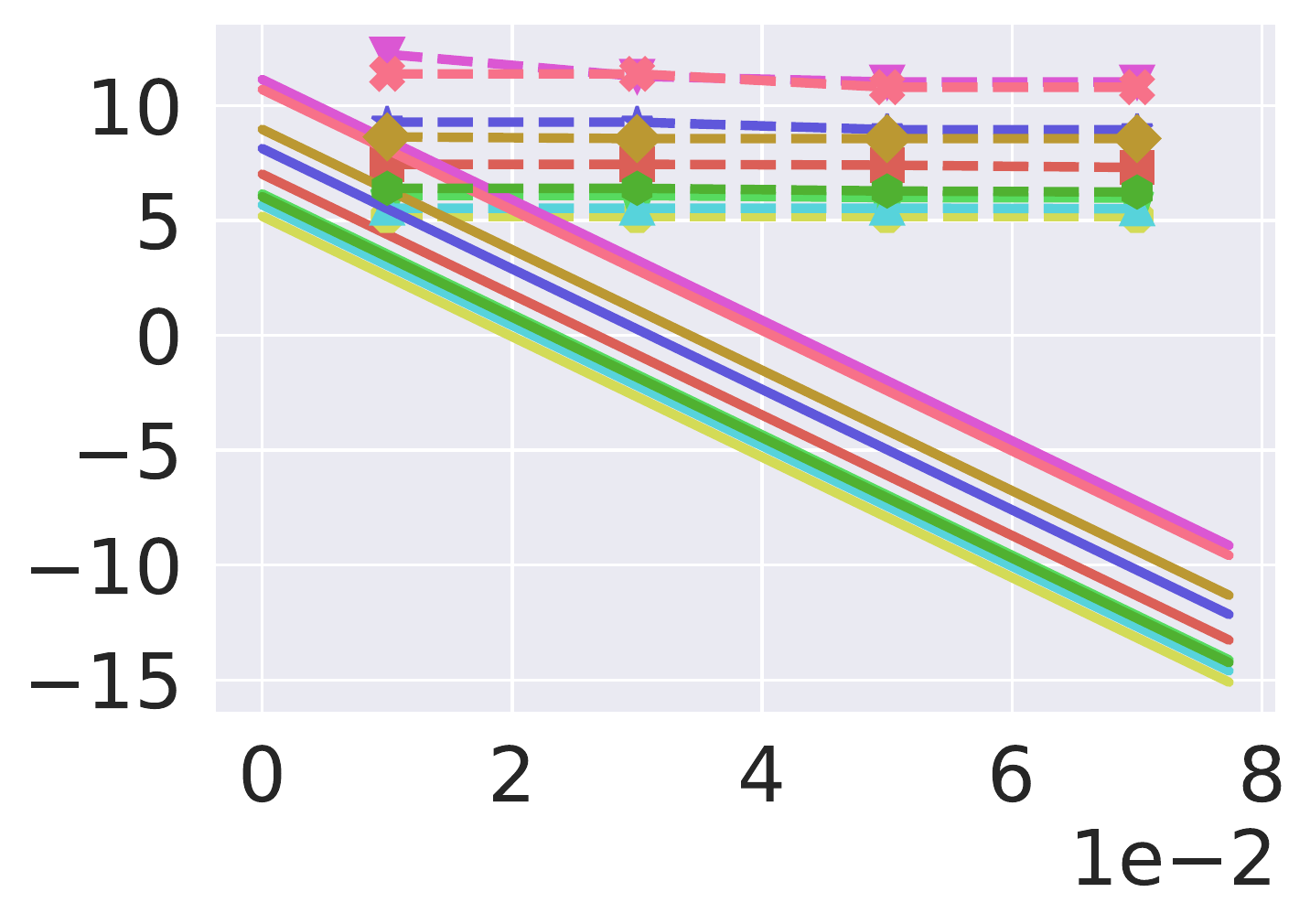}&
\includegraphics[height=\cputilheightcp]{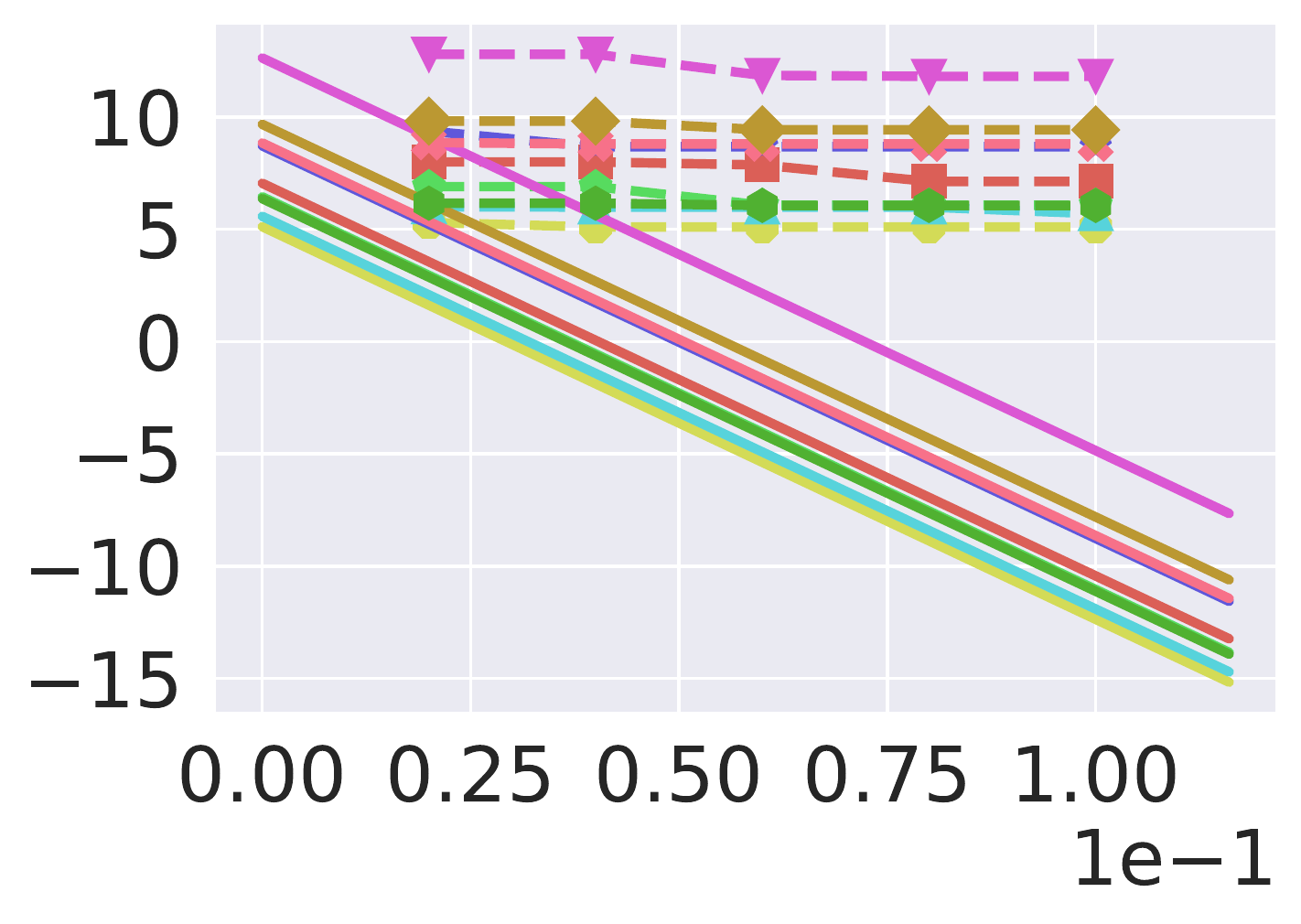}&
\includegraphics[height=\cputilheightcp]{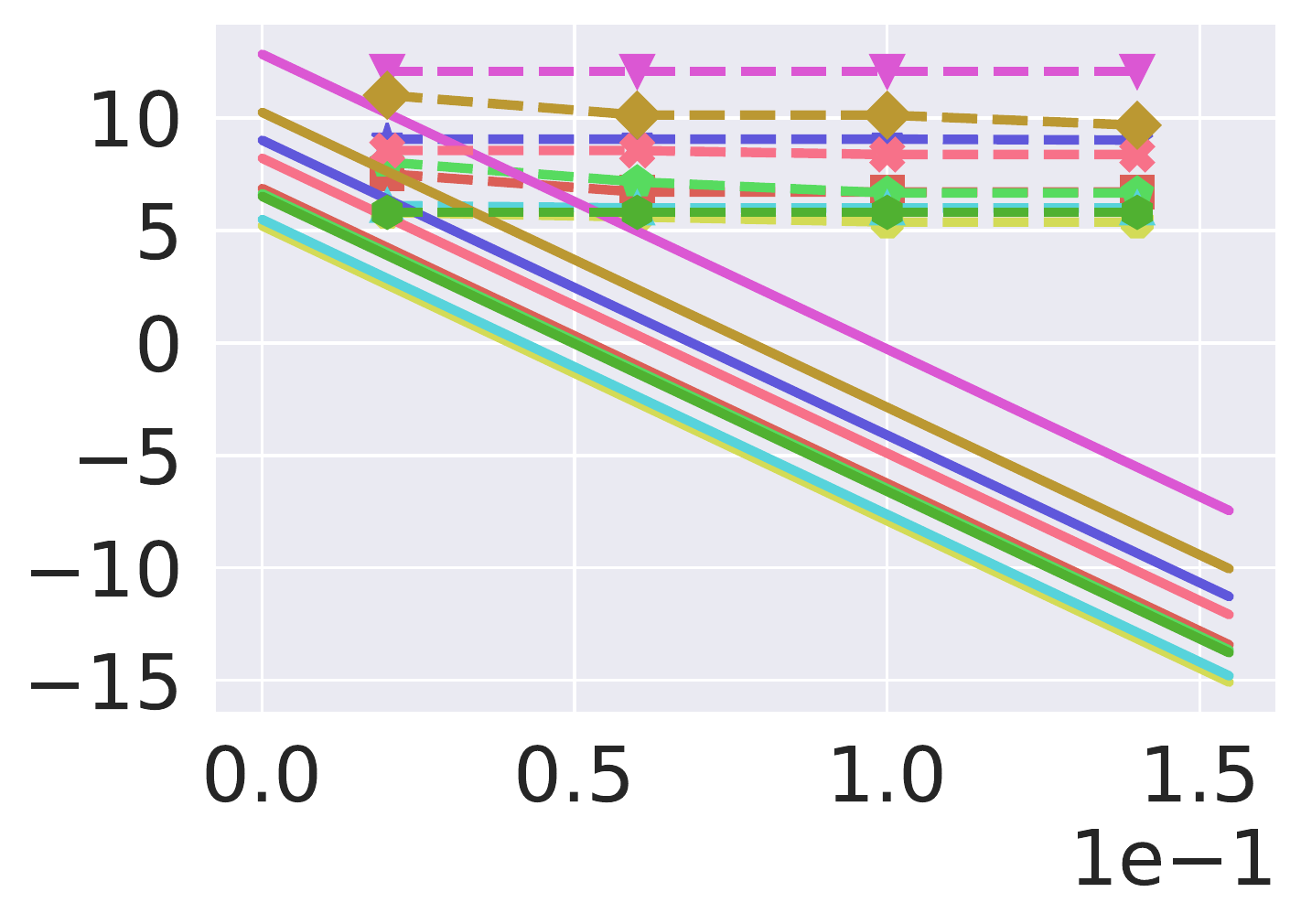}\\[-1.2ex]
\cprownamed{\makecell{(Global) $\ujp$}}&
\includegraphics[height=\cputilheightdp]{figures/highway-fast_global_median_sigma-0.005.pdf}&
\includegraphics[height=\cputilheightdp]{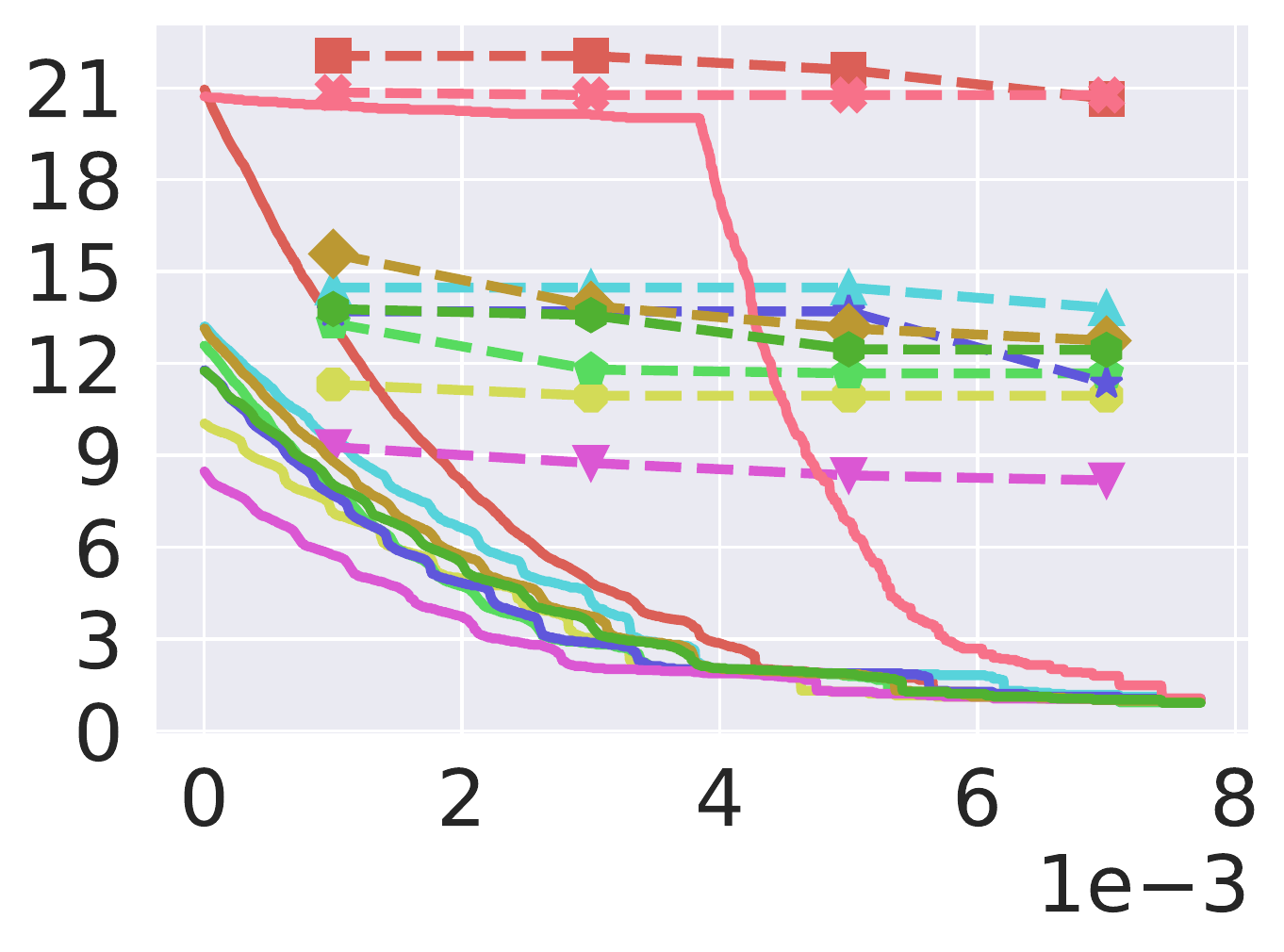}&
\includegraphics[height=\cputilheightdp]{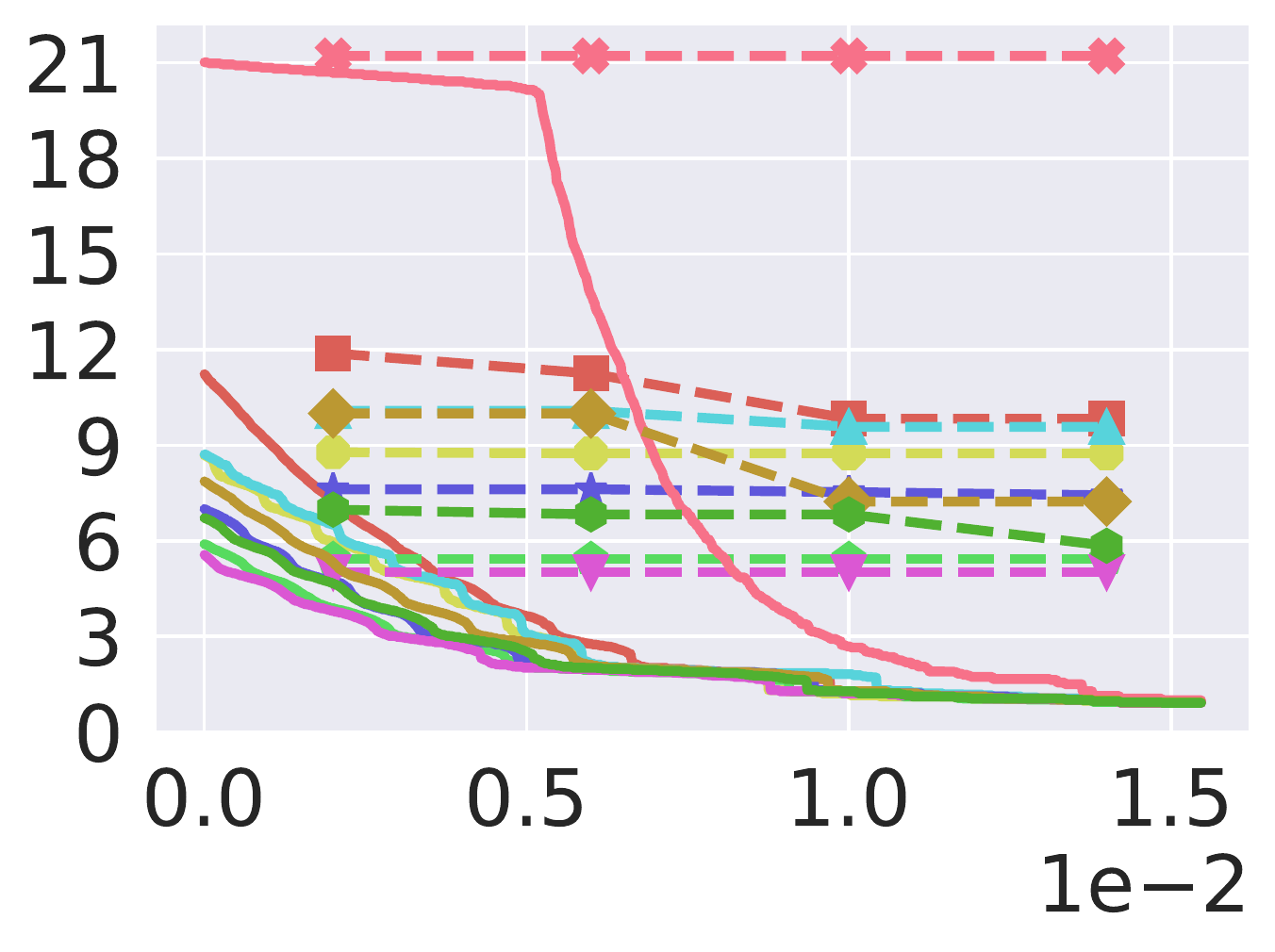}&
\includegraphics[height=\cputilheightdp]{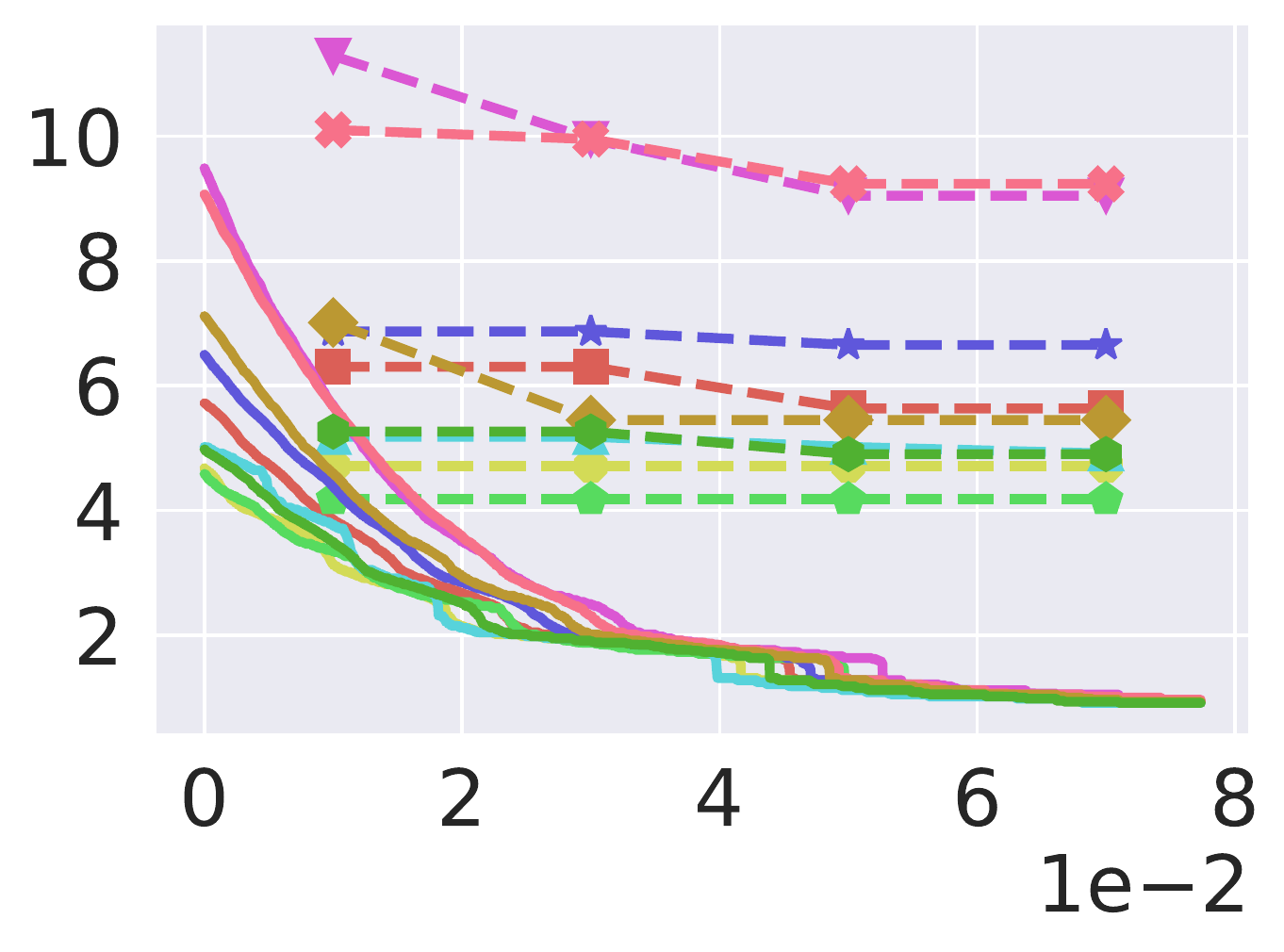}&
\includegraphics[height=\cputilheightdp]{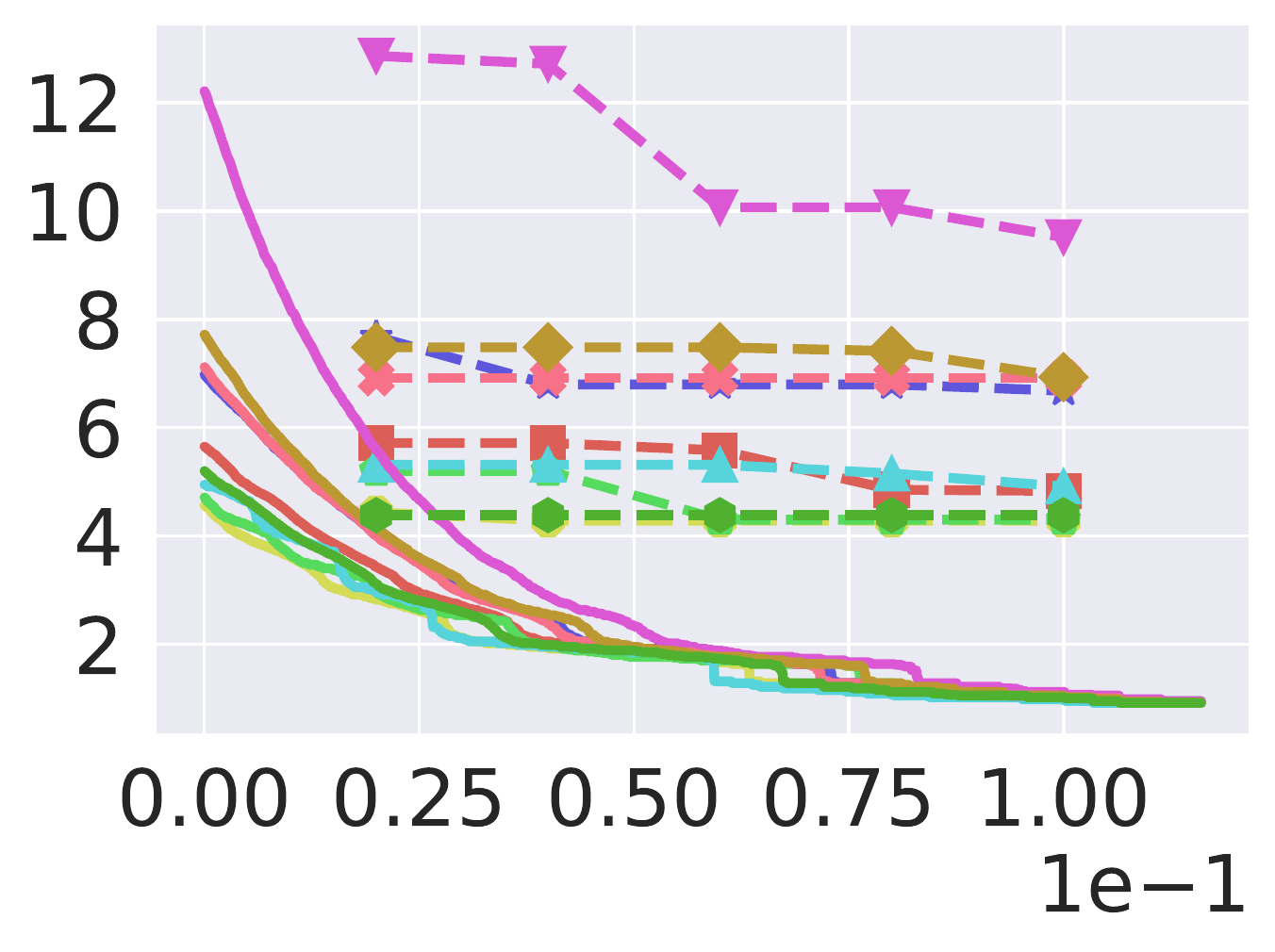}&
\includegraphics[height=\cputilheightdp]{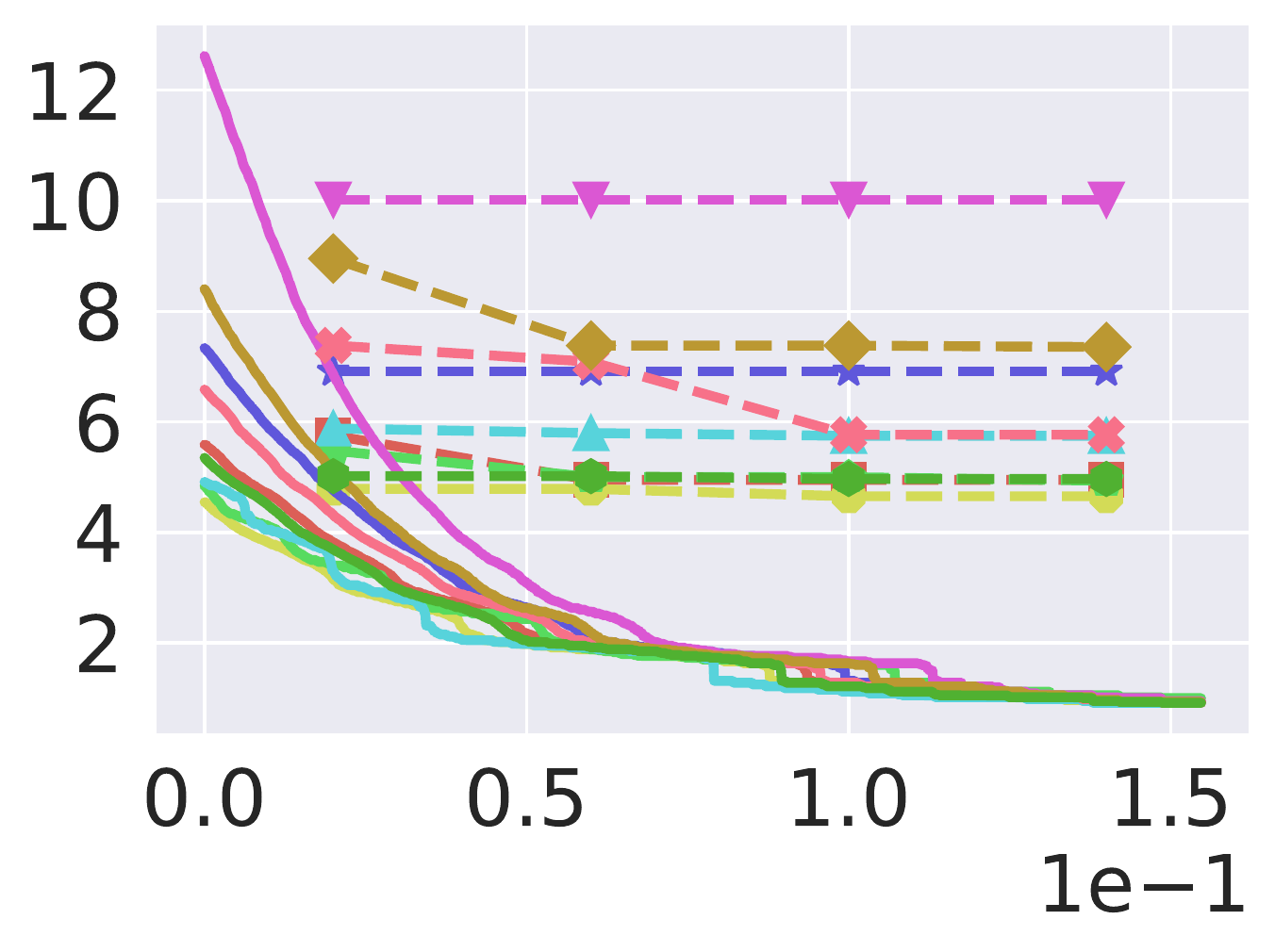}\\[-1.2ex]
\cprownamed{\makecell{(Local) $\uj$}}&
\includegraphics[height=\cputilheightaap]{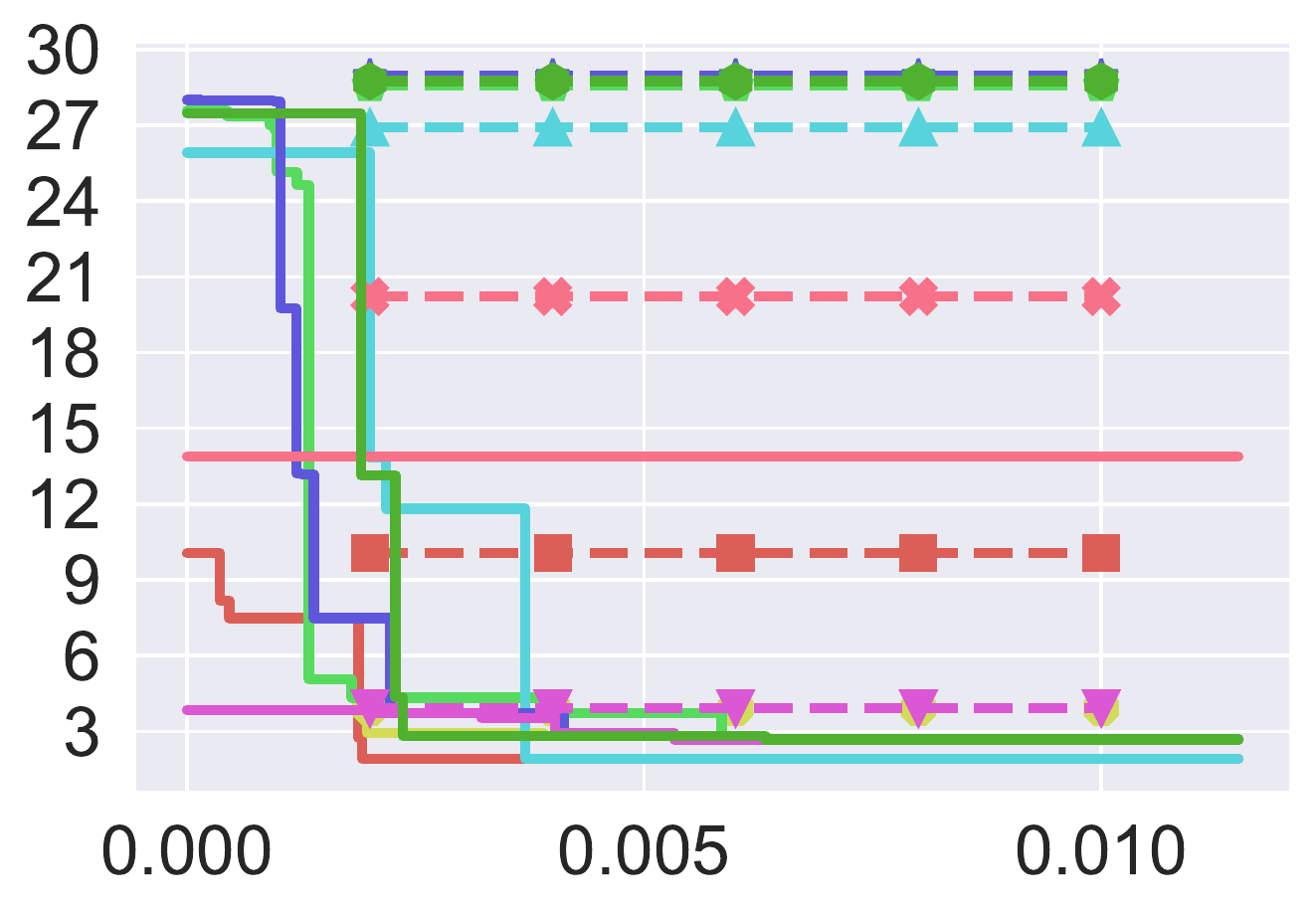}&
\includegraphics[height=\cputilheightaap]{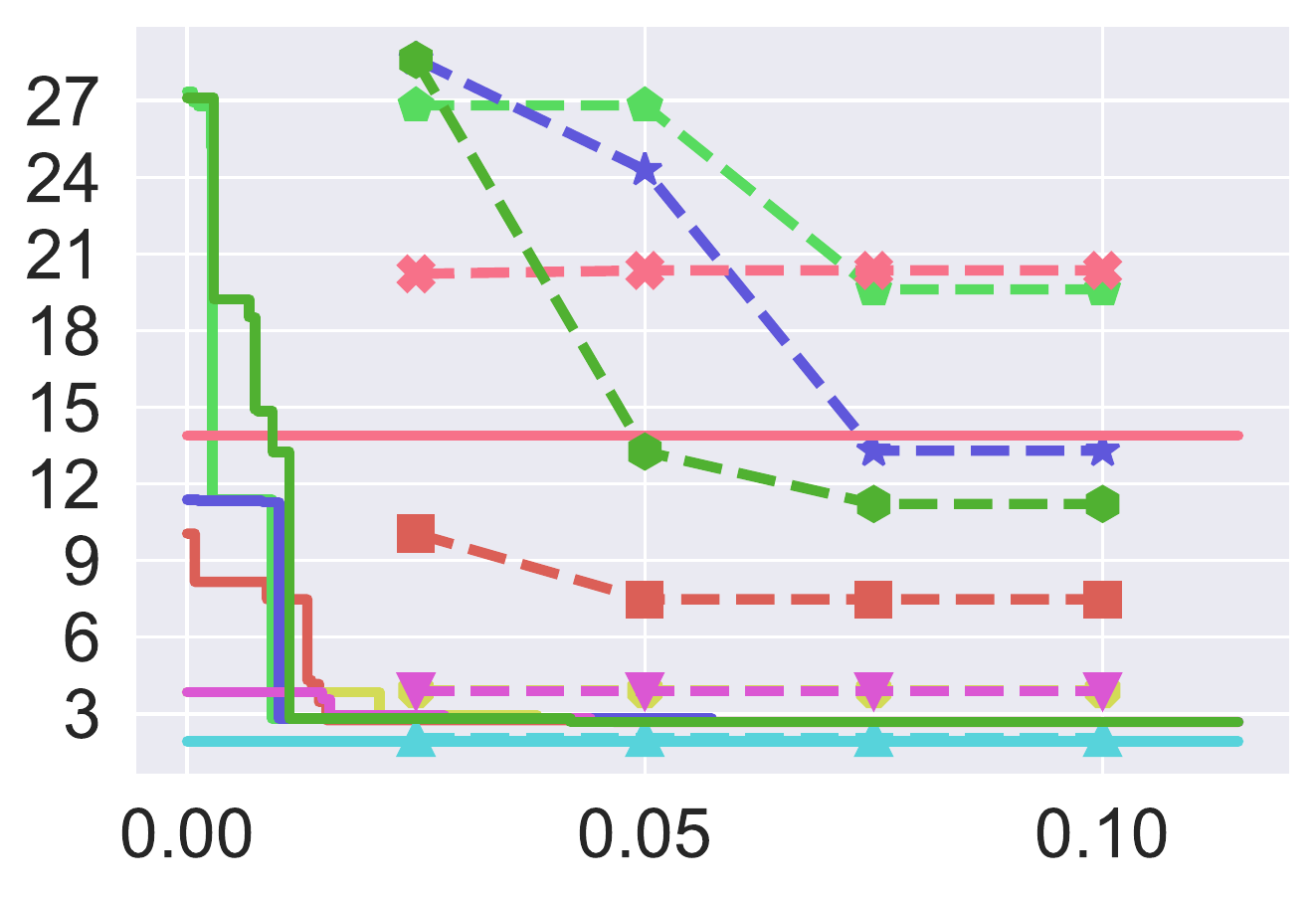}&
\includegraphics[height=\cputilheightaap]{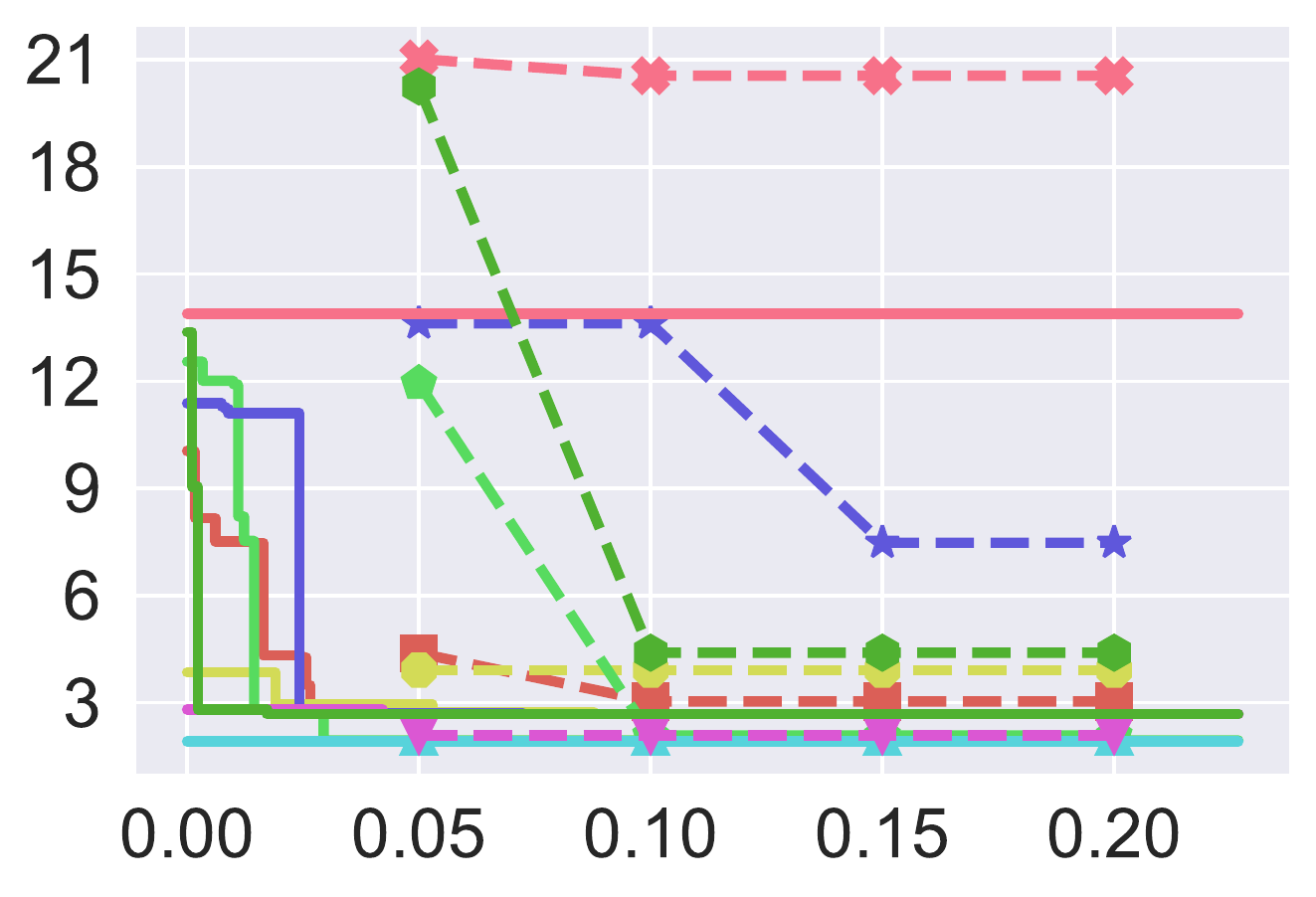}&
\includegraphics[height=\cputilheightaap]{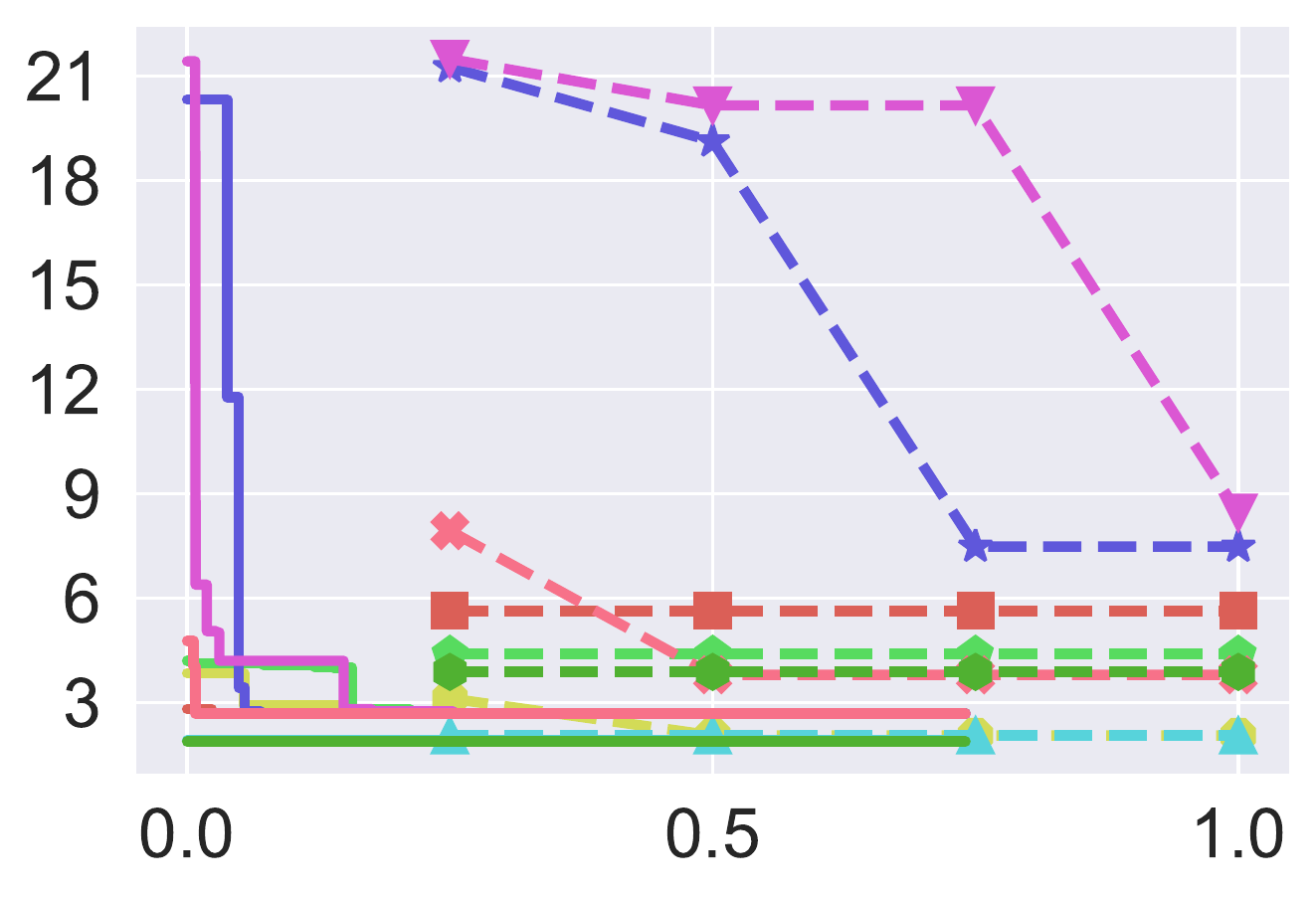}&
\includegraphics[height=\cputilheightaap]{figures/highway-fast_adasearch_sigma-0.75_cmp.pdf}&
\includegraphics[height=\cputilheightaap]{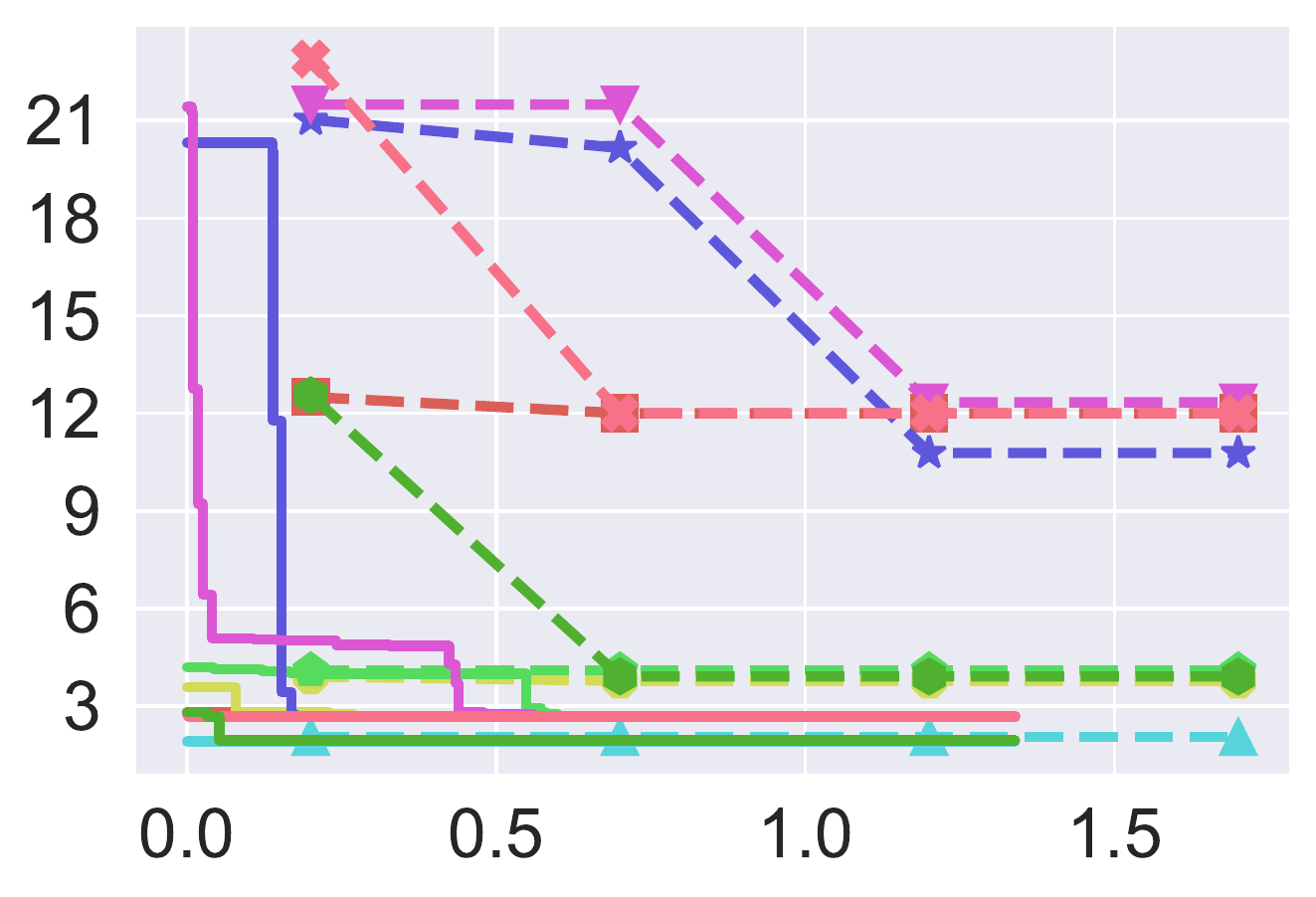}\\[-1.2ex]
        & \makecell{\tiny{attack $\eps$}}
        & \makecell{\tiny{attack $\eps$}}
        & \makecell{\tiny{attack $\eps$}}
        & \makecell{\tiny{attack $\eps$}}
        & \makecell{\tiny{attack $\eps$}}
        & \makecell{\tiny{attack $\eps$}}
\end{tabular}
}
\caption{\small Robustness certification as cumulative reward, including \textit{expectation bound} $\uje$, \textit{percentile bound} $\ujp$ ($p=50\%$), and \textit{absolute lower bound} $\uj$. 
% Each column corresponds to one smoothing variance. 
Solid lines represent the certified reward bounds of different methods, and dashed lines show the empirical performance under PGD attacks.}\label{tab:highway-bounds}
\end{subtable}
}

\caption{\small \textbf{Robustness certification on highway} in terms of (a-b): \textit{robustness of per-state action} and (c): \textit{lower bound of cumulative rewards}.
In (a) and (c), each column corresponds to one smoothing variance. 
% , including as cumulative reward, including \textit{expectation bound} $\uje$, \textit{percentile bound} $\ujp$ ($p=50\%$), and \textit{absolute lower bound} $\uj$  on (a) Freeway and (b) Pong. Each column corresponds to one smoothing variance. Solid lines represent the certified reward bounds of different methods, and dashed lines show the empirical performance under PGD attacks.
}\label{tab:highway-figs}
\end{figure}

}

% \subsection{More Results for \adasearch on Atari Game Breakout}
% \label{append:breakout}

% \input{fig_desc/fig_Breakout_lore}

% \input{tables/tab_Breakout_lore}

% We provide the results for \adasearch on a more complicated Atari game Breakout. 
% Concretely, we present the \textit{certified} cumulative reward lower bound computed via our \adasearch algorithm in~\Cref{fig:breakout-lore}, as well as the \textit{empirical} cumulative reward obtained under practical attacks in~\Cref{tab:breakout-lore}. 

% We obtain similar conclusions based on the new game and will describe the detailed observations below. 
% \textit{First},~\Cref{fig:breakout-lore} shows that GaussAug is the most certifiably robust, followed by AdvTrain; while StdTrain and SA-MDP (PGD) are shown to be non-robust for this game, whose lower bounds remain 0 for all smoothing parameters $\sigma$ we evaluate. 
% \textit{Second}, the robustness ranking given by our certification is aligned with the ranking under practical attacks. 
% From~\Cref{tab:breakout-lore}, we see that GaussAug can achieve much higher cumulative reward under larger perturbation magnitudes $\varepsilon$ compared with other RL methods, which is consistent with the conclusion drawn from~\Cref{fig:breakout-lore}. 
% \textit{Third}, our certification is relatively tight, especially for larger perturbation magnitudes. Notably, when $\sigma=0.05$, our lower bounds exactly match the attack results. 

\section{A Broader Discussion on Related Work}
\label{append:related}

\subsection{Evasion Attacks in RL}

\looseness=-1
We consider the adversarial attacks on state observations, where the attacker aims to add perturbations to the state during test time, so as to achieve certain \textbf{adversarial goals}, such as misleading the \textit{action} selection and minimizing the \textit{cumulative reward}. We discuss the related works that fall into these two categories below.

\textbf{Misleading the Action Selection.}\quad
It has been shown that different methods---applying \textit{random noise} to simply interfere with the action selection~\citep{kos2017delving}, or adopting \textit{adversarial attacks} (\eg, FGSM~\citep{goodfellow2014explaining} and CW attacks~\citep{carlini2017towards}) to deliberately alter the probability of action selection---are effective to different degrees. 
More concretely, when manipulating the action selection probability, some works aim to reduce the probability of selecting the optimal action~\citep{huang2017adversarial,kos2017delving,behzadan2017whatever}, while others try to increase the probability of selecting the worst action~\citep{lin2017tactics,pattanaik2018robust}.

\textbf{Minimizing the Cumulative Reward.}\quad
ATLA~\citep{zhang2021atla} and PA-AD~\citep{sun2021strongest} consider an \textit{optimal adversary} under the SA-MDP framework~\citep{zhang2021robust}, which aims to lead to minimal value functions under bounded state perturbations.
% the SA-MDP framework~\cite{zhang2021robust} is considered which characterizes the involvement of an adversary that can manipulate the states.
% Under this framework, the definition of an \textit{optimal adversary} is proposed---an adversary that can lead to minimal value functions under bounded state perturbations.
To find this \textit{optimal adversary} (\ie, the optimal adversarial state perturbation), 
% ATLA~\cite{zhang2021atla} and PA-AD~\cite{sun2021strongest} adopt different methods. 
ATLA~\citep{zhang2021atla} proposes to train an adversary whose action space is the perturbation set in the state space,
while PA-AD~\citep{sun2021strongest} further decouples the problem of finding state perturbations into finding policy perturbations plus finding the state that achieves the lowest value policy, thus addressing the challenge of large state space.
% Based on the SA-MDP framework~\cite{zhang2021robust}, ATLA~\cite{zhang2021atla} proposes to find the optimal adversarial state perturbation by training an adversary whose action space is the perturbation set.
% Given that ATLA~\cite{zhang2021atla} suffers from high complexity in MDPs with large state space, PA-AD~\cite{sun2021strongest} further decouples the attack into an RL problem to find the lowest value policy and an optimization problem to perturb the state such that the lowest value policy is achieved.
% Both ATLA~\cite{zhang2021atla} and PA-AD~\cite{sun2021strongest} achieve attack goal (b). 
% In the \textit{multi-agent environment} specifically, Gleave \etal~\citep{gleave2019adversarial} consider a \textit{different threat model} where the adversary can control the agent's opponent.
% They propose to craft natural adversarial observations via training the opponent's adversarial policy w.r.t. the agent's fixed trained policy, thus achieving the adversarial goal empirically.
% Different from these works which \textit{minimize} the cumulative reward in the original task, Tretschk \etal~\citep{tretschk2018sequential} try to \textit{maximize} an adversarial cumulative reward guided by an adversarial reward signal $r^{\rm adv}$. 
% Assuming whitebox access to the trained agent, they train an Adversarial Transformer Network (ATN)~\citep{baluja2018learning} along with the trained agent on the adversarial task to generate the perturbations.

\subsection{Robust RL}

\textbf{Distributionally Robust RL.}\quad
\citet{nilim2005robust} and~\citet{iyengar2005robust} consider the problem of distributionally robust RL, where there is normally an uncertainty set with prior information regarding the distribution of the uncertain parameters. 
\citet{iyengar2005robust} directly put constraints on environmental dynamics and reward functions by assuming that they take values in an uncertainty set.
Another line of research assumes the environmental dynamics and reward function as the uncertain parameters and are sampled from an uncertain distribution in the uncertain set. The uncertainty set is in general state-wise independent, \ie, ``s-rectangularity'' in~\citet{nilim2005robust}. 
Both types aim to derive a policy by solving a max-min problem, with the former taking the minimum directly over the parameters, while the latter taking the minimum over the distribution.

\noindent\textbf{Empirically Robust RL against Evasion Attacks.}\quad
We have briefly introduced related work on empirically robust RL in~\Cref{sec:related} in the main paper.
We emphasize two main differences between the contributions of these works and our \sysname. 
\textbf{First}, most of these robust RL methods only provide empirical robustness against perturbed state inputs during test time, but cannot provide theoretical guarantees for the performance of the trained models under any bounded perturbations; while our \sysname framework can provide practically computable \textit{certified robustness} w.r.t. two robustness certification criteria (per-state action stability and cumulative reward lower bound), \ie, for a given RL algorithm, we can compute certifications for it that indicate its robustness.
(We will discuss a few more related work that can provide certified robustness in~\Cref{append:related-cert-rl}.)
\textbf{Second}, most of these Robust RL methods focus on only the per-state decision, while we additionally consider the \textit{certification of trajectories} to obtain the lower bound of cumulative rewards.
Below, we provide a more comprehensive review of the categories of RL methods that demonstrate empirical robustness against evasion attacks.

\textit{Randomization methods}~\citep{tobin2017domain,akkaya2019solving} were first proposed to encourage exploration. 
% through large-scale exploration
This type of method was later systematically studied for its potential to improve model robustness.
NoisyNet~\citep{fortunato2017noisy} adds parametric noise to the network's weight during training, providing better resilience to both training-time and test-time attacks~\citep{behzadan2017whatever,behzadan2018mitigation},
also reducing the transferability of adversarial examples,
and enabling quicker recovery with fewer number of transitions during phase transition.

Under the \textit{adversarial training} framework,
\citet{kos2017delving} and \citet{behzadan2017whatever} show that re-training with random noise and FGSM perturbations increases the resilience against adversarial examples.
\citet{pattanaik2018robust} leverage attacks using an engineered loss function specifically designed for RL
to significant increase the robustness to parameter variations.
RS-DQN~\citep{fischer2019online} is an \textit{imitation learning} based approach that trains a robust student-DQN in parallel with a standard DQN in order to incorporate the constrains such as SOTA adversarial defenses~\citep{madry2017towards,mirman2018differentiable}. 
% We did not evaluate on their method due to no access to open-source code and missing details for reproducing the experiments.

SA-DQN~\citep{zhang2021robust} is a \textit{regularization} based method that 
adds regularizers to the training loss function to 
encourage the top-$1$ action to stay unchanged under perturbation.

Built on top of the \textit{neural network verification} algorithms~\citep{gowal2018effectiveness,weng2018towards}, Radial-RL~\citep{oikarinen2020robust} proposes to minimize an adversarial loss function that incorporates the upper bound of the perturbed loss, computed using certified bounds from verification algorithms.
CARRL~\citep{everett2021certifiable} aims to compute the  lower bounds of action-values under potential perturbation and select actions according to the worst-case bound, but it relies on linear bounds~\citep{weng2018towards} and is only suitable for low-dimensional environments.

\subsection{Robustness Certification for RL}
\label{append:related-cert-rl}

Despite the abundant literature in robustness certification in supervised learning, there is almost no work on robustness certification for RL, given the unclear criteria and the intrinsic difficulty of the task. In this part, we first introduce another two works that can achieve robustness certification at state-level, and then another concurrent work~\citet{kumar2021policy} which also aims to provide provable robustness regarding the cumulative reward for RL.

\textbf{Robustness Certification on State Level.}\quad
\citet{fischer2019online} and~\citet{zhang2021robust} have discussed the state level robustness certification as well, but the way they obtain the certification is different from our \sysname.
Concretely, we achieve robustness by probabilistic certification via randomized smoothing~\citep{cohen2019certified}, while~\citet{fischer2019online} and~\citet{zhang2021robust} directly achieve robustness by deterministic certification via leveraging the neural network verification techniques~\citep{zhang2018efficient,xu2020automatic,mirman2018differentiable}. 

Despite these works on state-level robustness certification for RL, we are the \textit{first} to provide the certification of cumulative rewards.
We emphasize that the certification of lower bound of the cumulative reward in RL is more challenging than the per-state certification considering the dynamic nature of RL. Our main contribution of the paper indeed mainly focuses on providing an efficient adaptive search based algorithm \adasearch to certify the lower bound of the cumulative reward together with rigorous analysis of the certification.

\textbf{Robustness Certification for Cumulative Rewards.}\quad
We compare our robustness certification for cumulative rewards with another concurrent work~\citet{kumar2021policy}.

In terms of the \textit{threat model}, both~\citet{kumar2021policy} and \sysname consider the type of adversary that can perturb the state observations of the agent during test time. In our CROP, we consider a perturbation budget per time step following previous works~\citep{huang2017adversarial,behzadan2017whatever,kos2017delving,pattanaik2018robust,zhang2021robust}, while they assume a budget for the entire episode. The two perspectives are closely related.
 
Regarding the \textit{certification criteria}, we certify both \textit{per-state action stability} and \textit{cumulative reward lower bound}. Although they also consider the cumulative reward in their certification goal, they formulate the certification as a classification problem---certifying whether the cumulative reward is above a threshold or not, in contrast to directly certifying the lower bound as in our case.
 
For the \textit{certification technique}, both works are developed based on randomized smoothing proposed in supervised learning~\citep{cohen2019certified}. We propose a global smoothing technique (\glbrs), as well as an adaptive search algorithm coupled with local smoothing (\adasearch) to achieve the certification of cumulative reward; while they propose an adaptive version of the Neyman-Pearson lemma, which is similar with our global smoothing. 
Among the three algorithms (\glbrs, \adasearch, and~\citet{kumar2020certifying}), \glbrs cannot defend against an adaptive adversary (as concretely stated in~\Cref{append:discuss-cert}) while the other two can; specifically, \adasearch is much more sophisticated and tight than the global smoothing ones, which is an important contribution in our work.

%%%%%%%%%%%%%%%%%%%%%%%%%%%%%%%%%%%%%%%%%%%%%%%%%%%%%%%%%%%%%%%%%%%%%%%%%%%%%%%
%%%%%%%%%%%%%%%%%%%%%%%%%%%%%%%%%%%%%%%%%%%%%%%%%%%%%%%%%%%%%%%%%%%%%%%%%%%%%%%

\end{document}

% This document was modified from the file originally made available by
% Pat Langley and Andrea Danyluk for ICML-2K. This version was created
% by Iain Murray in 2018, and modified by Alexandre Bouchard in
% 2019 and 2021. Previous contributors include Dan Roy, Lise Getoor and Tobias
% Scheffer, which was slightly modified from the 2010 version by
% Thorsten Joachims & Johannes Fuernkranz, slightly modified from the
% 2009 version by Kiri Wagstaff and Sam Roweis's 2008 version, which is
% slightly modified from Prasad Tadepalli's 2007 version which is a
% lightly changed version of the previous year's version by Andrew
% Moore, which was in turn edited from those of Kristian Kersting and
% Codrina Lauth. Alex Smola contributed to the algorithmic style files.